\documentclass[noinfoline]{book}

\RequirePackage[OT1]{fontenc}
\RequirePackage[lnms,amsmath]{imsart}
\RequirePackage{imsartprelims}
\RequirePackage{natbib}
\usepackage{mathbh}

\volumetitle{Pac-Bayesian Supervised Classification:
The Thermodynamics of Statistical Learning}

\volume{56}
\pubyear{2007}
\doi{10.1214/074921707000000391}
\usepackage{fancyhdr}
\usepackage{titlesec}
\usepackage{titletoc}
\usepackage{graphicx}
\usepackage{amstext}
\usepackage{amsbsy}
\usepackage{amssymb}
\usepackage{theorem}
\usepackage{dsfont}
\usepackage{eucal}
\usepackage{mathrsfs}
\newcommand{\ds}{\displaystyle}

\newcommand{\PP}{\mathbb{P}}
\newcommand{\EE}{\mathbb{E}}
\newcommand{\RR}{\mathbb{R}}
\newcommand{\NN}{\mathbb{N}}
\newcommand{\ZZ}{\mathbb{Z}}

\newcommand{\C}[1]{\mathcal{#1}}
\newcommand{\R}{{R'}}
\newcommand{\T}{\widetilde{\theta}}
\newcommand{\rr}{{r'}}

\newcommand{\B}[1]{\mathbh{#1}}
\DeclareMathOperator*{\essinf}{ess~inf}
\DeclareMathOperator{\supp}{supp}

\DeclareMathOperator*{\ess}{ess}

\DeclareMathOperator{\Var}{\mathbb{V}ar}
\DeclareMathOperator{\sign}{sign}
\numberwithin{equation}{chapter}
\newtheorem{thm}{Theorem}[section]
\newtheorem{prop}[thm]{Proposition}
\newtheorem{proposition}[thm]{Proposition}
\newtheorem{lemma}[thm]{Lemma}
\newtheorem{cor}[thm]{Corollary}

{
\newtheorem{dfn}{Definition}[section]
\newtheorem{rmk}{Remark}[section]}
\newenvironment{proof}{{\sc Proof.}}{\ $\square$}
\contentsmargin[30pt]{30pt}
\titlecontents{chapter}[15pt]{\vspace{6pt}\sc}
{\contentslabel[\thecontentslabel.]{15pt}}{\hspace*{-15pt}}
{\titlerule*[1pc]{.}\contentspage[\rm\thecontentspage]}
\titlecontents{section}[37pt]{\vspace{3pt}\small\sc}
{\contentslabel[\thecontentslabel.]{22pt}}{\hspace*{-22pt}}
{\titlerule*[1pc]{.}\contentspage[\normalsize\rm\thecontentspage]}
\titlecontents{subsection}[69pt]{\it}
{\contentslabel[\thecontentslabel.]{32pt}}{\hspace*{-32pt}}
{~\titlerule*[1pc]{.}\contentspage[\rm\thecontentspage]}

\allowdisplaybreaks
\newcommand{\thmref}[1]{\ref{#1} (page \pageref{#1})}
\newcommand{\myeq}[1]{(\ref{#1}, page \pageref{#1})}

\newcommand{\mypoint}{\makebox[1ex][r]{.\:\hspace*{1ex}}}
\RequirePackage[linktocpage]{hyperref}

\begin{document}

\begin{titlepage}
\editor{Olivier Catoni}
\end{titlepage}

\begin{copyrightpage}
\LCCN{2007939120}
\ISBN[13]{978-0-940600-72-0}
\ISBN[10]{0-940600-72-2}
\serieseditor{Anthony C.~Davison}
\treasurer{Rong Chen}
\executivedirector{Elyse Gustafson}
\end{copyrightpage}

\tableofcontents
\clearpage



\begin{preface}
\begin{frontmatter}
\title{Preface}
\addcontentsline{toc}{chapter}{Preface}
\end{frontmatter}

\thispagestyle{plain}

This monograph deals with adaptive
supervised classification, using tools borrowed from
statistical mechanics and information theory,
stemming from the PAC-Bayesian approach pioneered
by David McAllester and applied to a conception of
statistical learning theory forged by Vladimir Vapnik.
Using convex analysis
on the set of posterior probability measures,
we show how to get local measures
of the complexity of the classification model involving
the relative entropy of posterior distributions with
respect to Gibbs posterior measures. We then discuss
relative bounds, comparing the generalization error
of two classification rules, showing how the margin
assumption of Mammen and Tsybakov can be replaced with
some empirical measure of the covariance structure
of the classification model. We show how to
associate to any posterior distribution an \emph{effective
temperature} relating it to the Gibbs prior distribution
with the same level of expected error rate, and how
to estimate this effective temperature from data,
resulting in an estimator whose expected
error rate converges according to the best possible
power of the sample size adaptively under any margin and parametric
complexity assumptions. We describe and study
an alternative selection scheme based on relative bounds between
estimators, and present a two step localization
technique which can handle the selection
of a parametric model from a family of those.
We show how to extend systematically all the results
obtained in the inductive setting to transductive
learning, and use this to improve Vapnik's
generalization bounds,
extending them to the case when the sample
is made of independent non-identically distributed
pairs of patterns and labels.
Finally we review briefly the construction of Support Vector
Machines and show how to derive generalization bounds for
them, measuring the complexity either through the number
of support vectors or through the value of
the transductive or inductive
margin.

\begin{flushright}
{\sc Olivier Catoni}\\
CNRS -- Laboratoire de Probabilit\'es et Mod\`eles Al\'eatoires,
Universit\'e Paris 6 (site Chevaleret), 4 place Jussieu -- Case 188,
75 252 Paris Cedex 05.
\end{flushright}
\end{preface}

\clearpage
\thispagestyle{empty}

\mbox{}

\vfill

\hfill {\it to my son Nicolas}

\vfill

\vfill

\pagestyle{fancy}
\newcommand{\w}[1]{\widehat{#1}}
\newcommand{\ov}[1]{\overline{#1}}
\newcommand{\wh}[1]{\widehat{#1}}
\newcommand{\wt}[1]{\widetilde{#1}}
\newcommand{\wtheta}{\widehat{\theta}}
\chapter*{Introduction}
\addcontentsline{toc}{chapter}{Introduction}
\markboth{Introduction}{Introduction}

Among the possible approaches to pattern recognition,
statistical learning theory has received a lot of attention
in the last few years. Although a realistic pattern recognition
scheme involves data pre-processing and post-processing that
need a theory of their own, a central role is often played
by some kind of supervised learning algorithm. This central
building block is the subject we are going to analyse in
these notes.

Accordingly, we assume that we have prepared in some way or another
a \emph{sample} of $N$ labelled patterns $(X_i, Y_i)_{i=1}^N$,
where $X_i$ ranges in some pattern space $\C{X}$ and $Y_i$ ranges
in some finite label set $\C{Y}$. We also assume that we have devised
our experiment in such a way that the couples of random variables
$(X_i, Y_i)$ are independent (but not necessarily equidistributed).
Here, randomness should be understood to come from the way the
statistician has planned his experiment. He may for instance
have drawn the $X_i$s
at random from some larger population of patterns the algorithm
is meant to be applied to in a second stage. The labels $Y_i$
may have been set with the help of some external expertise
(which may itself be faulty or
contain some amount of randomness, so we do not assume
that $Y_i$ is a function of $X_i$, and allow the couple of
random variables $(X_i, Y_i)$ to follow any kind of joint distribution).
In practice, patterns will be extracted from some high dimensional and highly
structured data, such as digital images, speech signals, DNA sequences, etc.
We will not discuss this pre-processing stage  here,
although it poses crucial problems dealing with segmentation
and the choice of a representation. The aim of supervised classification
is to choose some classification rule $f: \C{X} \rightarrow \C{Y}$
which predicts $Y$ from $X$ making as few mistakes
as possible  on average.

The choice of $f$ will be driven by
a suitable use of the information provided by the sample
$(X_i, Y_i)_{i=1}^N$ on the joint distribution of
$X$ and $Y$. Moreover, considering all the possible
measurable functions $f$ from $\C{X}$ to $\C{Y}$ would
not be feasible in practice and maybe more importantly
not well founded from a statistical point of view,
at least as soon as the pattern space $\C{X}$ is large
and little is known in advance about the joint distribution
of patterns $X$ and labels $Y$. Therefore, we will
consider parametrized subsets of classification rules
$\{ f_{\theta}: \C{X} \rightarrow \C{Y}\,;\, \theta \in \Theta_m \}$,
$m \in M$, which may be grouped
to form a big parameter set $\Theta = \bigcup_{m \in M} \Theta_m$.

The subject of this monograph is to introduce
to statistical learning theory, and more precisely
to the theory of supervised classification,
a number of technical tools akin to statistical mechanics
and information theory, dealing with the concepts of
entropy and temperature.
A central task will in particular be to
control the mutual information between
an estimated parameter and the observed sample.
The focus will not be directly on the description
of the data to be classified, but on the description
of the classification rules. As we want to deal
with high dimensional data, we will be bound to consider
high dimensional sets of candidate classification rules,
and will analyse them with tools very similar to those used
in statistical mechanics to describe particle systems
with many degrees of freedom. More specifically,
the sets of classification rules will be described by
Gibbs measures defined on parameter sets and depending
on the observed sample value. A Gibbs measure is
the special kind of probability measure used in
statistical mechanics to describe the state of
a particle system driven by a given energy function
at some given temperature. Here, Gibbs measures
will emerge as minimizers of the average loss value
under entropy (or mutual information) constraints.
Entropy itself, more precisely the Kullback divergence
function between probability measures, will emerge
in conjunction with the use of exponential deviation
inequalities: indeed, the $\log$-Laplace transform
may be seen as the Legendre transform of the Kullback
divergence function, as will be stated in Lemma \ref{lemma1.3}
(page \pageref{lemma1.3}).

To fix notation, let $(X_i,Y_i)_{i=1}^N$ be the canonical process
on $\Omega = (\C{X} \times \C{Y})^N$ (which means
the coordinate process).
Let the pattern space
be provided with a sigma-algebra $\C{B}$ turning it into
a measurable space $(\C{X}, \C{B})$. On the finite label space $\C{Y}$,
we will consider the trivial algebra $\C{B}'$ made of all its subsets.
Let $\C{M}_+^1\bigl[(\C{K} \times \C{Y})^N, (\C{B}
\otimes \C{B}')^{\otimes N} \bigr]$ be our notation for
the set of probability measures (i.e. of positive measures
of total mass equal to $1$) on the measurable space
$\bigl[ (\C{X} \times \C{Y})^N, (\C{B} \times \C{B}')^{\otimes N}
\bigr]$.
Once some probability distribution
$\PP \in \C{M}_+^1\bigl[ (\C{X} \times \C{Y})^N, (\C{B} \otimes
\C{B}')^{\otimes N} \bigr]$ is chosen,
it turns $(X_i,Y_i)_{i=1}^N$
into the canonical realization of a stochastic process modelling the
observed sample (also called the training set).
We will assume that $\PP = \bigotimes_{i=1}^N P_i$, where
for each $i = 1, \dots, N$,
$P_i \in \C{M}_+^1(\C{X} \times \C{Y}, \C{B} \otimes \C{B}')$,
to reflect
the assumption that we observe independent pairs of patterns and labels.
We will also assume that we are provided with some indexed set of
possible classification rules
$$
\C{R}_{\Theta} = \bigl\{ f_{\theta}: \C{X} \rightarrow \C{Y};
\theta \in \Theta \bigr\},
$$
where $(\Theta, \C{T})$ is some measurable index set. Assuming
some indexation of the classification rules is just a matter
of presentation. Although it leads to heavier notation, it
allows us to integrate over the space of classification rules
as well as over $\Omega$, using the usual formalism of multiple
integrals. For this matter, we will assume that
$(\theta, x) \mapsto f_{\theta}(x): ( \Theta \times \C{X},
\C{B} \otimes \C{T} ) \rightarrow (\C{Y}, \C{B}')$
is a measurable function.

In many cases, as already mentioned,
$\Theta = \bigcup_{m \in M} \Theta_m$ will be a finite
(or more generally countable) union of subspaces, dividing the classification
model $\C{R}_{\Theta} = \bigcup_{m \in M} \C{R}_{\Theta_m}$ into a union of
sub-models. The importance of introducing such a structure has been
put forward by V.~Vapnik, as a way to avoid making strong hypotheses
on the distribution $\PP$ of the sample.
If neither the distribution of the sample nor the set of
classification rules were constrained, it is well known  that
no kind of statistical inference would be possible.
Considering a family of sub-models is a way to
provide for adaptive classification where
the choice of the model depends on the observed
sample. Restricting the set of classification rules is more realistic
than restricting the distribution of patterns, since the classification
rules are a processing tool left to the choice of the statistician,
whereas the distribution of the patterns is not fully under his control,
except for some planning of the learning experiment which may enforce
some weak properties like independence, but not the precise shapes of
the marginal distributions $P_i$ which are as a rule unknown distributions
on some high dimensional space.

In these notes, we will concentrate on general issues concerned with
a natural measure of risk, namely the \emph{expected error rate}
of each classification rule $f_{\theta}$, expressed as
\begin{equation}
\label{eq0.1}
R(\theta) = \frac{1}{N} \sum_{i=1}^N \PP\bigl[ f_{\theta}(X_i) \neq Y_i
\bigr].
\end{equation}
As this quantity is unobserved, we will be led to work with
the corresponding  \emph{empirical error rate}
\begin{equation}
\label{eq0.2}
r(\theta,\omega) = \frac{1}{N} \sum_{i=1}^N \B{1} \bigl[ f_{\theta}(X_i) \neq Y_i \bigr].
\end{equation}
This does not mean that practical learning algorithms will
always try to minimize this criterion. They often on the contrary
try to minimize some other criterion which is linked with
the structure of the problem and has some nice additional properties
(like smoothness and convexity, for example). Nevertheless, and independently
of  the precise form of the estimator $\wtheta: \Omega \rightarrow \Theta$
under study, the analysis of $R(\wtheta)$ is a natural question,
and often corresponds to what is required in practice.

Answering this question is not straightforward because,
although $R(\theta)$ is the expectation of $r(\theta)$,
a sum of independent Bernoulli random variables,
$R(\wtheta)$ is not the expectation of $r(\wtheta)$,
because of the dependence of $\wtheta$ on the sample,
and neither is $r(\wtheta)$ a sum of independent
random variables.
To circumvent this unfortunate situation,
some uniform control over the deviations of $r$ from $R$
is needed.

We will follow the PAC-Bayesian approach to this problem,
originated in the machine
learning community and pioneered by
\citet{McAllester,McAllester2}.
It can be seen as some variant of the more classical approach of $M$-estimators
relying on empirical process theory --- as described for instance in
\citet{VanDeGeer}.

It is built on some general principles:
\begin{itemize}
\item One idea is to embed the set of estimators of the type $\wtheta
: \Omega \rightarrow \Theta$ into the larger set of
regular conditional probability measures
$\rho: \bigl( \Omega,
(\C{B} \otimes \C{B}')^{\otimes N} \bigr) \rightarrow \C{M}_+^1(\Theta, \C{T})$.
We will call these conditional probability measures \emph{posterior distributions},
to follow standard terminology.
\item A second idea is to measure the fluctuations of $\rho$
with respect to the sample, using some prior distribution $\pi \in
\C{M}_+^1(\Theta, \C{T})$, and the Kullback divergence function
$\C{K}(\rho, \pi)$. The expectation $\PP \bigl\{ \C{K}(\rho, \pi) \bigr\}$
measures the randomness of $\rho$.
The optimal choice of
$\pi$ would be $\PP(\rho)$, resulting in a measure of the
randomness of $\rho$ equal to the mutual information between
the sample and the estimated parameter drawn from $\rho$.
Anyhow, since $\PP(\rho)$ is usually not better known than
$\PP$, we will have to be content with some less concentrated
prior distribution $\pi$, resulting in some looser measure
of randomness, as shown by the identity
$\PP \bigl[ \C{K}(\rho, \pi) \bigr] = \PP \bigl\{ \C{K}\bigl[\rho,
\PP(\rho)\bigr] \bigr\} + \C{K}\bigl[\PP(\rho), \pi\bigr]$.
\item A third idea is to analyse the fluctuations of the random
process $\theta \mapsto r(\theta)$ from its mean
process $\theta \mapsto R(\theta)$ through the $\log$-Laplace
transform
$$
- \frac{1}{\lambda}
\log \left\{ \iint \exp \bigl[ - \lambda r(\theta,\omega) \bigr]
\pi(d \theta) \PP(d \omega) \right\},
$$
as would be done in statistical mechanics, where this is called the free energy. This transform
is well suited
to relate $\min_{\theta \in \Theta} r(\theta)$
to $\inf_{\theta \in \Theta} R(\theta)$, since for large enough
values of the parameter $\lambda$, corresponding to low enough
values of the temperature, the system has small fluctuations
around its ground state.
\item A fourth idea deals with localization. It consists of considering
a prior distribution $\ov{\pi}$ depending on the unknown expected
error rate function $R$. Thus some central result of the theory
will consist in an empirical upper bound for $\C{K}\bigl[
\rho, \pi_{\exp( - \beta R)} \bigr]$, where $\pi_{\exp( - \beta R)}$,
defined by its density
$$
\frac{d}{d \pi} \bigl[ \pi_{\exp( - \beta R)} \bigr]  = \frac{\exp ( - \beta R) }{
\pi \bigl[ \exp( - \beta R) \bigr]},
$$
is a Gibbs distribution built from a known prior distribution
$\pi \in \C{M}_+^1(\Theta, \C{T})$, some inverse temperature parameter $\beta \in \RR_+$ and the
expected error rate $R$. This bound will in particular be used
when $\rho$ is a posterior Gibbs distribution, of the form
$\pi_{\exp( - \beta r)}$. The general idea will be to show that
in the case when $\rho$ is not too random, in the sense that it is
possible to find a prior (that is non-random) distribution
$\ov{\pi}$ such that $\C{K}(\rho, \ov{\pi})$ is small, then
$\rho(r)$ can be reliably taken for a good approximation of $\rho(R)$.
\end{itemize}

This monograph is divided into four chapters.
The first deals with the
inductive setting presented in these lines.
The second is devoted to relative bounds.
It shows that it is possible to obtain a tighter
estimate of the mutual information between
the sample and the estimated parameter by
comparing prior and posterior Gibbs distributions.
It shows how to use this idea to obtain adaptive
model selection schemes under very weak hypotheses.

The third chapter introduces the
\emph{transductive} setting of V.~Vapnik \citep{Vapnik},
which consists in comparing the performance of
classification rules on the learning sample
with their performance on a test sample instead
of their average performance. The fourth one
is a fast introduction to Support Vector Machines.
It is the occasion to show the implications
of the general results discussed in the three first
chapters when some particular choice is made
about the structure of the classification rules.

In the first chapter, two types of bounds are shown. \emph{Empirical bounds}
are useful to build, compare and select estimators.
\emph{Non random bounds} are useful to assess the speed of convergence
of estimators, relating this speed to the behaviour
of the Gibbs prior expected error rate $\beta \mapsto
\pi_{\exp ( - \beta R)}(R)$ and to covariance factors
related to the margin assumption of Mammen and Tsybakov
when a finer analysis is performed.
We will proceed from the most straightforward
bounds towards more elaborate ones, built to achieve a better
asymptotic behaviour. In this course towards more
sophisticated inequalities, we will introduce \emph{local bounds}
and \emph{relative bounds}.

The study of relative bounds is expanded in the third chapter,
where tighter comparisons between prior and posterior Gibbs
distributions are proved.
Theorems \ref{thm1.1.43} (page \pageref{thm1.1.43}) and
\ref{thm1.58} (page \pageref{thm1.58}) present two ways of
selecting some nearly optimal classification rule.
They are both proved to be adaptive in all the parameters
under Mammen and Tsybakov margin assumptions and parametric
complexity assumptions.
This is done in Corollary \ref{cor1.52} (page \pageref{cor1.52}) of Theorem
\ref{thm1.50} (page \pageref{thm1.50}) and in Theorem
\ref{thm1.71} (page \pageref{thm1.71}).
In the first approach, the performance of a randomized estimator
modelled by a posterior distribution is compared with the
performance of a prior Gibbs distribution. In the second
approach posterior distributions are directly
compared between themselves (and leads to slightly stronger
results, to the price of
using a more complex algorithm).
When there are more than one parametric model,
it is appropriate to use also some
\emph{doubly localized scheme}:
two step localization is presented for both approaches,
in Theorems \ref{thm1.59} (page \pageref{thm1.59})
and \ref{thm1.80} (page \pageref{thm1.80})
and provides bounds with a decreased influence
of the number of empirically inefficient
models included in the selection scheme.

We would not like to induce the
reader into thinking that the most sophisticated results
presented in these first two chapters are necessarily the most useful
ones, they are as a rule
only more efficient \emph{asymptotically}, whereas, being
more involved, they use looser constants leading to
less precision for small sample sizes.
In practice whether a sample is to be considered small
is a question of the ratio between
the number of examples and the complexity (roughly speaking the number
of parameters) of the model used for classification. Since our aim here is
to describe methods appropriate for complex data (images,
speech, DNA, \dots), we suspect that practitioners wanting to make use
of our proposals will often be confronted with small sample sizes;
thus we would advise them to try the simplest
bounds first and only afterwards see whether the asymptotically
better ones can bring some improvement.

We would also like to point out that the results
of the first two chapters are not of a purely
theoretical nature:
posterior parameter distributions
can indeed be computed effectively, using Monte Carlo techniques,
and there is well-established
know-how about these computations in Bayesian statistics.
Moreover, non-randomized estimators of the classical
form $\wh{\theta}: \Omega \rightarrow \Theta$ can be
efficiently approximated by posterior distributions
$\rho: \Omega \rightarrow \C{M}_+^1(\Theta)$ supported
by a fairly narrow neighbourhood of $\wh{\theta}$, more
precisely a neighbourhood of the size of the typical
fluctuations of $\wh{\theta}$, so that this randomized
approximation of $\wh{\theta}$ will most of the time
provide the same classification as $\wh{\theta}$ itself,
except for a small amount of dubious examples for which
the classification provided by $\wh{\theta}$ would
anyway be unreliable. This is explained on page
\pageref{eq1.1.2}.

As already mentioned, the third chapter is about
the \emph{transductive setting}, that is about
comparing the performance of estimators on a
training set and on a test set.
We show first that this comparison can be based
on a set of exponential deviation inequalities
which parallels the one used in the
inductive case. This gives the opportunity
to \emph{transport} all the results obtained
in the inductive case in a systematic way.
In the transductive setting, the use of prior distributions
can be extended to the use of \emph{partially
exchangeable posterior distributions} depending
on the union of training and test patterns,
bringing increased possibilities to adapt
to the data and giving rise to such crucial
notions of complexity as the Vapnik--Cervonenkis
dimension.

Having done so, we more specifically focus
on the \emph{small sample case}, where local and relative
bounds are not expected to be of great help. Introducing
a fictitious (that is unobserved) shadow sample, we study Vapnik-type
generalization bounds, showing how to tighten and extend them
with some original ideas, like making no Gaussian approximation to the
log-Laplace transform of Bernoulli random variables,  using a shadow sample
of arbitrary size. shrinking from the use of any symmetrization trick,
and using a suitable subset of the group of permutations to cover the
case of independent non-identically distributed data. The culminating
result of the third chapter is Theorem \thmref{thm2.3.3},
subsequent bounds showing the separate influence of the above ideas and
providing an easier comparison with Vapnik's original results.
Vapnik-type generalization bounds have a broad applicability, not
only through the concept of Vapnik--Cervonenkis dimension, but also through the use
of compression schemes \citep{Little}, which are briefly described
on page \pageref{compression}.

The beginning of the fourth chapter introduces Support Vector Machines,
both in the separable and in the non-separable case (using the
box constraint). We then describe different types of bounds. We
start with compression scheme bounds, to proceed with margin bounds.
We begin with transductive margin bounds, recalling on this
occasion in Theorem \thmref{th1} the growth bound for a family of classification
rules with given Vapnik--Cervonenkis dimension. In Theorem \thmref{chap5Th1.1} we give
the usual estimate of the Vapnik--Cervonenkis dimension of a family
of separating hyperplanes with
a given transductive margin (we mean by this that the margin
is computed on the union of the training and test sets).
We present an original probabilistic proof
inspired by a similar one from \cite{Cristianini},
whereas other proofs available usually rely on the informal
claim that  the simplex is the worst case.
We end this short review of Support Vector Machines with
a discussion of inductive margin bounds. Here the margin is
computed on the training set only, and a more involved combinatorial
lemma, due to \cite{Alon} and recalled in
Lemma \thmref{lemma3.1} is used. We use this lemma and the results
of the third chapter to establish a bound depending on the
margin of the training set alone.

In appendix, we finally discuss the textbook example
of classification by thresholding:
in this setting, each classification rule is built by thresholding a
series of measurements and taking a decision based on
these thresholded values. This relatively simple example
(which can be considered as an introduction
to the more technical case of classification trees)
can be used to give more flesh to the results of the first three
chapters.

It is a pleasure to end this introduction with my greatest thanks to
Anthony Davison, for his careful reading of the manuscript and his numerous
suggestions.

\chapter{Inductive PAC-Bayesian learning}
\setcounter{page}{1}
\pagenumbering{arabic}

The setting of inductive inference (as opposed to transductive
inference to be discussed later) is the one described in the
introduction.

When we will have to take the expectation of
a random variable $Z: \Omega \rightarrow \RR$ as well as of a function
of the parameter $h: \Theta \rightarrow \RR$ with respect to
some probability measure, we will as a rule use short  functional
 notation instead of resorting to the integral sign:
thus we will write $\PP(Z)$ for $\int_{\Omega} Z(\omega) \PP(d \omega)$
and $\pi(h)$ for $\int_{\Theta} h(\theta) \pi(d \theta)$.

A more traditional statistical approach would
focus on estimators $\wh{\theta}: \Omega \rightarrow \Theta$
of the parameter $\theta$ and be interested on the
relationship between the \emph{empirical error rate}
$r(\wh{\theta})$, defined by equation \myeq{eq0.1},
 which is the number of errors made
on the sample, and the expected error rate $R(\wh{\theta})$,
defined by equation \myeq{eq0.2},
which is the expected probability of error on new instances
of patterns. The PAC-Bayesian approach instead chooses a broader
perspective and allows the estimator $\wh{\theta}$ to
be drawn at random using some auxiliary source of
randomness to smooth the dependence of $\wh{\theta}$
on the sample. One
way of representing the supplementary randomness
allowed in the choice of $\wh{\theta}$, is to consider
what it is usual to call \emph{posterior distributions}
on the parameter space, that is probability measures
$\rho: \Omega \rightarrow \C{M}_+^1(\Theta, \C{T})$,
depending on the sample, or from a technical perspective,
regular conditional (or transition) probability measures.
Let us recall that we use the model described in the
introduction: the training sample is modelled by the canonical
process $(X_i, Y_i)_{i=1}^N$ on $\Omega = \bigl(
\C{X} \times \C{Y} \bigr)^N$, and a product probability
measure $\PP = \bigotimes_{i=1}^N P_i$ on $\Omega$
is considered to reflect the assumption that the
training sample is made of independent pairs
of patterns and labels.
The transition
probability measure $\rho$, along with $\PP \in \C{M}_+^1(\Omega)$,
defines a probability distribution on $\Omega \times \Theta$
and describes the conditional distribution
of the estimated parameter $\wh{\theta}$ knowing the sample
$(X_i, Y_i)_{i=1}^N$.

The main subject of this broadened theory becomes to
investigate the relationship between $\rho(r)$, the
average error rate of $\wh{\theta}$ on the training sample,
and $\rho(R)$, the expected error rate of $\wh{\theta}$
on new samples. The first step towards using some kind
of thermodynamics to tackle this question, is to consider
the Laplace transform of $\rho(R) - \rho(r)$,
a well known provider of non-asymptotic deviation
bounds. This transform takes the form
$$
\PP \Bigl\{ \exp \Bigl[ \lambda \bigl[ \rho(R) - \rho(r) \bigr] \Bigr] \Bigr\},
$$
where some inverse temperature parameter $\lambda \in \RR_+$, as a physicist would call it,
is introduced.
This Laplace transform would be easy to bound if $\rho$ did
not depend on $\omega \in \Omega$ (namely on the sample),
because $\rho(R)$ would then be non-random, and
$$
\rho(r) = \frac{1}{N} \sum_{i=1}^N \rho \bigl[ Y_i \neq f_{\theta}(X_i) \bigr],
$$
would be a sum of independent random variables.
It turns out, and this will be the subject of the next section,
that this annoying dependence of $\rho$ on $\omega$ can
be quantified, using the inequality
$$
\rho(R) - \rho(r)  \leq \lambda^{-1} \log \Bigl\{ \pi \Bigl[
\exp \bigl[ \lambda (R - r) \bigr] \Bigr] \Bigr\} + \lambda^{-1} \C{K}(\rho, \pi),
$$
which holds for any probability measure $\pi \in \C{M}_+^1(\Theta)$
on the parameter space; for our purpose it will be appropriate to consider
a \emph{prior}\ distribution $\pi$ that is non-random, as opposed
to $\rho$, which depends on the sample. Here, $\C{K}(\rho, \pi)$
is the Kullback divergence of $\rho$ from $\pi$,
whose definition will be recalled when we will come to technicalities;
it can be seen as an upper bound for the mutual information
between the $(X_i, Y_i)_{i=1}^N$ and the estimated parameter
$\wh{\theta}\,$.
This inequality will
allow us to relate the \emph{penalized}\ difference
$\rho(R) - \rho(r) - \lambda^{-1} \C{K}(\rho, \pi)$
with the Laplace transform of sums of independent random variables.

\section{Basic inequality}

Let us now come to the details of the investigation sketched
above. The first thing we will do is to study the
Laplace transform of $R(\theta) - r(\theta)$,
as a starting point for the more general study
of $\rho(R) - \rho(r)$: it corresponds to the
simple case where $\wh{\theta}$ is not random at all,
and therefore where $\rho$ is a Dirac mass at some
deterministic parameter value $\theta$.

In the setting described in the introduction,
let us consider the Bernoulli random variables
$\sigma_i(\theta) = \B{1} \bigl[ Y_i \neq f_{\theta} (X_i) \bigr]$, which
 indicates whether the classification rule $f_\theta$ made
an error on the $i$th component of the training sample.
Using independence and the concavity of the logarithm
function, it is readily seen that for any real constant $\lambda$
\begin{multline*}
\log \Bigl\{ \PP \bigl\{ \exp \bigl[ - \lambda r(\theta) \bigr]
\bigr\} \Bigr\}
= \sum_{i=1}^N \log \Bigl\{ \PP \Bigl[ \exp\bigl(
- \tfrac{\lambda}{N} \sigma_i \bigr) \Bigr] \Bigr\}
\\ \leq N \log \biggl\{ \frac{1}{N}\sum_{i=1}^N
\PP \Bigl[ \exp \bigl( - \tfrac{\lambda}{N}
\sigma_i \bigr) \Bigr]
\biggr\}.
\end{multline*}
The right-hand side of this inequality is the $\log$-Laplace
transform of a Bernoulli distribution with parameter
$\frac{1}{N} \sum_{i=1}^N \PP(\sigma_i) = R(\theta)$.
As any Bernoulli distribution is fully defined
by its parameter, this $\log$-Laplace transform
is necessarily a function of $R(\theta)$. It can
be expressed with the help of the family of functions
\begin{equation}
\label{eq1.1}
\Phi_{a}(p) = - a^{-1} \log \bigl\{
1 - \bigl[1 - \exp( - a)\bigr]
p \bigr\}, \quad a \in \RR, p \in (0,1).
\end{equation}
It is immediately seen that $\Phi_{a}$ is an increasing
one-to-one mapping of the unit interval onto itself, and that it
is convex when $a > 0$, concave when $a < 0$ and can be defined
by continuity to be the identity when $a = 0$.
Moreover the inverse of $\Phi_{a}$ is given by the
formula
$$
\Phi_{a}^{-1}(q) = \frac{1 - \exp (- a q )}{1 - \exp ( - a )},
\qquad a \in \RR, q \in (0,1).
$$
This formula may be used to extend $\Phi_a^{-1}$
to $q \in \RR$, and we will use this extension without
further notice when required.

Using this notation, the previous inequality becomes
$$
\log \Bigl\{ \PP \bigl\{ \exp \bigl[ - \lambda r(\theta)
\bigr] \bigr\} \Bigr\} \leq
- \lambda \Phi_{\frac{\lambda}{N}} \bigl[ R(\theta) \bigr],
\quad \text{proving}
$$

\begin{lemma}
\label{lemma1.1.1} \mypoint For any real constant $\lambda$ and
any parameter $\theta \in \Theta$,
$$
\PP \biggl\{ \exp \Bigl\{
\lambda \Bigl[ \Phi_{\frac{\lambda}{N}} \bigl[ R(\theta) \bigr]
- r(\theta) \Bigr]
\Bigr\} \biggr\} \leq 1.
$$
\end{lemma}

In previous versions of this study, we had used some Bernstein
bound, instead of this lemma. Anyhow, as it will turn out,
keeping the $\log$-Laplace transform of a Bernoulli instead of approximating
it provides simpler and tighter results.

Lemma \ref{lemma1.1.1} implies that
for any constants $\lambda \in \RR_+$ and $\epsilon \in )0,1)$,
$$
\PP \biggl[ \Phi_{\frac{\lambda}{N}}\bigl[ R(\theta) \bigr] +
\frac{\log(\epsilon)}{\lambda} \leq r(\theta) \biggr] \geq 1 - \epsilon.
$$
Choosing $\ds \overline{\lambda} \in \arg\max_{\RR_+}
\Phi_{\frac{\lambda}{N}}\bigl[ R(\theta) \bigr] + \frac{\log(\epsilon)}{\lambda}$,
we deduce
\begin{lemma}\mypoint
For any $\epsilon \in )0,1)$, any $\theta \in \Theta$,
$$
\PP \Biggl\{ R(\theta) \leq \inf_{\lambda \in \RR_+}
\Phi_{\frac{\lambda}{N}}^{-1} \biggl[
r(\theta) - \frac{\log(\epsilon)}{\lambda} \biggr] \Biggr\}
\geq 1 - \epsilon.
$$
\end{lemma}

We will illustrate throughout these notes the bounds we prove with
a small numerical example: in the case where $N = 1000$,
$\epsilon = 0.01$ and $r(\theta) = 0.2$,
we get with a confidence level of $0.99$ that $ R(\theta) \leq .2402$,
this being obtained for $\lambda = 234$.

Now, to proceed towards the analysis of posterior
distributions, let us put $U_{\lambda}(\theta,\break \omega) =
\lambda \Bigl[ \Phi_{\frac{\lambda}{N}} \bigl[ R(\theta) \bigr]
- r(\theta, \omega) \Bigr]
$  for short, and let us consider some prior probability distribution
$\pi \in \C{M}_+^1(\Theta, \C{T})$.
A proper choice of $\pi$ will be an important question,
underlying much of the material presented in this
monograph, so for the time being, let us only
say that we will let this choice be as open
as possible by writing inequalities which hold
for \emph{any}\ choice of $\pi$ . Let us insist on
the fact that when we say that $\pi$ is a prior distribution,
we mean that it \emph{does not}\ depend on the training sample
$(X_i, Y_i)_{i=1}^N$. The quantity of interest to obtain
the bound we are looking for is
$\log \Bigl\{ \PP \Bigl[ \pi \bigl[ \exp ( U_{\lambda}) \bigr] \Bigr] \Bigr\}$.
Using Fubini's theorem
for non-negative functions, we see that
$$
\log \Bigl\{ \PP \Bigl[ \pi \bigl[ \exp ( U_{\lambda}) \bigr] \Bigr] \Bigr\}
= \log \Bigl\{ \pi \Bigl[ \PP \bigl[ \exp ( U_{\lambda} ) \bigr] \Bigr]
\Bigr\} \leq 0.
$$

To relate this quantity
to the expectation $\rho(U_{\lambda})$ with respect to
any posterior distribution $\rho: \Omega \rightarrow \C{M}_+^1(\Theta)$,
we will use the properties of the Kullback divergence
$\C{K}(\rho, \pi)$
of $\rho$ with respect to $\pi$, which is defined as
$$
\C{K}(\rho, \pi) = \begin{cases}
\int \log( \frac{d\rho}{d \pi}) d \rho, & \begin{array}[t]{l}\text{when $\rho$
is absolutely continuous} \\ \text{with respect to $\pi$,}\end{array}\\
+ \infty, & \text{ \,otherwise}.
\end{cases}
$$
The following lemma shows in which sense the Kullback divergence
function can be thought of as the dual of the $\log$-Laplace
transform.
\begin{lemma} \mypoint
\label{lemma1.3}
For any bounded measurable function $h: \Theta \rightarrow \RR$,
and any probability distribution $\rho \in \C{M}_+^1(\Theta)$
such that $\C{K}(\rho,\pi) < \infty$,
$$
\log \bigl\{ \pi \bigl[ \exp (h) \bigr]
\bigr\} = \rho(h)
- \C{K}(\rho,\pi) + \C{K}(\rho, \pi_{\exp(h)}),
$$
where by definition $\ds \frac{d \pi_{\exp(h)}}{d \pi} =
\frac{\exp[h(\theta)]}{\pi[\exp(h)]}$. Consequently
$$
\log \bigl\{ \pi \bigl[ \exp (h)] \bigr] \bigr\}
= \sup_{\rho \in \C{M}_+^1(\Theta)} \rho (h)
- \C{K}(\rho, \pi).
$$
\end{lemma}

The proof is just a matter of writing down the definition
of the quantities involved and using the fact that the Kullback
divergence function is non-negative, and
can be found in \citet[page 160]{Cat7}.
In the duality between measurable functions and probability measures,
we thus see that the $\log$-Laplace transform with respect to
$\pi$ is the Legendre transform of the Kullback divergence function
with respect to $\pi$.
Using this, we get
$$
\PP \Bigl\{ \exp \bigl\{ \sup_{\rho \in \C{M}_+^1(\Theta)}
\rho [ U_{\lambda}(\theta) ] - \C{K}(\rho, \pi) \bigr\} \Bigr\} \leq 1,
$$
which, combined with the convexity of $\lambda \Phi_{\frac{\lambda}{N}}$, proves
the basic inequality we were looking for.
\begin{thm}
\label{thm2.3}
\mypoint For any real constant $\lambda$,
\begin{multline*}
\PP \biggl\{ \exp \biggl[
\sup_{\rho \in \C{M}_+^1(\Theta)} \lambda
\Bigl[ \Phi_{\frac{\lambda}{N}}\bigl[ \rho(R) \bigr]
- \rho(r) \Bigr] - \C{K}(\rho,\pi) \biggr] \biggr\}
\\ \leq
\PP \biggl\{ \exp \biggl[
\sup_{\rho \in \C{M}_+^1(\Theta)} \lambda
\Bigl[ \rho \bigl( \Phi_{\frac{\lambda}{N}}\!\circ\!R \bigr)
- \rho(r) \Bigr] - \C{K}(\rho,\pi) \biggr] \biggr\}
\leq 1.
\end{multline*}
\end{thm}

We insist on the fact that in this theorem, we take
a supremum in $\rho \in  \C{M}_+^1(\Theta)$ \emph{inside}\
the expectation with respect to $\PP$, the sample distribution.
This means that the proved inequality holds for any $\rho$
depending on the training sample, that is for any
posterior distribution: indeed, measurability questions
set aside,
\begin{multline*}
\PP \biggl\{ \exp \biggl[ \sup_{\rho \in \C{M}_+^1(\Theta)}
\lambda \Bigl[ \rho \bigl[ U_{\lambda}(\theta) \bigr] - \C{K}(\rho, \pi)
\Bigr] \biggr] \biggr\}
\\ = \sup_{ \rho: \Omega \rightarrow \C{M}_+^1(\Theta) }
\PP \biggl\{ \exp \biggl[ \lambda \bigl[ \rho \bigl[ U_{\lambda}
(\theta) \bigr] - \C{K}(\rho, \pi) \bigr] \Bigr] \biggr] \biggr\},
\end{multline*}
and more formally,
\begin{multline*}
\sup_{ \rho: \Omega \rightarrow \C{M}_+^1(\Theta) }
\PP \biggl\{ \exp \biggl[ \lambda \Bigl[ \rho \bigl[ U_{\lambda}
(\theta) \bigr] - \C{K}(\rho, \pi) \bigr] \Bigr] \biggr] \biggr\}
\\ \leq
\PP \biggl\{ \exp \biggl[ \sup_{\rho \in \C{M}_+^1(\Theta)}
\lambda \Bigl[ \rho \bigl[ U_{\lambda}(\theta) \bigr] - \C{K}(\rho, \pi)
\Bigr] \biggr] \biggr\},
\end{multline*}
where the supremum in $\rho$ taken in the left-hand side is restricted
to \emph{regular}\ conditional probability distributions.

The following sections will show how to use this theorem.
\section{Non local bounds}
At least three sorts of bounds can be deduced from Theorem \ref{thm2.3}.

The most interesting ones with which to build estimators and tune parameters,
as well as the first that have been considered in the development of
the PAC-Bayesian approach, are deviation bounds. They provide an
empirical upper bound for $\rho(R)$ --- that is a bound which can be computed from
observed data --- with some probability $1 - \epsilon$, where $\epsilon$
is a presumably small and tunable parameter setting the desired confidence
level.

Anyhow, most
of the results about the convergence speed of estimators to be found
in the statistical literature are concerned with the expectation $\PP \bigl[
\rho(R) \bigr]$, therefore it is also enlightening to bound this quantity.
In order to know at which rate it may be approaching $\inf_{\Theta} R$,
a non-random upper bound is required, which will relate the average of
the expected risk $\PP \bigl[ \rho(R) \bigr]$ with the properties of
the contrast function $\theta \mapsto R(\theta)$.

Since the values of constants do matter a lot when a bound is to be used
to select between various estimators using classification models of various
complexities, a third kind of bound, related to the first, may be considered
for the sake of its hopefully better constants: we will call them
\emph{unbiased empirical bounds},\ to stress the fact that they provide some
empirical quantity whose expectation under $\PP$ can be proved to
be an upper bound for $\PP \bigl[ \rho(R) \bigr]$, the average expected
risk. The price to pay for these better constants is of course the lack
of formal guarantee given by the bound: two random variables whose
expectations are ordered in a certain way may very well be ordered
in the reverse way with a large probability, so that basing the
estimation of parameters or the selection of an estimator on some
unbiased empirical bound is a hazardous business. Anyhow, since it is
common practice to use the inequalities provided by mathematical statistical
theory while replacing the proven constants with smaller values showing
a better practical efficiency, considering unbiased empirical bounds
as well as deviation bounds provides an indication about how much
the constants may be decreased while not violating the theory too
much.

\subsection{Unbiased empirical bounds}
Let $\rho: \Omega
\rightarrow \C{M}_+^1(\Theta)$ be some fixed (and arbitrary)
posterior distribution, describing some randomized estimator $\wh{\theta}
: \Omega \rightarrow \Theta$.
As we already mentioned, in these notes a posterior distribution
will always be a regular conditional probability measure. By this
we mean that
\begin{itemize}
\item for any $A \in \C{T}$, the map $\omega \mapsto \rho (\omega, A)
: \bigl(\Omega, ( \C{B} \otimes
\C{B}')^{\otimes N} \bigr) \rightarrow \RR_+$
is assumed to be measurable;
\item for any $\omega \in \Omega$, the map $A \mapsto \rho(\omega, A):
\C{T} \rightarrow \RR_+$
is assumed to be a probability measure.
\end{itemize}
We will also assume without further notice that the $\sigma$-algebras
we deal with are always countably generated.
The technical implications of these assumptions are standard
and discussed for instance in \citet[pages 50-54]{Cat7},
where, among other things, a detailed proof of the decomposition
of the Kullback Liebler divergence is given.

Let us restrict to the case when the constant $\lambda$ is positive.
We get from Theorem \ref{thm2.3} that
\begin{equation}
\label{eq2.2.1bis}
\exp \biggl[ \lambda \Bigl\{ \Phi_{\frac{\lambda}{N}}
\Bigl[ \PP \bigl[ \rho(R) \bigr]
\Bigr] - \PP \bigl[ \rho(r) \bigr] \Bigr\} - \PP \bigl[\C{K}(\rho, \pi)
\bigr] \biggr]
\leq 1,
\end{equation}
where we have used the convexity of the $\exp$ function and of $\Phi_{\frac{
\lambda}{N}}$.
Since we have restricted our attention to positive values of the constant $\lambda$,
equation \eqref{eq2.2.1bis} can also be written
$$
\PP \bigl[ \rho(R) \bigr]
\leq \Phi_{\frac{\lambda}{N}}^{-1} \Bigl\{
\PP \bigl[ \rho(r) + \lambda^{-1} \C{K}(\rho,\pi) \bigr] \Bigr\},
$$
leading to
\begin{thm}
\label{thm2.4}
\mypoint For any posterior distribution $\rho: \Omega \rightarrow \C{M}_+^1(\Theta)$,
for any positive parameter $\lambda$,
\begin{align*}
\PP \bigl[ \rho (R) \bigr]
& \leq \frac{\ds
1 - \exp \Bigl[ - N^{-1} \PP \bigl[
\lambda \rho(r) + \C{K}(\rho,\pi) \bigr]  \Bigr] }{\ds 1 - \exp( - \tfrac{\lambda}{N})} \\
& \leq \PP \Biggl\{ \frac{\lambda}{N \bigl[ 1 - \exp( - \frac{\lambda}{N}) \bigr]}
\left[ \rho(r) + \frac{\C{K}(\rho,\pi)}{\lambda} \right] \Biggr\}.
\end{align*}
\end{thm}

The last inequality provides the \emph{unbiased empirical upper
bound}\  for $\rho(R)$ we were looking for, meaning that the expectation of
\linebreak $\frac{\lambda}{N \bigl[ 1 - \exp( - \frac{\lambda}{N}) \bigr]}
\left[ \rho(r) + \frac{\C{K}(\rho,\pi)}{\lambda} \right]$
is larger than the expectation of $\rho(R)$. Let us notice that
$1 \leq \frac{\lambda}{N \bigl[ 1 - \exp( - \frac{\lambda}{N}) \bigr]} \leq
\bigl[ 1 - \frac{\lambda}{2N} \bigr]^{-1}$ and therefore that this
coefficient is close to $1$ when $\lambda$ is significantly smaller
than $N$.

If we are ready to believe in this bound (although this belief is not
mathematically well founded, as we already mentioned), we can use
it to optimize $\lambda$ and to choose $\rho$. While the optimal choice
of $\rho$ when $\lambda$ is fixed is, according to Lemma
\thmref{lemma1.3}, to take it equal to $\pi_{\exp( - \lambda r)}$,
a Gibbs posterior distribution, as it is sometimes called, we may for
computational reasons be more interested in choosing $\rho$ in some
other class of posterior distributions.

For instance, our real interest
may be to select some non-randomized estimator from a
family $\wtheta_m: \Omega \rightarrow
\Theta_m$, $m \in M$, of possible ones, where $\Theta_m$ are
measurable subsets of $\Theta$ and where $M$ is an arbitrary (non necessarily
countable) index set. We may for instance think of
the case when $\wtheta_m \in \arg\min_{\Theta_m} r$.
We may slightly randomize the estimators to start with,
considering for any $\theta \in \Theta_m$ and any $m \in M$,
$$
\Delta_m(\theta) = \Bigl\{ \theta' \in \Theta_m:
\bigl[ f_{\theta'}(X_i) \bigr]_{i=1}^N = \bigl[ f_{\theta}(X_i) \bigr]_{i=1}^N
\Bigr\},
$$
and defining $\rho_m$ by the formula
$$
\frac{d \rho_m}{d \pi} (\theta) = \frac{\B{1}\bigl[ \theta \in \Delta_m(\wtheta_m)
\bigr]}{\pi \bigl[ \Delta_m(\wtheta_m) \bigr]}.
$$
Our posterior minimizes $\C{K}(\rho, \pi)$ among those distributions
whose support is restricted to the values of $\theta$
in $\Theta_m$ for which the classification rule $f_{\theta}$
is identical to the estimated one $f_{\wtheta_m}$ on
the observed sample.
Presumably, in many practical situations, $f_{\theta}(x)$
will be $\rho_m$ almost surely identical to
$f_{\wtheta_m}(x)$ when $\theta$ is drawn from
$\rho_m$, for the vast majority of the values of $x \in \C{X}$
and all the sub-models $\Theta_m$ not plagued with too much overfitting
(since this is by construction the case when $x \in \{ X_i: i = 1, \dots, N \}$).
Therefore replacing $\wtheta_m$ with $\rho_m$ can be expected to be
a minor change in many situations. This change by the way can be
estimated in the (admittedly not so common) case when the
distribution of the patterns $(X_i)_{i=1}^N$ is known.
Indeed, introducing the pseudo distance
\begin{equation}
\label{eq1.1.2}
D(\theta, \theta') = \frac{1}{N} \sum_{i=1}^N
\PP \bigl[ f_{\theta}(X_i) \neq f_{\theta'}(X_i) \bigr], \qquad \theta, \theta' \in
\Theta,
\end{equation}
one immediately sees that $R(\theta') \leq R(\theta) + D(\theta, \theta')$,
for any $\theta, \theta' \in \Theta$, and
therefore that
$$
R(\wtheta_m) \leq \rho_m(R) + \rho_m\bigl[ D(\cdot,\wtheta_m) \bigr].
$$
Let us notice also that in the case where $\Theta_m
\subset \RR^{d_m}$, and $R$ happens to be convex on
$\Delta_m(\wtheta_m)$, then $\rho_m(R) \geq R \bigl[
\int \theta \rho_m(d \theta)\bigr]$, and we can replace
$\wtheta_m$ with $\T_m = \int \theta \rho_m( d\theta)$,
and obtain bounds for $R(\T_m)$.
This is not a very heavy assumption about $R$, in the case
where we consider $\wtheta_m \in \arg\min_{\Theta_m} r$.
Indeed, $\wtheta_m$, and therefore $\Delta_m(\wtheta_m)$,
will presumably be  close to $\arg\min_{\Theta_m} R$,
and requiring a function to be convex in the neighbourhood of
its minima is not a very strong assumption.

Since $r(\wtheta_m) = \rho_m(r)$,
and $\C{K}(\rho_m, \pi) = - \log \bigl\{
\pi\bigl[ \Delta_m(\wtheta_m) \bigr] \bigr\}$,
our unbiased empirical upper
bound in this context reads as
$$
\frac{\lambda}{N\bigl[ 1 - \exp( - \frac{\lambda}{N})\bigr]} \left\{
r(\wtheta_m) - \frac{\log\bigl\{ \pi \bigl[ \Delta_m(\wtheta_m) \bigr]
\bigr\}}{\lambda} \right\}.
$$
Let us notice that we obtain a complexity factor $- \log \bigl\{
\pi \bigl[ \Delta_m(\wtheta_m) \bigr] \bigr\}$ which may be
compared with the Vapnik--Cervonenkis dimension. Indeed, in the
case of binary classification, when using a classification model
with Vapnik--Cervonenkis dimension not greater than $h_m$, that is when any subset
of $\C{X}$ which can be split in any arbitrary way by some
classification rule $f_{\theta}$ of the model $\Theta_m$ has at most $h_m$
points, then
$$
\bigl\{ \Delta_m(\theta): \theta \in \Theta_m  \bigr\}
$$
is a partition of $\Theta_m$ with at most $\left( \frac{eN}{h_m} \right)^{h_m}$
components: these facts, if not already familiar to the
reader, will be proved in Theorems \ref{th1} and
\thmref{th2}. Therefore
$$
\inf_{\theta \in \Theta_m} - \log \bigl\{
\pi \bigl[ \Delta_m(\theta) \bigr] \bigr\} \leq h_m \log \left( \frac{e N}{h_m}
\right) - \log \bigl[ \pi(\Theta_m) \bigr].
$$
Thus, if the model and prior distribution are well suited to the classification
task, in the sense that there is more ``room'' (where room is measured with $\pi$)
between the two clusters defined by $\wtheta_m$ than between other partitions
of the sample of patterns $(X_i)_{i=1}^N$, then we will have
$$
-\log \bigl\{ \pi \bigl[ \Delta_m(\wtheta) \bigr] \bigr\} \leq h_m
\log \left( \frac{e N}{h_m} \right) - \log \bigl[ \pi(\Theta_m) \bigr].
$$
\newcommand{\wm}{\widehat{m}}
An optimal value $\wm$ may be selected so that
$$
\wm \in \arg\min_{m \in M} \left\{ \inf_{\lambda \in \RR_+}
\frac{\lambda}{N\bigl[ 1 - \exp( - \frac{\lambda}{N})\bigr]} \left(
r(\wtheta_m) - \frac{\log\bigl\{ \pi \bigl[ \Delta_m(\wtheta_m) \bigr] \bigr\}}{\lambda} \right) \right\}.
$$
Since $\rho_{\wm}$ is still another posterior distribution, we can be sure that
\begin{multline*}
\PP \Bigl\{ R(\wtheta_{\wm}) - \rho_{\wm} \bigl[ D(\cdot, \wtheta_{\wm}) \bigr]\Bigr\}
\leq \PP \bigl[ \rho_{\wm}(R) \bigr]
\\ \leq \inf_{\lambda \in \RR_+} \PP
\left\{ \frac{\lambda}{N\bigl[ 1 - \exp( - \frac{\lambda}{N})\bigr]} \left(
r(\wtheta_{\wm}) - \frac{\log\bigl\{ \pi \bigl[ \Delta_{\wm}
(\wtheta_{\wm}) \bigr] \bigr\}}{\lambda} \right) \right\}.
\end{multline*}
Taking the infimum in $\lambda$ inside the expectation with respect to $\PP$
would be possible at the price of some supplementary technicalities
and a slight increase of the bound that we prefer to postpone to the discussion
of deviation bounds, since they are the only ones to provide a rigorous mathematical
foundation to the adaptive selection of estimators.

\subsection{Optimizing explicitly the exponential parameter $\lambda$}
In this section we address some technical issues we think
helpful to the understanding of Theorem \ref{thm2.4}
(page \pageref{thm2.4}): namely to investigate
how the upper bound it provides could be optimized, or at least approximately
optimized, in $\lambda$. It turns out that this can be done quite
explicitly.

So we will consider in this discussion the
posterior distribution $\rho: \Omega \rightarrow \C{M}_+^1(\Theta)$
to be fixed, and our aim will be to eliminate the constant $\lambda$
from the bound by choosing its value in some nearly optimal way as
a function of $\PP\bigl[ \rho(r) \bigr]$, the average of the
empirical risk, and of
$\PP \bigl[ \C{K}(\rho, \pi) \bigr]$, which controls overfitting.

Let the bound be written as
$$
\varphi ( \lambda) = \bigl[ 1 - \exp( - \tfrac{\lambda}{N}) \bigr]^{-1}
\left\{ 1 - \exp \Bigl[ - \tfrac{\lambda}{N} \PP \bigl[ \rho(r) \bigr]
- N^{-1}\PP \bigl[ \C{K}(\rho,\pi) \bigr] \Bigr] \right\}.
$$
We see that
$$
N \frac{\partial}{\partial \lambda} \log \bigl[ \varphi(\lambda) \bigr]
= \frac{\PP\bigl[\rho(r)\bigr]}{\exp \Bigl[ \frac{\lambda}{N} \PP\bigl[\rho(r)\bigr]
+ N^{-1} \PP\bigl[ \C{K}(\rho, \pi) \bigr] \Bigr] - 1} -
\frac{1}{\exp(\frac{\lambda}{N}) - 1}.
$$
Thus, the optimal value for $\lambda$ is such that
$$
\bigl[ \exp( \tfrac{\lambda}{N}) - 1 \bigr] \PP \bigl[\rho(r)\bigr]
= \exp \Bigl[ \tfrac{\lambda}{N} \PP \bigl[ \rho(r) \bigr] + N^{-1}
\PP \bigl[ \C{K}(\rho, \pi) \bigr] \Bigr] - 1.
$$
Assuming that $1 \gg \frac{\lambda}{N} \PP \bigl[ \rho(r) \bigr]
\gg \frac{\PP [ \C{K}(\rho,\pi) ]}{N}$,
and keeping only higher order terms, we are led to choose
$$
\lambda = \sqrt{ \frac{2 N \PP \bigl[ \C{K}(\rho,\pi) \bigr]}{\PP \bigl[ \rho(r) \bigr]
\bigl\{ 1 - \PP \bigl[\rho(r) \bigr] \bigr\}}},
$$
obtaining
\begin{thm}
\label{thm1.6}
\mypoint For any posterior distribution $\rho: \Omega \rightarrow \C{M}_+^1(\Theta)$,
$$
\PP \bigl[ \rho(R) \bigr] \leq
\frac{ 1 - \exp \left\{ - \sqrt{\frac{ 2 \PP [ \C{K}(\rho,\pi) ] \PP [
\rho(r)]}{N \{ 1 - \PP [ \rho(r) ] \}}} -
\frac{\PP [ \C{K}(\rho,\pi) ]}{N} \right\}}{
1 - \exp \left\{ - \sqrt{ \frac{ 2 \PP [ \C{K}(\rho,\pi) ]}{
N \PP [ \rho(r) ] \{1 - \PP [ \rho(r) ] \}}}
\right\}}.
$$
\end{thm}

This result of course is not very useful in itself, since neither of the
two quantities $\PP\bigl[ \rho(r) \bigr]$ and $\PP\bigl[ \C{K}(\rho, \pi) \bigr]$
are easy to evaluate. Anyhow it gives a hint that replacing them boldly
with $\rho(r)$ and $\C{K}(\rho, \pi)$ could produce something close to
a legitimate empirical upper bound for $\rho(R)$. We will see in the subsection
about deviation bounds that this is indeed essentially true.

Let us remark that in the third chapter of this monograph,
we will see another way of bounding
$$
\inf_{\lambda \in \RR_+} \Phi_{\frac{\lambda}{N}}^{-1}
\left(q + \frac{d}{\lambda}\right),\text{ leading to}
$$
\begin{thm}\mypoint
\label{thm1.1.6}
For any prior distribution $\pi \in \C{M}_+^1(\Theta)$,
for any posterior distribution $\rho: \Omega \rightarrow \C{M}_+^1(\Theta)$,
\begin{multline*}
\PP \bigl[ \rho(R) \bigr] \leq
\left(1 + \frac{2\PP\bigl[\C{K}(\rho, \pi) \bigr]}{N}\right)^{-1}
\Biggl\{ \PP \bigl[ \rho(r) \bigr] + \frac{\PP\bigl[\C{K}(\rho, \pi)\bigr]}{N}
\\* \shoveright{+ \sqrt{ \frac{2 \PP \bigl[ \C{K}(\rho, \pi) \bigr] \PP \bigl[ \rho(r) \bigr]
\bigl\{ 1 - \PP \bigl[ \rho(r) \bigr] \bigr\}}{N} + \frac{
\PP\bigl[\C{K}(\rho,\pi)\bigr]^2}{N^2}} \Biggr\},}\\
\text{as soon as }
\PP \bigl[ \rho(r)  \bigr] + \sqrt{ \frac{\PP \bigl[ \C{K}(\rho, \pi) \bigr]}{2N}}
\leq \frac{1}{2},\\
\text{and }
\PP\bigl[\rho(R)\bigr] \leq \PP\bigl[\rho(r)\bigr] +
\sqrt{\frac{\PP\bigl[\C{K}(\rho,\pi)\bigr]}{2N}} \text{ otherwise.}
\end{multline*}
\end{thm}

This theorem enlightens the influence of three terms on the average expected
risk:

$\bullet$ the average empirical risk, $\PP \bigl[ \rho(r) \bigr]$, which
as a rule will decrease as the size of the classification model increases,
acts as a \emph{bias}\ term, grasping the ability of the model to
account for the observed sample itself;
\begin{itemize}
\item a \emph{variance}\ term $\frac{1}{N}\PP \bigl[ \rho(r) \bigr] \bigl\{ 1 - \PP \bigl[ \rho(r) \bigr]
\bigr\}$ is due to the random fluctuations of $\rho(r)$;

\item
a \emph{complexity}\ term $\PP \bigl[ \C{K}(\rho, \pi) \bigr]$, which as a rule will
increase with the size of the classification model,
eventually acts as a multiplier of the variance term.
\end{itemize}

We observed numerically that the bound provided by Theorem \ref{thm1.6}
is better than the more classical Vapnik-like bound of Theorem \ref{thm1.1.6}.
For instance, when $N = 1000$, $\PP\bigl[\rho(r) \bigr] = 0.2$
and $\PP\bigl[\C{K}(\rho,\pi)\bigr] = 10$, Theorem \ref{thm1.6} gives a bound
lower than $0.2604$, whereas the more classical Vapnik-like approximation
of Theorem \ref{thm1.1.6} gives a bound larger than $0.2622$. Numerical simulations tend to suggest
the two bounds are always ordered in the same way,
although this could be a little tedious
to prove mathematically.
\eject

\subsection{Non random bounds}
It is time now to come to less tentative results and
see how far is the average expected error rate $\PP \bigl[ \rho(R) \bigr]$
from its best possible value $\inf_{\Theta} R$.

Let us notice first that
$$
\lambda \rho(r) + \C{K}(\rho,\pi) =
\C{K}(\rho, \pi_{\exp( - \lambda r)})
- \log \Bigl\{ \pi \bigl[ \exp ( - \lambda r) \bigr] \Bigr\}.
$$
Let us remark moreover that $r \mapsto \log \Bigl[ \pi \bigl[
\exp ( - \lambda r) \bigr] \Bigr]$ is a convex functional,
a property which  from
a technical point of view can be dealt with in the following way:
\begin{multline}
\label{eq1.1.3Ter}
\PP \Bigl\{ \log \Bigl[ \pi \bigl[ \exp ( - \lambda r) \bigr]
\Bigr] \Bigr\}
= \PP \Bigl\{ \sup_{\rho \in \C{M}_+^1(\Theta)}
- \lambda \rho(r) - \C{K}(\rho,\pi) \Bigr\}
\\ \geq \sup_{\rho \in \C{M}_+^1(\Theta)} \PP \Bigl\{
- \lambda \rho(r) - \C{K}(\rho, \pi) \Bigr\}
= \sup_{\rho \in \C{M}_+^1(\Theta)} - \lambda \rho(R) - \C{K}(\rho, \pi)
\\ = \log \Bigl\{ \pi \bigl[ \exp ( - \lambda R) \bigr] \Bigr\}
= - {\textstyle \int_{0}^{\lambda}} \pi_{\exp( - \beta R)}(R) d \beta.
\end{multline}
These remarks applied to Theorem \ref{thm2.4} lead to
\begin{thm}
\label{thm2.5}
\mypoint For any posterior distribution $\rho: \Omega \rightarrow \C{M}_+^1(\Theta)$,
for any positive parameter $\lambda$,
\begin{align*}
\PP \bigl[ \rho(R) \bigr] &
\leq
\frac{1 - \exp \left\{ - \frac{1}{N} \int_0^{\lambda} \pi_{\exp( - \beta R)}(R)
d \beta - \frac{1}{N} \PP \bigl[ \C{K}(\rho, \pi_{\exp(- \lambda r)}) \bigr]
\right\}}{
1 - \exp( - \frac{\lambda}{N})}
\\ & \leq \frac{1}{N \bigl[ 1 - \exp ( - \frac{\lambda}{N}) \bigr]}
\Bigl\{ {\textstyle \int_0^{\lambda}} \pi_{\exp( - \beta R)}(R) d \beta
+ \PP \bigl[ \C{K}(\rho, \pi_{\exp( - \lambda r)}) \bigr]  \Bigr\}.
\end{align*}
\end{thm}

This theorem is particularly well suited to the case
of the Gibbs posterior distribution $\rho = \pi_{\exp(- \lambda r)}$,
where the entropy factor cancels and where
$\PP \bigl[ \pi_{\exp( - \lambda r)}(R) \bigr]$
is shown to get close to $\inf_{\Theta} R$ when $N$ goes to $+ \infty$,
as soon as $\lambda/N$ goes to $0$ while $\lambda$ goes to $+ \infty$.

We can elaborate on Theorem \ref{thm2.5} and define a notion of dimension
of $(\Theta, R)$, with margin $\eta \geq 0$ putting
\begin{multline}
\label{eq1.1.3Bis}
d_{\eta} (\Theta, R) = \sup_{\beta \in \RR_+} \beta \bigl[
\pi_{\exp( - \beta R)}(R) - \ess\inf_{\pi} R - \eta \bigr]
\\ \leq - \log \Bigl\{ \pi \bigl[ R \leq \ess\inf_{\pi} R + \eta \bigr] \Bigr\}.
\end{multline}
This last inequality can be established by the chain of inequalities:
\begin{multline*}
\beta \pi_{\exp( - \beta R)}(R) \leq {\textstyle \int_0^{\beta}}
\pi_{\exp( - \gamma R)}(R) d \gamma =
- \log \Bigl\{ \pi \bigl[
\exp ( - \beta R) \bigr] \Bigr\} \\ \leq \beta \Bigl( \ess \inf_{\pi} R
+ \eta \Bigr) - \log \Bigl[ \pi\bigl( R \leq \ess \inf_{\pi} R + \eta
\bigr) \Bigr],
\end{multline*}
where we have used successively the fact that $\lambda \mapsto
\pi_{\exp( - \lambda R)}(R)$ is decreasing (because it is
the derivative of the concave function $ \lambda \mapsto -\log
\bigl\{ \pi \bigl[ \exp( - \lambda R) \bigr] \bigr\}$)
and the fact that the exponential function takes positive values.

In typical ``parametric'' situations $d_0(\Theta, R)$ will be finite,
and in all circumstances $d_{\eta}(\Theta, R)$
will be finite for any $\eta > 0$ (this is a direct consequence
of the definition of the essential infimum).
Using this notion of dimension, we see that
\begin{multline*}
\int_{0}^{\lambda} \pi_{\exp( -\beta R)}(R) d \beta \leq
\lambda  \bigl( \ess \inf_{\pi} R  + \eta \bigr)
\\ \shoveright{+ \int_{0}^{\lambda} \left[ \frac{d_{\eta}}{\beta} \wedge (1 - \ess
\inf_{\pi} R - \eta)
\right] d \beta \quad}\\ = \lambda \bigl(\ess \inf_{\pi} R + \eta \bigr) +
d_{\eta}(\Theta, R) \log \left[ \frac{e \lambda}{d_{\eta}(\Theta, R)}
\bigl(1 - \ess \inf_{\pi} R - \eta \bigr) \right].
\end{multline*}
This leads to
\begin{cor}
With the above notation, for any margin $\eta \in \RR_+$,
for any posterior distribution
$\rho: \Omega \rightarrow \C{M}_+^1(\Theta)$,
$$
\PP \bigl[ \rho(R) \bigr] \leq \inf_{\lambda \in \RR_+}
\Phi_{\frac{\lambda}{N}}^{-1} \left[ \ess \inf_{\pi} R + \eta +
\frac{d_{\eta}}{\lambda} \log \left( \frac{e \lambda}{d_{\eta}} \right)
+ \frac{\PP \bigl\{ \C{K}\bigl[\rho, \pi_{\exp( - \lambda r)}\bigr] \bigr\}}{\lambda}
\right].
$$
\end{cor}

If one wants a posterior distribution with a small support,
the theorem can also be applied to the case when $\rho$ is obtained by truncating $\pi_{\exp ( - \lambda r)}$
to some level set to reduce its support: let
$\Theta_{p} = \{ \theta \in \Theta: r(\theta) \leq p \}$,
and let us define for any $q \in )0,1)$ the level
$p_{q} = \inf \{ p: \pi_{\exp( - \lambda r)}(\Theta_p) \geq
q \}$,
let us then define $\rho_{q}$ by its density
$$
\frac{\ds d \rho_q}{\ds d \pi_{\exp(- \lambda r)}} (\theta)
= \frac{\ds \B{1}(\theta \in \Theta_{p_q})}{\ds \pi_{\exp( - \lambda r)}(\Theta_{p_q})},
$$
then $\rho_0 = \pi_{\exp ( - \lambda r)}$ and for any $q \in (0,1($,
\begin{align*}
\PP \bigl[ \rho_q(R) \bigr] &
\leq
\frac{1 - \exp \left\{ - \frac{1}{N} \int_0^{\lambda} \pi_{\exp( - \beta R)}(R)
d \beta - \frac{\log(q)}{N}
\right\}}{
1 - \exp( - \frac{\lambda}{N})} \\
& \leq \frac{1}{N \bigl[ 1 - \exp ( - \frac{\lambda}{N}) \bigr]}
\Bigl\{ {\textstyle \int_0^{\lambda}} \pi_{\exp( - \beta R)}(R) d \beta
- \log(q) \Bigr\}.
\end{align*}

\subsection{Deviation bounds}
They provide results holding under the distribution $\PP$
of the sample with probability at least $1 - \epsilon$, for any
given confidence level, set by the choice of $\epsilon \in )0, 1($.
Using them is the only way to be quite (i.e. with probability $1-\epsilon$)
sure to do the right thing,
although this right thing may be over-pessimistic, since
deviation upper bounds are larger than corresponding non-biased bounds.

Starting again
from Theorem \ref{thm2.3} (page \pageref{thm2.3}),
and using Markov's inequality $\PP \bigl[
\exp (h) \geq 1 \bigr] \leq \PP \bigl[ \exp(h) \bigr]$, we
obtain
\begin{thm}
\label{thm2.7}
\mypoint For any positive parameter $\lambda$, with $\PP$ probability at least $1 - \epsilon$,
for any posterior distribution $\rho: \Omega \rightarrow
\C{M}_+^1(\Theta)$,
\begin{align*}
\rho(R) & \leq \Phi_{\frac{\lambda}{N}}^{-1} \left\{
\rho(r) + \frac{\C{K}(\rho, \pi) - \log(\epsilon)}{\lambda} \right\}\\
& = \frac{\ds 1 - \exp \left\{ - \frac{\lambda \rho(r)}{N}
- \frac{\C{K}(\rho,\pi) - \log(\epsilon)}{N} \right\}}{\ds 1
- \exp\bigl( - \tfrac{\lambda}{N}\bigr)} \\
& \leq \frac{\lambda}{\ds N \left[ 1 - \exp \left( -
\tfrac{\lambda}{N} \right) \right]}
\left[ \rho(r)+ \frac{ \C{K}(\rho, \pi) - \log(\epsilon)}{\lambda}
\right].
\end{align*}
\end{thm}

We see that for a fixed value of the parameter $\lambda$,
the upper bound is optimized when the posterior is chosen
to be the Gibbs distribution $\rho = \pi_{\exp( - \lambda r)}$.

In this theorem, we have bounded $\rho(R)$, the average expected
risk of an estimator $\wh{\theta}$ drawn from the posterior
$\rho$. This is what we will do most of the time in this
study. This is the error rate we will get if we classify
a large number of test patterns, drawing a new $\wh{\theta}$
for each one. However, we can also be interested in
the error rate we get if we draw only one $\wh{\theta}$
from $\rho$ and use this single draw of $\wh{\theta}$
to classify a large number of test patterns. This
error rate is $R(\wh{\theta})$. To state a result
about its deviations, we can start back from
Lemma \ref{lemma1.1.1} (page \pageref{lemma1.1.1})
and integrate it with respect to the prior distribution
$\pi$ to get for any real constant $\lambda$
$$
\PP \biggl\{ \pi \biggl[ \exp \Bigl\{ \lambda \Bigl[
\Phi_{\frac{\lambda}{N}} \bigl( R \bigr) - r \Bigr] \Bigr\} \biggr] \biggr\}
\leq 1.
$$
For any posterior distribution $\rho: \Omega \rightarrow
\C{M}_+^1(\Theta)$, this can be rewritten as
$$
\PP \biggl\{ \rho \biggl[ \exp \Bigl\{ \lambda \Bigl[
\Phi_{\frac{\lambda}{N}} \bigl( R \bigr) - r
\Bigr] - \log \bigl( \tfrac{d \rho}{d \pi} \bigr)
+ \log(\epsilon) \Bigr] \Bigr\}
\biggr] \biggr\} \leq \epsilon,
$$
proving
\begin{thm}
\label{thm1.11}
For any positive real parameter $\lambda$, for any posterior
distribution $\rho: \Omega \rightarrow \C{M}_+^1(\Theta)$,
with $\PP \rho$ probability at least $1 - \epsilon$,
\begin{align*}
R(\wh{\theta}\,) & \leq \Phi_{\frac{\lambda}{N}}^{-1}
\biggl\{ r(\wh{\theta}\,) + \lambda^{-1} \log \biggl(
\epsilon^{-1} \frac{d \rho}{d \pi} \biggr) \biggr\} \\
& \leq \frac{\lambda}{N \bigl[ 1 - \exp ( - \frac{\lambda}{N})
\bigr]} \biggl[ r(\wh{\theta}\,)
+ \lambda^{-1} \log \biggl( \epsilon^{-1} \frac{d \rho}{d \pi} \biggr)
\biggr].
\end{align*}
\end{thm}
Let us remark that the bound provided here is the exact counterpart
of the bound of Theorem \ref{thm2.7}, since $\log \bigl(
\frac{d \rho}{d \pi} \bigr)$ appears as a \emph{disintegrated}\
version of the divergence $\C{K}(\rho, \pi)$.
The parallel between the two theorems is particularly
striking in the special case when $\rho = \pi_{\exp( - \lambda r)}$.
Indeed Theorem \ref{thm2.7}
proves that with $\PP$ probability at least $1 - \epsilon$,
$$
\pi_{\exp( - \lambda r)}(R) \leq
\Phi_{\frac{\lambda}{N}}^{-1} \biggl\{
- \frac{\log \bigl\{ \pi \bigl[ \exp \bigl( - \lambda r \bigr)
\bigr] \bigr\} + \log(\epsilon)}{\lambda} \biggr\},
$$
whereas Theorem \ref{thm1.11} proves that with $\PP \pi_{\exp( - \lambda r)}$
probability at least $1 - \epsilon$
$$
R(\wh{\theta}\,) \leq
\Phi_{\frac{\lambda}{N}}^{-1} \biggl\{
- \frac{\log \bigl\{ \pi \bigl[ \exp \bigl( - \lambda r \bigr)
\bigr] \bigr\} + \log(\epsilon)}{\lambda} \biggr\},
$$
showing that we get the same deviation bound for $\pi_{\exp( - \lambda r)}(R)$
under $\PP$
and for $\wh{\theta}$ under $\PP \pi_{\exp( - \lambda r)}$.

We would like to show now how to optimize with respect to
$\lambda$ the bound given by Theorem
\ref{thm2.7}
(the same discussion would apply to Theorem \ref{thm1.11}).
Let us notice first that values of $\lambda$ less than $1$
are not interesting (because they provide a bound larger than
one, at least as soon as $\epsilon \leq \exp(-1)$). Let us consider some real parameter
$\alpha > 1$, and the set $\Lambda =
\{ \alpha^k ; k \in \NN \}$, on which we put
the probability measure $\nu(\alpha^k) = [(k+1)(k+2)]^{-1}$.
Applying Theorem \ref{thm2.7} to $\lambda = \alpha^k$ at
confidence level $1 - \frac{\epsilon}{(k+1)(k+2)}$,
and using a union bound, we see that
with probability at least $1 - \epsilon$,
for any posterior distribution $\rho$,
$$
\rho(R) \leq \inf_{\lambda' \in \Lambda}
\Phi_{\frac{\lambda'}{N}}^{-1}
\left\{ \rho(r) + \frac{\C{K}(\rho,\pi) - \log(\epsilon) +
2 \log \Bigl[\tfrac{\log(\alpha^2\lambda')}{\log(\alpha)} \Bigr]}{
\lambda'}
\right\}.
$$
Now we can remark that for any $\lambda \in (1, + \infty($,
there is $\lambda' \in \Lambda$ such that $\alpha^{-1} \lambda \leq \lambda' \leq
\lambda$. Moreover, for any $q \in (0,1)$, $\beta \mapsto \Phi_{\beta}^{-1}(q)$
is increasing on $\RR_+$. Thus
with probability at least $1 - \epsilon$,
for any posterior distribution $\rho$,
\begin{align*}
\rho(R) & \leq \inf_{\lambda \in (1, \infty(}
\Phi_{\frac{\lambda}{N}}^{-1}
\left\{ \rho(r) + \frac{\alpha}{\lambda} \left[
\C{K}(\rho,\pi) - \log(\epsilon) + 2 \log
\Bigl( \tfrac{\log(\alpha^2 \lambda)}{\log(\alpha)} \Bigr)
\right] \right\} \\
& = \inf_{\lambda \in (1, \infty(}\frac{ 1 - \exp \left\{ - \frac{\lambda}{N}\rho(r) -
\frac{\alpha}{N}\left[ \C{K}(\rho,\pi) - \log(\epsilon) +
2 \log \Bigl( \frac{\log(\alpha^2 \lambda)}{\log(\alpha)}
\Bigr) \right] \right\}}{ 1 -
\exp( - \frac{\lambda}{N} )}.
\end{align*}
Taking the approximately optimal value
$$
\lambda = \sqrt{ \frac{2 N \alpha \left[ \C{K}(\rho,\pi) - \log (\epsilon) \right]}{
\rho(r)[ 1 - \rho(r) ]}},
$$
we obtain
\begin{thm}
\label{thm1.1.11}
\mypoint With probability $1 - \epsilon$, for any posterior distribution
$\rho: \Omega \rightarrow \C{M}_+^1(\Theta)$, putting
$d(\rho,\epsilon) = \C{K}(\rho,\pi) - \log(\epsilon)$,
\begin{multline*}
\rho(R)
 \leq \inf_{k \in \NN}\frac{\ds 1 - \exp \left\{ -
 \frac{\alpha^k}{N}\rho(r) -
\frac{1}{N}\Bigl[ d(\rho,\epsilon)+
\log \bigl[
(k+1)(k+2)\bigr] \Bigr] \right\}}{\ds 1 -
\exp \left( - \frac{\alpha^k}{N} \right)} \\
\leq \frac{\ds 1 - \exp \left\{ - \sqrt{\frac{2 \alpha \rho(r)
d(\rho,\epsilon)}{N [1 - \rho(r)]}} - \frac{\alpha}{N}
\Biggl[ d(\rho,\epsilon)+
2 \log \biggl( \tfrac{\log \left( \alpha^2
\sqrt{\frac{2 N \alpha d(\rho,\epsilon)}{
\rho(r)[1 - \rho(r)]}}\right)}{\log(\alpha)} \biggr) \Biggr] \right\}}{\ds
1 - \exp \left[ - \sqrt{\frac{2 \alpha d(\rho,\epsilon)}{
N \rho(r) [1 - \rho(r)]}} \right]}.
\end{multline*}
Moreover with probability at least $1 - \epsilon$, for any
posterior distribution $\rho$ such that $\rho(r) = 0$,
$$
\rho(R) \leq 1 - \exp \left[ - \frac{\C{K}(\rho,\pi) - \log(\epsilon)}{N} \right].
$$
\end{thm}

We can also elaborate on the results in an other direction by introducing
the \emph{empirical dimension}\
\begin{equation}
\label{eq1.1.3}
d_e = \sup_{\beta \in \RR_+} \beta \bigl[ \pi_{\exp( - \beta r)}(r) -
\ess\inf_{\pi} r
\bigr] \leq - \log \bigl[ \pi \bigl( r = \ess \inf_{\pi} r\bigr) \bigr].
\end{equation}
There is no need to introduce a margin in this definition, since $r$ takes
at most $N$ values, and therefore $\pi \bigl( r = \ess \inf_{\pi}
r \bigr)$
is strictly positive.
This leads to
\begin{cor}
\label{cor1.1.12}
\mypoint
For any positive real constant $\lambda$,
with $\PP$ probability at least $1 - \epsilon$, for any posterior distribution
$\rho: \Omega \rightarrow \C{M}_+^1(\Theta)$,
$$
\rho(R) \leq \Phi_{\frac{\lambda}{N}}^{-1}
\left[ \ess \inf_{\pi} r + \frac{d_e}{\lambda} \log \left( \frac{e \lambda}{d_e}
\right) + \frac{\C{K}\bigl[ \rho, \pi_{\exp( - \lambda r)} \bigr]- \log(\epsilon)
}{\lambda} \right].
$$
\end{cor}

We could then make the bound uniform in $\lambda$ and optimize this parameter
in a way similar to what was done to obtain Theorem \ref{thm1.1.11}.

\section{Local bounds}
In this section, better bounds will be achieved through a better choice
of the prior distribution. This better prior distribution turns out to
depend on the unknown sample distribution $\PP$, and some work is required to
circumvent this and obtain empirical bounds.
\subsection{Choice of the prior}
\label{mutual}
As mentioned in the introduction, if one is
willing to minimize the bound in expectation provided by Theorem
\ref{thm2.4} (page \pageref{thm2.4}),
one is led to consider the optimal choice $\pi =
\PP(\rho)$. However, this is only an ideal choice, since
$\PP$ is in all conceivable situations unknown. Nevertheless it
shows that it is possible through Theorem \ref{thm2.4} to measure
the \emph{complexity} of the classification model
with $\PP \bigl\{ \C{K}\bigl[\rho, \PP(\rho) \bigr] \bigr\}$,
which is nothing but the \emph{mutual information}
between the random sample $(X_i,Y_i)_{i=1}^N$
and the estimated parameter $\Hat{\theta}$,
under the joint distribution $\PP \rho$.

In practice, since we cannot choose $\pi = \PP(\rho)$,
we have to be content with a \emph{flat} prior $\pi$,
resulting in a bound measuring complexity according to
$\PP \bigl[ \C{K}(\rho,\pi) \bigr] = \PP \bigl\{ \C{K} \bigl[ \rho, \PP(\rho) \bigr]
\bigr\} + \C{K} \bigl[ \PP(\rho), \pi \bigr]$ larger by the entropy
factor $\C{K}\bigl[ \PP(\rho), \pi \bigr]$ than the optimal one
(we are still commenting on Theorem \ref{thm2.4}).

If we want to base the choice of $\pi$ on Theorem \ref{thm2.5}
(page \pageref{thm2.5}), and if we
choose
$\rho = \pi_{\exp( - \lambda r)}$
to optimize this bound, we will be inclined to choose some $\pi$ such
that
$$
\frac{1}{\lambda} {\textstyle \int_0^{\lambda}} \pi_{\exp( - \beta R)}(R) d \beta
= - \frac{1}{\lambda} \log \Bigl\{ \pi \bigl[ \exp( - \lambda R) \bigr] \Bigr\}
$$
is as far as possible close to $\inf_{\theta \in \Theta} R(\theta)$ in all circumstances. To give
a more specific example, in
the case when the distribution of the design $(X_i)_{i=1}^N$ is known,
one can introduce on the parameter space $\Theta$ the metric $D$
already defined by equation (\ref{eq1.1.2}, page \pageref{eq1.1.2})
(or some available upper bound for this distance). In view of the fact that
$R(\theta) - R(\theta') \leq D(\theta, \theta')$, for any $\theta$, $\theta'
\in \Theta$, it can be meaningful, at least theoretically,
to choose $\pi$ as
$$
\pi = \sum_{k=1}^{\infty} \frac{1}{k(k+1)} \pi_k,
$$
where $\pi_k$ is the uniform measure on some minimal (or close
to minimal) $2^{-k}$-net $\C{N}(\Theta,
D,2^{-k})$ of the metric space $(\Theta, D)$. With this choice
\begin{multline*}
- \frac{1}{\lambda} \log \Bigl\{ \pi \bigl[ \exp (- \lambda R) \bigr] \Bigr\}
\leq \inf_{\theta \in \Theta} R(\theta)
\\ + \inf_k \left\{ 2^{-k} + \frac{\log ( \lvert \C{N}(\Theta, D, 2^{-k}) \rvert
) + \log[k(k+1)]}{\lambda} \right\}.
\end{multline*}

Another possibility, when we have to deal with real valued parameters,
meaning that $\Theta \subset \RR^d$, is to code each real component
$\theta_i \in \RR$ of $\theta = (\theta_i)_{i=1}^d$ to some precision
and to use a prior $\mu$ which is atomic on dyadic numbers. More
precisely let us parametrize the set of dyadic real numbers as
\begin{multline*}
\C{D} = \Biggl\{
r\bigl[ s, m, p, (b_j)_{j=1}^p\bigr] = s 2^m \biggl( 1 + \sum_{j=1}^p b_j 2^{-j}
\biggr)\,\\:\,
s \in \{-1, +1\}, m \in \ZZ, p \in \NN, b_j \in \{0,1\} \Biggr\},
\end{multline*}
where, as can be seen, $s$ codes the sign, $m$ the order of magnitude,
$p$ the precision and $(b_j)_{j=1}^p$ the binary representation of
the dyadic number $r\bigl[ s,m,p, (b_j)_{j=1}^p \bigr]$. We can for
instance consider on $\C{D}$ the probability distribution
\begin{equation}
\label{eq1.1.4bis}
\mu\bigl\{ r\bigl[ s,m,p,(b_j)_{j=1}^p \bigr] \bigr\}
= \Bigl[ 3 (\lvert m \rvert + 1)(\lvert m \rvert + 2) (p+1)(p+2) 2^p  \Bigr]^{-1},
\end{equation}
and define $\pi \in \C{M}_+^1(\RR^d)$ as $\pi = \mu^{\otimes d}$.
This kind of ``coding'' prior distribution can be used also to define
a prior on the integers (by renormalizing the restriction of $\mu$
to integers to get a probability distribution).
Using $\mu$ is somehow equivalent to picking up a representative of
each dyadic interval, and makes it possible to restrict to the
case when the posterior $\rho$ is a Dirac mass without losing
too much (when $\Theta = (0,1)$, this approach is somewhat equivalent
to considering as prior distribution the Lebesgue measure and using
as posterior distributions the uniform probability measures on dyadic
intervals, with the advantage of obtaining non-randomized estimators).
When one uses in this way an atomic prior and Dirac masses as posterior
distributions, the bounds proven so far can be obtained through a
simpler union bound argument. This is so true that some of the
detractors of the PAC-Bayesian approach (which, as a newcomer,
has sometimes received a suspicious greeting among statisticians)
have argued that it cannot bring anything that elementary union bound
arguments could not essentially provide. We do not share of course
this derogatory opinion, and while we think that allowing for
non atomic priors and posteriors is worthwhile, we also would
like to stress that the upcoming local and relative bounds could
hardly be obtained with the only help of union bounds.

Although the choice of a \emph{flat} prior seems at first glance to be
the only alternative when nothing is known about the sample distribution
$\PP$, the previous discussion shows that this type of choice is
lacking proper localisation, and namely that we loose a factor
$\C{K}\bigl\{ \PP\bigl[\pi_{\exp(- \lambda r)}\bigr],\pi \bigr\}$, the divergence
between the bound-optimal prior $\PP\bigl[ \pi_{\exp( - \lambda r)} \bigr]$,
which is concentrated near the minima of $R$ in favourable situations,
and the flat prior $\pi$. Fortunately, there are technical ways to
get around this difficulty and to obtain more local empirical bounds.

\subsection{Unbiased local empirical bounds}
The idea is to start with some flat prior $\pi \in \C{M}_+^1(\Theta)$, and the
posterior distribution $\rho = \pi_{\exp( - \lambda r)}$ minimizing the bound of
Theorem \ref{thm2.4}
(page \pageref{thm2.4}), when $\pi$ is used as a prior. To improve the bound, we
would like to use $\PP \bigl[ \pi_{\exp(- \lambda r)}\bigr]$ instead of $\pi$,
and we are going to make the guess that we could approximate it with $\pi_{\exp(
- \beta R)}$ (we have replaced the parameter $\lambda$ with some distinct
parameter $\beta$ to give some more freedom to our investigation,
and also because, intuitively, $\PP \bigl[ \pi_{\exp( - \lambda r)} \bigr]$
may be expected to be less concentrated than each of the $\pi_{\exp( - \lambda r)}$
it is mixing,
which suggests that the best approximation of $\PP \bigl[
\pi_{\exp( - \lambda r)} \bigr]$ by some $\pi_{\exp( - \beta R)}$
may be obtained for some parameter $\beta < \lambda$). We are then
led to look for some empirical upper bound of $\C{K}\bigl[
\rho, \pi_{\exp( -\beta R)} \bigr]$. This is happily provided by the
following computation
\begin{multline*}
\PP \bigl\{ \C{K}\bigl[ \rho, \pi_{\exp( - \beta R)} \bigr] \bigr\}
= \PP \bigl[ \C{K}(\rho, \pi) \bigr] + \beta \PP \bigl[ \rho (R) \bigr]
+ \log \Bigl\{ \pi \bigl[ \exp( - \beta R) \bigr] \Bigr\}
\\ = \PP \bigl\{ \C{K}\bigl[ \rho, \pi_{\exp( - \beta r)}\bigr] \bigr\}
+ \beta \PP \bigl[ \rho(R-r) \bigr]
\\ + \log \Bigl\{ \pi \bigl[ \exp( - \beta R) \bigr] \Bigr\}
- \PP \Bigl\{ \log \pi \bigl[ \exp( - \beta r) \bigr] \Bigr\}.
\end{multline*}
Using the convexity of $r \mapsto \log \bigl\{ \pi \bigl[
\exp ( - \beta r) \bigr] \bigr\}$ as in equation
\eqref{eq1.1.3Ter} on page \pageref{eq1.1.3Ter}, we conclude that
$$
0 \leq \PP \bigl\{ \C{K}\bigl[ \rho, \pi_{\exp( - \beta R)}\bigr] \bigr\}
\leq \beta \PP \bigl[ \rho(R - r) \bigr] + \PP \bigl\{ \C{K} \bigl[ \rho,
\pi_{\exp( - \beta r)} \bigr] \bigr\}.
$$
This inequality has an interest of its own, since it provides a lower
bound for $\PP \bigl[ \rho(R) \bigr]$. Moreover we can plug it
into Theorem \ref{thm2.4} (page \pageref{thm2.4}) applied to the prior distribution
$\pi_{\exp( - \beta R)}$ and obtain for any posterior distribution $\rho$
and any positive parameter $\lambda$ that
$$
\Phi_{\frac{\lambda}{N}} \bigl\{ \PP \bigl[ \rho(R) \bigr] \bigr\}
\leq \PP \biggl\{ \rho(r) + \frac{\beta}{\lambda} \rho(R-r)
+ \frac{1}{\lambda} \PP \Bigl\{ \C{K}\bigl[
\rho, \pi_{\exp( - \beta r)} \bigr] \Bigr\} \biggr\}.
$$
In view of this, it it convenient to introduce the function
\newcommand{\TPhi}{\widetilde{\Phi}}
\begin{multline*}
\TPhi_{a,b}(p) = (1 - b)^{-1}
\bigl[ \Phi_a(p) - bp \bigr] \\
= - (1 - b)^{-1} \Bigl\{ a^{-1} \log \bigl\{ 1 - p
\bigl[ 1 - \exp( - a) \bigr] \bigr\} + bp \Bigr\},\\
p \in (0,1), a \in )0,\infty(, b \in (0,1(.
\end{multline*}
This is a convex function of $p$, moreover
$$
\TPhi_{a,b}'(0)
= \Bigl\{ a^{-1} \bigl[ 1 - \exp(- a) \bigr] - b \Bigr\} (1 - b)^{-1},$$
showing that it is an increasing one to one convex map of the unit interval unto
itself as soon as $b \leq a^{-1}
\bigl[ 1 - \exp( - a ) \bigr]$.
Its convexity, combined with the value of its derivative at the origin, shows
that
$$
\TPhi_{a,b}(p) \geq \frac{a^{-1} \bigl[ 1 - \exp ( - a) \bigr] - b}{1-b} p.
$$
Using this notation and remarks, we can state
\begin{thm}
\label{thm3.1}
\mypoint For any positive real constants
$\beta$ and $\lambda$ such that
$0 \leq \beta < N [1 - \exp( - \frac{\lambda}{N})]$, for any posterior distribution $\rho: \Omega \rightarrow \C{M}_+^1(\Theta)$,
\begin{multline*}
\PP \biggl\{ \rho(r) - \frac{ \C{K} \bigl[ \rho, \pi_{\exp( - \beta r)} \bigr]}{\beta}
\biggr\} \leq
\PP \bigl[ \rho(R) \bigr] \\ \leq
\TPhi_{\frac{\lambda}{N}, \frac{\beta}{\lambda}}^{-1}
\biggl\{ \PP \biggl[ \rho(r) + \frac{\C{K}\bigl[ \rho, \pi_{\exp( - \beta r)}
\bigr]}{\lambda - \beta}
\biggr] \biggr\}
\\ \leq
\frac{\lambda - \beta}{N [ 1 - \exp( - \frac{\lambda}{N})] - \beta}
\PP \biggl[ \rho(r) + \frac{\C{K} \bigl[ \rho, \pi_{\exp( - \beta r)}
\bigr]}{\lambda - \beta} \biggr].
\end{multline*}
Thus (taking $\lambda = 2 \beta$), for any $\beta$ such that $0 \leq \beta < \frac{N}{2}$,
$$
\PP \bigl[ \rho(R) \bigr]
\leq \frac{1}{1 - \frac{2 \beta}{N}} \PP \biggl\{ \rho(r) + \frac{\C{K}\bigl[
\rho, \pi_{\exp(- \beta r)} \bigr]}{\beta} \biggr\}.
$$
\end{thm}

Note that the last inequality is obtained using the fact that
$1 - \exp( - x) \geq x - \frac{x^2}{2}$, $x \in \RR_+$.
\begin{cor}
\label{cor3.2}
\mypoint For any $\beta \in (0,N($,
\begin{multline*}
\PP \bigl[ \pi_{\exp( - \beta r)}(r) \bigr] \leq
\PP \bigl[ \pi_{\exp(- \beta r)}(R) \bigr] \\
\leq \inf_{\lambda \in (- N \log(1 - \frac{\beta}{N}),
\infty(} \frac{\lambda - \beta}{N[1 - \exp( - \frac{\lambda}{N})] - \beta}
\PP \bigl[ \pi_{\exp( - \beta r)}(r) \bigr]
\\ \leq \frac{1}{1 - \frac{2 \beta}{N}} \PP \bigl[
\pi_{\exp( - \beta r)}(r) \bigr],
\end{multline*}
the last inequality holding only when $\beta < \frac{N}{2}$.
\end{cor}

It is interesting to compare the upper bound provided by
this corollary with Theorem \thmref{thm2.4}
when the posterior is a Gibbs measure $\rho = \pi_{\exp( - \beta r)}$.
We see that we have got rid of the entropy term
$\C{K}\bigl[\pi_{\exp( - \beta r)}, \pi \bigr]$, but at the price
of an increase of the multiplicative factor, which for small values of
$\frac{\beta}{N}$ grows from $( 1 - \frac{\beta}{2N})^{-1}$
(when we take $\lambda = \beta$ in Theorem \ref{thm2.4}),
to $(1 - \frac{2 \beta}{N})^{-1}$. Therefore non-localized bounds
have an interest of their own, and are superseded by localized
bounds only in favourable circumstances (presumably when the sample
is large enough when compared with the complexity of the classification
model).

Corollary \ref{cor3.2} shows that when $\frac{2 \beta}{N}$ is
small, $\pi_{\exp( - \beta r)}(r)$ is a tight approximation of
$\pi_{\exp( - \beta r)}(R)$ in the mean (since we have
an upper bound and a lower bound which are close together).

Another corollary is obtained by optimizing the bound
given by Theorem \ref{thm3.1} in $\rho$, which is done
by taking $\rho = \pi_{\exp( - \lambda r)}$.
\begin{cor}
\mypoint For any positive real constants $\beta$ and $\lambda$ such that
$0 \leq \beta < N[1 - \exp( - \frac{\lambda}{N})]$,
\begin{multline*}
\PP \bigl[ \pi_{\exp( - \lambda r)}(R) \bigr]
\leq \TPhi_{\frac{\lambda}{N}, \frac{\beta}{\lambda}}^{-1}
\biggl\{ \PP \biggl[ \frac{1}{\lambda - \beta} \int_{\beta}^{\lambda}
\pi_{\exp( - \gamma r)}(r) d \gamma \biggr] \biggr\}
\\ \leq \frac{1}{N[1 - \exp( - \frac{\lambda}{N})] - \beta} \PP
\Bigr[ {\textstyle \int_{\beta}^{\lambda}}
\pi_{\exp( - \gamma r)}(r) d \gamma \Bigr].
\end{multline*}
\end{cor}

Although this inequality gives by construction a better
upper bound for\break $\inf_{\lambda \in \RR_+} \PP \bigl[
\pi_{\exp( - \lambda r)}(R) \bigr]$ than Corollary
\ref{cor3.2}, it is not easy to tell which one of the two inequalities
is the best to bound $\PP \bigl[ \pi_{\exp( - \lambda r)}(R)\bigr]$
for a fixed (and possibly suboptimal) value of
$\lambda$, because in this case, one factor is improved while the other is worsened.

Using the \emph{empirical dimension} $d_e$ defined by equation \eqref{eq1.1.3}
on page \pageref{eq1.1.3}, we see that
$$
\frac{1}{\lambda - \beta} \int_{\beta}^{\lambda} \pi_{\exp( - \gamma r)}(r)
d \gamma \leq \ess \inf_{\pi} r + d_e \log \left( \frac{\lambda}{\beta} \right).
$$
Therefore, in the case when we keep the ratio $\frac{\lambda}{\beta}$
bounded, we get a better dependence on the empirical dimension $d_e$
than in Corollary \ref{cor1.1.12} (page \pageref{cor1.1.12}).

\subsection{Non random local bounds} Let us come now to the localization
of the non-random upper
bound given by Theorem \thmref{thm2.5}.
According to Theorem \ref{thm2.4} (page \pageref{thm2.4})
applied to the localized prior $\pi_{\exp( - \beta R)}$,
\begin{multline*}
\lambda \Phi_{\frac{\lambda}{N}} \bigl\{ \PP \bigl[ \rho(R) \bigr] \bigr\}
\leq \PP \Bigl\{ \lambda \rho(r) + \C{K}(\rho, \pi) + \beta \rho(R) \Bigr\}
+ \log \bigl\{ \pi \bigl[ \exp( - \beta R) \bigr] \bigr\} \\
= \PP \Bigl\{ \C{K}\bigl[\rho, \pi_{\exp( - \lambda r)}\bigr]
- \log \bigl\{ \pi \bigl[ \exp( - \lambda r) \bigr] \bigr\} +
\beta \rho(R) \Bigr\} + \log \bigl\{ \pi \bigl[ \exp (- \beta R) \bigr] \bigr\}\\
\leq \PP \Bigl\{ \C{K}\bigl[\rho, \pi_{\exp( - \lambda r)}\bigr]
+ \beta \rho(R) \Bigr\} - \log \bigl\{ \pi \bigl[ \exp( - \lambda R) \bigr]
\bigr\} + \log \bigl\{ \pi \bigl[ \exp ( - \beta R) \bigr] \bigr\},
\end{multline*}
where we have used as previously inequality \eqref{eq1.1.3Ter}
(page \pageref{eq1.1.3Ter}).
This proves
\begin{thm}
\mypoint For any posterior distribution $\rho: \Omega \rightarrow \C{M}_+^1(\Theta)$,
for any real parameters $\beta$ and $\lambda$ such that
$0 \leq \beta < N \bigl[ 1 - \exp( - \frac{\lambda}{N}) \bigr]$,
\begin{multline*}
\PP \bigl[ \rho(R) \bigr]
\leq \TPhi_{\frac{\lambda}{N}, \frac{\beta}{\lambda}}^{-1}
\biggl\{
\frac{1}{ \lambda - \beta} \int_{\beta}^{\lambda}
\pi_{\exp( - \gamma R)}(R) d \gamma + \PP \biggl[ \frac{\C{K}\bigl[ \rho,
\pi_{\exp( - \lambda r)}\bigr]}{\lambda - \beta} \biggr] \biggr\} \\
\leq \frac{ 1}{N \bigl[ 1 - \exp( - \frac{\lambda}{N} )
\bigr] - \beta} \biggl\{
\int_{\beta}^{\lambda}
\pi_{\exp( - \gamma R)}(R) d \gamma + \PP \Bigl\{ \C{K}\bigl[
\rho, \pi_{\exp( - \lambda r)}\bigr] \Bigr\} \biggr\}.
\end{multline*}
\end{thm}

Let us notice in particular that this theorem contains Theorem \ref{thm2.5}
(page \pageref{thm2.5})
which corresponds to the case $\beta = 0$. As a corollary, we see also,
taking $\rho = \pi_{\exp( - \lambda r)}$ and $\lambda = 2 \beta$,
and noticing that $\gamma \mapsto \pi_{\exp( -\gamma R)}(R)$ is decreasing, that
\begin{align*}
\PP \bigl[ \pi_{\exp( - \lambda r)}(R) \bigr]
& \leq  \inf_{\beta, \beta < N[ 1 - \exp( - \frac{\lambda}{N})]}
\frac{\beta}{N \bigl[ 1 - \exp( - \frac{\lambda}{N} ) \bigr]
- \beta} \pi_{\exp( - \beta R)}(R)
\\ & \leq \frac{1}{1 - \frac{\lambda}{N}} \pi_{\exp( - \frac{\lambda}{2} R)}(R).
\end{align*}
We can use this inequality in conjunction with the notion of
dimension with margin $\eta$ introduced by equation
\eqref{eq1.1.3Bis} on page \pageref{eq1.1.3Bis},
to see that the Gibbs posterior achieves for
a proper choice of $\lambda$ and any margin parameter $\eta \geq 0$
(which can be chosen to be equal to zero in parametric
situations)
\begin{multline}
\label{eq1.1.7}
\inf_{\lambda} \PP \bigl[ \pi_{\exp( - \lambda r)}(R) \bigr]
\leq \ess \inf_{\pi} R + \eta + \frac{4 d_{\eta}}{N} \\ +
2 \sqrt{ \frac{2d_{\eta} \bigl( \ess \inf_{\pi} R + \eta
\bigr) }{N} + \frac{4 d_{\eta}^2}{N^2}}.
\end{multline}
Deviation bounds to come next will show that the optimal
$\lambda$ can be estimated from empirical data.

Let us propose a little numerical example as an illustration: assuming
that $d_{0} = 10$, $N=1000$ and $\ess \inf_{\pi}
R = 0.2$, we obtain from equation
\eqref{eq1.1.7} that
$\inf_{\lambda} \PP \bigl[ \pi_{\exp(-\lambda r)}(R) \bigr]
\leq 0.373$.
\subsection{Local deviation bounds}
When it comes to deviation bounds, for technical reasons we will
choose a slightly more involved change of prior distribution and
apply Theorem \ref{thm2.7} (page \pageref{thm2.7}) to the prior $
\pi_{\exp [ - \beta \Phi_{- \frac{\beta}{N}}
\circ R ]}$. The advantage of tweaking $R$ with the nonlinear function
$\Phi_{- \frac{\beta}{N}}$ will appear in the search for an empirical upper
bound of the local entropy term.
Theorem \ref{thm2.3} (page \pageref{thm2.3}), used with the above-mentioned local prior,
shows that
\begin{equation}
\label{eq1.1.4}
\PP \Biggl\{ \sup_{\rho \in \C{M}_+^1(\Theta)}
\lambda \Bigl\{ \rho \bigl(\Phi_{\frac{\lambda}{N}}\!\circ\!R \bigr)
- \rho(r) \Bigr\} - \C{K}\bigl[\rho, \pi_{\exp (- \beta \Phi_{- \frac{\beta}{N}}
\!\circ R)}\bigr] \Biggr\} \leq 1.
\end{equation}
\newcommand{\Brho}{\Bar{\rho}}
Moreover
\begin{multline}
\label{eq1.1.5bis}
\C{K}\bigl[ \rho, \pi_{\exp[ - \beta \Phi_{- \frac{\beta}{N}}\circ R ]} \bigr]
= \C{K}\bigl[ \rho,\pi_{\exp( - \beta r)}
\bigr] + \beta \rho \Bigl[ \Phi_{- \frac{\beta}{N}}\!\circ\!R - r \Bigr] \\*
+ \log \Bigl\{ \pi \Bigl[ \exp \bigl( - \beta \Phi_{- \frac{\beta}{N}}\!\circ\!R
\bigr) \Bigr] \Bigr\} - \log \Bigl\{ \pi \Bigl[ \exp ( - \beta r) \Bigr]
\Bigr\},
\end{multline}
which is an invitation to find an upper bound for
$\log \Bigl\{ \pi \Bigl[ \exp \bigl[ - \beta \Phi_{- \frac{\lambda}{N}}\!\circ R
\big] \Bigr] \Bigr\} - \log \Bigl\{ \pi \bigl[ \exp ( - \beta r) \bigr] \Bigr\}$.
\newcommand{\Bpi}{\overline{\pi}}
For conciseness, let us call our localized prior distribution  $\Bpi$, thus defined
by its density
$$
\frac{d \Bpi}{d \pi}(\theta)
= \frac{\ds
\exp \Bigl\{ - \beta \Phi_{- \frac{\beta}{N}} \bigl[ R(\theta) \bigr] \Bigr\}}{\ds
\pi \Bigl\{ \exp \bigl[ - \beta
\Phi_{- \frac{\beta}{N}}\!\circ\!R \bigr] \Bigr\}}.
$$
Applying once again Theorem \ref{thm2.3} (page \pageref{thm2.3}),
but this time to $- \beta$, we see that
\begin{multline}
\label{eq1.1.5}
\PP \biggl\{ \exp \biggl[
\log \Bigl\{ \pi \Bigl[ \exp \bigl( - \beta \Phi_{-
\frac{\beta}{N}}\!\circ\!R
\bigr) \Bigr] \Bigr\}
- \log \Bigl\{ \pi \bigl[ \exp ( - \beta r) \bigr] \Bigr\} \biggr] \biggr\}
\\ = \PP \biggl\{ \exp \biggl[
\log \Bigl\{ \pi \Bigl[ \exp \bigl( - \beta \Phi_{- \frac{\beta}{N}}\!\circ\!R)
\bigr) \Bigr] \Bigr\}
+ \inf_{\rho \in \C{M}_+^1(\Theta)}
\beta \rho(r) + \C{K}(\rho, \pi)  \biggr] \biggr\}
\\ \leq \PP \biggl\{ \exp \biggl[
\log \Bigl\{ \pi \Bigl[ \exp \bigl( - \beta \Phi_{- \frac{\beta}{N}}\!\circ\!R)
\bigr) \Bigr] \Bigr\}  + \beta \Bpi(r)
+ \C{K}(\Bpi , \pi) \biggr] \biggr\}
\\ = \PP \biggl\{ \exp \biggl[
\beta \Bigl[ \Bpi(r) - \Bpi \bigl( \Phi_{- \frac{\beta}{N}}\!\circ\!R \bigr)
\Bigr] + \C{K}(\Bpi,\Bpi) \biggl]
\biggr\} \leq 1.
\end{multline}
Combining equations \eqref{eq1.1.5bis} and \eqref{eq1.1.5}
and using the concavity of $\Phi_{- \frac{\beta}{N}}$,
we see that with $\PP$ probability at least $1 - \epsilon$,
for any posterior distribution $\rho: \Omega \rightarrow \C{M}_+^1(\Theta)$,
$$
0 \leq \C{K}(\rho, \Bpi) \leq \C{K} \bigl[\rho, \pi_{\exp(-\beta r)}\bigr]
+ \beta \Bigl[ \Phi_{-\frac{\beta}{N}}\bigl[ \rho(R) \bigr] - \rho(r) \Bigr]
- \log(\epsilon).
$$
We have proved a lower deviation bound:
\begin{thm} For any positive real constant $\beta$,
with $\PP$ probability at least $1 - \epsilon$,
for any posterior distribution $\rho: \Omega \rightarrow
\C{M}_+^1(\Theta)$,
$$
\frac{\ds \exp \biggl\{ \frac{\beta}{N} \biggl[
\rho(r) - \frac{\C{K}[\rho, \pi_{\exp( - \beta r)}]
- \log(\epsilon)}{\beta} \biggr] \biggr\} - 1}{\ds
\exp\bigl( \tfrac{\beta}{N} \bigr) - 1} \leq \rho (R).
$$
\end{thm}

We can also obtain a lower deviation bound for $\wh{\theta}$.
Indeed equation \eqref{eq1.1.5} can
also be written as
$$
\PP \biggl\{ \pi_{\exp( - \beta r)} \biggl[ \exp \Bigl\{ \beta
\Bigl[ r - \Phi_{- \frac{\beta}{N}} \circ \, R \Bigr] \Bigr\} \biggr] \biggr\}
\leq 1.
$$
This means that for any posterior distribution $\rho: \Omega \rightarrow
\C{M}_+^1(\Theta)$,
$$
\PP \Bigl\{ \rho \Bigl[ \exp \bigl\{ \beta \bigl[ r - \Phi_{-\frac{\beta}{N}}
\circ R \bigr] - \log \bigl( \tfrac{d \rho}{d \pi_{\exp( - \beta r)}} \bigr)
\bigr\} \Bigr] \Bigr\} \leq 1.
$$
We have proved
\begin{thm}
For any positive real constant $\beta$, for any posterior distribution
$\rho: \Omega \rightarrow \C{M}_+^1(\Theta)$, with $\PP \rho$
probability at least $1 - \epsilon$,
\begin{align*}
R(\wh{\theta}\,) & \geq \Phi_{- \frac{\beta}{N}}^{-1}
\biggl[ r(\wh{\theta}\,) - \frac{\log \bigl( \frac{d \rho}{d \pi_{\exp( - \beta r)}}
\bigr) - \log(\epsilon)}{\beta} \biggr]\\
& = \frac{\ds \exp \biggl\{ \frac{\beta}{N} \biggl[ r(\wh{\theta}\,) -
\frac{\log \bigl( \frac{d \rho}{d \pi_{\exp( - \beta r)}}
\bigr) - \log(\epsilon)}{\beta} \biggr] \biggr\} - 1}{\ds
\exp \biggl( \frac{\beta}{N} \biggr) - 1}.
\end{align*}
\end{thm}

Let us now resume our investigation of the upper deviations
of $\rho(R)$.
Using the Cauchy-Schwarz inequality to combine
equations \myeq{eq1.1.4} and \myeq{eq1.1.5},
we obtain
\begin{multline}
\label{eq1.1.11Bis}
\PP \biggl\{ \exp \biggl[ \frac{1}{2}
\sup_{\rho \in \C{M}_+^1(\Theta)} \lambda
\rho \bigl( \Phi_{\frac{\lambda}{N}}\!\circ\!R \bigr) - \beta
\rho \bigl( \Phi_{- \frac{\beta}{N}}\!\circ\!R \bigr) - (\lambda - \beta)
\rho(r) - \C{K}\bigl[ \rho, \pi_{\exp(- \beta r)}\bigr] \biggr] \biggr\}
\\ =
\PP \biggl\{ \exp \biggl[
\tfrac{1}{2} \sup_{\rho \in \C{M}_+^1(\Theta)} \biggl(\lambda \Bigl\{
\rho \bigl( \Phi_{\frac{\lambda}{N}}\!\circ\!R \bigr)
- \rho(r) \Bigr\} - \C{K}(\rho, \Bpi) \biggr) \bigg] \\
\times \exp \biggl[ \tfrac{1}{2}
\biggl( \log \Bigl\{ \pi \Bigl[
\exp\bigl( - \beta \Phi_{- \frac{\beta}{N}}\!\circ\!R\bigr)
\Bigr] \Bigr\} - \log \Bigl\{ \pi \Bigl[
\exp ( - \beta r) \Bigr] \Bigr\} \biggr) \biggr] \biggr\}
\\ \leq
\PP \biggl\{ \exp \biggl[
\sup_{\rho \in \C{M}_+^1(\Theta)} \biggl(\lambda \Bigl\{
\rho \bigl( \Phi_{\frac{\lambda}{N}}\!\circ\!R \bigr)
- \rho(r) \Bigr\} - \C{K}(\rho, \Bpi) \biggr) \biggl] \biggr\}^{1/2}\\
\times \PP \biggl\{ \exp \biggl[
\biggl( \log \Bigl\{ \pi \Bigl[
\exp\bigl( - \beta \Phi_{- \frac{\beta}{N}}\!\circ\!R\bigr)
\Bigr] \Bigr\} - \log \Bigl\{ \pi \Bigl[
\exp ( - \beta r) \Bigr] \Bigr\} \biggr) \biggr] \biggr\}^{1/2}
\leq 1.
\end{multline}
Thus with $\PP$ probability
at least $1 - \epsilon$, for any posterior distribution $\rho$,
\begin{multline*}
\lambda \Phi_{\frac{\lambda}{N}}\bigl[ \rho(R) \bigr]
- \beta \Phi_{- \frac{\beta}{N}} \bigl[ \rho(R) \bigr] \\ \leq
\lambda \rho \bigl( \Phi_{\frac{\lambda}{N}} \circ \, R \bigr)
- \beta \rho \bigl( \Phi_{-\frac{\beta}{N}} \circ \,R \bigr)
\\ \leq (\lambda - \beta) \rho(r) + \C{K}(\rho, \pi_{\exp(- \beta r)})
- 2 \log(\epsilon).
\end{multline*}
(It would have been more straightforward to use a union bound on
deviation inequalities instead of the Cauchy-Schwarz
inequality on exponential moments, anyhow, this would have led
to replace $- 2 \log(\epsilon)$ with the worse factor
$2 \log(\frac{2}{\epsilon})$.)
Let us now recall that
\begin{multline*}
\lambda \Phi_{\frac{\lambda}{N}}(p) - \beta \Phi_{-\frac{\beta}{N}}(p)
= - N \log \Bigl\{ 1 - \bigl[ 1 - \exp\bigl(- \tfrac{\lambda}{N}\bigr)\bigr] p
\Bigr\} \\ - N \log \Bigl\{ 1 + \bigl[\exp\bigl( \tfrac{\beta}{N} \bigr) - 1\bigr] p
\Bigr\},
\end{multline*}
and let us put
\begin{multline*}
B  = (\lambda - \beta) \rho(r) + \C{K}\bigl[ \rho, \pi_{\exp(- \beta r)}\bigr]
- 2 \log(\epsilon) \\
= \C{K}\bigl[ \rho, \pi_{\exp( - \lambda r)} \bigr]
+ {\textstyle \int_{\beta}^{\lambda}} \pi_{\exp( - \xi r)}(r) d \xi - 2 \log(\epsilon).
\end{multline*}
Let us consider moreover the change of variables
$\alpha = 1 - \exp( - \frac{\lambda}{N})$ and $\gamma = \exp(\frac{\beta}{N}) - 1$.
We obtain
$
\bigl[ 1 - \alpha \rho(R)  \big] \bigl[ 1 + \gamma \rho(R) \bigr]
\geq \exp( - \tfrac{B}{N}),
$
leading to
\begin{thm}
\label{thm1.1.17}\mypoint
For any positive constants $\alpha$, $\gamma$, such that $0 \leq \gamma < \alpha <1$,
with $\PP$ probability at least $1 - \epsilon$, for any posterior distribution
$\rho: \Omega \rightarrow \C{M}_+^1(\Theta)$,
the bound
\begin{align*}
M(\rho) & = - \frac{\log\bigl[ (1 - \alpha)(1 + \gamma) \bigr]}{\alpha - \gamma} \rho(r)
+ \frac{\ds \C{K}(\rho, \pi_{\exp[ - N \log( 1 + \gamma)r ]})
- 2 \log(\epsilon)}{\ds N (\alpha - \gamma)} \\
& = \frac{\ds \C{K}\bigl[ \rho, \pi_{\exp[ N\log(1 - \alpha) r]}\bigr]
+ \int_{N \log(1 + \gamma)}^{- N \log(1 - \alpha)} \pi_{\exp( - \xi r)}(r)
d \xi - 2 \log(\epsilon)}{N (\alpha - \gamma)},
\end{align*}
is such that
$$
\rho(R) \leq \frac{\alpha - \gamma}{2 \alpha \gamma}
\left( \sqrt{1+ \frac{4 \alpha \gamma}{(\alpha - \gamma)^2} \bigl\{ 1 - \exp\bigl[
- (\alpha - \gamma) M(\rho) \bigr]  \bigr\}}- 1 \right) \leq
M(\rho),
$$
\end{thm}

Let us now give an upper bound for $R(\wh{\theta}\,)$.
Equation (\ref{eq1.1.11Bis} page \pageref{eq1.1.11Bis})
can also be written as
$$
\PP \biggl\{ \biggl[ \pi_{\exp( - \beta r)} \Bigl\{ \exp
\Bigl[ \lambda \Phi_{\frac{\lambda}{N}} \circ \,R - \beta
\Phi_{- \frac{\beta}{N}} \circ \, R  - (\lambda - \beta) r \Bigr]
\Bigr\} \biggr]^{\frac{1}{2}} \biggr\} \leq 1.
$$
This means that for any posterior distribution $\rho:
\Omega \rightarrow \C{M}_+^1(\Theta)$,
$$
\PP \biggl\{ \biggl[ \rho \Bigl\{ \exp
\Bigl[ \lambda \Phi_{\frac{\lambda}{N}} \circ \,R - \beta
\Phi_{- \frac{\beta}{N}} \circ \, R  - (\lambda - \beta) r -
\log \bigl( \tfrac{d \rho}{d \pi_{\exp( - \beta r)}} \bigr) \Bigr]
\Bigr\} \biggr]^{\frac{1}{2}} \biggr\} \leq 1.
$$

Using the concavity of the square root function, this
inequality can be weakened to
$$
\PP \biggl\{ \rho \biggl[  \exp  \Bigl\{
\tfrac{1}{2}
\Bigl[ \lambda \Phi_{\frac{\lambda}{N}} \circ \,R - \beta
\Phi_{- \frac{\beta}{N}} \circ \, R  - (\lambda - \beta) r
- \log \bigl( \tfrac{d \rho}{d \pi_{\exp( - \beta r)}} \bigr) \Bigr]
\Bigr\} \biggr] \biggr\} \leq 1.
$$
We have proved
\begin{thm}
\mypoint
\label{thm1.21}
For any positive real constants $\lambda$
and $\beta$ and for any posterior distribution
$\rho: \Omega \rightarrow \C{M}_+^1(\Theta)$,
with $\PP \rho$ probability at least $1 - \epsilon$,
$$
\lambda \Phi_{\frac{\lambda}{N}} \bigl[ R(\wh{\theta}\,)\bigr]
- \beta \Phi_{- \frac{\beta}{N}} \bigl[ R( \wh{\theta}\,)\bigr]
\leq (\lambda - \beta)\, r(\wh{\theta}\,) + \log \Bigl[ \tfrac{
d \rho}{d \pi_{\exp( - \beta r)}}(\wh{\theta}\,) \Bigr] - 2 \log(\epsilon).
$$
Putting $\alpha = 1 - \exp \bigl( - \tfrac{\lambda}{N} \bigr)$,
$\gamma = \exp \bigl( \tfrac{\beta}{N} \bigr) - 1$ and
\begin{align*}
M(\theta) & = - \frac{\log \bigl[ (1 - \alpha) ( 1 + \gamma) \bigr]}{\alpha - \gamma}
r(\theta) + \frac{ \log \Bigl[ \frac{d \rho}{d \pi_{\exp[ -N \log(1 + \gamma) r]}}
(\theta) \Bigr]
- 2 \log(\epsilon)}{N(\alpha - \gamma)}\\
& = \frac{\ds  \log \Bigl[ \tfrac{d \rho}{d \pi_{\exp [ N \log(1 - \alpha) r]}}
(\theta) \Bigr]
+ \int_{N \log(1 + \gamma)}^{- N \log(1 - \alpha)} \pi_{\exp( - \xi r)}
(r)\, d \xi - 2 \log(\epsilon)}{N(\alpha - \gamma)},
\end{align*}
we can also, in the case when $\gamma < \alpha$,  write this inequality as
$$
R(\wh{\theta}\,) \leq \frac{\alpha - \gamma}{2 \alpha \gamma}
\Biggl( \sqrt{1 + \frac{4 \alpha \gamma}{(\alpha - \gamma)^2}
\Bigl\{ 1 - \exp \Bigl[ - (\alpha - \gamma) M(\wh{\theta}\,) \Bigr]
\Bigr\}} - 1 \Biggr) \leq M(\wh{\theta}).
$$
\end{thm}

It may be enlightening to
introduce the \emph{empirical dimension} $d_e$ defined by equation \eqref{eq1.1.3}
on page \pageref{eq1.1.3}.
It provides the upper bound
$$
\int_{\beta}^{\lambda} \pi_{\exp(- \xi r)}(r) d \xi
\leq (\lambda - \beta) \ess \inf_{\pi} r + d_e \log \left( \frac{\lambda}{\beta} \right),
$$
which shows that in Theorem \ref{thm1.1.17} (page \pageref{thm1.1.17}),
\begin{multline*}
M(\rho) \leq \frac{\log\bigl[ (1+\gamma)(1-\alpha) \bigr]}{\gamma - \alpha}
\ess \inf_{\pi} r \\
+ \frac{d_e
\log \left[ \frac{ - \log( 1- \alpha)}{\log(1 + \gamma)} \right]
+ \C{K}\bigl[ \rho, \pi_{\exp [ N \log(1 - \alpha)r]}\bigr] - 2 \log(\epsilon)}{
N(\alpha - \gamma)}.
\end{multline*}
Similarly, in Theorem \ref{thm1.21} above,
\begin{multline*}
M(\theta) \leq \frac{\log \bigl[ (1 + \gamma)(1 - \alpha) \bigr]}{\gamma - \alpha}
\ess \inf_{\pi} r
\\ + \frac{ d_e \log \Bigl[ \frac{- \log(1-\alpha)}{\log(1 + \gamma)} \Bigr]
+ \log \Bigl[ \tfrac{d \rho}{d \pi_{\exp[ N\log(1-\alpha)r]}}
(\theta) \Bigr] -
2 \log(\epsilon)}{N(\alpha - \gamma)}
\end{multline*}

\label{illustration}
Let us give a little numerical illustration: assuming that
$d_e = 10$, $N = 1000$, and $\ess \inf_{\pi} r = 0.2$,
 taking $\epsilon = 0.01$,
$\alpha = 0.5$ and $\gamma = 0.1$, we obtain from
Theorem \ref{thm1.1.17} $\pi_{\exp[ N\log(1-\alpha)r]}(R) \simeq \pi_{\exp(- 693 r)}(R)
\leq 0.332\leq 0.372$, where we have given respectively the non-linear and
the linear bound. This shows the practical interest of keeping the non-linearity.
Optimizing the values of the parameters
$\alpha$ and $\gamma$ would not have yielded a significantly lower bound.

The following corollary is obtained by taking $\lambda = 2 \beta$ and
keeping only the linear bound; we give it for the sake of its simplicity:
\begin{cor}\mypoint
For any positive real constant $\beta$ such that
$\exp(\frac{\beta}{N})
+ \exp( - \frac{2 \beta}{N}) < 2$, which is the case when $\beta < 0.48 N$,
with $\PP$ probability at least $1 - \epsilon$, for any posterior distribution
$\rho: \Omega \rightarrow \C{M}_+^1(\Theta)$,
\begin{multline*}
\rho(R) \leq \frac{ \beta \rho(r) + \C{K}\bigl[ \rho, \pi_{\exp( - \beta r)}\bigr]
- 2 \log(\epsilon)}{N \bigl[ 2 - \exp\bigl( \frac{\beta}{N}\bigr) -
\exp \bigl( - \frac{2 \beta}{N} \bigr) \bigr]}
\\ = \frac{
\int_{\beta}^{2 \beta}
\pi_{\exp( - \xi r)}(r) d \xi + \C{K}\bigl[ \rho, \pi_{\exp( - 2 \beta r)}\bigr] - 2 \log(\epsilon)}{
N \bigl[ 2 - \exp( \frac{\beta}{N}) - \exp( - \frac{2 \beta}{N}) \bigr]}.
\end{multline*}
\end{cor}

Let us mention that this corollary applied to the above numerical example
gives $\pi_{\exp(-200 r)}(R) \leq 0.475$ (when we take $\beta = 100$, consistently
with the choice $\gamma = 0.1$).

\subsection{Partially local bounds}

Local bounds are suitable when the lowest values of the empirical
error rate $r$ are reached only on a small part of the parameter
set $\Theta$. When $\Theta$ is the disjoint union of sub-models
of different complexities, the minimum of $r$ will as a rule
not be ``localized'' in a way that calls for the use of
local bounds. Just think for instance of the case when
$\Theta = \bigsqcup_{m=1}^M \Theta_m$, where the sets $\Theta_1 \subset
\Theta_2 \subset \dots \subset \Theta_M$ are nested.
In this case we will have $\inf_{\Theta_1} r \geq \inf_{\Theta_2} r
\geq \dots \geq \inf_{\Theta_M} r$, although $\Theta_M$ may be
too large to be the right model to use. In this situation, we
do not want to localize the bound completely. Let us make a
more specific fanciful but typical pseudo computation.
Just imagine we have a countable collection $(\Theta_m)_{m \in M}$ of
sub-models.
Let us assume we are interested in choosing between the
estimators $\wtheta_m \in \arg\min_{\Theta_m} r$,
maybe randomizing them (e.g. replacing them
with $\pi^m_{\exp( - \lambda r)}$). Let us
imagine moreover that we are in a typically parametric
situation, where, for some priors $\pi^m \in \C{M}_+^1(\Theta_m)$,
$m \in M$, there is a ``dimension'' $d_m$ such that
$\lambda \bigl[ \pi^m_{\exp( - \lambda r)}(r) - r(\wtheta_m)
\bigr] \simeq d_m$. Let $\mu \in \C{M}_+^1(M)$ be some distribution
on the index set $M$.
It is easy to see that $(\mu \pi)_{\exp( - \lambda r)}$ will
typically not be properly local, in the sense that
typically
\begin{multline*}
(\mu \pi)_{\exp( - \lambda r)}(r) =
\frac{\ds \mu \Bigl\{ \pi_{\exp( - \lambda r)}(r) \pi \bigl[ \exp( - \lambda r) \bigr]
\Bigr\}}{
\mu \Bigl\{ \pi  \bigl[ \exp( - \lambda r) \bigr] \Bigr\}
} \\ \simeq
\frac{\ds \sum_{m \in M}
\bigl[ (\inf_{\Theta_m} r) + \tfrac{d_m}{\lambda} \bigr] \exp \bigl[ - \lambda
(\inf_{\Theta_m} r) - d_m \log\bigl(\tfrac{e \lambda}{d_m}\bigr) \bigr]
\mu(m)}{\ds
\sum_{m \in M} \exp \Bigl[ - \lambda (\inf_{\Theta_m} r) - d_m \log \bigl(\tfrac{e
\lambda}{d_m}
\bigr) \Bigr] \mu(m)}
\\ \simeq \biggl\{ \inf_{m \in M} (\inf_{\Theta_m} r) + \tfrac{d_m}{\lambda}
\log \bigl(
\tfrac{e \lambda}{d_m}\bigr) - \tfrac{1}{\lambda} \log[\mu(m)] \biggr\} \\ + \log
\biggl\{ \sum_{m \in M}
\exp \bigl[ - d_m \log(\tfrac{\lambda}{d_m})\bigr] \mu(m)\biggr\}.
\end{multline*}
where we have used the approximations
\begin{multline*}
- \log \Bigl\{ \pi \bigl[ \exp( - \lambda r) \bigr]
\Bigr\} = \int_0^{\lambda} \pi_{\exp( - \beta r)}(r) d \beta
\\ \simeq \int_0^{\lambda } (\inf_{\Theta_m} r) + \bigl[
\tfrac{d_m}{\beta} \wedge 1 \bigr]
d \beta \simeq  \lambda (\inf_{\Theta_m} r) + d_m
\bigl[ \log \bigl( \tfrac{\lambda}{d_m} \bigr) + 1 \bigr],
\end{multline*}
and $\ds \frac{\sum_m h(m) \exp[ - h(m) ] \nu(m)}{ \sum_m \exp[-h(m)] \nu(m)} \simeq
\inf_m h(m) - \log[\nu(m)], \nu \in \C{M}_+^1(M)$,
taking $\ds \nu(m) = \frac{\mu(m) \exp \bigl[ - d_m\log \bigl( \tfrac{\lambda}{d_m}
\bigr) \bigr]}{\sum_{m'} \mu(m') \exp \bigl[ - d_{m'} \log \bigl(
\tfrac{\lambda}{d_{m'}} \bigr) \bigr]}$.\\[1ex]

These approximations have no pretension to be rigorous or
very accurate, but they nevertheless give the best order
of magnitude we can expect in typical situations, and
show that this order of magnitude is not what we are
looking for: mixing different models with the help
of $\mu$ spoils the localization, introducing a multiplier
$\log \bigl( \tfrac{\lambda}{d_m} \bigr)$ to the dimension
$d_m$ which is precisely what we would have got if we had
not localized the bound at all. What we would
really like to do in such situations is to use a \emph{partially
localized} posterior distribution, such as
$\pi^{\widehat{m}}_{\exp( - \lambda r)}$, where
$\widehat{m}$ is an estimator of the best sub-model
to be used. While the most straightforward way to
do this is to use a union bound on results obtained
for each sub-model $\Theta_m$, here we are going
to show how to allow arbitrary posterior distributions
on the index set (corresponding to a randomization of
the choice of $\widehat{m}$).

Let us consider the framework we just mentioned: let the
measurable parameter
set $(\Theta, \C{T})$ be a union of measurable sub-models,
$\Theta = \bigcup_{m \in M} \Theta_m$. Let the index set $(M, \C{M})$ be
some measurable space (most of the time it will be a countable set).
Let $\mu \in \C{M}_+^1(M)$ be a prior probability distribution on
$(M, \C{M})$. Let $\pi: M \rightarrow \C{M}_+^1(\Theta)$ be a regular
conditional probability measure such that $\pi(m,\Theta_m) = 1$,
for any $m \in M$.
Let $\mu \pi \in \C{M}_+^1(M \times \Theta)$ be the product probability
measure defined
for any bounded measurable function $h: M \times \Theta \rightarrow \RR$
by
$$
\mu\pi(h) = \int_{m \in M} \left( \int_{\theta \in \Theta} h(m,\theta)
\pi(m, d \theta) \right) \mu(dm).
$$
For any bounded measurable function $h: \Omega \times M \times \Theta
\rightarrow \RR$,
let $\pi_{\exp(h)}: \Omega \times M \rightarrow \C{M}_+^1(\Theta)$ be the regular
conditional posterior probability measure defined by
$$
\frac{d \pi_{\exp(h)}}{d \pi} (m, \theta) = \frac{ \exp\bigl[ h(m, \theta) \bigr]}{
\pi \bigl[ m, \exp(h) \bigr]},
$$
where consistently with previous notation $\pi(m,h) = \int_{\Theta}
h(m,\theta) \pi(m, d \theta)$ (we will also often use the less explicit
notation $\pi(h)$).
For short, let
$$
U(\theta, \omega) = \lambda \Phi_{\frac{\lambda}{N}}\bigl[ R(\theta) \bigr] -
\beta \Phi_{- \frac{\beta}{N}}\bigl[ R(\theta) \bigr] - (\lambda - \beta) r
(\theta, \omega).
$$
Integrating with respect to $\mu$ equation
\myeq{eq1.1.11Bis},
written in each sub-model $\Theta_m$ using the prior distribution $\pi(m, \cdot)$,
we see that
\begin{multline*}
\PP \biggl\{ \exp \biggl[
\sup_{\nu \in \C{M}_+^1(M)} \sup_{\rho: M \rightarrow \C{M}_+^1(\Theta)}
\frac{1}{2} \Bigl[ (\nu \rho)(U) - \nu \bigl\{
\C{K}(\bigl[ \rho, \pi_{\exp( - \beta r)}\bigr] \bigr\} \Bigl] - \C{K}(\nu,\mu)
\biggr] \biggr\}
\\ \leq
\PP \biggl\{ \exp \biggl[
\sup_{\nu \in \C{M}_+^1(M)} \frac{1}{2} \nu \biggl( \sup_{\rho: M \rightarrow \C{M}_+^1(\Theta)}
\rho(U) - \C{K}(\rho, \pi_{\exp( - \beta r)}) \biggr)
- \C{K}(\nu, \mu) \biggr] \biggr\}
\\ =
\PP \biggl\{ \mu \biggl[ \exp \Bigl\{ \tfrac{1}{2} \sup_{\rho: M \rightarrow
\C{M}_+^1(\Theta)} \Bigl[ \rho(U) - \C{K} \bigl[ \rho, \pi_{\exp( - \beta r)}\bigr]
\Bigr] \Bigr\} \biggr] \biggr\}\\
= \mu \biggl\{ \PP \biggl[ \exp \Bigl\{ \tfrac{1}{2} \sup_{\rho: M \rightarrow
\C{M}_+^1(\Theta)} \Bigl[ \rho(U) - \C{K} \bigl[ \rho, \pi_{\exp( - \beta r)}\bigr]
\Bigr] \Bigr\} \biggr] \biggr\} \leq 1.
\end{multline*}
This proves that
\begin{multline}
\label{eq1.1.10}
\PP \Biggl\{ \exp \Biggl[ \frac{1}{2}
\sup_{\nu \in \C{M}_+^1(M)} \sup_{\rho:M\rightarrow \C{M}_+^1(\Theta)}
\nu \rho \bigl[ \lambda \Phi_{\frac{\lambda}{N}} (R)
- \beta \Phi_{-\frac{\beta}{N}} (R) \bigr]
\\ -(\lambda - \beta) \nu \rho(r) - 2 \C{K}(\nu,\mu) - \nu \bigl\{
\C{K} \bigl[ \rho,
\pi_{\exp( - \beta r)}\bigr] \bigr\} \Biggr] \Biggr\} \leq 1.
\end{multline}
\newcommand{\sR}{R^{\star}}
\newcommand{\sr}{r^{\star}}
\newcommand{\stheta}{\theta^{\star}}
Introducing the optimal value of $r$ on each sub-model
$\sr(m) = \ess \inf_{\pi(m,\cdot)} r$ and the empirical dimensions
$$
d_e(m) = \sup_{\xi \in \RR_+} \xi \bigl[
\pi_{\exp( - \xi r)}(m,r) - \sr(m) \bigr],
$$
we can thus state
\begin{thm}
\label{thm1.1.20}
\mypoint
For any positive real constants $\beta < \lambda$,
with $\PP$ probability at least $1 - \epsilon$,
for any posterior distribution $\nu: \Omega \rightarrow \C{M}_+^1(M)$,
for any conditional posterior distribution $\rho: \Omega \times
M \rightarrow \C{M}_+^1(\Theta)$,
$$
\nu \rho \bigl[
\lambda \Phi_{\frac{\lambda}{N}}(R) - \beta
\Phi_{- \frac{\beta}{N}}(R) \bigr] \leq \lambda \Phi_{\frac{\lambda}{N}} \bigl[ \nu \rho(R) \bigr]
- \beta \Phi_{-\frac{\beta}{N}} \bigl[ \nu \rho(R) \bigr]
\leq B_1(\nu, \rho),
$$
\begin{multline*}
\text{where } B_1(\nu, \rho) =
(\lambda - \beta) \nu \rho(r) + 2\C{K}(\nu,\mu)+
\nu \bigl\{ \C{K}\bigl[ \rho, \pi_{\exp( - \beta r)} \bigr] \bigr\} - 2
\log(\epsilon)\\
= \nu \biggl[ \int_{\beta}^{\lambda}
\pi_{\exp ( - \alpha r)}(r) d\alpha \biggr] + 2 \C{K}(\nu, \mu)
+ \nu \bigl\{ \C{K}\bigl[ \rho, \pi_{\exp( - \lambda r)} \bigr] \bigr\}
- 2 \log(\epsilon)
\\
= - 2 \log \biggl\{ \mu \biggl[ \exp \biggl( - \frac{1}{2}
\int_{\beta}^{\lambda} \pi_{\exp( - \alpha r)}(r) d \alpha \biggr)
\biggr] \biggr\} \\
\shoveright{+ 2 \C{K}\bigl[ \nu, \mu_{\left(\frac{\pi[\exp(-\lambda r)]}{
\pi[\exp(-\beta r)]}\right)^{1/2}}\bigr] + \nu \bigl\{ \C{K}\bigl[
\rho, \pi_{\exp( - \lambda r)} \bigr] \bigr\} - 2 \log(\epsilon),}\\
\shoveleft{\text{and therefore }
B_1(\nu,\rho) \leq  \nu \Bigl[ (\lambda - \beta) \sr + \log \Bigl( \tfrac{\lambda}{\beta}
\Bigr) d_e
\Bigr] + 2 \C{K}(\nu, \mu)} \\\shoveright{ + \nu \bigl\{ \C{K} \bigl[
\rho, \pi_{\exp( - \lambda r)} \bigr] \bigr\} - 2 \log(\epsilon),}
\\\shoveleft{\text{as well as }
B_1(\nu, \rho) \leq - 2 \log \biggl\{ \mu \biggl[
\exp \biggl( - \tfrac{(\lambda - \beta)}{2} \sr - \tfrac{1}{2}
\log \bigl( \tfrac{\lambda}{\beta} \bigr) d_e \biggr) \biggr] \biggr\}
}\\+ 2 \C{K} \bigl[ \nu, \mu_{\bigl(\frac{\pi[\exp( -
\lambda r)]}{\pi[\exp( - \beta r)]}\bigr)^{1/2}}
\bigr] + \nu \bigl\{ \C{K}\bigl[ \rho, \pi_{\exp( - \lambda r)} \bigr]
- 2 \log(\epsilon).
\end{multline*}
Thus, for any real constants $\alpha$ and $\gamma$ such that
$0 \leq \gamma < \alpha < 1$, with $\PP$ probability
at least $1 - \epsilon$, for any posterior distribution
$\nu: \Omega \rightarrow \C{M}_+^1(M)$ and any conditional posterior
distribution $\rho: \Omega \times M \rightarrow \C{M}_+^1(\Theta)$,
the bound
\begin{multline*}
B_2(\nu,\rho) = - \tfrac{\log \bigl[ (1 - \alpha)(1 + \gamma)\bigr]}{\alpha-\gamma}
\nu\rho(r) + \tfrac{ 2 \C{K}(\nu,\mu) + \nu \bigl\{ \C{K}\bigl[
\rho, \pi_{(1 + \gamma)^{-Nr}}\bigr] \bigr\} - 2 \log(\epsilon)}{N (\alpha - \gamma)}
\\ = \frac{1}{N(\alpha - \gamma)} \Biggl\{
2 \C{K}\biggl[ \nu, \mu_{\left( \frac{ \pi [ (1 -\alpha)^{Nr}]}{\pi [ (1 + \gamma)^{-N
r}]}\right)^{1/2}} \biggr]
+ \nu \Bigl\{ \C{K}\bigl[\rho, \pi_{(1 - \alpha)^{Nr}}\bigr] \Bigr\}
\Biggr\} \\ - \frac{2}{N(\alpha - \gamma)}
\log \Biggl\{ \mu \Biggl[ \exp \biggl[ - \frac{1}{2}
\int_{N \log(1 + \gamma)}^{- N \log(1 - \alpha)} \pi_{\exp( - \xi r)}(\cdot,r) d \xi
\biggr] \Biggr] \Biggr\}
\\
- \frac{2 \log(\epsilon)}{N(\alpha - \gamma)}
\end{multline*}
satisfies
\begin{multline*}
\nu \rho(R) \leq \frac{\alpha - \gamma}{2 \alpha \gamma}
\left( \sqrt{1 + \frac{4 \alpha \gamma}{(\alpha - \gamma)^2} \Bigl\{
1 - \exp \bigl[ - (\alpha - \gamma) B_2(\nu,\rho) \bigr] \Bigr\}} - 1
\right) \\ \leq B_2(\nu,\rho).
\end{multline*}
\end{thm}

If one is willing to bound the deviations with respect to
$\PP \nu \rho$, it is enough to remark that
the equation preceding equation \myeq{eq1.1.10}
can also be written as
$$
\PP \Biggl\{  \mu \Biggl[ \biggl\{
\pi_{\exp( - \beta r)} \biggl[ \exp \Bigl\{
\lambda \Phi_{\frac{\lambda}{N}} \circ R - \beta \Phi_{- \frac{\beta}{N}}
\circ R - (\lambda - \beta) r \Bigr\} \biggr] \biggr\}^{1/2} \Biggr] \Biggr\}
\leq 1.
$$
Thus for any posterior distributions $\nu: \Omega \rightarrow
\C{M}_+^1(M)$ and $\rho: \Omega \times M \rightarrow \C{M}_+^1(\Theta)$,
\begin{multline*}
\PP \biggl\{ \nu \biggl[ \Bigl\{ \rho \Bigl[
\exp \bigl\{ \lambda \Phi_{\frac{\lambda}{N}} \circ R
- \beta \Phi_{- \frac{\beta}{N}} \circ R
\\ - (\lambda - \beta) r - 2 \log \bigl( \tfrac{d \nu}{d \mu} \bigr)
- \log \bigl( \tfrac{d \rho}{d \pi_{\exp ( - \beta r)}} \bigr) \bigr\} \Bigr] \Bigr\}^{1/2}
\biggr] \biggr\} \leq 1.
\end{multline*}
Using the concavity of the square root function to pull the integration
with respect to $\rho$ out of the square root, we get
\begin{multline*}
\PP \nu \rho \biggl\{ \exp \biggl[ \frac{1}{2}
\Bigl\{ \lambda \Phi_{\frac{\lambda}{N}} \circ R
- \beta \Phi_{- \frac{\beta}{N}} \circ R \\
- (\lambda - \beta) r - 2 \log \bigl( \tfrac{d \nu}{d \pi} \bigr)
- \log \bigl( \tfrac{d \rho}{d \pi_{\exp( - \beta r)}} \bigr) \Bigr\}  \biggr] \biggr\}
\leq 1.
\end{multline*}
This leads to
\begin{thm} \mypoint
For any positive real constants $ \beta < \lambda$,
for any posterior distributions $\nu: \Omega \rightarrow \C{M}_+^1(M)$
and $\rho: \Omega \times M \rightarrow \C{M}_+^1(\Theta)$,
with $\PP \nu \rho$ probability at least $1 - \epsilon$,
\begin{multline*}
\lambda \Phi_{\frac{\lambda}{N}}\bigl[ R(\wh{m}, \wh{\theta} \,) \bigr]
- \beta \Phi_{- \frac{\beta}{N}} \bigl[ R(\wh{m}, \wh{\theta}\,) \bigr]
\leq (\lambda - \beta) r(\wh{m}, \wh{\theta})
\\ + 2 \log \bigl[ \tfrac{d \nu}{d \mu}(\wh{m}\,) \bigr]
+ \log \bigl[ \tfrac{ d \rho}{d \pi_{\exp( - \beta r)}}
(\wh{m}, \wh{\theta}\,) \bigr] - 2 \log(\epsilon) \\
= \int_{\beta}^{\lambda} \pi_{\exp ( - \alpha r)}(r) d \alpha
\\ + 2 \log \bigl[ \tfrac{d \nu}{d \mu}(\wh{m}) \bigr]
+ \log \bigl[ \tfrac{ d \rho}{d \pi_{\exp( - \lambda r)}}(\wh{m},
\wh{\theta}\,) \bigr] - 2 \log(\epsilon)
\\ = 2 \log \biggl\{ \mu \biggl[ \exp \biggl( - \frac{1}{2} \int_{\beta}^{\lambda}
\pi_{\exp( - \alpha r)}(r) d \alpha \biggr) \biggr] \biggr\} \\
\\ + 2 \log \bigl[ \tfrac{d \nu}{d \mu_{\bigl( \frac{\pi[ \exp( - \lambda r)]}{
\pi [ \exp( - \beta r)]} \bigr)^{1/2}}} ( \wh{m}) \bigr] + \log \bigl[
\tfrac{d \rho}{d \pi_{\exp( - \lambda r)}} (\wh{m}, \wh{\theta}\,) \bigr]
- 2 \log(\epsilon).
\end{multline*}
Another way to state the same inequality is to say that for any
real constants $\alpha$ and $\gamma$ such that
$0 \leq \gamma < \alpha < 1$, with $\PP\nu\rho$
probability at least $1 - \epsilon$,
\begin{multline*}
R(\wh{m}, \wh{\theta}) \\ \leq \frac{\alpha - \gamma}{2 \alpha \gamma}
\biggl( \sqrt{1 + \frac{4 \alpha \gamma}{(\alpha - \gamma)^2} \Bigl\{
1 - \exp \bigl[ - (\alpha - \gamma) B(\wh{m}, \wh{\theta})
\bigr] \Bigr\}} - 1 \biggr) \\ \leq B(\wh{m},\wh{\theta}),
\end{multline*}
where
\begin{multline*}
B(\wh{m}, \wh{\theta}) = - \frac{
\log \bigl[ (1 - \alpha)(1+\gamma)\bigr]}{
\alpha - \gamma} r(\wh{m}, \wh{\theta}) \\
\shoveright{+
\frac{ 2 \log \Bigl[
\frac{d \nu}{d \mu}(\wh{m}) \Bigr]
+ \log \Bigl[\frac{d \rho}{d \pi_{(1+\gamma)^{-Nr}}} (\wh{m},
\wh{\theta}) \Bigr] - 2 \log(\epsilon)}{N(\alpha - \gamma)}
\quad}\\ \shoveleft{ \quad = \frac{2}{N(\alpha-\gamma)}  \log \biggl[
\frac{d \nu}{d \mu_{\Bigl(
\frac{\pi [(1-\alpha)^{Nr} ]}{\pi[
(1 + \gamma)^{-Nr}]}\Bigr)^{1/2}}}(\wh{m}) \biggr]
}\\ + \frac{\log \Bigl[ \frac{d \rho}{d \pi_{(1-\alpha)^{Nr}}}(\wh{m}, \wh{\theta})
\Bigr]- 2 \log(\epsilon)}{N(\alpha - \gamma)} \qquad \qquad \\
+ \frac{2}{N(\alpha - \gamma)} \log \biggl\{ \mu \biggl[
\exp \biggl( - \frac{1}{2}  \int_{\beta}^{\lambda}
\pi_{\exp( - \alpha r)}(r) d \alpha \biggr) \biggr] \biggr\}.
\end{multline*}
\end{thm}

Let us remark that in the case when $\nu = \mu_{\left( \frac{
\pi[(1 - \alpha)^{Nr}]}{\pi[(1 + \gamma)^{-Nr}]} \right)^{1/2}}$
and $\rho = \pi_{(1-\alpha)^{Nr}}$,
we get as desired a bound that is adaptively local in all the $\Theta_m$
(at least when $M$ is countable and $\mu$ is atomic):
\begin{multline*}
B(\nu,\rho) \leq - \tfrac{2}{N(\alpha - \gamma)}
\log \Biggl\{ \mu \biggl\{
\exp \biggl[ \tfrac{N}{2} \log\bigl[(1+\gamma)(1 - \alpha)\bigr]
\sr  \\\shoveright{ - \log \left( \tfrac{-\log(1-\alpha)}{\log(1 + \gamma)}
\right) \tfrac{d_e}{2} \biggr] \biggr\} \Biggr\}
- \frac{2 \log(\epsilon)}{N(\alpha - \gamma)}\qquad}
\\\shoveleft{\qquad \qquad \leq \inf_{m \in M} \biggl\{
- \tfrac{\log\bigl[ (1- \alpha)(1+\gamma)\bigr]}{\alpha
-\gamma} \sr(m)} \\ +
\log \left( \tfrac{- \log(1 - \alpha)}{\log(1 + \gamma)}\right)
\tfrac{d_e(m)}{N(\alpha - \gamma)} -
2 \tfrac{\log\bigl[\epsilon \mu(m) \bigr]}{N(\alpha - \gamma)} \biggr\}.
\end{multline*}
The penalization by the \emph{empirical dimension} $d_e(m)$ in each sub-model
is as desired linear in $d_e(m)$. Non random partially local bounds could
be obtained in a way that is easy to imagine. We leave this investigation
to the reader.
\eject

\subsection{Two step localization}

We have seen that the bound optimal choice of the posterior
distribution $\nu$ on the index set in Theorem \ref{thm1.1.20}
(page \pageref{thm1.1.20}) is such that
$$
\frac{d\nu}{d \mu}(m)  \sim
\left( \frac{\pi \bigl[ \exp\bigl( - \lambda r(m, \cdot) \bigr) \bigr]}{\pi
\bigl[ \exp\bigl( - \beta r(m,\cdot) \bigr)  \bigr]}\right)^{\frac{1}{2}}
= \exp \biggl[ - \frac{1}{2} \int_{\beta}^{\lambda}
\pi_{\exp( - \alpha r)}(m,r)  d \alpha \biggr].
$$
This suggests replacing the prior distribution $\mu$ with $\ov{\mu}$
defined by its density
\begin{multline}
\label{eq1.13}
\frac{d \ov{\mu}}{d \mu} (m) = \frac{ \exp \bigl[ - h(m) \bigr]}{\mu
\bigl[ \exp( - h ) \bigr]},
\\ \text{ where }
h(m) = - \xi \int_{\beta}^{\gamma} \pi_{\exp( - \alpha \Phi_{- \frac{\eta}{N}}
\circ R)} \bigl[ \Phi_{- \frac{\eta}{N}}\!\circ\!R(m, \cdot) \bigr] d \alpha.
\end{multline}
The use of $\Phi_{- \frac{\eta}{N}}\!\circ\!R$ instead of $R$ is motivated
by technical reasons which will appear in subsequent computations.
Indeed, we will need to bound
$$
\nu \biggl[ \int_{\beta}^{\lambda} \pi_{\exp ( - \alpha
\Phi_{- \frac{\eta}{N}} \circ R)} \bigl(
\Phi_{- \frac{\eta}{N}}\!\circ\!R \bigr) d \alpha \biggr]
$$
in order to handle $\C{K}(\nu, \ov{\mu})$.
In the spirit of equation (\ref{eq1.1.4}, page \pageref{eq1.1.4}),
starting back from Theorem \ref{thm2.3} (page \pageref{thm2.3}),
applied in each sub-model $\Theta_m$ to the prior
distribution $\pi_{\exp( - \gamma \Phi_{-\frac{\eta}{N}} \circ
R )}$ and integrated with respect to
$\ov{\mu}$, we see that for any
positive real constants $\lambda$, $\gamma$ and $\eta$,
with $\PP$ probability at least $1 - \epsilon$,
for any posterior distribution $\nu: \Omega \rightarrow \C{M}_+^1(M)$ on the index set
and any conditional posterior distribution $\rho: \Omega \times M \rightarrow
\C{M}_+^1(\Theta)$,
\begin{multline}
\label{eq1.1.13}
\nu \rho \bigl( \lambda \Phi_{\frac{\lambda}{N}}\!\circ\!R - \gamma
\Phi_{-\frac{\eta}{N}}\!\circ\!R \bigr) \leq \lambda \nu \rho(r) \\ +
\nu \C{K}(\rho, \pi)
+ \C{K}(\nu, \ov{\mu}) +
\nu \Bigl\{ \log \Bigl[ \pi \bigl[ \exp \bigl(
- \gamma \Phi_{- \frac{\eta}{N}}\!\circ\!R \bigr) \bigr]  \Bigr] \Bigr\} -
\log(\epsilon).
\end{multline}
Since $x \mapsto f(x) \overset{\text{\rm def}}{=}
\lambda \Phi_{\frac{\lambda}{N}}
- \gamma \Phi_{- \frac{\eta}{N}}(x)$ is a convex function, it is such
that
$$
f(x) \geq x f'(0)= x N \Bigl\{
 \bigl[1 - \exp( - \tfrac{\lambda}{N}) \bigr] + \tfrac{\gamma}{\eta}
\bigl[ \exp( \tfrac{\eta}{N}) - 1 \bigr] \Bigr\}.
$$
Thus if we put
\begin{equation}
\label{eq1.14}
\gamma = \frac{\eta \bigl[ 1 - \exp (- \frac{\lambda}{N}) \bigr]}{\exp(
\frac{\eta}{N}) - 1},
\end{equation}
we obtain that $f(x) \geq 0$, $x \in \RR$, and therefore that
the left-hand side of equation \eqref{eq1.1.13} is non-negative.
We can moreover introduce the prior conditional distribution $\ov{\pi}$ defined
by
$$
\frac{d \ov{\pi}}{d \pi}(m, \theta) =
\frac{ \exp \bigl[ - \beta \Phi_{- \frac{\eta}{N}} \circ R(\theta) \bigr]}{
\pi \bigl\{m, \exp \bigl[ - \beta \Phi_{- \frac{\eta}{N}} \circ R \bigr] \bigr\}}.
$$
With $\PP$ probability at least $1 - \epsilon$, for any posterior distributions
$\nu : \Omega \rightarrow \C{M}_+^1(M)$ and $\rho: \Omega \times M \rightarrow
\C{M}_+^1(\Theta)$,
\begin{multline*}
\beta \nu \rho(r) + \nu \bigl[ \C{K}( \rho, \pi) \bigr] =
\nu \bigl\{ \C{K}\bigl[ \rho, \pi_{\exp (- \beta r)} \bigr] \bigr\} -
\nu \biggl[ \log \Bigl\{ \pi \bigl[ \exp ( - \beta r) \bigr] \Bigr\}  \biggr]
\\ \leq \nu \bigl\{ \C{K} \bigl[ \rho, \pi_{\exp( - \beta r)} \bigr] \bigr\}
+ \beta \nu \ov{\pi} (r) + \nu \bigl[ \C{K}(\ov{\pi}, \pi) \bigr] \\
\leq \nu \bigl\{ \C{K} \bigl[ \rho, \pi_{\exp ( - \beta r)} \bigr] \bigr\}
+ \beta \nu \ov{\pi} \bigl( \Phi_{- \frac{\eta}{N}}\!\circ\!R \bigr)
\\\shoveright{+ \tfrac{\beta}{\eta} \bigl[ \C{K}(\nu, \ov{\mu})- \log(\epsilon) \bigr]
+ \nu \bigl[ \C{K}(\ov{\pi}, \pi) \bigr] \qquad}
\\\shoveleft{\qquad
= \nu \bigl\{ \C{K} \bigl[ \rho, \pi_{\exp ( - \beta r)} \bigr] \bigr\}
- \nu \Bigl\{ \log \Bigl[ \pi \bigl[ \exp \bigl( -
\beta \Phi_{-\frac{\eta}{N}}\!\circ\!R \bigr) \bigr] \Bigr] \Bigr\}}
\\ + \tfrac{\beta}{\eta} \bigl[ \C{K}(\nu, \ov{\mu}) - \log(\epsilon) \bigr].
\end{multline*}
Thus, coming back to equation \eqref{eq1.1.13}, we see that under condition
\eqref{eq1.14},
with $\PP$ probability at least $1 - \epsilon$,
\begin{multline*}
0 \leq (\lambda - \beta) \nu \rho(r) + \nu \bigl\{ \C{K}\bigl[ \rho, \pi_{
\exp( - \beta r)}\bigr] \bigr\} \\ - \nu \biggl[
\int_{\beta}^{\gamma} \pi_{\exp( - \alpha \Phi_{- \frac{\eta}{N}} \circ R)}
\bigl( \Phi_{- \frac{\eta}{N}}\!\circ\!R \bigr) d \alpha \biggr]
+ (1 + \tfrac{\beta}{\eta}) \bigl[ \C{K}(\nu, \ov{\mu}) + \log(\tfrac{2}{\epsilon})
\bigr].
\end{multline*}
Noticing moreover that
\begin{multline*}
(\lambda - \beta) \nu \rho(r) + \nu \bigl\{ \C{K} \bigl[
\rho, \pi_{\exp( - \beta r)}\bigr] \bigr\} \\ =
\nu \bigl\{ \C{K}\bigl[ \rho, \pi_{\exp ( - \lambda r)}\bigr] \bigr\}
+ \nu \biggl[ \int_{\beta}^{\lambda} \pi_{\exp( - \alpha r)}(r) d \alpha \biggr],
\end{multline*}
and choosing $\rho = \pi_{\exp( - \lambda r)}$, we have proved
\begin{thm}\mypoint
For any positive real constants $\beta$, $\gamma$ and $\eta$, such that
\linebreak $\gamma < \eta \bigl[ \exp( \frac{\eta}{N}) - 1 \bigr]^{-1}$, defining
$\lambda$ by condition \eqref{eq1.14}, so that \linebreak
$\lambda = - N \log \Bigl\{ 1 - \frac{\gamma}{\eta} \bigl[ \exp(
\frac{\eta}{N}) - 1 \bigr] \Bigr\}$,
with $\PP$ probability at least $1 - \epsilon$,
for any posterior distribution $\nu: \Omega \rightarrow \C{M}_+^1(M)$,
any conditional posterior distribution $\rho: \Omega \times M
\rightarrow \C{M}_+^1(\Theta)$,
\begin{multline*}
\nu \biggl[ \int_{\beta}^{\gamma}
\pi_{\exp( - \alpha \Phi_{- \frac{\eta}{N}}\circ R)}
\bigl( \Phi_{- \frac{\eta}{N}}\!\circ\!R \bigr) d \alpha \biggr]
\\ \leq \nu \biggl[ \int_{\beta}^{\lambda} \pi_{\exp( - \alpha r)}(r)
d \alpha \biggr] + \bigl( 1 + \tfrac{\beta}{\eta} \bigr)
\bigl[ \C{K}(\nu, \ov{\mu}) + \log\bigl(\tfrac{2}{\epsilon}\bigr) \bigr].
\end{multline*}
\end{thm}

Let us remark that this theorem does not require that $\beta < \gamma$,
and thus provides both an upper and a lower bound for the quantity of
interest:
\begin{cor}
\mypoint
For any positive real constants $\beta$, $\gamma$ and $\eta$
such that
$\max \{ \beta,\break \gamma \} < \eta \bigl[ \exp(\frac{\eta}{N}) - 1 \bigr]^{-1}$,
with $\PP$ probability at least $1- \epsilon$, for any posterior distributions
$\nu: \Omega \rightarrow \C{M}_+^1(M)$ and $\rho: \Omega \times M \rightarrow
\C{M}_+^1(\Theta)$,
\begin{multline*}
\nu \biggl[ \int_{- N \log \{ 1 - \frac{\beta}{N} [
\exp (\frac{\eta}{N}) -1 ] \}}^{\gamma} \pi_{\exp( - \alpha r)}(r) d \alpha \biggr]
- \bigl( 1 + \tfrac{\gamma}{\eta} \bigr)\bigl[ \C{K}(\nu, \ov{\mu}) +
\log \bigl( \tfrac{3}{\epsilon} \bigr) \bigr]
\\ \shoveleft{\qquad \leq  \nu \biggl[ \int_{\beta}^{\gamma} \pi_{\exp( - \alpha
\Phi_{- \frac{\eta}{N}}\circ R)} \bigl(
\Phi_{- \frac{\eta}{N}}\!\circ\!R \bigr) d \alpha \biggr] }
\\ \leq \nu \biggl[ \int_{\beta}^{- N \log \{ 1 - \frac{\gamma}{\eta}
[ \exp(\frac{\eta}{N})-1 ] \}}
\pi_{\exp( - \alpha r)}(r) d \alpha \biggr]
\\ + \bigl( 1 + \tfrac{\beta}{\eta} \bigr) \bigl[
\C{K}(\nu, \ov{\mu}) + \log \bigl( \tfrac{3}{\epsilon} \bigr) \bigr].
\end{multline*}
\end{cor}

We can then remember that
$$
\C{K}(\nu, \ov{\mu}) = \xi \bigl( \nu - \ov{\mu} \bigr)  \biggl[ \int_{\beta}^{\gamma}
\pi_{\exp( - \alpha \Phi_{- \frac{\eta}{N}}\circ R)} \bigl(
\Phi_{- \frac{\eta}{N}}\!\circ\!R \bigr) d \alpha \biggr] + \C{K}(\nu, \mu) -
\C{K}(\ov{\mu}, \mu),
$$
to conclude that, putting
\begin{equation}
\label{eq1.16}
G_{\eta}(\alpha) =
-N \log \bigl\{ 1 - \tfrac{\alpha}{\eta} \bigl[
\exp \bigl( \tfrac{\eta}{N}) - 1 \bigr] \bigr\} \geq \alpha, \qquad \alpha \in \RR_+,
\end{equation}
and
\begin{equation}
\label{eq1.15}
\frac{d \w{\nu}}{d \mu} (m) \overset{\text{\rm def}}{=}
\frac{\exp \bigl[ - h(m) \bigr]}{\mu \bigl[ \exp( - h)\bigr]}
\text{ where }
h(m) = \xi \int_{G_{\eta}(\beta)}^{\gamma} \pi_{\exp( - \alpha r)}(m, r) d \alpha,
\end{equation}
the divergence of $\nu$ with respect to the local prior $\ov{\mu}$ is bounded by
\begin{multline*}
\bigl[ 1 - \xi \bigl( 1 + \tfrac{\beta}{\eta} \bigr) \bigr]
\C{K}(\nu, \ov{\mu}) \\
\shoveleft{\qquad \leq \xi \nu \biggl[ \int_{\beta}^{
G_{\eta}(\gamma)}
\pi_{\exp( - \alpha r)}(r) d \alpha \biggr]
- \xi \ov{\mu} \biggl[ \int_{G_{\eta}(\beta)}^{\gamma} \pi_{\exp( - \alpha r)}(r)
d \alpha \biggr]} \\ \shoveright{+ \C{K}(\nu, \mu)
- \C{K}(\ov{\mu}, \mu)
+ \xi \bigl( 2 +
\tfrac{\beta + \gamma}{\eta} \bigr)
\log\bigl(\tfrac{3}{\epsilon}\bigr)} \\
\shoveleft{\qquad \leq \xi \nu \biggl[ \int_{\beta}^{G_{\eta}(\gamma)} \pi_{\exp( - \alpha r)}(r)
d \alpha \biggr] + \C{K}(\nu, \mu)} \\ +
\log \biggl\{ \mu \biggl[ \exp \biggl( - \xi \int_{G_{\eta}(\beta)}^{\gamma}
\pi_{\exp(- \alpha r)}(r) d \alpha \biggr) \biggr] \biggr\}
\\
\shoveright{+ \xi \bigl( 2 +
\tfrac{\beta + \gamma}{\eta} \bigr)
\log\bigl(\tfrac{3}{\epsilon}\bigr)}
\\
\shoveleft{\qquad = \C{K}(\nu, \w{\nu}) + \xi \nu \biggl[ \biggl( \int_{\beta}^{G_{\eta}(\beta)}
+ \int_{\gamma}^{G_{\eta}(\gamma)}\biggr)  \pi_{\exp( - \alpha r)}(r) d \alpha \biggr]}
\\
+ \xi \bigl( 2 + \tfrac{\beta+\gamma}{\eta} \bigr) \log \bigl( \tfrac{3}{\epsilon}
\bigr).
\end{multline*}
We have proved
\begin{thm}
\mypoint
\label{thm1.23}
For any positive constants $\beta$, $\gamma$ and $\eta$ such that
\linebreak $\max \{ \beta, \gamma \}
< \eta \bigl[ \exp( \frac{\eta}{N}) - 1 \bigr]^{-1}$,
with $\PP$ probability at least $1 - \epsilon$, for any posterior distribution
$\nu: \Omega \rightarrow \C{M}_+^1(M)$ and any conditional posterior distribution
$\rho: \Omega \times M \rightarrow \C{M}_+^1(\Theta)$,
\begin{multline*}
\C{K}(\nu, \ov{\mu}) \leq \Bigl[1 - \xi\Bigl(1
+ \frac{\beta}{\eta}\Bigr)\Bigr]^{-1}
\biggl\{
\C{K}(\nu, \w{\nu})
\\
+ \xi \nu \biggl[ \biggl( \int_{\beta}^{G_{\eta}(\beta)}
+ \int_{\gamma}^{G_{\eta}(\gamma)}\biggr)
\pi_{\exp( - \alpha r)} (r) d \alpha \biggr]
\\\shoveright{ + \xi \bigl( 2 + \tfrac{\beta+\gamma}{\eta} \bigr)
\log \bigl( \tfrac{3}{\epsilon}
\bigr) \biggr\}}
\\ \shoveleft{ \leq  \Bigl[ 1 - \xi\Bigl(1 + \frac{\beta}{\eta}\Bigr) \Bigr]^{-1}
\biggl\{ \C{K}(\nu, \w{\nu})}\\ + \xi \nu \biggl[
\bigl[ G_{\eta}(\gamma)
- \gamma  + G_{\eta}(\beta)- \beta \bigr] \sr +
\log \biggl( \frac{G_{\eta}(\beta)
G_{\eta}(\gamma)}{\beta \gamma}\biggr)
d_e \biggr] \\ +
\xi \bigl( 2 + \tfrac{\beta+\gamma}{\eta} \bigr) \log \bigl(
\tfrac{3}{\epsilon} \bigr) \biggr\},
\end{multline*}
where the local prior $\ov{\mu}$ is defined by equation \myeq{eq1.13}
and the local posterior $\w{\nu}$ and the function
$G_{\eta}$ are defined by equation \myeq{eq1.15}.
\end{thm}

We can then use this theorem to give a local version of Theorem
\thmref{thm1.1.20}. To get something pleasing
to read, we can apply Theorem \ref{thm1.23} with constants
$\beta'$, $\gamma'$ and $\eta$ chosen so that
$ \frac{2 \xi}{1 - \xi(1 + \frac{\beta'}{\eta})} = 1,$
$G_{\eta}(\beta') = \beta$ and $\gamma' = \lambda$, where
$\beta$ and $\lambda$ are the constants appearing in Theorem
\ref{thm1.1.20}. This gives
\begin{thm}\mypoint
\label{thm1.24}
For any positive real constants $\beta < \lambda$ and $\eta$
such that $\lambda < \eta \bigl[ \exp(\frac{\eta}{N}) - 1 \bigr]^{-1}$,
with $\PP$ probability at least $1 - \epsilon$, for any posterior distribution
$\nu: \Omega \rightarrow \C{M}_+^1(M)$, for any conditional posterior distribution
$\rho: \Omega \times M \rightarrow \C{M}_+^1(\Theta)$,
\begin{multline*}
\nu \rho \bigl[
\lambda \Phi_{\frac{\lambda}{N}}(R)
- \beta \Phi_{-\frac{\beta}{N}}(R)
\bigr] \leq \lambda \Phi_{\frac{\lambda}{N}}
\bigl[ \nu \rho(R) \bigr]
- \beta \Phi_{- \frac{\beta}{N}} \bigl[ \nu \rho(R) \bigr]
\leq B_3(\nu, \rho),\\
\text{where } B_3(\nu, \rho) =
\nu \biggl[ \int_{G_{\eta}^{-1} (\beta)}^{G_{\eta}(\lambda)}
\pi_{\exp( - \alpha r)}(r) d \alpha \biggr]
\hfill \\ + \Bigl(3 + \tfrac{G_{\eta}^{-1}(\beta)}{
\eta} \Bigr) \C{K}\bigl[ \nu, \mu_{\exp \bigl[ - \bigl(3
+ \frac{G_{\eta}^{-1}(\beta)}{\eta}\bigr)^{-1}
\int_{\beta}^{\lambda} \pi_{\exp( - \alpha
r)}(r) d \alpha \bigr]}\bigr]
\\\shoveright{ + \nu \bigl\{ \C{K}(\rho,
\pi_{\exp( - \lambda r)}\bigr] \bigr\} + \Bigl( 4 +
\tfrac{G_{\eta}^{-1}(\beta)+\lambda}{\eta} \Bigr) \log \bigl( \tfrac{4}{\epsilon}
\bigr)}\\
\shoveleft{\qquad \leq \nu \Bigl[ \bigl[ G_{\eta}(\lambda) - G_{\eta}^{-1}(\beta)  \bigr]
\sr + \log \Bigl(\tfrac{G_{\eta}(\lambda)}{G_{\eta}^{-1}(\beta)} \Bigr) d_e
\Bigr]}
\\
+ \Bigl(3 + \tfrac{G_{\eta}^{-1}(\beta)}{
\eta} \Bigr) \C{K}\bigl[ \nu, \mu_{\exp \bigl[ - \bigl(3+\frac{
G_{\eta}^{-1}(\beta)}{\eta}\bigr)^{-1} \int_{\beta}^{\lambda} \pi_{\exp( - \alpha
r)}(r) d \alpha \bigr]}\bigr]
\\ + \nu \bigl\{ \C{K}(\rho,
\pi_{\exp( - \lambda r)}\bigr] \bigr\} + \Bigl( 4 +
\tfrac{G_{\eta}^{-1}(\beta)+\lambda}{\eta} \Bigr) \log \bigl( \tfrac{4}{\epsilon}
\bigr),
\end{multline*}
and where the function $G_{\eta}$ is defined by equation
\myeq{eq1.16}.
\end{thm}

A first remark: if we had the stamina to use Cauchy Schwarz inequalities
(or more generally H\"older inequalities) on exponential moments
instead of using weighted union bounds on deviation inequalities, we could have
replaced $\log(\frac{4}{\epsilon})$ with $- \log(\epsilon)$ in the above inequalities.

We see that we have achieved the desired kind of localization of Theorem
\ref{thm1.1.20} (page \pageref{thm1.1.20}), since the new empirical
entropy term \\\mbox{} \hfill$\C{K}[\nu, \mu_{\exp [
- \xi \int_{\beta}^{\lambda} \pi_{\exp( - \alpha r)}(r) d\alpha ]}]$
\hfill\mbox{}\\
cancels for a value of the posterior distribution on the index set $\nu$
which is of the same form as the one minimizing the bound $B_1(\nu, \rho)$
of Theorem \ref{thm1.1.20} (with a decreased constant, as could be expected).
In a typical parametric setting, we will have
$$
\int_{\beta}^{\lambda} \pi_{\exp( - \alpha r)}(r) d\alpha
\simeq (\lambda - \beta) \sr(m) + \log \left( \tfrac{\lambda}{\beta} \right)
d_e(m),
$$
and therefore, if we choose for $\nu$ the Dirac mass at\\\mbox{}\hfill
$\w{m} \in \arg \min_{m \in M} \sr(m) +
\frac{\log(\frac{\lambda}{\beta})}{\lambda - \beta} d_e(m)$,\hfill
\mbox{}\\
and $\rho(m,\cdot) = \pi_{\exp( - \lambda r)}(m, \cdot)$,
we will get, in the case when the index set $M$ is countable,
\begin{multline*}
B_3(\nu, \rho) \lesssim
\max \left\{ \bigl[ G_{\eta}(\lambda) - G_{\eta}^{-1}(\beta) \bigr]
, (\lambda - \beta)\tfrac{\log\bigl[\frac{G_{\eta}(\lambda)}{
G_{\eta}^{-1}(\beta)}\bigr]}{
\log(\frac{\lambda}{\beta})}\right\}
\\ \shoveright{\times \Bigl[ \sr(\w{m}) + \tfrac{\log(\frac{\lambda}{\beta})}{\lambda - \beta}
d_e(\w{m}) \Bigr]\quad}\\
\shoveleft{\quad + \Bigl( 3 +
\tfrac{G_{\eta}^{-1}(\beta)}{\eta} \Bigr)
\log \Biggl\{ \sum_{m \in M} \frac{\mu(m)}{\mu(\w{m})}
\exp \biggl[ - \Bigl( 3 + \tfrac{G_{\eta}^{-1} (\beta)}{\eta}\Bigr)^{-1}}\\
\times
\Bigl\{ (\lambda - \beta) \bigl[ \sr(m) - \sr(\w{m}) \bigr]
+ \log \bigl( \tfrac{\lambda}{\beta} \bigr)
\bigl[ d_e(m)- d_e(\w{m}) \bigr] \Bigr\} \biggr] \Biggr\} \\
+ \Bigl(4 + \tfrac{G_{\eta}^{-1}(\beta)+\lambda}{\eta}\Bigr)\log\bigl(\tfrac{4}{
\epsilon}\bigr).
\end{multline*}
This shows that the impact on the bound
of the addition of supplementary models depends on their
penalized minimum empirical risk
$\sr(m) + \frac{\log(\frac{\lambda}{\beta})}{\lambda - \beta}
\,d_e(m)$. More precisely the adaptive and local complexity factor
\begin{multline*}
\log \Biggl\{ \sum_{m \in M} \frac{\mu(m)}{\mu(\w{m})}
\exp \biggl[ - \Bigl( 3 + \tfrac{G_{\eta}^{-1} (\beta)}{\eta}\Bigr)^{-1}\\
\times
\Bigl\{ (\lambda - \beta) \bigl[ \sr(m) - \sr(\w{m}) \bigr]
+ \log \bigl( \tfrac{\lambda}{\beta} \bigr)
\bigl[ d_e(m)- d_e(\w{m}) \bigr] \Bigr\} \biggr] \Biggr\}
\end{multline*}
replaces in this bound the non local factor
$$
\C{K}(\nu, \mu) = - \log \bigl[ \mu(\wh{m})\bigr]
= \log \Biggl[\, \sum_{m \in M} \frac{\mu(m)}{\mu(\wh{m})} \Biggr]
$$
which appears when applying Theorem \thmref{thm1.1.20} to the
Dirac mass $\nu = \delta_{\wh{m}}$. Thus in the local bound,
the influence of models decreases exponentially fast when
their penalized empirical risk increases.

One can deduce a result about the deviations with respect to the posterior
$\nu \rho$ from Theorem \thmref{thm1.24} without much supplementary work:
it is enough for that purpose to remark that
with $\PP$ probability at least $1 - \epsilon$, for any
posterior distribution $\nu: \Omega \rightarrow \C{M}_+^1(M)$,
\begin{multline*}
\nu \biggl[ \log \Bigl\{ \pi_{\exp( - \lambda r)} \Bigl[ \exp \bigl\{
\lambda \Phi_{\frac{\lambda}{N}} (R) - \beta \Phi_{- \frac{\beta}{N}}
(R) \bigr\} \Bigr] \Bigr\} \biggr]
\\ - \nu \left( \int_{G_{\eta}^{-1}(\beta)}^{G_{\eta}(\lambda)} \pi_{\exp( - \alpha r)}
(r) d \alpha \right) \qquad \qquad \qquad \qquad \qquad \qquad \qquad
\\ - \Bigl( 3 + \tfrac{G_{\eta}^{-1}(\beta)}{\eta} \Bigr)
\C{K} \bigl[ \nu, \mu_{\exp \Bigl[ - \bigl( 3 + \frac{G_{\eta}^{-1}(\beta)}{\eta}
\bigr)^{-1} \int_{\beta}^{\lambda}
\pi_{\exp( - \alpha r)}(r)d \alpha \Bigr]} \bigr] \\ -
\Bigl(4 + \tfrac{G_{\eta}^{-1}(\beta)+\lambda}{\eta} \Bigr) \log
\Bigl(\tfrac{4}{\epsilon} \Bigr) \leq 0,
\end{multline*}
this inequality being obtained by taking a supremum in $\rho$ in
Theorem \thmref{thm1.24}. One can then take a supremum in $\nu$,
to get, still with $\PP$ probability at least $1 - \epsilon$,
\begin{multline*}
\log \Biggl\{ \mu_{\exp
\Bigl[ - \bigl(3 + \frac{G_{\eta}^{-1}(\beta)}{\eta} \Bigr)^{-1}
\int_{\beta}^{\lambda} \pi_{\exp( - \alpha r)}(r) d \alpha \Bigr]}
\Biggl[ \\
\Bigl\{ \pi_{\exp( - \lambda r)} \Bigl[ \exp \bigl\{
\lambda \Phi_{\frac{\lambda}{N}}
(R) - \beta \Phi_{- \frac{\beta}{N}}(R) \bigr\} \Bigr] \Bigr\}^{\bigl(
3 + \frac{G_{\eta}^{-1}(\beta)}{\eta} \bigr)^{-1}} \\
\times \exp \Biggl( - \Bigl(3 + \tfrac{G_{\eta}^{-1}(\beta)}{\eta}
\Bigr)^{-1} \int_{G_{\eta}^{-1}(\beta)}^{G_{\eta}(\lambda)}
\pi_{\exp(-\alpha r)}(r) d \alpha \Biggr) \Biggr] \Biggr\}
\\ \leq \frac{4 + \frac{G_{\eta}^{-1}(\beta)+\lambda}{\eta}}{
3 + \frac{G_{\eta}^{-1}(\beta)}{\eta}} \log \bigl( \tfrac{4}{\epsilon}
\bigr).
\end{multline*}
Using the fact that $x \mapsto x^\alpha$ is concave when
$\alpha = \bigl(3 + \frac{G_{\eta}^{-1}(\beta)}{\eta}\bigr)^{-1}
< 1$, we get for any posterior conditional distribution
$\rho: \Omega \times M \rightarrow \C{M}_+^1(\Theta)$,
\begin{multline*}
\mu_{\exp
\Bigl[ - \bigl(3 + \frac{G_{\eta}^{-1}(\beta)}{\eta} \Bigr)^{-1}
\int_{\beta}^{\lambda} \pi_{\exp( - \alpha r)}(r) d \alpha \Bigr]}
\rho \Biggl\{
\\ \exp \Biggl[ \Bigl(3 + \tfrac{G_{\eta}^{-1}(\beta)}{\eta}
\Bigr)^{-1} \Biggl(
\lambda \Phi_{\frac{\lambda}{N}}(R)
- \beta \Phi_{- \frac{\beta}{N}}(R) - \int_{G_{\eta}^{-1}(\beta)}^{G_{\eta}(\lambda)}
\pi_{\exp(-\alpha r)}(r) d \alpha \\
+ \log \left[ \frac{d \rho}{d \pi_{\exp( - \lambda r)}}
(\wh{m}, \wh{\theta}\,) \right] \Biggr) \Biggr] \Biggr\}
\\ \leq \exp \Biggl( \frac{4 + \frac{G_{\eta}^{-1}(\beta)+\lambda}{\eta}}{
3 + \frac{G_{\eta}^{-1}(\beta)}{\eta}} \log \bigl( \tfrac{4}{\epsilon}
\bigr) \Biggr).
\end{multline*}
We can thus state
\begin{thm}\mypoint
For any $\epsilon \in )0,1($, with $\PP$ probability at least $1 - \epsilon$,
for any posterior distribution $\nu: \Omega \rightarrow \C{M}_+^1(M)$
and conditional posterior distribution $\rho: \Omega \times M \rightarrow
\C{M}_+^1(\Theta)$, for any $\xi \in )0,1($,
with $\nu \rho$ probability at least $1 - \xi$,
\begin{multline*}
\lambda \Phi_{\frac{\lambda}{N}}(R) - \beta \Phi_{- \frac{\beta}{N}}(R)
\leq \int_{G_{\eta}^{-1}(\beta)}^{G_{\eta}(\lambda)} \pi_{\exp( - \alpha r)}(r)
d \alpha
\\
+ \Bigl(3 + \tfrac{G_{\eta}^{-1}(\beta)}{
\eta} \Bigr) \log \left[ \frac{d \nu}{d \mu_{\exp \Bigl[
- \bigl(3 + \frac{G_{\eta}^{-1}(\beta)}{\eta} \bigr)^{-1}
\int_{\beta}^{\lambda}\pi_{\exp( - \alpha r)}(r) d \alpha \Bigr]}}
(\wh{m}) \right] \\
+ \log \left[ \frac{d \rho}{d \pi_{\exp( - \lambda r)}}(\wh{m}, \wh{\theta}
\,) \right]
+ \Bigl( 4 + \tfrac{G_{\eta}^{-1}(\beta) + \lambda}{\eta} \Bigr)
\log \bigl( \tfrac{4}{\epsilon} \bigr) -
\Bigl( 3 + \tfrac{G_{\eta}^{-1}(\beta)}{\eta} \Bigr) \log (\xi).
\end{multline*}
\end{thm}

Note that the given bound consequently holds with $\PP \nu \rho$
probability at least $(1 - \epsilon)(1 - \xi) \geq 1 - \epsilon - \xi$.
\eject

\section{Relative bounds}
The behaviour of the minimum
of the empirical process $\theta \mapsto r(\theta)$
is known to depend on the covariances between pairs $\bigl[
r(\theta), r(\theta') \bigr]$, $\theta, \theta' \in \Theta$.
In this respect, our previous study, based on the analysis of the variance
of $r(\theta)$ (or technically on some exponential moment playing
quite the same role), loses some accuracy in some circumstances
(namely when $\inf_{\Theta} R$ is not close enough to zero).

In this section, instead of bounding the expected risk $\rho(R)$
of any posterior distribution,
we are going to upper bound the difference $\rho(R) - \inf_{\Theta} R$,
and more generally $\rho(R) - R(\T)$, where $\T \in \Theta$ is some
fixed parameter value.

In the next section
we will analyse $\rho(R) - \pi_{\exp( - \beta R)}(R)$, allowing us to compare the expected error
rate of a posterior distribution $\rho$ with the error rate
of a Gibbs prior distribution.
We will also analyse $\rho_1(R) - \rho_2(R)$, where $\rho_1$
and $\rho_2$ are two arbitrary posterior distributions,
using comparison with a Gibbs prior distribution as a tool,
and in particular as a tool to establish the required
Kullback divergence bounds.

Relative bounds do not provide the same kind of
results as direct bounds on the error rate:
it is not possible to estimate
$\rho(R)$ with an order of precision higher than $(\rho(R) / N)^{1/2}$,
so that relative bounds cannot of course achieve that,
but they provide a way to reach a faster rate
for $\rho(R) - \inf_{\Theta} R$, that is for the relative
performance of the estimator within a restricted model.

The study of PAC-Bayesian relative bounds was initiated in
the second and third parts of J.-Y.~Audibert's dissertation \citep{Audibert2}.

In this section and the next,
we will suggest a series of possible uses of relative bounds.
As usual, we will start with the simplest inequalities
and proceed towards more sophisticated techniques with
better theoretical properties, but at the same time
less precise constants, so that which one is the more
fitted will depend on the size of the training sample.

The first thing we will do is to compute for any
posterior distribution $\rho: \Omega \rightarrow \C{M}_+^1(\Theta)$
a relative performance bound bearing on $\rho(R) - \inf_{\Theta} R$.
We will also compare the classification model indexed by
$\Theta$ with a sub-model indexed by one of its measurable subsets
$\Theta_1 \subset \Theta$. For this purpose we will form
the difference  $\rho(R) - R(\wt{\theta})$,
where $\T \in \Theta_1$ is some possibly unobservable
value of the parameter in the sub-model defined by $\Theta_1$,
typically chosen in $\arg\min_{\Theta_1} R$.
If this is so and $\rho(R) - R(\T)
= \rho(R) - \inf_{\Theta_1} R$, a negative upper bound
indicates that it is definitely
worth using a randomized estimator $\rho$ supported by
the larger parameter set $\Theta$ instead of using only
the classification model defined by the smaller set $\Theta_1$.

\subsection{Basic inequalities}
Relative bounds in this section are based on the control of
$r(\theta) - r(\T)$, where $\theta, \T \in \Theta$. These
differences are related to the random variables
$$
\psi_i(\theta, \T) = \sigma_i(\theta) - \sigma_i(\T)
= \B{1} \bigl[ f_{\theta}(X_i) \neq Y_i \bigr] -
\B{1} \bigl[ f_{\T}(X_i) \neq Y_i \bigr].
$$

Some supplementary technical difficulties, as compared to
the previous sections, come from the fact that
$\psi_i(\theta, \T)$ takes three values, whereas $\sigma_i(\theta)$
takes only two. Let
\begin{equation}
\label{eq1.19Bis}
\rr(\theta, \T) = r(\theta) - r(\T) = \frac{1}{N}
\sum_{i=1}^N \psi_i(\theta, \T), \quad \theta, \T \in \Theta,
\end{equation}
and $\R(\theta, \T) = R(\theta) - R(\T) = \PP \bigl[
r'(\theta, \T) \bigr] $. We have as usual from
independence that
\begin{multline*}
\log \Bigl\{ \PP \Bigl[ \exp \bigl[
- \lambda \rr(\theta, \T) \bigr] \Bigr] \Bigr\}
= \sum_{i=1}^N \log \Bigl\{ \PP \Bigl[
\exp \bigl[ - \tfrac{\lambda}{N} \psi_i(\theta, \T) \bigr] \Bigr] \Bigr\}
\\ \leq N \log \biggl\{ \frac{1}{N} \sum_{i=1}^N \PP
\Bigl\{ \exp \Bigl[ - \frac{\lambda}{N} \psi_i(\theta, \T) \Bigr] \Bigr\} \biggr\}.
\end{multline*}
Let $C_i$ be the distribution of $\psi_i(\theta, \T)$ under $\PP$ and let
$\Bar{C} = \frac{1}{N} \sum_{i=1}^N C_i \in
\C{M}_+^1\bigl( \{-1, 0,\break 1\} \bigr)$.
With this notation
\begin{equation}
\label{eq2.2.2Bis}
\log \Bigl\{ \PP \Bigl[ \exp \bigl[ - \lambda \rr( \theta, \T) \bigr]
\Bigr] \Bigr\} \leq N \log \biggl\{ \int_{\psi \in \{-1,0,1\}} \exp \Bigl( - \frac{\lambda}{N}
\psi \Bigr) \Bar{C}(d \psi) \biggr\}.
\end{equation}
\newcommand{\BM}{{M'}}
The right-hand side of this inequality is a function of $\Bar{C}$. On the
other hand, $\Bar{C}$ being a probability measure on a three point set, is
defined by two parameters, that we may take equal to $\int \psi \Bar{C}(d \psi)$ and
$\int \psi^2 \Bar{C}(d \psi)$. To this purpose, let us introduce
$$
\BM(\theta, \T) = \int \psi^2 \Bar{C}(d \psi) = \Bar{C}(+1)
+ \Bar{C}(-1) = \frac{1}{N} \sum_{i=1}^N \PP \bigl[
\psi_i^2(\theta, \T) \bigr], \quad \theta, \T \in \Theta.
$$
It is a pseudo distance
(meaning that it is symmetric and satisfies the triangle inequality),
since it can also be written as
$$
\BM(\theta, \T) = \frac{1}{N} \sum_{i=1}^N
\PP \Bigl\{ \Bigl\lvert \B{1} \bigl[ f_{\theta}(X_i) \neq Y_i \bigr]
- \B{1} \bigl[ f_{\T}(X_i) \neq Y_i \bigr] \Bigr\rvert \Bigr\},
\quad \theta, \T \in \Theta.
$$
It is readily seen that
$$
N \log \left\{ \int \exp \left( - \frac{\lambda}{N} \psi \right) \Bar{C}(d \psi)
\right\} = - \lambda \Psi_{\frac{\lambda}{N}} \bigl[ R'(\theta, \T), M'(\theta, \T) \bigr],
$$
where
\begin{align}
\nonumber \Psi_a(p,m) & = - a^{-1}
\log \Bigl[ (1 - m) + \frac{m+p}{2} \exp(-a)
+ \frac{m-p}{2} \exp (a) \Bigr]
\\
\label{eq1.19}
& = - a^{-1} \log \Bigl\{
1 - \sinh(a) \bigl[ p - m \tanh(\tfrac{a}{2}) \bigr] \Bigr\}.
\end{align}
Thus plugging this equality into inequality \myeq{eq2.2.2Bis} we get
\begin{thm}\mypoint
For any real parameter $\lambda$,
$$
\log \Bigl\{ \PP \Bigl[ \exp \bigl[ - \lambda \rr( \theta, \T) \bigr]
\Bigr] \Bigr\} \leq - \lambda \Psi_{\frac{\lambda}{N}}
\bigl[ \R(\theta, \T), \BM(\theta, \T) \bigr], \quad \theta, \T \in \Theta,
$$
where $r'$ is defined by equation \myeq{eq1.19Bis} and $\Psi$
and $M'$ are defined just above.
\end{thm}

To make a link with previous work of  Mammen and Tsybakov
--- see e.g. \citet{Mammen} and \citet{Tsybakov} --- we may consider the
pseudo-distance $D$ on $\Theta$ defined by equation
\myeq{eq1.1.2}.
This distance only depends on the distribution of the patterns. It
is often used to formulate margin assumptions, in the sense of Mammen
and Tsybakov.
Here we are going to work rather with
$\BM$: as it is dominated by $D$ in the sense that
$\BM(\theta, \T) \leq D(\theta, \T)$, $\theta, \T \in \Theta$, with equality
in the important case of binary classification, hypotheses formulated on
$D$ induce hypotheses on $M'$, and working with $M'$ may only sharpen the
results when compared to working with $D$.

Using the same reasoning as in the previous section, we deduce
\begin{thm}
\label{thm4.1}
\mypoint For any real parameter $\lambda$, any $\T \in \Theta$,
any prior distribution $\pi \in \C{M}_+^1(\Theta)$,
$$
\PP \biggl\{ \exp \biggl[ \sup_{\rho \in \C{M}_+^1(\Theta)}
\lambda \Bigl[ \rho \bigl\{ \Psi_{\frac{\lambda}{N}} \bigl[
\R(\cdot, \T\,), \BM(\cdot, \T\,) \bigr]  \bigr\}
- \rho\bigl[\rr(\cdot, \T) \bigr] \Bigr]
- \C{K}(\rho, \pi) \biggr] \biggr\} \leq 1.
$$
\end{thm}

We are now going to derive some other type of relative
exponential inequality. In Theorem \ref{thm4.1}
we obtained an inequality comparing one observed quantity
$\rho\bigl[r'(\cdot, \T\,)\bigr]$ with two unobserved ones, $\rho\bigl[R'(
\cdot, \T\,)\bigr]$ and $\rho\bigl[M'(\cdot, \T\,) \bigr]$,
--- indeed,
because of the convexity of the function $\lambda \Psi_{\frac{\lambda}{N}}$,
$$
\lambda \rho
\bigl\{ \Psi_{\frac{\lambda}{N}}\bigl[R'(\cdot, \T\,),M'(\cdot, \T\,) \bigr]
\bigr\} \geq
\lambda \Psi_{\frac{\lambda}{N}} \bigl\{ \rho\bigl[R'(\cdot, \T\,)\bigr],
\rho\bigl[ M'(\cdot, \T\,) \bigr] \bigr\}.
$$
This may be inconvenient when looking for
an empirical bound for $\rho\bigl[ R'(\cdot, \T) \bigr]$, and we are going now to seek
an inequality comparing $\rho\bigl[R'(\cdot, \T\,)\bigr]$ with empirical quantities
only.

This is possible by considering the $\log$-Laplace
transform of some modified random variable $\chi_i(\theta, \T)$.
We may consider more precisely the change of variable
defined by the equation
$$
\exp \left( - \frac{\lambda}{N} \chi_i \right)
= 1 - \frac{\lambda}{N} \psi_i,
$$
which is possible when $\frac{\lambda}{N} \in \; )\!-\!\!1, 1($
and leads to define
$$
\chi_i = - \frac{N}{\lambda} \log \left( 1 - \frac{\lambda}{N}\psi_i \right).
$$
We may then work on the $\log$-Laplace transform
\begin{multline*}
\log \Biggl\{ \PP \Biggl[ \exp \biggl\{ -  \frac{\lambda}{N} \sum_{i=1}^N
\chi_i(\theta, \T) \biggr\}  \Biggr] \Biggr\} =
\log \Biggl\{ \PP \Biggl[ \prod_{i=1}^N \biggl( 1 - \frac{\lambda}{N}
\psi_i(\theta, \T) \biggr) \Biggr] \Biggr\}
\\ = \log \Biggl\{ \PP \Biggl[ \exp \biggl\{ \sum_{i=1}^N
\log \biggl[ 1 - \frac{\lambda}{N} \psi_i(\theta, \T) \biggr] \biggr\}
\Biggr] \Biggr\}.
\end{multline*}
We may now follow the same route as previously, writing
\begin{multline*}
\log \Biggl\{ \PP \Biggl[ \exp \biggl\{ \sum_{i=1}^N \log \biggl[
1 - \frac{\lambda}{N} \psi_i(\theta, \T)
\biggr] \biggr\} \Biggr] \Biggr\}
\\= \sum_{i=1}^N \log \biggl[ 1 - \frac{\lambda}{N} \PP \bigl[ \psi_i
(\theta, \T) \bigr]  \biggr]
\leq N  \log \Bigl[ 1 - \frac{\lambda}{N} R'(\theta,\T\,) \Bigr].
\end{multline*}
Let us also introduce the random pseudo distance
\begin{multline}
\label{eq1.3}
m'(\theta, \T) = \frac{1}{N} \sum_{i=1}^N \psi_i(\theta,\T)^2
\\ = \frac{1}{N} \sum_{i=1}^N \Bigl\lvert \B{1} \bigl[
f_{\theta}(X_i) \neq Y_i \bigr] - \B{1} \bigl[ f_{\T}(
X_i) \neq Y_i \bigr] \Bigr\rvert, \quad \theta, \T \in \Theta.
\end{multline}
This is the empirical counterpart of $M'$, implying that $\PP(m') = M'$.
Let us notice that
\begin{multline*}
\frac{1}{N} \sum_{i=1}^N \log \bigl[ 1 - \tfrac{\lambda}{N} \psi_i(\theta, \T) \bigr]
= \frac{\log(1 - \tfrac{\lambda}{N}) - \log(1 + \tfrac{\lambda}{N})}{2} r'(\theta, \T)
\\ \shoveright{+ \frac{\log(1 - \tfrac{\lambda}{N}) +
\log(1 + \tfrac{\lambda}{N})}{2} m'(\theta,\T)
\qquad} \\
\\ = \frac{1}{2} \log \left( \frac{1 - \tfrac{\lambda}{N}}{1 + \tfrac{\lambda}{N}} \right)
r'\bigl(\theta, \T\,\bigr) + \frac{1}{2} \log
\bigl( 1 - \tfrac{\lambda^2}{N^2} \bigr)
m'\bigl(\theta, \T\,\bigr).
\end{multline*}
Let us put
$ \gamma = \frac{N}{2} \log \biggl( \frac{1 + \frac{\lambda}{N}}{1 -
\frac{\lambda}{N}} \biggr),
$
so that
$$
\lambda  = N \tanh \bigl( \tfrac{\gamma}{N} \bigr) \text{ and }
\tfrac{N}{2} \log \Bigl( 1 - \tfrac{\lambda^2}{N^2} \Bigr) =
- N \log \bigl[ \cosh(\tfrac{\gamma}{N}) \bigr].
$$
With this notation, we can
conveniently write the previous inequality as
\begin{multline*}
\PP \Bigl\{ \exp \Bigl[ -N \log \bigl[ 1 - \tanh \bigl(
\tfrac{\gamma}{N} \bigr) R'(\theta, \T) \bigr]
\\ - \gamma r'\bigl(\theta,
\T\,\bigr) - N \log \bigl[ \cosh( \tfrac{\gamma}{N})
\bigr]  m'\bigl(\theta, \T\, \bigr) \Bigr] \Bigr\}
\leq 1.
\end{multline*}
Integrating with respect to a prior probability measure $\pi \in \C{M}_+^1(\Theta)$,
we obtain
\begin{thm}
\label{thm2.2.18}
\mypoint For any real parameter $\gamma$, for any $\T \in \Theta$,
for any prior probability distribution $\pi \in \C{M}_+^1(\Theta)$,
\begin{multline*}
\PP \Biggl\{ \exp \Biggl[ \sup_{\rho \in \C{M}_+^1(\Theta)} \biggl\{
-N \rho \Bigl\{ \log \bigl[ 1 - \tanh\bigl(
\tfrac{\gamma}{N}\bigr) R'(\cdot, \T\,) \bigr] \Bigr\}
\\ - \gamma
\rho \bigl[r'(\cdot, \T\,)\bigr] - N \log \bigl[ \cosh(\tfrac{
\gamma}{N} ) \bigr]
\rho\bigl[m'(\cdot, \T\,) \bigr]
- \C{K}(\rho, \pi) \biggr\} \Biggr] \Biggr\} \leq 1.
\end{multline*}
\end{thm}

\subsection{Non random bounds}
Let us first deduce a non-random bound from Theorem \thmref{thm4.1}.
This theorem can be conveniently taken advantage of by
throwing the non-linearity into a localized prior, considering
the prior probability measure $\mu$ defined by its density
$$
\frac{d \mu}{d \pi}(\theta) = \frac{\exp \bigl\{ - \lambda \Psi_{\frac{\lambda}{N}}
\bigl[ R'(\theta, \T\,), \BM(\theta, \T\,) \bigr] + \beta \R(\theta, \T\,) \bigr\}}
{\pi \Bigl\{ \exp \bigl\{ - \lambda \Psi_{\frac{\lambda}{N}}
\bigl[ R'(\cdot, \T\,), \BM(\cdot, \T\,) \bigr] + \beta \R(\cdot, \T\,) \bigr\}
\Bigr\}}.
$$
Indeed, for any posterior distribution $\rho: \Omega \rightarrow \C{M}_+^1(\Theta)$,
\begin{multline*}
\C{K}(\rho,\mu) = \C{K}(\rho,\pi) + \lambda \rho \Bigl\{
\Psi_{\frac{\lambda}{N}} \bigl[ R'(\cdot, \T\,),M'(\cdot, \T\,) \bigr]
\Bigr\} - \beta \rho \bigl[ R'(\cdot, \T\,) \bigr] \\ +
\log \Bigl\{ \pi \Bigl[ \exp \bigl\{
- \lambda \Psi_{\frac{\lambda}{N}}\bigl[ R'(\cdot, \T\,),
M'(\cdot, \T\,) \bigr] + \beta R'(\cdot, \T\,) \bigr] \bigr\} \Bigr] \Bigr\}.
\end{multline*}
Plugging this into Theorem \thmref{thm4.1} and using the convexity of the
exponential function, we see that for any posterior probability distribution
$\rho: \Omega \rightarrow \C{M}_+^1(\Theta)$,
\begin{multline*}
\beta \PP \bigl\{ \rho \bigl[ R'(\cdot, \T\,) \bigr] \bigr\}
\leq \lambda \PP \bigl\{ \rho \bigl[ r'(\cdot, \T\,) \bigr] \bigr\}
+ \PP \bigl[ \C{K}(\rho, \pi) \bigr] \\ +
\log \Bigl\{ \pi \Bigl[ \exp \bigl\{
- \lambda \Psi_{\frac{\lambda}{N}}\bigl[ R'(\cdot, \T\,),
M'(\cdot, \T\,) \bigr] + \beta R'(\cdot, \T\,) \bigr] \bigr\} \Bigr] \Bigr\}.
\end{multline*}
We can then recall that
$$
\lambda \rho\bigl[ r'(\cdot, \T\,) \bigr] + \C{K}(\rho, \pi)
= \C{K}\bigl[ \rho, \pi_{\exp( - \lambda r)}\bigr] - \log
\Bigl\{ \pi \Bigl[ \exp \bigl[ - \lambda r'(\cdot, \T\,) \bigr] \Bigr] \Bigr\},
$$
and notice moreover that
$$
- \PP \biggl\{ \log \Bigl\{ \pi \Bigl[
\exp \bigl[ - \lambda r'(\cdot, \T\,) \bigr] \Bigr] \Bigr\} \biggr\}
\leq
- \log \Bigl\{ \pi \Bigl[
\exp \bigl[ - \lambda R'(\cdot, \T\,) \bigr] \Bigr] \Bigr\},
$$
since $R' = \PP(r')$ and $h \mapsto \log \Bigl\{ \pi \bigl[ \exp ( h) \bigr] \Bigr\}$
is a convex functional. Putting these two remarks together, we obtain
\begin{thm}
\mypoint \label{thm2.2.19}
For any real positive parameter $\lambda$, for any prior distribution $\pi
\in \C{M}_+^1(\Theta)$, for any posterior distribution $\rho: \Omega
\rightarrow \C{M}_+^1(\Theta)$,
\begin{multline*}
\PP \bigl\{ \rho \bigl[ R'(\cdot, \T\,) \bigr] \bigr\}
\leq \frac{1}{\beta} \PP \bigl[ \C{K}(\rho, \pi_{\exp( - \lambda r)}) \bigr]
\\ + \frac{1}{\beta} \log \Bigl\{ \pi \Bigl[ \exp \bigl\{
- \lambda \Psi_{\frac{\lambda}{N}}\bigl[ R'(\cdot, \T\,),
M'(\cdot, \T\,) \bigr] + \beta R'(\cdot, \T\,) \bigr] \bigr\} \Bigr] \Bigr\}\\
\shoveright{- \frac{1}{\beta} \log \Bigl\{ \pi \Bigl[
\exp \bigl[ - \lambda R'(\cdot, \T\,) \bigr] \Bigr] \Bigr\}\quad}\\\shoveleft{\qquad
\leq \frac{1}{\beta} \PP \bigl[ \C{K}(\rho, \pi_{\exp( - \lambda r)})\bigr]}
\\ + \frac{1}{\beta} \log \Bigl\{ \pi \Bigl[
\exp \bigl\{ - \bigl[ N \sinh(\tfrac{\lambda}{N}) - \beta \bigl] R'(\cdot, \T\,)
\\ \shoveright{+ 2 N \sinh(\tfrac{\lambda}{2N})^2 M'(\cdot, \T\,) \bigr\} \Bigr] \Bigr\}
\qquad} \\ - \frac{1}{\beta} \log \Bigl\{ \pi \Bigl[
\exp \bigl[ - \lambda R'(\cdot, \T\,) \bigr] \Bigr] \Bigr\}.
\end{multline*}
\end{thm}

It may be interesting to derive some more suggestive (but slightly weaker)
bound in the important case when $\Theta_1 = \Theta$ and $R(\T) = \inf_{\Theta} R$.
In this case, it is convenient to introduce the \emph{expected margin function}
\begin{equation}
\label{eq1.1.16Bis}
\varphi(x) = \sup_{\theta \in \Theta} \BM(\theta, \T) -
x \R(\theta, \T), \quad x \in \RR_+.
\end{equation}
We see that $\varphi$ is convex and non-negative on $\RR_+$.
Using the bound $M'(\theta, \T\,) \leq x R'(\theta, \T\,) + \varphi(x)$,
we obtain
\begin{multline*}
\PP \bigl\{ \rho \bigl[ R'(\cdot, \T\,) \bigr] \bigr\}
\leq \frac{1}{\beta} \PP \bigl[ \C{K}(\rho, \pi_{\exp( - \lambda r)})\bigr]
\\ + \frac{1}{\beta} \log \biggl\{ \pi \biggl[
\exp \Bigl\{ -
\bigl\{ N \sinh(\tfrac{\lambda}{N})\bigl[
1 - x\tanh(\tfrac{\lambda}{2N})\bigr] - \beta \bigr\}
R'(\cdot, \T\,) \Bigr\}
\biggr] \biggr\}
\\ + \frac{N \sinh(\tfrac{\lambda}{N}) \tanh(\tfrac{\lambda}{2N})}{\beta} \varphi(x)
- \frac{1}{\beta} \log \Bigl\{ \pi \Bigl[
\exp \bigl[ - \lambda R'(\cdot, \T\,) \bigr] \Bigr] \Bigr\}.
\end{multline*}
Let us make the change of variable $\gamma =
N \sinh(\tfrac{\lambda}{N})\bigl[
1 - x\tanh(\tfrac{\lambda}{2N})\bigr] - \beta$ to obtain
\begin{cor}
\label{cor1.1.21}\mypoint
For any real positive parameters $x$, $\gamma$ and $\lambda$ such that
$x \leq \tanh(\frac{\lambda}{2N})^{-1}$ and $0 \leq \gamma <
N \sinh(\frac{\lambda}{N}) \bigl[ 1 - x \tanh(\frac{\lambda}{2N}) \bigr]$,
\begin{multline*}
\PP \bigl[ \rho(R) \bigr] - \inf_{\Theta} R
\leq \Bigl\{
N \sinh(\tfrac{\lambda}{N}) \bigl[ 1 - x
\tanh(\tfrac{\lambda}{2N})\bigr] - \gamma \Bigr\}^{-1} \\
\shoveleft{\qquad \times
\biggl\{ \int_{\gamma}^{\lambda}
\bigl[ \pi_{\exp( - \alpha R)}(R) - \inf_{\Theta} R\bigr]
d \alpha }\\ + N \sinh\bigl(\tfrac{\lambda}{N}\bigr) \tanh\bigl(\tfrac{\lambda}{2N}\bigr)
\varphi(x) + \PP \bigl[ \C{K}(\rho, \pi_{\exp( - \lambda r)}) \bigr]
\biggr\}.
\end{multline*}
\end{cor}

Let us remark that these results, although well suited to study Mammen and Tsybakov's
margin assumptions, hold in the general case: introducing the convex \emph{expected
margin function} $\varphi$ is a substitute for making hypotheses about the relations
between $R$ and $D$.

Using the fact that $R'(\theta, \T\,) \geq 0$, $\theta \in \Theta$  and
that $\varphi(x) \geq 0$, $x \in \RR_+$, we can weaken and simplify
the preceding corollary even more to get
\begin{cor}
\label{cor4.3}
\mypoint For any real parameters $\beta$, $\lambda$ and $x$ such that
$x \geq 0$ and $0 \leq \beta < \lambda - x \frac{\lambda^2}{2N}$,
for any posterior distribution $\rho: \Omega \rightarrow \C{M}_+^1(\Theta)$,
\begin{multline*}
\PP \bigl[ \rho(R) \bigr] \leq \inf_{\Theta} R
\\ +
\Bigl[\lambda - x \tfrac{\lambda^2}{2N} - \beta \Bigr]^{-1}
\biggl\{ \int_{\beta}^{\lambda}
\bigl[ \pi_{\exp( - \alpha R)}(R) - \inf_{\Theta} R \bigr]  d \alpha
\\ + \PP \bigl\{ \C{K}\bigl[\rho, \pi_{\exp( - \lambda r)} \bigr] \bigr\}
+ \varphi(x) \frac{\lambda^2}{2N} \biggr\}.
\end{multline*}
\end{cor}

Let us apply this bound under the \emph{margin assumption}
first considered by Mammen and Tsybakov \citep{Mammen,Tsybakov},
which says that for some real positive constant $c$ and some
real exponent $\kappa \geq 1$,
\begin{equation}
\label{eq1.1.17Bis}
\R(\theta, \T) \geq
c D(\theta, \T)^{\kappa}, \qquad \theta \in \Theta.
\end{equation}
In the
case when $\kappa = 1$, then $\varphi(c^{-1}) = 0$, proving that
\begin{align*}
\PP \bigl\{ \pi_{\exp( - \lambda r)}\bigl[  \R(\cdot, \T\,) \bigr] \bigr\}
& \leq \frac{\int_{\beta}^{\lambda} \pi_{\exp(
- \gamma R)}\bigl[ \R(\cdot, \T\,)\bigr]
d \gamma}{N \sinh(\frac{\lambda}{N})
\bigl[ 1 - c^{-1} \tanh(\frac{\lambda}{2N}) \bigr] - \beta}
\\ & \leq \frac{ \int_{\beta}^{\lambda} \pi_{\exp( - \gamma R)}\bigl[
\R(\cdot, \T\,)\bigr]
d \gamma}{
\lambda - \frac{ \lambda^2}{2 c N} - \beta}.
\end{align*}
Taking for example  $\lambda = \frac{cN}{2}$, $\beta = \frac{\lambda}{2}
= \frac{cN}{4}$,
we obtain
\begin{align*}
\PP \bigl[ \pi_{\exp( - 2^{-1} c N r)}(R) \bigr] & \leq \inf R +
\frac{8}{cN} \int_{\frac{c N}{4}}^{\frac{cN}{2}}
\pi_{\exp( - \gamma R)}\bigl[\R(\cdot, \T)\bigr]
d \gamma \\* & \leq \inf R + 2 \pi_{\exp(- \frac{cN}{4} R)}\bigl[ \R(\cdot, \T\,)\bigr].
\end{align*}
If moreover the behaviour of the prior distribution $\pi$ is parametric,
meaning that $\pi_{\exp( - \beta R)}\bigl[ \R(\cdot, \T\,) \bigr]
\leq \frac{d}{\beta}$,
for some positive real constant $d$ linked with the dimension of the
classification model, then
$$
\PP \bigl[ \pi_{\exp( - \frac{c N}{2} r)}(R) \bigr]
\leq \inf R + \frac{8 \log(2) d}{cN}
\leq \inf R + \frac{5.55 \, d}{cN}.
$$
In the case when $\kappa > 1$,
$$\varphi(x) \leq (\kappa -1) \kappa^{- \frac{\kappa}{
\kappa -1}} (c x)^{- \frac{1}{\kappa - 1}} = (1 - \kappa^{-1})(\kappa c x)^{-\frac{1}{
\kappa - 1}},$$
\begin{multline*}
\hspace{-10pt}\text{thus }\PP \bigl\{ \pi_{\exp(- \lambda r)}\bigl[ \R(\cdot, \T\,)\bigr] \bigr\}
\\ \leq \frac{\int_{\beta}^{\lambda} \pi_{\exp( - \gamma R)}\bigl[ \R(\cdot, \T\,)\bigr] d \gamma
+ (1 - \kappa^{-1}) (\kappa c x)^{-\frac{1}{\kappa - 1}}
\frac{\lambda^2}{2N} }{
\lambda - \frac{x\lambda^2}{2N}  - \beta}.
\end{multline*}
Taking for instance $\beta = \frac{\lambda}{2}$, $x = \frac{N}{2 \lambda}$,
and putting $b = (1 - \kappa^{-1}) (c \kappa)^{- \frac{1}{\kappa -1}}$,
we obtain
$$
\PP \bigl[ \pi_{\exp( - \lambda r)}(R) \bigr] - \inf R
\leq \frac{4}{\lambda} \int_{\lambda/2}^{\lambda}
\pi_{\exp( - \gamma R)}\bigl[ \R(\cdot, \T\,)\bigr] d \gamma + b \left(\frac{2 \lambda}{N}\right)^{\frac{
\kappa}{\kappa -1}}.
$$
In the \emph{parametric} case when $\pi_{\exp( - \gamma R)}\bigl[ \R(\cdot, \T\,)\bigr]
\leq \frac{d}{\gamma}$,
we get
$$
\PP \bigl[ \pi_{\exp( - \lambda r)}(R) \bigr] - \inf R
\leq \frac{4 \log(2) d}{\lambda} + b \left( \frac{2 \lambda}{N} \right)^{\frac{
\kappa}{\kappa - 1}}.
$$
Taking
\newcommand{\Blambda}{\overline{\lambda}}
$$
\Blambda = 2^{-1} \bigl[ 8 \log(2) d \bigr]^{\frac{\kappa-1}{2 \kappa -1}}
(\kappa c)^{\frac{1}{2 \kappa -1}}
N^{\frac{\kappa}{2 \kappa -1 }},
$$
we obtain
$$
\PP \bigl[ \pi_{\exp( - \Blambda r)}(R) \bigr] - \inf R
\leq (2 - \kappa^{-1}) (\kappa c)^{-\frac{1}{2 \kappa - 1}}
\left( \frac{ 8 \log(2) d}{N} \right)^{\frac{\kappa}{2 \kappa - 1}}.
$$
We see that this formula coincides with the result for $\kappa = 1$.
We can thus reduce the two cases to a single one and state
\begin{cor}
\mypoint
\label{cor1.1.23} Let us assume that for some $\T \in \Theta$, some
positive real constant $c$, some real exponent $\kappa \geq 1$
and for any $\theta \in \Theta$,
$R(\theta)\geq R(\T) + c D(\theta, \T)^{\kappa}$.
Let us also assume that for some positive real
constant $d$ and any positive real parameter $\gamma$,
$\pi_{\exp( - \gamma R)}(R) - \inf R \leq \frac{d}{\gamma}$.
Then
\begin{multline*}
\PP \Bigl[ \pi_{\exp \bigl\{ -
2^{-1}[ 8 \log(2) d ]^{\frac{\kappa-1}{2 \kappa -1}}
(\kappa c)^{\frac{1}{2 \kappa -1}}
N^{\frac{\kappa}{2 \kappa -1 }}
r\bigr\}}(R) \Bigr]
\\ \leq \inf R + (2 - \kappa^{-1}) (\kappa c)^{-\frac{1}{2 \kappa - 1}}
\left( \frac{ 8 \log(2) d}{N} \right)^{\frac{\kappa}{2 \kappa - 1}}.
\end{multline*}
\end{cor}

Let us remark that the exponent of $N$ in this corollary is
known to be the minimax exponent under these assumptions:
it is unimprovable, whatever estimator is used in place of
the Gibbs posterior shown here (at least in the worst case
compatible with the hypotheses). The interest of the corollary
is to show not only the minimax exponent in $N$, but also
an explicit non-asymptotic bound with reasonable and simple
constants. It is also clear that we could have got slightly
better constants if we had kept the full strength of Theorem
\ref{thm2.2.19} (page \pageref{thm2.2.19})
instead of using the weaker Corollary \ref{cor4.3}
(page \pageref{cor4.3}).

We will prove in the following empirical bounds showing
how the constant $\lambda$ can be estimated from the data
instead of being chosen according to some margin and
complexity assumptions.

\subsection{Unbiased empirical bounds}
We are going to define an empirical counterpart for the
\emph{expected margin function} $\varphi$. It will appear
in empirical bounds having otherwise the same structure as
the non-random bound we just proved. Anyhow, we will not
launch into trying to compare the behaviour of our proposed
\emph{empirical margin function} with the \emph{expected margin function},
since the margin function involves taking a supremum
which is not straightforward to handle. When we will
touch the issue of building \emph{provably} adaptive estimators, we will
instead formulate another type of bounds based on integrated
quantities, rather than try to analyse the properties of
the empirical margin function.

Let us start as in the previous subsection with the inequality
\begin{multline*}
\beta \PP \Bigl\{ \rho\bigl[\R(\cdot,\T\,) \bigr] \Bigr\} \leq
\PP \Bigl\{ \lambda \rho\bigl[ r'(\cdot, \T\,) \bigr]+ \C{K}(\rho, \pi) \Bigr\}
\\ + \log \Bigl\{ \pi \Bigl[ \exp \bigl\{ - \lambda \Psi_{\frac{\lambda}{N}}\bigl[\R
(\cdot, \T\,), \BM(\cdot, \T\,) \bigr] + \beta \R(\cdot, \T\,) \, \bigr\} \Bigr]
\Bigr\} .
\end{multline*}
We have already defined by equation \myeq{eq1.3} the empirical pseudo-distance
\newcommand{\m}{{m'}}
$$
\m( \theta, \T\,) = \frac{1}{N} \sum_{i=1}^N \psi_i(\theta, \T\,)^2.
$$
Recalling that $\PP \bigl[ \m(\theta, \T\,) \bigr] = \BM(\theta, \T\,)$,
and using the convexity of $h \mapsto \break \log \Bigl\{ \pi \bigl[ \exp( h ) \bigr] \Bigr\}$,
leads to the following inequalities:
\begin{multline*}
\log \Bigl\{ \pi \Bigl[ \exp \bigl\{ - \lambda \Psi_{\frac{\lambda}{N}}\bigl[
\R(\cdot, \T\,), \BM(\cdot, \T\,)\bigr] + \beta \R(\cdot, \T\,) \bigr\} \Bigr] \Bigr\}
\\*\shoveleft{\qquad \leq \log \Bigl\{ \pi \Bigl[ \exp \bigl\{
- N \sinh(\tfrac{\lambda}{N}) \R(\cdot, \T\,)  }
\\ \shoveright{+  N \sinh(\tfrac{\lambda}{N})\tanh(\tfrac{\lambda}{2N}) \BM(\cdot, \T\,)
+ \beta \R(\cdot,\T\,) \bigr] \bigr\} \Bigr] \Bigr\} \qquad}
\\* \leq \PP \biggl\{
\log \Bigl\{ \pi \Bigl[
\exp \bigl\{ - \bigl[N \sinh(\tfrac{\lambda}{N})
- \beta \bigr] \rr(\cdot, \T\,)
\\ + N \sinh(\tfrac{\lambda}{N}) \tanh(\tfrac{\lambda}{2N})
\m(\cdot, \T\,) \bigr\} \Bigr] \Bigr\} \biggr\}.
\end{multline*}
We may moreover remark that
\begin{multline*}
\lambda \rho\bigl[ \rr(\cdot, \T\,) \bigr]
+ \C{K}(\rho, \pi)
= \bigl[ \beta - N \sinh(\tfrac{\lambda}{N}) + \lambda \bigr]
\rho \bigl[ \rr(\cdot, \T\,)\bigr] \\ + \C{K}\bigl[ \rho, \pi_{\exp \{-[ N \sinh(\frac{\lambda}{N}) - \beta
] r \}} \bigr] \\ - \log \Bigl\{ \pi \Bigl[ \exp \bigl\{
- \bigl[ N \sinh(\tfrac{\lambda}{N}) - \beta \bigr] \rr(\cdot, \T\,) \bigr\} \Bigr]
\Bigr\}.
\end{multline*}
This establishes
\begin{thm}
\mypoint For any positive real parameters $\beta$ and $\lambda$,
for any posterior distribution $\rho: \Omega \rightarrow \C{M}_+^1(\Theta)$,
\begin{multline*}
\PP \bigl\{ \rho\bigl[ \R(\cdot, \T\,) \bigr] \bigr\}
\leq \PP \biggl\{
\biggl[ 1 - \frac{ N \sinh(\frac{\lambda}{N}) - \lambda}{\beta} \biggr]
\rho\bigl[ \rr(\cdot, \T\,)\bigr]
\\\shoveright{ + \frac{\C{K}\bigl[\rho, \pi_{\exp \{ - [ N \sinh(\frac{\lambda}{N})
- \beta ] r \}} \bigr]}{\beta} \qquad}
\\ + \beta^{-1}
\log \Bigl\{
\pi_{\exp \{ - [N \sinh(\frac{\lambda}{N}) - \beta ] r \}} \Bigl[
\exp \bigl[ N \sinh(\tfrac{\lambda}{N}) \tanh(\tfrac{\lambda}{2N})\m(\cdot, \T\,)
\bigr] \Bigr] \Bigr\} \biggr\}.
\end{multline*}
\end{thm}

Taking $\beta = \frac{N}{2} \sinh (\frac{\lambda}{N})$, using the
fact that $\sinh(a) \geq a$, $a \geq 0$ and expressing
$\tanh(\frac{a}{2}) = a^{-1} \bigl[ \sqrt{1 + \sinh(a)^2}- 1 \bigr]$
and $a = \log \bigl[ \sqrt{1 + \sinh(a)^2} + \sinh(a) \bigr]$,
we deduce
\begin{cor}
\mypoint For any positive real constant $\beta$ and any posterior distribution
$\rho: \Omega \rightarrow \C{M}_+^1(\Theta)$,
\begin{multline*}
\PP \bigl\{ \rho\bigl[ \R(\cdot, \T\,) \bigr] \bigr\} \leq
\PP \Biggl\{ \underbrace{\biggl[ \tfrac{N}{\beta}\log \Bigl(
\sqrt{1 + \tfrac{4 \beta^2}{N^2}} + \tfrac{2 \beta}{N} \Bigr) - 1  \biggr]}_{\leq 1}
\rho\bigl[ \rr(\cdot, \T\,) \bigr]  \\
\shoveleft{\qquad
+ \frac{1}{\beta} \biggl\{ \C{K}\bigl[ \rho,\pi_{\exp( - \beta r)} \bigr]}
\\ + \log \biggl[ \pi_{\exp( - \beta r)} \Bigl\{ \exp \Bigl[ N\Bigl(
\sqrt{1 + \tfrac{4 \beta^2}{N^2}}
- 1 \Bigr) \m(\cdot, \T\,) \Bigr] \Bigr\} \biggr] \biggr\} \Biggr\}.
\end{multline*}
\end{cor}

This theorem and its corollary are really analogous to
Theorem \ref{thm2.2.19} (page \pageref{thm2.2.19}), and
it could easily be proved that under Mammen and Tsybakov margin assumptions
we obtain an upper bound of the same order as Corollary \ref{cor1.1.23}
(page \pageref{cor1.1.23}).
Anyhow, in order to obtain an empirical bound, we are now going to take
a supremum over all possible values of $\T$, that is over $\Theta_1$.
Although we believe that taking this supremum will not spoil the bound
in cases when over-fitting remains under control, we will not try
to investigate precisely if and when this is actually true, and
provide our empirical bound as such. Let us say only that on qualitative
grounds, the values of the margin function quantify the steepness of the
contrast function $R$ or its empirical counterpart $r$, and
that the definition
of the empirical margin function is obtained by substituting $\PP$, the true
sample distribution, with $\overline{\PP} = \bigl( \frac{1}{N} \sum_{i=1}^N
\delta_{(X_i, Y_i)}\bigr)^{\otimes N}$, the empirical sample distribution,
in the definition of the expected margin function. Therefore, on qualitative
grounds, it seems hopeless to presume that $R$ is steep when $r$ is
not, or in other words that a classification model that would be inefficient
at estimating a bootstrapped sample according to our non-random bound
would be by some miracle efficient at estimating the true sample distribution
according to the same bound. To this extent, we feel that our empirical
bounds bring a satisfactory counterpart of our non-random bounds.
Anyhow, we will also produce estimators which can be proved
to be adaptive
using PAC-Bayesian tools in the next section, at the price of
a more sophisticated construction involving comparisons between
a posterior distribution and a Gibbs prior distribution or
between two posterior distributions.

\newcommand{\Btheta}{\widehat{\theta}}
Let us now restrict discussion to the important case when $\T \in \arg\min_{\Theta_1} R$.
To obtain an observable bound, let $\Btheta \in \arg\min_{\theta
\in \Theta} r(\theta)$ and let us introduce the \emph{empirical margin
functions}
\newcommand{\Tphi}{\widetilde{\varphi}}
\newcommand{\Bphi}{\overline{\varphi}}
\begin{align*}
\Bphi(x) & = \sup_{\theta \in \Theta} \m(\theta, \Btheta) - x \bigl[
r(\theta) - r(\Btheta) \bigr], \quad x \in \RR_+,\\
\Tphi(x) & = \sup_{\theta \in \Theta_1} \m(\theta, \Btheta) - x \bigl[
r(\theta) - r(\Btheta) \bigr], \quad x \in \RR_+.
\end{align*}
Using the fact that $\m(\theta, \T) \leq \m(\theta, \Btheta)
+ \m(\Btheta, \T)$, we get
\begin{cor}
\mypoint For any positive real parameters $\beta$ and $\lambda$,
for any posterior distribution $\rho: \Omega
\rightarrow \C{M}_+^1(\Theta)$,
\begin{multline*}
\PP \bigl[ \rho (R) \bigr] - \inf_{\Theta_1} R
\leq \PP \biggl\{
\Bigl[ 1 - \tfrac{ N \sinh(\frac{\lambda}{N}) - \lambda}{\beta}
\Bigr] \bigl[ \rho(r) - r(\Btheta)\bigr] \\
+ \frac{ \C{K}\bigl[ \rho, \pi_{\exp\{-[N \sinh(\frac{\lambda}{N})
- \beta]r\}} \bigr]}{\beta}\\
+ \beta^{-1} \log \Bigl\{ \pi_{\exp \{-[N \sinh(\frac{\lambda}{N})
- \beta]r\}} \Bigl[ \exp \bigl[
N \sinh\bigl(\tfrac{\lambda}{N}\bigr) \tanh\bigl(\tfrac{\lambda}{2N}\bigr) \m(\cdot,\Btheta)
\bigr] \Bigr] \Bigr\} \\ +
\beta^{-1}N \sinh(\tfrac{\lambda}{N}) \tanh(\tfrac{\lambda}{2N})
\Tphi \biggl[ \frac{\beta}{N\sinh(\frac{\lambda}{N}) \tanh(\frac{\lambda}{
2N})} \left(1 - \frac{N\sinh(\frac{\lambda}{N}) - \lambda}{\beta}
\right)\biggr] \biggr\}.
\end{multline*}
Taking $\beta = \frac{N}{2} \sinh(\frac{\lambda}{N})$, we also
obtain
\begin{multline*}
\PP \bigl[ \rho(R) \bigr] - \inf_{\Theta_1} R \leq
\PP \Biggl\{ \underbrace{\biggl[ \tfrac{N}{\beta}\log \Bigl(
\sqrt{1 + \tfrac{4 \beta^2}{N^2}}
+ \tfrac{2 \beta}{N} \Bigr) - 1  \biggr]}_{\leq 1}
\bigl[ \rho(r) - r(\Btheta) \bigr] \\
\shoveleft{\qquad + \frac{1}{\beta} \biggl\{ \C{K}\bigl[
\rho,\pi_{\exp( - \beta r)} \bigr]}
\\\qquad + \log \biggl[ \pi_{\exp( - \beta r)} \Bigl\{ \exp \Bigl[ N\Bigl(
\sqrt{1 + \tfrac{4 \beta^2}{N^2}}
- 1 \Bigr) \m(\cdot, \Btheta) \Bigr] \Bigr\} \biggr] \biggr\} \\
+ \frac{N}{\beta}\Bigl(\sqrt{1 + \tfrac{4 \beta^2}{N^2}} - 1\Bigr)
\Tphi \Biggl[ \frac{\log \Bigl( \sqrt{1 + \frac{4 \beta^2}{N^2}}
+ \frac{2 \beta}{N} \Bigr) - \frac{\beta}{N}}{\Bigl(
\sqrt{1 + \frac{4 \beta^2}{N^2}} - 1 \Bigr)}\Biggr]
\Biggr\}.
\end{multline*}
\end{cor}

Note that we could also use the upper bound
$\m(\theta, \Btheta) \leq x \bigl[ r(\theta) - r(\Btheta)
\bigr] + \Bphi(x)$ and put $\alpha =
N \sinh(\frac{\lambda}{N}) \bigl[ 1 -
x \tanh(\frac{\lambda}{2N}) \bigr] - \beta$, to obtain
\begin{cor}
\label{cor1.1.27}
\mypoint For any non-negative
real parameters $x$, $\alpha$ and $\lambda$,
such that $\alpha < N \sinh(\frac{\lambda}{N}) \bigl[
1 - x \tanh(\frac{\lambda}{2N}) \bigr]$, for any posterior distribution $\rho: \Omega \rightarrow \C{M}_+^1(\Theta)$,
\begin{multline*}
\PP \bigl[ \rho(R) \bigr] - \inf_{\Theta_1} R
\\ \shoveleft{\quad \leq \PP
\Biggl\{ \biggl[ 1 - \frac{N\sinh(\frac{\lambda}{N})\bigl[1 - x
\tanh(\frac{\lambda}{2N})\bigr] - \lambda}{
N \sinh(\frac{\lambda}{N})\bigl[ 1 - x \tanh(\frac{\lambda}{2N})
\bigr] - \alpha} \biggr] \bigl[ \rho(r) - r(\Btheta) \bigr]}
\\ \shoveleft{\quad \qquad \qquad + \frac{\C{K} \bigl[ \rho, \pi_{\exp(- \alpha r)} \bigr]}{
N \sinh(\frac{\lambda}{N})\bigl[1 - x \tanh(\frac{\lambda}{2N})\bigr]
- \alpha} }\\
\shoveleft{\quad\qquad \qquad + \frac{N\sinh(\tfrac{\lambda}{N})
\tanh(\tfrac{\lambda}{2N})}{
N \sinh(\frac{\lambda}{N}) \bigl[ 1 - x \tanh(\frac{\lambda}{2N}) \bigr]
- \alpha}}\\\times
\biggl[ \Bphi(x) + \Tphi \biggl(
\frac{\lambda - \alpha}{N \sinh(\frac{\lambda}{N})
\tanh(\frac{\lambda}{2N})}\biggr) \biggr] \Biggr\}.
\end{multline*}
\end{cor}

Let us notice that in the case when $\Theta_1 = \Theta$,
the upper bound provided by this corollary
has the same general form as the upper bound provided by Corollary
\ref{cor1.1.21} (page \pageref{cor1.1.21}), with the sample
distribution $\PP$ replaced with
the empirical distribution of the sample $\overline{\PP}
= \bigl( \frac{1}{N} \sum_{i=1}^N \delta_{(X_i, Y_i)} \bigr)^{\otimes N}$.
Therefore, our empirical bound can be of a larger order of magnitude
than our non-random bound only in the case when our non-random
bound applied to the bootstrapped sample distribution $\overline{\PP}$
would be of a larger order of magnitude than when applied to
the true sample distribution $\PP$. In other words, we can say that
our empirical bound is close to our non-random bound in every situation
where the bootstrapped sample distribution $\overline{\PP}$ is not
harder to bound than the true sample distribution $\PP$. Although
this does not prove that our empirical bound is always of the same
order as our non-random bound, this is a good qualitative hint that
this will be the case in most practical situations of interest,
since in situations of ``under-fitting'', if they exist, it is likely
that the choice of the classification model is inappropriate to the data
and should be modified.

Another reassuring remark is that the empirical margin functions
$\Bphi$ and $\Tphi$ behave well in the case when $\inf_{\Theta} r
= 0$. Indeed in this case $m'(\theta, \wtheta)
= r'(\theta, \wtheta) = r(\theta)$, $\theta \in \Theta$,
and thus $\Bphi(1) = \Tphi(1) = 0$, and\\
\mbox{}\hfill $\Tphi(x)
\leq - (x -1 ) \inf_{\Theta_1} r$, $x \geq 1$.\hfill \mbox{}\\
This shows that in this case we recover the same
accuracy as with non-relative local empirical bounds.
Thus the bound of Corollary \ref{cor1.1.27} does not
collapse in presence of massive over-fitting in the larger
model, causing $r(\wtheta) = 0$, which is another hint
that this may be an accurate bound in many situations.

\subsection{Relative empirical deviation bounds}

It is natural to make use of Theorem \thmref{thm2.2.18}
to obtain
empirical deviation bounds, since this theorem provides an empirical
variance term.

Theorem \ref{thm2.2.18} is written in a way which exploits the
fact that $\psi_i$ takes only the three values $-1$, 0 and $+1$.
However, it will be more convenient for the following computations
to use it in its more general form, which only makes use of the
fact that $\psi_i \in\; (-1, 1)$.
With notation to be
explained hereafter, it can indeed also be written as
\newcommand{\BP}{\overline{P}}
\begin{multline}
\label{eq2.2.2}
\PP \Biggl\{ \exp \Biggl[ \sup_{\rho \in \C{M}_+^1(\Theta)} \biggl\{
- N \rho \Bigl\{ \log \Bigl[ 1 - \lambda P(\psi) \Bigr] \Bigr\}
\\ + N \rho \Bigl\{ \BP \Bigl[ \log(1 - \lambda \psi) \Bigr]
\Bigr\} - \C{K}(\rho,\pi) \biggr\} \Biggr] \Biggr\} \leq 1.
\end{multline}
We have used the following notation in this inequality. We have put
$$
\BP = \frac{1}{N} \sum_{i=1}^N \delta_{(X_i,Y_i)},
$$
so that $\BP$ is our notation for the empirical distribution of the
process \linebreak $(X_i,Y_i)_{i=1}^N$. Moreover we have also used
$$
P = \PP(\BP) = \frac{1}{N} \sum_{i=1}^N P_i,
$$
where it should be remembered that the joint distribution of the
process $(X_i,Y_i)_{i=1}^N$ is $\PP = \bigotimes_{i=1}^N P_i$.
We have considered $\psi(\theta, \T)$ as a function defined on $\C{X} \times \C{Y}$ as
$\psi(\theta, \T) (x,y) = \B{1}\bigl[ y \neq f_{\theta}(x) \bigr] - \B{1} \bigl[
y \neq f_{\T}(x) \bigr]$, $(x,y) \in \C{X} \times \C{Y}$
so that it should be understood that
\begin{multline*}
P(\psi) = \frac{1}{N} \sum_{i=1}^N \PP \bigl[ \psi_i(\theta, \T) \bigr]
\\ = \frac{1}{N} \sum_{i=1}^N \PP \Bigl\{
\B{1} \bigl[ Y_i \neq f_{\theta}(X_i) \bigr] - \B{1} \bigl[
Y_i \neq f_{\T}(X_i) \bigr] \Bigr\} = R'(\theta, \T).
\end{multline*}
In the same way
$$
\BP \Bigl[ \log(1 - \lambda \psi) \Bigr]
= \frac{1}{N} \sum_{i=1}^N \log \bigl[ 1 - \lambda \psi_i(\theta, \T) \bigr].
$$
Moreover integration with respect to $\rho$ bears on the index $\theta$,
so that
\begin{align*}
\rho \Bigl\{ \log \Bigl[ 1 - \lambda P(\psi) \Bigr] \Bigr\}
& = \int_{\theta \in \Theta} \log \biggl\{ 1 - \frac{\lambda}{N}
\sum_{i=1}^N \PP\bigl[ \psi_i(\theta, \T) \bigr] \biggr\} \rho(d \theta),\\
\rho \Bigl\{ \BP \Bigl[ \log (1 - \lambda \psi) \Bigr] \Bigr\}
& = \int_{\theta \in \Theta} \biggl\{ \frac{1}{N} \sum_{i=1}^N \log \bigl[
1 - \lambda \psi_i(\theta, \T) \bigr] \biggr\} \rho(d \theta).
\end{align*}

We have chosen concise notation, as we did throughout these notes,
in order to make the computations easier to follow.

To get an alternate version of empirical relative deviation bounds,
we need to find some convenient way to localize the choice of
the prior distribution $\pi$ in equation (\ref{eq2.2.2},
page \pageref{eq2.2.2}).
Here we propose replacing
$\pi$ with $\mu = \pi_{\exp \{ - N \log[1 + \beta P(\psi)] \}}$,
which can also be written $\pi_{\exp \{ - N \log[1 + \beta
R'(\cdot, \T)]\}}$. Indeed we see that
\begin{multline*}
\C{K}(\rho, \mu)
= N \rho \Bigl\{ \log \bigl[ 1 + \beta P(\psi) \bigr] \Bigr\}
+ \C{K}(\rho, \pi)
\\ + \log \Bigl\{ \pi \Bigl[ \exp \bigl\{
- N \log \bigl[ 1 + \beta P(\psi) \bigr] \bigr\} \Bigr] \Bigr\}.
\end{multline*}
Moreover, we deduce from our deviation inequality applied
to $- \psi$, that (as long as $\beta > -1$),
$$
\PP \biggl\{ \exp \biggl[ N \mu \Bigl\{ \BP \bigl[
\log( 1 + \beta \psi) \bigr] \Bigr\}
-N \mu \Bigl\{ \log \bigl[ 1 + \beta P(\psi) \bigr] \Bigr\}
\biggr] \biggr\} \leq 1.
$$
Thus
\begin{multline*}
\PP \biggl\{ \exp \biggl[
\log \Bigl\{ \pi \Bigl[ \exp \bigl\{
- N \log \bigl[ 1 + \beta P(\psi) \bigr] \bigr\} \Bigr] \Bigr\}
\\ \shoveright{- \log \Bigl\{ \pi \Bigl[ \exp \bigl\{
- N \BP \bigl[ \log(1 + \beta \psi) \bigr] \bigr\} \Bigr] \Bigr\}
\biggr] \bigg\}\qquad}
\\ \leq
\PP \biggl\{ \exp \biggl[
- N \mu \Bigl\{ \log \bigl[ 1 + \beta P(\psi) \bigr] \Bigr\}
- \C{K}(\mu,\pi) \\ + N \mu \Bigl\{
\BP \bigl[ \log(1 + \beta \psi) \bigr] \Bigr\} + \C{K}(\mu, \pi) \biggr] \biggr\}
\leq 1.
\end{multline*}
This can be used to handle $\C{K}(\rho, \mu)$, making use
of the Cauchy--Schwarz inequality as follows
\begin{multline*}
\PP \Biggl\{ \exp \Biggl[ \frac{1}{2} \biggl[
-N \log \Bigl\{ \Bigl( 1 - \lambda \rho\bigl[P(\psi)\bigr] \Bigr)
\Bigl( 1 + \beta \rho \bigl[ P (\psi) \bigr] \Bigr) \Bigr\}
\\* \shoveright{ \begin{aligned} + N \rho \Bigl\{ & \BP \Bigl[ \log
( 1 - \lambda \psi) \Bigr] \Bigr\}
\\* & - \C{K}(\rho, \pi) - \log \Bigl\{ \pi \Bigl[
\exp \bigl\{ - N \BP \bigl[ \log(1 + \beta \psi) \bigr]
\bigr\} \Bigr] \Bigr\} \biggr] \Biggr] \Biggr\}\end{aligned}}
\\* \shoveleft{\qquad \leq \PP \Biggl\{ \exp \Biggl[ - N \log \Bigl\{ \Bigl(
1 - \lambda \rho \bigl[ P(\psi) \bigr] \Bigr) \Bigr\}}
\\*\shoveright{ + N \rho \Bigl\{ \BP \Bigl[ \log(1 - \lambda \psi) \Bigr] \Bigr\}
- \C{K}(\rho, \mu) \Biggr] \Biggr\}^{1/2} \qquad} \\
\shoveleft{\qquad \times \PP \Biggl\{ \exp \Biggl[ \log
\Bigl\{ \pi \Bigl[ \exp \bigl\{
- N \log \bigl[1 + \beta P(\psi)\bigr] \bigr\} \Bigr] \Bigr\} }
\\*- \log \Bigl\{ \pi \Bigl[ \exp \bigl\{ - N \BP \bigl[
\log(1 + \beta \psi) \bigr] \bigr\} \Bigr] \Bigr\} \Biggr] \Biggr\}^{1/2}
\leq 1.
\end{multline*}
This implies that with $\PP$ probability at least $1 - \epsilon$,
\begin{multline*}
-N \log \Bigl\{ \Bigl( 1 - \lambda \rho\bigl[P(\psi)\bigr] \Bigr)
\Bigl( 1 + \beta \rho \bigl[ P (\psi) \bigr] \Bigr) \Bigr\}
\\ \begin{aligned} \leq -N \rho & \Bigl\{ \BP \Bigl[ \log
( 1 - \lambda \psi) \Bigr] \Bigr\}
\\ & + \C{K}(\rho, \pi) + \log \Bigl\{ \pi \Bigl[
\exp \bigl\{ - N \BP \bigl[ \log(1 + \beta \psi) \bigr]
\bigr\} \Bigr] \Bigr\} -
2 \log(\epsilon).\end{aligned}
\end{multline*}
It is now convenient to remember that
$$
\BP \Bigl[\log(1 - \lambda \psi) \Bigr]
= \frac{1}{2} \log \left( \frac{1 - \lambda}{1 + \lambda} \right) r'(\theta, \T)
+ \frac{1}{2} \log (1 - \lambda^2) m'(\theta, \T).
$$
We thus can write the previous inequality as
\begin{multline*}
- N \log \Bigl\{ \Bigl( 1 - \lambda \rho\bigl[R'(\cdot,\T) \bigr] \Bigr)
\Bigl(1 + \beta \rho \bigl[ R'(\cdot,\T) \bigr] \Bigr) \Bigr\} \\ \leq
\frac{N}{2} \log \left( \frac{1+\lambda}{1-\lambda}\right)
\rho \bigl[ r'(\cdot,\T) \bigr] - \frac{N}{2} \log(1 - \lambda^2)
\rho \bigl[ m'(\cdot, \T) \bigr] +
\C{K}(\rho, \pi) \\ \begin{aligned}+ \log \biggl\{ \pi \biggl[
\exp \Bigl\{ & - \frac{N}{2}
\log \Bigl( \frac{1 + \beta}{1 - \beta} \Bigr) r'(\cdot, \T)
\\ & - \frac{N}{2} \log( 1 - \beta^2) m'(\cdot, \T) \Bigr\} \biggr] \biggr\}
- 2 \log(\epsilon).\end{aligned}
\end{multline*}
Let us assume now that $\T \in \arg\min_{\Theta_1} R$.
Let us introduce $\Btheta \in \arg\min_{\Theta} r$.
Decomposing
$r'(\theta, \T) = r'(\theta, \Btheta) + r'(\Btheta,\T)$ and
considering that \\
\mbox{} \hfill $m'(\theta, \T) \leq m'(\theta,
\Btheta) + m'(\Btheta,\T)$, \hfill \mbox{}\\
we see that with $\PP$ probability at least $1 - \epsilon$,
for any posterior distribution $\rho:
\Omega \rightarrow \C{M}_+^1(\Theta)$,

\begin{multline*}
- N \log \Bigl\{ \Bigl( 1 -
\lambda \rho \bigl[ R'(\cdot, \T) \bigr] \Bigr) \Bigl(
1 + \beta \rho \bigl[ R'(\cdot, \T) \Bigr) \Bigr\}
\\* \leq \frac{N}{2} \log \biggl( \frac{1 + \lambda}{1 - \lambda} \biggr)
\rho \bigl[ r'(\cdot, \Btheta) \bigr] -
\frac{N}{2} \log(1 - \lambda^2) \rho \bigl[ m'(\cdot, \Btheta) \bigr]
+ \C{K}(\rho,\pi) \\* + \log \biggl\{ \pi \biggl[
\exp \Bigl\{ - \tfrac{N}{2} \log \Bigl( \tfrac{1+\beta}{1-\beta} \Bigr)
\bigl[r'(\cdot, \Btheta\,) \bigr] - \tfrac{N}{2} \log(1 - \beta^2) m'(\cdot, \Btheta\,)
\Bigr\} \biggr] \biggr\} \\*
+ \tfrac{N}{2} \log \Bigl[ \tfrac{(1 + \lambda)(1 - \beta)}{(1 - \lambda)(1 + \beta)}
\Bigr] \bigl[ r(\Btheta\,) - r(\T) \bigr]
\\* - \tfrac{N}{2} \log \bigl[ (1 - \lambda^2)(1 - \beta^2) \bigr] m'(\Btheta\,,\T)
- 2 \log(\epsilon).
\end{multline*}

Let us now define for simplicity the posterior $\nu: \Omega \rightarrow
\C{M}_+^1(\Theta)$ by the identity
$$
\frac{d \nu}{d \pi}(\theta) = \frac{ \exp \Bigl\{
- \frac{N}{2}  \log \Bigl( \frac{1+\lambda}{1-\lambda} \Bigr)
r'(\theta,\Btheta) + \frac{N}{2} \log(1 - \lambda^2) m'(\theta, \Btheta)
\Bigr\}}{ \pi
\biggl[ \exp \Bigl\{
- \frac{N}{2}  \log \Bigl( \frac{1+\lambda}{1-\lambda} \Bigr)
r'(\cdot,\Btheta) + \frac{N}{2} \log(1 - \lambda^2) m'(\cdot, \Btheta)
\Bigr\}\biggl]}.
$$
Let us also introduce the random bound
\begin{multline*}
B =
\frac{1}{N} \log \biggl\{ \nu \biggl[ \exp \Bigl[ \tfrac{N}{2} \log \Bigl[
\tfrac{(1 + \lambda)(1 - \beta)}{(1 - \lambda) (1 + \beta) } \Bigr]
r'(\cdot, \Btheta) \\ \shoveright{- \tfrac{N}{2} \log \bigl[ (1 - \lambda^2)
(1 - \beta^2) \bigr] m'(\cdot, \Btheta\,) \Bigr] \biggr] \biggr\}\qquad} \\
\shoveleft{\qquad + \sup_{\theta \in \Theta_1}
\frac{1}{2} \log \Big[\tfrac{(1 - \lambda)(1 + \beta)}{(1 + \lambda)(1 - \beta)}
\Bigr]
r'(\theta,\Btheta\,)} \\ - \frac{1}{2} \log\bigl[ (1 - \lambda^2)(1 - \beta^2)\bigr]
m'(\theta,\Btheta\,) - \frac{2}{N} \log(\epsilon).
\end{multline*}
\begin{thm}\mypoint
Using the above notation, for any real constants $0 \leq \beta < \lambda < 1$,
for any prior distribution $\pi \in \C{M}_+^1(\Theta)$,
for any subset $\Theta_1 \subset \Theta$,
with $\PP$ probability at least $1 - \epsilon$,
for any posterior distribution $\rho: \Omega \rightarrow \C{M}_+^1(\Theta)$,
$$
- \log \Bigl\{ \Bigl( 1 - \lambda \bigl[ \rho (R) - \inf_{\Theta_1} R \bigr]
\Bigr)\Bigl(1 + \beta \bigl[ \rho (R) - \inf_{\Theta_1} R \bigr] \Bigr) \Bigr\}
\leq \frac{\C{K}(\rho, \nu)}{N} + B.
$$
Therefore,
\begin{multline*}
\rho(R) - \inf_{\Theta_1} R \\* \leq \frac{\lambda - \beta}{2 \lambda \beta}
\left( \sqrt{1 + 4 \frac{\lambda \beta}{(\lambda - \beta)^2}
\left[ 1 - \exp \left( - B - \frac{\C{K}(\rho, \nu)}{N} \right) \right]}-1\right)
\\ \leq \frac{1}{\lambda - \beta} \left( B + \frac{\C{K}(\rho,\nu)}{N} \right).
\end{multline*}
\end{thm}

Let us define the posterior $\widehat{\nu}$ by the identity
$$
\frac{d\widehat{\nu}}{d\pi} (\theta) = \frac{\exp
\Bigl[ - \frac{N}{2} \log \left(
\frac{1+\beta}{1-\beta}\right) r'(\theta, \Btheta) - \frac{N}{2}
\log(1 - \beta^2) m'(\theta, \Btheta)\Bigr]}{
\pi \Bigl\{ \exp
\Bigl[ - \frac{N}{2} \log \left(
\frac{1+\beta}{1-\beta}\right) r'(\cdot, \Btheta) - \frac{N}{2}
\log(1 - \beta^2) m'(\cdot, \Btheta)\Bigr]\Bigr\}}.
$$
It is useful to remark that
\begin{multline*}
\frac{1}{N} \log \biggl\{ \nu \biggl[ \exp \Bigl[ \frac{N}{2} \log \Bigl(
\frac{(1 + \lambda)(1 - \beta)}{(1 - \lambda) (1 + \beta) } \Bigr)
r'(\cdot, \Btheta) \\ \shoveright{- \frac{N}{2} \log \bigl[ (1 - \lambda^2)
(1 - \beta^2) \bigr] m'(\cdot, \Btheta) \Bigr] \biggr] \biggr\}\qquad} \\
\\ \shoveleft{\qquad \leq
\widehat{\nu}
\biggl\{ \frac{1}{2}
\log \Bigl( \frac{(1+\lambda)(1-\beta)}{(1 - \lambda)(1+\beta)}\Bigr)
r'( \cdot, \Btheta) }\\ - \frac{1}{2} \log\bigl[ (1 - \lambda^2)(1 - \beta^2) \bigr]
m'(\cdot, \Btheta) \biggr\}.
\end{multline*}
This inequality is a special case of
\begin{multline*}
\log \Bigl\{ \pi \bigl[ \exp (g) \bigr] \Bigr\} -
\log \Bigl\{ \pi \bigl[ \exp (h) \bigr] \Bigr\} \\ =
\int_{\alpha = 0}^1 \pi_{\exp[ h + \alpha (g-h)]}(g-h) d\alpha
\leq \pi_{\exp(g)}(g-h),
\end{multline*}
which is a consequence of the convexity of $\alpha \mapsto
\log \Bigl\{ \pi \Bigl[ \exp \bigl[ h + \alpha(g-h) \bigr] \Bigr] \Bigr\}$.

Let us introduce as previously
$
\Bphi(x) = \sup_{\theta \in \Theta} m'(\theta, \Btheta) -
x \, r'(\theta, \Btheta)$, $x \in \RR_+$.
Let us moreover consider $
\Tphi(x) = \sup_{\theta \in \Theta_1} m'(\theta, \Btheta) -
x \, r'(\theta, \Btheta)$, $x \in \RR_+$. These functions can be
used to produce a result which is slightly weaker, but maybe easier
to read and understand. Indeed, 
we see that, for any $x \in \RR_+$, with $\PP$ probability at least $1 - \epsilon$,
for any posterior distribution $\rho$,
\begin{multline*}
- N \log \Bigl\{\Bigl( 1 - \lambda \rho \bigl[R'(\cdot, \T)\bigr] \Bigr)
\Bigl(1 + \beta \rho \bigl[ R'(\cdot, \T) \bigr] \Bigr) \Bigr\}
\\*\shoveleft{\qquad \leq \frac{N}{2} \log \left[ \frac{(1+\lambda)}{(1-\lambda)(1 - \lambda^2)^x}\right]
\rho \bigl[ r'(\cdot, \Btheta) \bigr] }
\\*\shoveleft{\qquad\qquad - \frac{N}{2} \log\bigl[ (1 - \lambda^2)(1 -
\beta^2) \bigr]  \Bphi(x)} + \C{K}(\rho, \pi)
\\*\shoveleft{\qquad\qquad + \log \biggl\{ \pi \biggl[ \exp \Bigl\{
- \tfrac{N}{2} \log \Bigl[ \tfrac{(1+\beta)}{(1-\beta)(1 - \beta^2)^x}\Bigr]
r'(\cdot, \Btheta) \Bigr\} \biggr] \biggr\}
}\\* \shoveleft{\qquad\qquad - \frac{N}{2} \log\bigl[
(1-\lambda^2)(1-\beta^2) \bigr]
\Tphi \left( \frac{ \log \left[ \frac{(1+\lambda)(1-\beta)}{(1-\lambda)(1+\beta)}
\right]}{- \log\left[ (1 - \lambda^2)(1 - \beta^2) \right]} \right)
}\\*\shoveright{- 2 \log(\epsilon)\qquad}
\\ \shoveleft{ \qquad =
\int_{\frac{N}{2} \log \left[ \frac{(1+\beta)}{(1 - \beta)(1 - \beta^2)^x} \right]}^{
\frac{N}{2} \log \left[ \frac{(1+\lambda)}{(1 - \lambda)(1 - \lambda^2)^x} \right]}
\pi_{\exp (- \alpha r)}\bigl[ r'(\cdot, \Btheta)\bigr] d \alpha}
\\* \shoveright{+ \C{K}(\rho, \pi_{\exp \{ - \frac{N}{2} \log [ \frac{(1+\lambda)}{(1-\lambda)
(1-\lambda^2)^x}] r \}}) - 2 \log (\epsilon)\quad}
\\* - \frac{N}{2} \log \bigl[ (1 - \lambda^2)(1 - \beta^2) \bigr]
\left[ \Bphi(x) + \Tphi \left( \frac{\log \left[ \frac{(1+\lambda)(1-\beta)}{(1-\lambda)
(1 + \beta)} \right]}{- \log [ (1 - \lambda^2)(1 - \beta^2) ]} \right) \right].
\end{multline*}
\begin{thm}\mypoint
With the previous notation, for any real constants $0 \leq \beta < \lambda < 1$,
for any positive real constant $x$, for any prior probability distribution
$\pi \in \C{M}_+^1(\Theta)$, for any subset $\Theta_1 \subset \Theta$,
with $\PP$ probability at least $1 - \epsilon$,
for any posterior distribution $\rho: \Omega \rightarrow \C{M}_+^1(\Theta)$,
putting
\begin{multline*}
B(\rho) =
\frac{1}{N(\lambda - \beta)}
\int_{\frac{N}{2} \log \left[ \frac{(1+\beta)}{(1 - \beta)(1 - \beta^2)^x} \right]}^{
\frac{N}{2} \log \left[ \frac{(1+\lambda)}{(1 - \lambda)(1 - \lambda^2)^x} \right]}
\pi_{\exp (- \alpha r)}\bigl[ r'(\cdot, \Btheta)\bigr] d \alpha
\\ + \frac{\C{K}(\rho, \pi_{\exp \{ - \frac{N}{2} \log [ \frac{(1+\lambda)}{(1-\lambda)
(1-\lambda^2)^x}] r \}}) - 2 \log (\epsilon)}{N(\lambda - \beta)}\\
- \frac{1}{2(\lambda - \beta)} \log \bigl[ (1 - \lambda^2)(1 - \beta^2) \bigr]
\left[ \Bphi(x) + \Tphi \left( \frac{\log \left[ \frac{(1+\lambda)(1-\beta)}{(1-\lambda)
(1 + \beta)} \right]}{- \log [ (1 - \lambda^2)(1 - \beta^2) ]} \right) \right]
\\ \shoveleft{\leq
\frac{1}{N(\lambda - \beta)}
d_e \log \left( \frac{\log \Bigl[ \frac{(1+\lambda)}{(1-\lambda)(1-\lambda^2)^x}\Bigr]}{
\log \Bigl(\frac{(1+\beta)}{(1-\beta)(1-\beta^2)^x}\Bigr)}\right)}
\\ + \frac{\C{K}(\rho, \pi_{\exp \{ - \frac{N}{2} \log [ \frac{(1+\lambda)}{(1-\lambda)
(1-\lambda^2)^x}] r \}}) - 2 \log (\epsilon)}{N(\lambda - \beta)}\\
- \frac{1}{2(\lambda - \beta)} \log \bigl[ (1 - \lambda^2)(1 - \beta^2) \bigr]
\left[ \Bphi(x) + \Tphi \left( \frac{\log \left[ \frac{(1+\lambda)(1-\beta)}{(1-\lambda)
(1 + \beta)} \right]}{- \log [ (1 - \lambda^2)(1 - \beta^2) ]} \right) \right],
\end{multline*}
the following bounds hold true:
\begin{multline*}
\rho(R) - \inf_{\Theta_1} R \\ \leq \frac{\lambda - \beta}{2 \lambda \beta}
\Biggl(
\sqrt{
1 + \frac{4 \lambda \beta}{(\lambda - \beta)^2}
\Bigl\{ 1 - \exp \bigl[ - (\lambda - \beta)  B(\rho)
\bigr] \Bigr\}} - 1 \Biggr) \\ \leq B(\rho).
\end{multline*}
\end{thm}

Let us remark that this alternative way of handling
relative deviation bounds
made it possible to carry on with non-linear bounds up to the final result.
For instance, if $\lambda = 0.5$, $\beta = 0.2$ and $B(\rho) = 0.1$,
the non-linear bound gives $\rho(R) - \inf_{\Theta_1} R \leq 0.096$.

\chapter{Comparing posterior distributions to Gibbs priors}

\section{Bounds relative to a Gibbs distribution}
We now come to an approach to relative bounds
whose performance can be analysed with PAC-Bayesian
tools.

The empirical bounds at the end of the previous chapter
involve taking suprema in $\theta \in \Theta$, and replacing the
\emph{expected margin function} $\varphi$ with some empirical counterparts
$\Bphi$ or $\Tphi$, which may prove unsafe
when using very complex classification models.

We are now going to focus on
the control of the divergence $\C{K}
\bigl[ \rho, \pi_{\exp( - \beta R)} \bigr]$.
It is already obvious, we hope, that controlling this
divergence is the crux of the matter, and that it is
a way to upper bound the mutual information between
the training sample and the parameter, which can be expressed
as $\C{K}\bigl[ \rho, \PP(\rho) \bigr]
= \C{K} \bigl[ \rho, \pi_{\exp( - \beta R)} \bigr] -
\C{K} \bigl[ \PP(\rho), \pi_{\exp( - \beta R)} \bigr]$,
as explained on page \pageref{mutual}.

Through the identity
\begin{multline}
\label{eq1.26}
\C{K} \bigl[ \rho, \pi_{\exp( - \beta R)} \bigr]
= \beta \bigl[ \rho(R) - \pi_{\exp( - \beta R)}(R) \bigr] \\ +
\C{K} (\rho, \pi )- \C{K} \bigl[ \pi_{\exp( - \beta R)}, \pi \bigr],
\end{multline}
we see that the control of this divergence is related to
the control of the difference $\rho(R) - \pi_{\exp( - \beta R)}(R)$.
This is the route we will follow first.

Thus comparing any posterior distribution
with a Gibbs prior distribution will provide a first way to build
an estimator which can be proved
to reach adaptively the best possible asymptotic
error rate under Mammen
and Tsybakov margin assumptions and parametric complexity
assumptions (at least as long as orders of magnitude are concerned,
we will not discuss the question of asymptotically optimal constants).

Then we will provide an empirical bound for the Kullback
divergence $\C{K} \bigl[ \rho,\break \pi_{\exp( - \beta R)} \bigr]$ itself.
This will serve to address the question of model selection,
which will be achieved by comparing the performance of
two posterior distributions possibly supported by two
different models. This will also provide a second way
to build estimators which can be proved to be adaptive
under Mammen and Tsybakov margin assumptions and
parametric complexity assumptions (somewhat weaker
than with the first method).

Finally, we will present two-step localization strategies, in which
the performance of the posterior distribution to be analysed is
compared with a \emph{two-step} Gibbs prior.

\subsection{Comparing a posterior distribution with a Gibbs prior}
Similarly to Theorem \ref{thm2.2.18} (page \pageref{thm2.2.18}) we can prove that for any prior distribution
$\wt{\pi} \in \C{M}_+^1(\Theta)$,
\begin{multline}
\label{eq1.1.15}
\PP \Biggl\{ \wt{\pi} \otimes \wt{\pi} \biggl\{ \exp \biggl[ -
N \log (1 - N \tanh \bigl(
\tfrac{\gamma}{N} \bigr) R') \\ - \gamma r' - N \log \bigl[
\cosh(\tfrac{\gamma}{N}) \bigr] m' \biggr] \biggr\}
\Biggr\} \leq 1.
\end{multline}
Replacing $\wt{\pi}$ with $\pi_{\exp( - \beta R)}$ and considering
the posterior distribution $\rho \otimes \pi_{\exp( - \beta R)}$,
provides a starting point in the comparison of
$\rho$ with $\pi_{\exp( - \beta R)}$; we can indeed
state with $\PP$ probability at least $1 - \epsilon$ that
\begin{multline}
\label{eq1.1.17}
- N \log \Bigl\{ 1 - \tanh \bigl( \tfrac{\gamma}{N} \bigr) \Bigl[
\rho(R) - \pi_{\exp( - \beta R)}(R) \Bigr] \Bigr\}
\\ \leq \gamma
\bigl[ \rho(r) - \pi_{\exp(- \beta R)}(r) \bigr]
+ N \log \bigl[ \cosh(\tfrac{\gamma}{N}) \bigr] \bigl[
\rho \otimes \pi_{\exp( - \beta R)} \bigr]
(m') \\ + \C{K}\bigl[ \rho, \pi_{\exp(- \beta R)} \bigr] - \log(\epsilon).
\end{multline}
Using equation \myeq{eq1.26} to handle the entropy term, we get
\begin{multline}
\label{eq1.1.20}
- N \log \Bigl\{ 1 - \tanh(\tfrac{\gamma}{N})
\Bigl[ \rho(R) - \pi_{\exp( - \beta R)}(R) \Bigr] \Bigr\}
- \beta \bigl[ \rho(R) - \pi_{\exp( - \beta R)}(R) \bigr]
\\ \leq \gamma \bigl[ \rho(r) - \pi_{\exp( - \beta R)}(r) \bigr]
+ N \log \bigl[ \cosh \bigl(\tfrac{\gamma}{N} \bigr)\bigr]
\rho \otimes \pi_{\exp( - \beta R)}(m')
\\ + \C{K}(\rho, \pi) - \C{K}\bigl[ \pi_{\exp( - \beta R)}, \pi \bigr]
- \log(\epsilon).
\end{multline}

We can then decompose in the right-hand side
$\gamma \bigl[ \rho(r) - \pi_{\exp( - \beta R)}(r) \bigr]$ into
$(\gamma - \lambda) \bigl[ \rho(r) - \pi_{\exp( - \beta R)}(r) \bigr]
+ \lambda \bigl[ \rho(r) - \pi_{\exp( - \beta R)}(r) \bigr]$
for some parameter $\lambda$ to be set later on
and use the fact that
\begin{multline*}
\lambda \bigl[ \rho(r) - \pi_{\exp( - \beta R)}(r) \bigr]
+ N \log \bigl[ \cosh(\tfrac{\gamma}{N}) \bigr] \rho \otimes
\pi_{\exp( - \beta R)}(m') \\ \shoveright{+ \C{K}(\rho, \pi)
- \C{K}\bigl[ \pi_{\exp( - \beta R)}, \pi \bigr]}
\\ \leq \lambda \rho(r) + \C{K}(\rho, \pi) + \log \Bigl\{
\pi \Bigl[ \exp \bigl\{ - \lambda r + N \log \bigl[ \cosh(\tfrac{\gamma}{N}) \bigr] \rho(m') \bigr\}
\Bigr] \Bigr\} \\
= \C{K}\bigl[ \rho, \pi_{\exp( - \lambda r)}\bigr]
+ \log \Bigl\{ \pi_{\exp( - \lambda r)} \Bigl[ \exp \bigl\{ N \log \bigl[
\cosh(\tfrac{\gamma}{N}) \bigr] \rho(m') \bigr\} \Bigr] \Bigr\},
\end{multline*}
to get rid of the appearance of the unobserved Gibbs prior $\pi_{\exp( - \beta R)}$
in most places of the right-hand side of our inequality, leading to
\begin{thm}
\mypoint
\label{thm1.1.41Bis}
For any real constants $\beta$ and $\gamma$,
with $\PP$ probability at least $1 - \epsilon$, for any posterior distribution
$\rho: \Omega \rightarrow \C{M}_+^1(\Theta)$, for any real constant $\lambda$,
\begin{multline*}
\bigl[ N \tanh(\tfrac{\gamma}{N}) - \beta \bigr]
\bigl[ \rho(R) - \pi_{\exp( - \beta R)}(R) \bigr] \\
\shoveleft{\qquad \leq - N \log \Bigl\{ 1 - \tanh(\tfrac{\gamma}{N}) \Bigl[ \rho(R)
- \pi_{\exp( - \beta R)}(R) \Bigr] \Bigr\} }
\\ \shoveright{- \beta \bigl[ \rho(R) - \pi_{\exp( - \beta R)}(R) \bigr]}
\\ \shoveleft{\qquad \leq (\gamma - \lambda) \bigl[
\rho(r) - \pi_{\exp( - \beta R)}(r) \bigr]
+ \C{K}\bigl[ \rho, \pi_{\exp( - \lambda r)}\bigr]}
\\\shoveright{ + \log \Bigl\{ \pi_{\exp( - \lambda r)} \Bigl[ \exp \bigl\{ N
\log \bigl[ \cosh(\tfrac{\gamma}{N}) \bigr] \rho(m') \bigr\} \Bigr] \Bigr\} -
\log(\epsilon)}
\\ \shoveleft{\qquad = \C{K}\bigl[ \rho, \pi_{\exp (- \gamma r)} \bigr] }
\\ + \log \Bigl\{ \pi_{\exp( - \gamma r)} \Bigl[
\exp \bigl\{ (\gamma - \lambda) r + N \log \bigl[ \cosh(\tfrac{\gamma}{N})
\bigr] \rho(m') \bigr\} \Bigr] \Bigr\} \\
-( \gamma - \lambda) \pi_{\exp( - \beta R)}(r)
- \log(\epsilon).
\end{multline*}
\end{thm}

We would like to have a fully empirical upper bound even in the case when $\lambda
\neq \gamma$. This can be done by using the theorem twice. We will
need a lemma.
\begin{lemma}
\label{lemma1.38}
For any probability distribution $\pi \in \C{M}_+^1(\Theta)$,
for any bounded measurable functions $g,h: \Theta \rightarrow \RR$,
$$
\pi_{\exp( -g )}(g) - \pi_{\exp(-h)}(g) \leq
\pi_{\exp(-g)}(h) - \pi_{\exp(-h)}(h).
$$
\end{lemma}
\begin{proof}
Let us notice that
\begin{multline*}
0 \leq \C{K}(\pi_{\exp( - g)}, \pi_{\exp( - h )})
= \pi_{\exp( - g)}(h)
+ \log \bigl\{ \pi \bigl[ \exp ( - h) \bigr] \bigr\} + \C{K}(\pi_{\exp( - g)}, \pi)
\\ = \pi_{\exp( - g)}(h) - \pi_{\exp( - h)}(h) - \C{K}(\pi_{\exp( - h)}, \pi)
+ \C{K}(\pi_{\exp( - g)}, \pi)
\\ = \pi_{\exp( - g)}(h) - \pi_{\exp( - h)}(h) - \C{K}(\pi_{\exp( - h)}, \pi)
- \pi_{\exp( - g)}(g) - \log \bigl\{ \pi \bigl[ \exp ( - g) \bigr] \bigr\}.
\end{multline*}
Moreover
$$
- \log \bigl\{ \pi \bigl[ \exp( - g) \bigr] \bigr\} \leq \pi_{\exp( - h)}(g)
+ \C{K}(\pi_{\exp( - h)}, \pi),
$$
which ends the proof.
\end{proof}

For any positive real constants $\beta$ and $\lambda$,
we can then apply Theorem \ref{thm1.1.41Bis} to $\rho = \pi_{\exp( - \lambda r)}$,
and use the inequality
\begin{equation}
\label{eq1.1.22}
\frac{\lambda}{\beta} \bigl[
\pi_{\exp( - \lambda r)}(r) - \pi_{\exp( - \beta R)}(r) \bigr]
\leq \pi_{\exp( - \lambda r)}(R) -
\pi_{\exp( - \beta R) }(R)
\end{equation}
provided by the previous lemma.
We thus obtain with $\PP$ probability at least $1 - \epsilon$
\begin{multline*}
- N \log \Bigl\{ 1 - \tanh(\tfrac{\gamma}{N}) \tfrac{\lambda}{\beta}
\Bigl[ \pi_{\exp
(- \lambda r)} (r) - \pi_{\exp( - \beta R)}(r) \Bigr] \Bigr\}
\\ \shoveright{- \gamma \bigl[
\pi_{\exp( - \lambda r)}(r) - \pi_{\exp( - \beta R)}(r) \bigr] }
\\ \leq \log \Bigl\{ \pi_{\exp( - \lambda r)} \Bigl[
\exp \bigl\{ N \log \bigl[ \cosh(\tfrac{\gamma}{N}) \bigr] \pi_{\exp( - \lambda r)}
(m') \bigr\} \Bigr] \Bigr\} - \log(\epsilon).
\end{multline*}
Let us
introduce the convex function
$$
F_{\gamma, \alpha}(x) = - N \log \bigl[ 1 - \tanh(\tfrac{\gamma}{N})
x \bigr] - \alpha x \geq \bigl[ N \tanh(\tfrac{\gamma}{N}) - \alpha \bigr] x.
$$
With $\PP$ probability at least $1 - \epsilon$,
\begin{multline*}
- \pi_{\exp( - \beta R)}(r)
\leq \inf_{\lambda \in \RR_+^*} \biggl\{ - \pi_{\exp( - \lambda r)}(r) \\*
+ \frac{\beta}{\lambda} F_{\gamma,
\frac{\beta \gamma}{\lambda}}^{-1} \biggl[
\log \Bigl\{ \pi_{\exp(- \lambda r)} \Bigl[ \exp
\bigl\{ N \log \bigl[ \cosh(\tfrac{\gamma}{N}) \bigr]
\pi_{\exp( - \lambda r)}(m') \bigr\} \Bigr] \Bigr\}
\\ - \log(\epsilon) \biggr] \biggr\}.
\end{multline*}
Since Theorem \ref{thm1.1.41Bis} holds uniformly for any posterior distribution
$\rho$, we can apply it again to some arbitrary posterior distribution $\rho$.
We can moreover make the result uniform in $\beta$ and $\gamma$ by considering
some atomic measure $\nu \in \C{M}_+^1(\RR)$ on the real line and using a union bound.
This leads to
\begin{thm}
\mypoint
\label{thm1.1.43}
For any atomic probability distribution on the positive real line
$\nu \in \C{M}_+^1(\RR_+)$,
with $\PP$ probability
at least $1 - \epsilon$, for any posterior distribution $\rho:
\Omega \rightarrow \C{M}_+^1(\Theta)$, for any positive real constants $\beta$
and $\gamma$,
\begin{multline*}
\bigl[ N \tanh(\tfrac{\gamma}{N}) - \beta \bigr] \bigl[ \rho(R) -
\pi_{\exp( - \beta R)}(R) \bigr]
\\* \shoveright{\leq
F_{\gamma, \beta}\bigl[ \rho(R) - \pi_{\exp( - \beta R)}(R) \bigr]
\leq  B(\rho, \beta, \gamma), \text{ where}}\\\shoveleft{B(\rho, \beta, \gamma) = \inf_{
\substack{\lambda_1 \in \RR_+, \lambda_1 \leq \gamma\\
\lambda_2 \in \RR, \lambda_2 >
\frac{\beta \gamma}{N} \tanh(\frac{\gamma}{N})^{-1}
}} \Biggr\{
\C{K}\bigl[ \rho, \pi_{\exp( - \lambda_1 r)} \bigr] }
\\\shoveleft{\qquad + (\gamma - \lambda_1) \bigl[ \rho(r)
- \pi_{\exp( - \lambda_2 r)}(r) \bigr]}
\\\shoveleft{\qquad + \log \Bigl\{ \pi_{\exp( - \lambda_1 r)} \Bigl[ \exp \bigl\{
N \log \bigl[ \cosh(\tfrac{\gamma}{N}) \bigr] \rho(m') \bigr\} \Bigr] \Bigr\}
- \log \bigl[ \epsilon \nu(\beta) \nu(\gamma) \bigr]}\\
\shoveleft{\qquad + (\gamma - \lambda_1) \frac{\beta}{\lambda_2}
F_{\gamma, \frac{\beta \gamma}{\lambda_2}}^{-1}  \biggl[
\log \Bigl\{ }\\ \pi_{\exp( - \lambda_2 r)} \Bigl[ \exp \bigl\{
N \log \bigl[ \cosh(\tfrac{\gamma}{N}) \bigr] \pi_{\exp( - \lambda_2 r)}(m')
\bigr\} \Bigr] \Bigr\} \\\shoveright{ - \log \bigl[ \epsilon \nu(\beta)
\nu(\gamma)\bigr]  \biggr] \Biggr\}}
\\\shoveleft{\leq  \inf_{
\substack{\lambda_1 \in \RR_+, \lambda_1 \leq \gamma\\
\lambda_2 \in \RR, \lambda_2 >
\frac{\beta \gamma}{N} \tanh(\frac{\gamma}{N})^{-1}
}} \Biggr\{
\C{K}\bigl[ \rho, \pi_{\exp( - \lambda_1 r)} \bigr]
}\\\shoveleft{\qquad+ (\gamma - \lambda_1) \bigl[
\rho(r) - \pi_{\exp( - \lambda_2 r)}(r) \bigr]}
\\\shoveleft{\qquad+ \log \Bigl\{ \pi_{\exp( - \lambda_1 r)} \Bigl[ \exp \bigl\{
N \log \bigl[ \cosh(\tfrac{\gamma}{N}) \bigr] \rho(m') \bigr\} \Bigr] \Bigr\}}
\\\shoveleft{\qquad + \frac{\beta}{\lambda_2} \frac{(1 - \frac{\lambda_1}{\gamma})}{
\bigl[ \frac{N}{\gamma} \tanh(\frac{\gamma}{N}) - \frac{\beta}{\lambda_2}\bigr]}
\log \Bigl\{ \pi_{\exp( - \lambda_2 r)} \Bigl[ }
\\ \exp \bigl\{
N \log \bigl[ \cosh(\tfrac{\gamma}{N}) \bigr] \pi_{\exp( - \lambda_2 r)}(m')
\bigr\} \Bigr] \Bigr\} \\
- \Bigl\{ 1 + \frac{\beta}{\lambda_2} \tfrac{(1 - \frac{\lambda_1}{\gamma})}{
[ \frac{N}{\gamma} \tanh(\frac{\gamma}{N}) - \frac{\beta}{\lambda_2}]} \Bigr\}
\log \bigl[ \epsilon \nu( \beta) \nu( \gamma) \bigr]  \Biggr\},
\end{multline*}
where we have written for short $\nu(\beta)$ and $\nu(\gamma)$ instead
of $\nu(\{\beta\})$ and $\nu(\{\gamma\})$.
\end{thm}

Let us notice that $B(\rho, \beta, \gamma) = + \infty$ when $\nu(\beta) = 0$
or $\nu(\gamma) = 0$, the uniformity in $\beta$ and $\gamma$ of the
theorem therefore necessarily bears on a countable number of values of these parameters.
We can typically choose distributions  for $\nu$ such as the one
used in Theorem \thmref{thm1.1.11}:
namely we can put for some positive real ratio $\alpha > 1$
$$
\nu(\alpha^k) = \frac{1}{(k+1)(k+2)}, \qquad k \in \NN,
$$
or alternatively, since we are interested in values of the parameters
less than $N$, we can prefer
$$
\nu(\alpha^k) = \frac{\log(\alpha)}{\log(\alpha N)},
\qquad 0 \leq k < \frac{\log(N)}{\log(\alpha)}.
$$
We can also use such a coding distribution on dyadic numbers
as the one defined by equation \myeq{eq1.1.4bis}.

Following the same route as for Theorem \ref{thm1.24} (page
\pageref{thm1.24}), we can also prove the following result
about the deviations under any posterior distribution $\rho$:
\begin{thm}
For any $\epsilon \in )0,1($, with $\PP$ probability at least
$1 - \epsilon$, for any posterior distribution $\rho: \Omega
\rightarrow \C{M}_+^1(\Theta)$,
with $\rho$ probability at least $1 - \xi$,
\begin{multline*}
F_{\gamma, \beta} \bigl[ R(\wh{\theta}) - \pi_{\exp( - \beta R)}(R)
\bigr]
\leq
\inf_{
\substack{
\lambda_1 \in \RR_+, \lambda_1 \leq \gamma,\\
\lambda_2 \in \RR, \lambda_2 > \frac{\beta \gamma}{N} \tanh(\frac{
\gamma}{N})^{-1}}
}
\Biggl\{ \log \Biggl[ \frac{d \rho}{d \pi_{\exp ( - \lambda_1 r)}}(\wh{\theta}
\,) \Biggr]  \\ +
(\gamma - \lambda_1) \bigl[ r( \wh{\theta}\,) - \pi_{\exp( - \lambda_2 r)}(r)
\bigr] \\
+ \log \Bigl\{ \pi_{\exp( - \lambda_1 r)} \Bigl[
\exp \bigl\{ N \log \bigl[ \cosh (\tfrac{\gamma}{N}) \bigr]
m'(\cdot , \wh{\theta}\,) \bigr\} \Bigr] \Bigr\}
- \log \bigl[ \epsilon \xi \nu(\beta) \nu(\gamma) \bigr] \\
+ ( \gamma - \lambda_1) \frac{\beta}{\lambda_2}
F_{\gamma, \frac{\beta \gamma}{\lambda_2}}^{-1}
\biggl[ \log \Bigl\{ \\
\pi_{\exp( - \lambda_2 r)} \Bigl[ \exp \bigl\{ N \log \bigl[
\cosh(\tfrac{\gamma}{N}) \bigr] \pi_{\exp( - \lambda_2 r)}(m') \bigr\}
\Bigr] \Bigr\} \\
- \log \bigl[ \epsilon \nu(\beta) \nu(\gamma) \bigr] \biggr] \Biggr\}.
\end{multline*}
\end{thm}

The only tricky point is to justify that we can still take an infimum
in $\lambda_1$ without using a union bound. To justify this, we have
to notice that the following variant of Theorem \ref{thm1.1.41Bis}
(page \pageref{thm1.1.41Bis}) holds:
with $\PP$ probability at least $1 - \epsilon$, for any posterior
distribution $\rho: \Omega \rightarrow \C{M}_+^1(\Theta)$,
for any real constant $\lambda$,
\begin{multline*}
\rho \Bigl\{ F_{\gamma, \beta} \bigl[ R - \pi_{\exp( - \beta R)}(R) \bigr]
\Bigr\}
\leq \C{K}\bigl[ \rho, \pi_{\exp( - \gamma r)} \bigr]
\\ + \rho \biggl[ \inf_{\lambda \in \RR}
\log \Bigl\{ \pi_{\exp( - \gamma r)} \Bigl[
\exp \bigl\{ (\gamma - \lambda) r + N \log \bigl[ \cosh(\tfrac{\gamma}{N}
\bigr) \bigr] m'(\cdot, \wh{\theta}\,) \bigr\} \Bigr] \Bigr\}
\\ - (\gamma - \lambda) \pi_{\exp(- \beta R)}(r) \biggr]  - \log(\epsilon).
\end{multline*}
We leave the details as an exercise.

\subsection{The effective temperature of a posterior distribution}
Using the parametric approximation $\pi_{\exp( - \alpha r)}(r)
- \inf_{\Theta} r \simeq \frac{d_e}{\alpha}$, we get as an order of magnitude
\begin{multline*}
B(\pi_{\exp( - \lambda_1 r)}, \beta, \gamma) \lesssim
- (\gamma - \lambda_1) d_e \bigl[ \lambda_2^{-1} - \lambda_1^{-1} \bigr]
\\ \shoveleft{\qquad + 2 d_e \log \frac{\lambda_1}{ \lambda_1
- N\log\bigl[ \cosh(\tfrac{\gamma}{N}) \bigr] x}}\\*
\qquad\qquad + 2 \frac{\beta}{\lambda_2} \frac{(1 - \frac{\lambda_1}{\gamma})}{
\bigl[ \frac{N}{\gamma}\tanh(\tfrac{\gamma}{N}) - \frac{\beta}{\lambda_2} \bigr]} d_e \log
\left( \frac{ \lambda_2}{\lambda_2 - N \log \bigl[ \cosh(\tfrac{\gamma}{N}) \bigr] x}
\right) \\*
\qquad\qquad\qquad\qquad + 2 N \log \bigl[ \cosh(\tfrac{\gamma}{N}) \bigr] \biggl[ 1 + \frac{\beta}{\lambda_2}
\frac{(1 - \frac{\lambda_1}{\gamma})}{ \bigl[ \frac{N}{\gamma}
\tanh(\frac{\gamma}{N}) - \frac{\beta}{\lambda_2} \bigr]} \biggr] \Tphi(x)
\\ - \Bigl\{ 1 + \frac{\beta}{\lambda_2}
\frac{(1 - \frac{\lambda_1}{\gamma})}{[\frac{N}{\gamma} \tanh(\tfrac{\gamma}{N})
- \frac{\beta}{\lambda_2}]} \Bigr\} \log\bigl[ \nu(\beta) \nu(\gamma) \epsilon
\bigr].
\end{multline*}
Therefore, if the empirical dimension $d_e$ stays bounded when $N$ increases,
we are going to obtain a negative upper bound for any values of the constants
$\lambda_1 > \lambda_2 > \beta$, as soon as $\gamma$ and $\frac{N}{\gamma}$
are chosen to be large enough.
This ability to obtain negative values for the bound $B(\pi_{\exp( - \lambda_1 r)},
\gamma, \beta)$, and more generally $B(\rho, \gamma, \beta)$, leads the way
to introducing the new concept of the \emph{effective temperature} of an estimator.
\begin{dfn}
For any posterior distribution $\rho: \Omega \rightarrow \C{M}_+^1(\Theta)$ we define
the \emph{effective temperature} $T(\rho) \in
\RR \cup \{ - \infty, + \infty \}$ of $\rho$ by the equation
$$
\rho(R) = \pi_{\exp( - \frac{R}{T(\rho)})}(R).
$$
\end{dfn}
Note that $\beta \mapsto \pi_{\exp( - \beta R)}(R): \RR \cup \{ - \infty, + \infty \}
\rightarrow (0,1)$ is continuous and strictly decreasing from $\ess \sup_{\pi} R$
to $\ess \inf_{\pi} R$ (as soon as these two bounds do not coincide). This shows
that the effective temperature $T(\rho)$ is a well-defined random variable.

Theorem \ref{thm1.1.43} provides a bound for $T(\rho)$, indeed:
\begin{prop}\mypoint
\label{prop1.1.37}
Let
$$
\w{\beta}(\rho) = \sup \bigl\{ \beta \in \RR; \inf_{\gamma, N \tanh(\frac{\gamma}{N})
> \beta}
B(\rho, \beta, \gamma) \leq 0 \bigr\},
$$
where $B(\rho, \beta, \gamma)$ is as in Theorem \thmref{thm1.1.43}.
Then with $\PP$ probability at least $1 - \epsilon$, for any posterior
distribution $\rho: \Omega \rightarrow \C{M}_+^1(\Theta)$,
$T(\rho) \leq \w{\beta}(\rho)^{-1}$, or equivalently
$\rho(R) \leq \pi_{\exp[ - \w{\beta}(\rho)  R]}(R)$.
\end{prop}

This notion of \emph{effective temperature} of a (randomized) estimator
$\rho$ is interesting for two reasons:
\begin{itemize}
\item the difference $\rho(R) - \pi_{\exp( - \beta R)}(R)$ can be estimated
with better accuracy than $\rho(R)$ itself, due to the use of relative deviation
inequalities, leading to convergence rates up to $1/N$ in favourable situations,
even when $\inf_{\Theta} R$ is not close to zero;

\item and of course $\pi_{\exp( - \beta R)}(R)$ is a decreasing function
of $\beta$, thus being able to estimate $\rho(R) - \pi_{\exp( - \beta R)}(R)$
with some given accuracy, means being able to discriminate between values
of $\rho(R)$ with the same accuracy, although doing so through the
parametrization $\beta \mapsto \pi_{\exp( - \beta R)}(R)$, which can neither
be observed nor estimated with the same precision!
\end{itemize}
\eject

\subsection{Analysis of an empirical bound for the effective temperature}
We are now going to launch into a mathematically rigorous analysis of
the bound $B(\pi_{\exp( - \lambda_1 r), \beta, \gamma})$
provided by Theorem \thmref{thm1.1.43},
to show that $\inf_{\rho \in \C{M}_+^1(\Theta)}\break
\pi_{\exp[ - \w{\beta}(\rho) R]}(R)$ converges indeed to $\inf_{\Theta} R$
at some optimal rate in favourable situations.

It is more convenient for this purpose to use deviation inequalities involving
$M'$ rather than $m'$. It is straightforward to extend Theorem
\thmref{thm4.1} to
\begin{thm}
\mypoint
For any real constants $\beta$ and $\gamma$, for any prior distributions
$\pi, \mu \in \C{M}_+^1(\Theta)$, with $\PP$ probability at least $1 - \eta$,
for any posterior distribution $\rho: \Omega \rightarrow \C{M}_+^1(\Theta)$,
$$
\gamma \rho \otimes \pi_{\exp( - \beta R)} \bigl[ \Psi_{\frac{\gamma}{N}}(R', M') \bigr]
\leq \gamma \rho \otimes \pi_{\exp( - \beta R)}(r') + \C{K}(\rho, \mu) - \log(\eta).
$$
\end{thm}

In order to transform the left-hand side into a linear expression and
in the same time localize this theorem, let us choose $\mu$ defined by its density
\begin{multline*}
\frac{d \mu}{d \pi}(\theta_1)
= C^{-1} \exp \biggl[ - \beta R(\theta_1)
\\* - \gamma \int_{\Theta} \Bigl\{
\Psi_{\frac{\gamma}{N}} \bigl[ R'(\theta_1, \theta_2),
M'(\theta_1, \theta_2) \bigr] \\* - \tfrac{N}{\gamma} \sinh(\tfrac{\gamma}{N})
R'(\theta_1, \theta_2) \Bigr\}  \pi_{\exp( - \beta R)}(d \theta_2) \biggr],
\end{multline*}
where $C$ is such that $\mu(\Theta) = 1$.
We get
\begin{multline*}
\C{K}(\rho, \mu) = \beta \rho(R) + \gamma
\rho \otimes \pi_{\exp( - \beta R)} \bigl[
\Psi_{\frac{\gamma}{N}} (R', M') - \tfrac{N}{\gamma} \sinh(\tfrac{\gamma}{N})
R' \bigr] + \C{K}(\rho, \pi) \\
\shoveleft{\qquad + \log \biggl\{ \int_{\Theta} \exp \biggl[ - \beta R(\theta_1)}
\\ - \gamma \int_{\Theta} \Bigl\{
\Psi_{\frac{\gamma}{N}} \bigl[ R'(\theta_1, \theta_2), M'(\theta_1,
\theta_2) \bigr]\\\shoveright{ - \tfrac{N}{\gamma} \sinh(\tfrac{\gamma}{N})
R'(\theta_1, \theta_2) \Bigr\} \pi_{\exp( -
\beta R)}(d \theta_2) \biggr] \pi ( d \theta_1) \biggr\}}
\\\shoveleft{\quad= \beta \bigl[ \rho(R) - \pi_{\exp( - \beta R)}(R) \bigr]}\\
+ \gamma \rho \otimes \pi_{\exp ( - \beta R)} \bigl[
\Psi_{\frac{\gamma}{N}}(R', M') - \tfrac{N}{\gamma} \sinh(\tfrac{\gamma}{N})
R' \bigr]
\\\shoveright{+ \C{K}(\rho, \pi) - \C{K}(\pi_{\exp( - \beta R)}, \pi)
\qquad}\\
\shoveleft{\qquad + \log \biggl\{ \int_{\Theta} \exp
\biggl[ - \gamma \int_{\Theta} \Bigl\{ \Psi_{\frac{\gamma}{N}}
\bigl[ R'(\theta_1, \theta_2),M'(\theta_1, \theta_2) \bigr]
}\\ - \tfrac{N}{\gamma} \sinh(\tfrac{\gamma}{N})
R'(\theta_1, \theta_2) \Bigr\} \pi_{\exp( - \beta R)}(d \theta_2)
\biggr] \pi_{\exp( - \beta R)}(d \theta_1) \biggr\}.
\end{multline*}
Thus with $\PP$ probability at least $1 - \eta$,
\begin{multline}
\label{eq1.1.23}
\bigl[ N \sinh(\tfrac{\gamma}{N}) - \beta \bigr]
\bigl[ \rho(R) - \pi_{\exp( - \beta R)}(R) \bigr]
\\\shoveleft{\qquad \leq \gamma \bigl[ \rho(r) - \pi_{\exp ( - \beta R)}(r) \bigr] +
\C{K}(\rho, \pi) - \C{K}(\pi_{\exp( - \beta R)}, \pi) - \log(\eta) +
C(\beta, \gamma)}
\\
\shoveleft{\text{where } C(\beta, \gamma) = \log \biggl\{ \int_{\Theta} \exp
\biggl[ - \gamma \int_{\Theta} \Bigl\{ \Psi_{\frac{\gamma}{N}}
\bigl[ R'(\theta_1, \theta_2),M'(\theta_1, \theta_2) \bigr]
}\\- \tfrac{N}{\gamma} \sinh(\tfrac{\gamma}{N})
R'(\theta_1, \theta_2) \Bigr\} \pi_{\exp( - \beta R)}(d \theta_2)
\biggr] \pi_{\exp( - \beta R)}(d \theta_1) \biggr\}.
\end{multline}
Remarking that
$$
\C{K}\bigl[ \rho, \pi_{\exp( - \beta R)}\bigr]
= \beta \bigl[ \rho(R) - \pi_{\exp( - \beta R)}(R) \bigr]
+ \C{K}(\rho, \pi) - \C{K}(\pi_{\exp( - \beta R)}, \pi),
$$
we deduce from the previous inequality
\begin{thm}\mypoint
\label{thm1.1.45}
For any real constants $\beta$ and $\gamma$, with $\PP$ probability
at least $1 - \eta$, for any posterior distribution $\rho: \Omega
\rightarrow \C{M}_+^1(\Theta)$,
\begin{multline*}
N \sinh(\tfrac{\gamma}{N}) \bigl[ \rho(R) - \pi_{\exp( - \beta R)}(R)
\bigr] \leq \gamma \bigl[ \rho(r) - \pi_{\exp( - \beta R)}(r) \bigr]
\\ + \C{K}\bigl[ \rho, \pi_{\exp( - \beta R)}\bigr] - \log(\eta)
+ C(\beta, \gamma).
\end{multline*}
\end{thm}

We can also go into a slightly different direction, starting
back again from equation \myeq{eq1.1.23}  and
remarking that for any real constant $\lambda$,
\begin{multline*}
\lambda \bigl[ \rho(r) - \pi_{\exp( - \beta R)}(r) \bigr]
+ \C{K}(\rho, \pi) - \C{K}(\pi_{\exp(- \beta R)}, \pi)
\\ \leq \lambda \rho(r) + \C{K}(\rho, \pi) + \log \bigl\{
\pi \bigl[ \exp ( - \lambda r) \bigr] \bigr\} =
\C{K}\bigl[ \rho, \pi_{\exp( - \lambda r)} \bigr].
\end{multline*}
This leads to
\begin{thm}\mypoint
\label{thm1.43}
For any real constants $\beta$ and $\gamma$, with $\PP$ probability at least $1 - \eta$,
for any real constant $\lambda$,
\begin{multline*}
\bigl[ N \sinh(\tfrac{\gamma}{N}) - \beta \bigr]
\bigl[ \rho(R) - \pi_{\exp( - \beta R)}(R) \bigr]
\\ \leq (\gamma - \lambda)
\bigl[ \rho(r) - \pi_{\exp ( - \beta R)}(r) \bigr] +
\C{K}\bigl[ \rho, \pi_{\exp( - \lambda r)} \bigr] - \log(\eta) + C(\beta, \gamma),
\end{multline*}
where the definition of $C(\beta, \gamma)$ is given by equation
\myeq{eq1.1.23}.
\end{thm}

We can now use this inequality in the case when $\rho = \pi_{\exp( - \lambda r)}$
and combine it with Inequality \myeq{eq1.1.22}
to obtain
\begin{thm}
For any real constants $\beta$ and $\gamma$,
with $\PP$ probability at least $1 - \eta$, for any real constant
$\lambda$,
$$
\bigl[ \tfrac{N \lambda}{\beta} \sinh(\tfrac{\gamma}{N}) - \gamma \bigr]
\bigl[ \pi_{\exp( - \lambda r)}(r) - \pi_{\exp( - \beta R)}(r) \bigr]
\leq C(\beta, \gamma) - \log(\eta).
$$
\end{thm}

We deduce from this theorem
\begin{prop}
For any real positive constants $\beta_1$, $\beta_2$ and
$\gamma$, with $\PP$ probability at least $1 - \eta$, for any real constants
$\lambda_1$ and $\lambda_2$, such that $\lambda_2 < \beta_2 \frac{\gamma}{N}
\sinh(\frac{\gamma}{N})^{-1}$ and $\lambda_1 > \beta_1 \frac{\gamma}{N}
\sinh(\frac{\gamma}{N})^{-1}$,
\begin{multline*}
\pi_{\exp( - \lambda_1 r)}(r) - \pi_{\exp( - \lambda_2 r)}(r)
\leq \pi_{\exp( - \beta_1 R)}(r) - \pi_{\exp( - \beta_2 R)}(r)
\\ + \frac{C(\beta_1, \gamma) + \log( 2 /\eta)}{\frac{N\lambda_1}{\beta_1}
\sinh(\frac{\gamma}{N})- \gamma}
+ \frac{C(\beta_2, \gamma) + \log( 2 /\eta)}{\gamma - \frac{N\lambda_2}{\beta_2}
\sinh(\frac{\gamma}{N})}.
\end{multline*}
\end{prop}

Moreover, $\pi_{\exp( - \beta_1 R)}$ and $\pi_{\exp( - \beta_2 R)}$
being prior distributions,
with $\PP$ probability at least $1 - \eta$,
\begin{multline*}
\gamma \bigl[ \pi_{\exp( - \beta_1 R)}(r) - \pi_{\exp( - \beta_2 R)}(r) \bigr]
\\ \leq \gamma \pi_{\exp( - \beta_1 R)} \otimes \pi_{\exp( - \beta_2 R)}
 \bigl[ \Psi_{- \frac{\gamma}{N}}(R',M') \bigr] - \log( \eta).
\end{multline*}
Hence
\begin{prop}
\label{prop1.46}
For any positive real constants $\beta_1$, $\beta_2$ and $\gamma$,
with $\PP$ probability at least $1 - \eta$,
for any positive real constants $\lambda_1$ and $\lambda_2$
such that $\lambda_2 < \beta_2 \frac{\gamma}{N} \sinh(\tfrac{\gamma}{N})^{-1}$
and $\lambda_1 > \beta_1 \frac{\gamma}{N} \sinh(\frac{\gamma}{N})^{-1}$,
\begin{multline*}
\pi_{\exp ( - \lambda_1 r)}(r) - \pi_{\exp( - \lambda_2 r)}(r)
\\ \leq \pi_{\exp( - \beta_1 R)} \otimes
\pi_{\exp( - \beta_2 R)} \bigl[ \Psi_{- \frac{\gamma}{N}} (R',M')\bigr] \\
+ \frac{\log(\frac{3}{\eta})}{\gamma} + \frac{C(\beta_1,\gamma) + \log(\frac{3}{\eta})}{
\frac{N \lambda_1}{\beta_1} \sinh(\frac{\gamma}{N})- \gamma}
+ \frac{C(\beta_2, \gamma) + \log (\frac{3}{\eta})}{\gamma -
\frac{N \lambda_2}{\beta_2} \sinh(\frac{\gamma}{N})}.
\end{multline*}
\end{prop}

In order to achieve the analysis of the bound
$B(\pi_{\exp( - \lambda_1 r)}, \beta, \gamma)$
given by Theorem \thmref{thm1.1.43},
it now remains to bound quantities of the
general form
\begin{multline*}
\log \Bigl\{ \pi_{\exp( - \lambda r)} \Bigl[
\exp \bigl\{ N \log \bigl[ \cosh(\tfrac{\gamma}{N}) \bigr] \pi_{\exp(
- \lambda r)}(m') \bigr\} \Bigr] \Bigr\} \\
= \sup_{\rho \in \C{M}_+^1(\Theta)}
N \log \bigl[ \cosh(\tfrac{\gamma}{N}) \bigr] \rho \otimes
\pi_{\exp( - \lambda)}(m') -
\C{K}\bigl[\rho, \pi_{\exp( - \lambda r)}\bigr].
\end{multline*}

Let us consider the prior distribution $\mu \in \C{M}_+^1(\Theta \times \Theta)$
on couples of parameters defined by the density
$$
\frac{d \mu}{d (\pi \otimes \pi)} (\theta_1, \theta_2)
= C^{-1} \exp \Bigl\{
- \beta R(\theta_1) - \beta R(\theta_2) + \alpha
\Phi_{- \frac{\alpha}{N}} \bigl[ M'(\theta_1, \theta_2) \bigr] \Bigr\},
$$
where the normalizing constant $C$ is such that $\mu( \Theta \times \Theta) = 1$.
Since for fixed values of the parameters $\theta$
and $\theta' \in \Theta$, $m'(\theta, \theta')$, like $r(\theta)$, is a sum
of independent Bernoulli random variables, we can easily
adapt the proof of Theorem \ref{thm2.3} on page \pageref{thm2.3},
to establish that with $\PP$ probability at least $1 - \eta$,  for any posterior distribution
$\rho$ and any real constant $\lambda$,
\begin{multline*}
\alpha \rho \otimes \pi_{\exp( - \lambda r)}(m')
\leq \alpha \rho \otimes \pi_{\exp( - \lambda r)} \bigl[ \Phi_{- \frac{\alpha}{N}}(M') \bigr]
\\\shoveright{ +  \C{K}(\rho \otimes \pi_{\exp( - \lambda r)}, \mu) -
\log( \eta)} \\
\shoveleft{\qquad = \C{K}\bigl[ \rho, \pi_{\exp( - \beta R)}\bigr] + \C{K}\bigl[
\pi_{\exp( - \lambda r)}, \pi_{\exp( - \beta R)}\bigr] }
\\* + \log \Bigl\{ \pi_{\exp( - \beta R)} \otimes \pi_{\exp( - \beta
R)} \Bigl[ \exp \bigl( \alpha \Phi_{-\frac{\alpha}{N}}\!\circ\!M' \bigr)
\Bigr] \Bigr\} - \log(\eta).
\end{multline*}
Thus for any real constant $\beta$ and any positive real constants
$\alpha$ and $\gamma$,
with $\PP$ probability at least $1 - \eta$,  for any real constant
$\lambda$,
\begin{multline}
\label{eq1.1.24}
\log \Bigl\{ \pi_{\exp( - \lambda r)} \Bigl[ \exp
\bigl\{ N \log \bigl[ \cosh(\tfrac{\gamma}{N})\bigr] \pi_{\exp( - \lambda r)}
(m') \bigr\} \Bigr] \Bigr\}
\\ \leq \sup_{\rho \in \C{M}_+^1(\Theta)} \biggl(
\tfrac{N}{\alpha} \log \bigl[ \cosh(\tfrac{\gamma}{N})\bigr]
\Bigl\{ \C{K}\bigl[ \rho, \pi_{\exp( - \beta R)}\bigr]
+ \C{K} \bigl[ \pi_{\exp( - \lambda r)}, \pi_{\exp( - \beta R)} \bigr]
\\
+ \log \bigl\{ \pi_{\exp( - \beta R)} \otimes \pi_{\exp(- \beta R)}
\bigl[ \exp ( \alpha \Phi_{- \frac{\alpha}{N}}\!\circ\!M') \bigr] \bigr\}
\\ - \log( \eta) \Bigr\} - \C{K}\bigl[ \rho, \pi_{\exp( - \lambda r)}\bigr] \biggr).
\end{multline}

To finish, we need some appropriate upper bound for the entropy
\linebreak $\C{K}\bigl[ \rho, \pi_{\exp( - \beta R)} \bigr]$. This question can
be handled in the following way:
using Theorem \thmref{thm1.1.45},
we see that for any positive real constants $\gamma$ and $\beta$,
with $\PP$ probability at least $1 - \eta$, for any posterior distribution
$\rho$,
\begin{multline*}
\C{K}\bigl[ \rho, \pi_{\exp( - \beta R)} \bigr]
= \beta \bigl[ \rho(R) - \pi_{\exp( - \beta R)}(R) \bigr]
+ \C{K}(\rho, \pi) - \C{K}(\pi_{\exp( - \beta R)}, \pi)
\\ \shoveleft{\qquad \leq \frac{\beta}{N \sinh(\frac{\gamma}{N})} \biggl[
\gamma \bigl[ \rho(r) - \pi_{\exp( - \beta R)}(r) \bigr] }
 \\ + \C{K}\bigl[ \rho, \pi_{\exp( - \beta R)}\bigr]
- \log(\eta) + C(\beta, \gamma) \biggr]\\\shoveright{+ \C{K}(\rho, \pi)
- \C{K}(\pi_{\exp( - \beta R)}, \pi)\qquad}
\\ \shoveleft{\qquad \leq \C{K} \bigl[ \rho, \pi_{\exp( - \frac{\beta \gamma}{N
\sinh(\frac{\gamma}{N})} r)}
\bigr]} \\ + \frac{\beta}{N \sinh(\frac{\gamma}{N})}
\Bigl\{ \C{K}\bigl[ \rho, \pi_{\exp( - \beta R)}\bigr]
+ C(\beta, \gamma) - \log(\eta) \Bigr\}.
\end{multline*}
In other words,
\begin{thm}
\label{thm2.1.12}
\mypoint
For any positive real constants $\beta$ and $\gamma$ such that
$\beta < N\times \sinh(\tfrac{\gamma}{N})$, with $\PP$ probability at least $1 - \eta$, for any posterior
distribution $\rho: \Omega \rightarrow \C{M}_+^1(\Theta)$,
$$
\C{K}\bigl[ \rho, \pi_{\exp( - \beta R)} \bigr]
\leq \frac{\ds \C{K} \bigl[ \rho, \pi_{\exp[ - \beta \frac{\gamma}{N}
\sinh(\frac{\gamma}{N})^{-1} r]} \bigr]}{\ds 1 - \frac{\beta}{N \sinh(\frac{\gamma}{N})}}
+ \frac{\ds C(\beta, \gamma) - \log(\eta)}{\ds \frac{N \sinh(\frac{\gamma}{N})}{\beta}
- 1},
$$
where  the quantity $C(\beta, \gamma)$ is defined
by equation \myeq{eq1.1.23}.
Equivalently, it will be in some cases more convenient to
use this result in the form: for any positive real constants
$\lambda$ and $\gamma$, with $\PP$ probability at least $1 - \eta$,
for any posterior distribution $\rho: \Omega \rightarrow \C{M}_+^1(\Theta)$,
$$
\C{K} \bigl[ \rho, \pi_{\exp[ - \lambda \frac{N}{\gamma} \sinh(
\frac{\gamma}{N}) R]}\bigr]
\leq \frac{\C{K}\bigl[ \rho, \pi_{\exp( - \lambda r)} \bigr]}{
1 - \frac{\lambda}{\gamma}} +
\frac{C(\lambda \frac{N}{\gamma}\sinh(\frac{\gamma}{N}), \gamma)
- \log(\eta)}{\frac{\lambda}{\beta} - 1}.
$$
\end{thm}

Choosing in equation \myeq{eq1.1.24}
$\ds \alpha = \frac{N \log \bigl[ \cosh(\frac{\gamma}{N})\bigr]}{1
- \frac{\beta}{N \sinh(\frac{\gamma}{N})}}$ and
$\beta = \lambda \frac{N}{\gamma} \sinh(\frac{\gamma}{N})$, so that
$\ds \alpha = \frac{N \log \bigl[ \cosh(\frac{\gamma}{N})\bigr]}{1 - \frac{\lambda}{\gamma}
}$, we obtain with $\PP$
probability at least $1 - \eta$,
\begin{multline*}
\log \Bigl\{ \pi_{\exp( - \lambda r)} \Bigl[
\exp \bigl\{ N \log \bigl[ \cosh(\tfrac{\gamma}{N})\bigr] \pi_{\exp( -
\lambda r)}(m') \bigr\} \Bigr] \Bigr\}
\\ \shoveleft{\qquad \leq \tfrac{2 \lambda}{\gamma} \bigl[
C(\beta, \gamma) + \log( \tfrac{2}{\eta}) \bigr]
} \\ + \Bigl( 1 - \tfrac{\lambda}{\gamma} \Bigr) \biggl[ \log \Bigl\{ \pi_{\exp( - \beta R)} \otimes \pi_{\exp( - \beta R)}
\bigl[ \exp( \alpha \Phi_{-\frac{\alpha}{N}}\!\circ\!M')\bigr] \Bigr\} \\+
\log( \tfrac{2}{\eta}) \biggr].
\end{multline*}
This proves
\begin{prop}
\mypoint
\label{prop1.48}
For any positive real constants $\lambda < \gamma$,
with $\PP$ probability at least $1 - \eta$,
\begin{multline*}
\log \Bigl\{ \pi_{\exp( - \lambda r)} \Bigl[
\exp \bigl\{ N \log \bigl[ \cosh(\tfrac{\gamma}{N})\bigr] \pi_{\exp( -
\lambda r)}(m') \bigr\} \Bigr] \Bigr\} \\
\shoveleft{\qquad \leq
\frac{2 \lambda}{\gamma} \bigl[ C( \tfrac{N \lambda}{\gamma} \sinh(
\tfrac{\gamma}{N}), \gamma)
+ \log ( \tfrac{2}{\eta}) \bigr]}
\\\shoveleft{\qquad\qquad + \Bigl(1 - \tfrac{\lambda}{\gamma}\Bigr)
\log \biggl\{ \pi_{\exp[ - \frac{N\lambda}{\gamma} \sinh(\frac{\gamma}{N}) R]
}^{\otimes 2}
\biggl[}\\\shoveright{
\exp \biggl( \frac{N \log [ \cosh(\tfrac{\gamma}{N})]}{1 - \frac{\lambda}{\gamma}}
\Phi_{- \frac{\log[\cosh(\frac{\gamma}{N})]}{1 - \frac{\lambda}{\gamma}}}\!\circ\!M'
\biggr)
\biggr] \biggr\}\qquad}\\
+ \Bigl( 1 - \tfrac{\lambda}{\gamma} \Bigr) \log( \tfrac{2}{\eta}).
\end{multline*}
\end{prop}

We are now ready to analyse the bound
$B(\pi_{\exp( - \lambda_1 r)}, \beta, \gamma)$ of
Theorem \thmref{thm1.1.43}.
\begin{thm}\mypoint
\label{thm1.1.52}
For any positive real constants $\lambda_1$, $\lambda_2$, $\beta_1$,
$\beta_2$, $\beta$ and $\gamma$, such that
\begin{align*}
\lambda_1 & < \gamma,&
\beta_1 & < \tfrac{N \lambda_1}{\gamma} \sinh(\tfrac{\gamma}{N}),\\
\lambda_2 & < \gamma, & \beta_2 & > \tfrac{N \lambda_2}{\gamma} \sinh(\tfrac{\gamma}{N}),\\
& & \beta & < \tfrac{N \lambda_2}{\gamma} \tanh(\tfrac{\gamma}{N}),
\end{align*}
with $\PP$ probability $1 - \eta$, the bound
$B(\pi_{\exp( - \lambda_1 r)}, \beta, \gamma)$
of Theorem \thmref{thm1.1.43} satisfies
\begin{multline*}
B(\pi_{\exp( - \lambda_1 r)}, \beta, \gamma) \\ \leq
(\gamma - \lambda_1) \Biggl\{ \pi_{\exp( - \beta_1 R)} \otimes
\pi_{\exp( - \beta_2 R)} \bigl[ \Psi_{- \frac{\gamma}{N}} (R',M') \bigr]
+ \frac{\log(\frac{7}{\eta})}{\gamma} \\*
\shoveright{+ \frac{C(\beta_1, \gamma) + \log( \frac{7}{\eta})}{
\frac{N \lambda_1}{\beta_1} \sinh(\frac{\gamma}{N}) - \gamma}
+ \frac{C(\beta_2, \gamma)+ \log(\frac{7}{\eta})}{\gamma -
\frac{N\lambda_2}{\beta_2} \sinh( \frac{\gamma}{N})}
\Biggr\}} \\*
\qquad+ \frac{2 \lambda_1}{\gamma}
\Bigl[ C \bigl(\tfrac{N \lambda_1}{\gamma} \sinh(\tfrac{\gamma}{N}), \gamma\bigr)
+ \log(\tfrac{7}{\eta}) \Bigr] \\*
\shoveleft{\qquad + \left( 1 - \tfrac{\lambda_1}{\gamma} \right)
\log \biggl\{ \pi_{\exp [ - \frac{N \lambda_1}{\gamma} \sinh(\frac{\gamma}{N})
R]}^{\otimes 2} \biggl[}\\\shoveright{ \exp \biggl( \tfrac{N \log [ \cosh(\frac{\gamma}{N})] }{1
- \frac{\lambda_1}{\gamma}} \Phi_{- \frac{\log[\cosh(\frac{\gamma}{N})]}{1
- \frac{\lambda_1}{\gamma}}}\!\circ\!M'\biggr)\biggr] \biggr\} }
\\* + \Bigl( 1 - \tfrac{\lambda_1}{\gamma} \Bigr)
\log(\tfrac{7}{\eta}) - \log\bigl[ \nu(\{\beta\}) \nu(\{\gamma\})\epsilon
\bigr]\\*
\shoveleft{\qquad+ (\gamma - \lambda_1) \tfrac{\beta}{\lambda_2}
F_{\gamma, \frac{\beta \gamma}{\lambda_2}}^{-1} \Biggl\{
\frac{2 \lambda_2}{\gamma}
\Bigl[ C \bigl( \tfrac{N \lambda_2}{\gamma} \sinh(\tfrac{\gamma}{N}), \gamma \bigr)
+ \log \bigl( \tfrac{7}{\eta}\bigr) \Bigr]}\\*
\shoveleft{\qquad \qquad + \Bigl( 1 - \tfrac{\lambda_2}{\gamma}
\Bigr)
\log \biggl\{
\pi_{\exp[ - \frac{N \lambda_2}{\gamma} \sinh(\frac{\gamma}{N})R]}^{\otimes 2}
\biggl[}\\
\exp \biggl( \frac{N\log[\cosh(\frac{\gamma}{N})]}{1 - \frac{\lambda_2}{\gamma}}
\Phi_{- \frac{\log[\cosh(\frac{\gamma}{N})]}{1 - \frac{\lambda_2}{\gamma}}}\!\circ\!M'
\biggr) \biggr] \biggr\} \\* + \Bigl(1 - \tfrac{\lambda_2}{\gamma} \Bigr)
\log\bigl(\tfrac{7}{\eta}\bigr) - \log\bigl[\nu(\{\beta\}) \nu(\{\gamma\})\epsilon\bigr]
\Biggr\},
\end{multline*}
where the function $C(\beta, \gamma)$ is defined by equation \myeq{eq1.1.23}.
\end{thm}
\subsection{Adaptation to parametric and margin assumptions}
To help understand the previous theorem, it may be useful to
give linear upper-bounds to the factors appearing in the
right-hand side of the previous inequality.
Introducing $\T$ such that $R(\T) = \inf_{\Theta} R$
(assuming that such a parameter exists) and remembering that
\begin{align*}
\Psi_{-a}(p,m) & \leq a^{-1} \sinh(a) p + 2 a^{-1} \sinh(\tfrac{a}{2})^2 m, & a \in \RR_+,\\
\Phi_{-a}(p) & \leq a^{-1} \bigl[ \exp(a)-1 \bigr] p, & a \in \RR_+,\\
\Psi_{a}(p,m) & \geq a^{-1} \sinh(a) p - 2a^{-1}\sinh(\tfrac{a}{2})^2 m, & a \in \RR_+,\\
M'(\theta_1, \theta_2) & \leq M'(\theta_1, \T) + M'(\theta_2, \T), & \theta_1, \theta_2
\in \Theta,\\
M'(\theta_1, \T) & \leq x R'(\theta_1, \T) + \varphi(x), & x \in \RR_+, \theta_1 \in
\Theta,
\end{align*}
the last inequality being rather
a consequence of the definition of $\varphi$ than a property of $M'$,
we easily see that
\begin{multline*}
\pi_{\exp( - \beta_1 R)}\otimes \pi_{\exp( - \beta_2 R)}
\bigl[ \Psi_{- \frac{\gamma}{N}}(R',M') \bigr]
\\\shoveleft{\quad  \leq
\tfrac{N}{\gamma} \sinh(\tfrac{\gamma}{N})
\bigl[ \pi_{\exp( - \beta_1 R)}(R) - \pi_{\exp( - \beta_2 R)}(R) \bigr]}
\\\shoveright{+  \tfrac{2N}{\gamma}\sinh(\tfrac{\gamma}{2N})^{2}
\pi_{\exp( - \beta_1 R)} \otimes \pi_{\exp( - \beta_2 R)}
(M')\qquad} \\
\shoveleft{\quad\leq \tfrac{N}{\gamma} \sinh(\tfrac{\gamma}{N}) \bigl[ \pi_{
\exp( - \beta_1 R)}(R) -
\pi_{\exp( - \beta_2 R)}(R) \bigr]} \\
\qquad + \frac{2xN}{\gamma} \sinh(\tfrac{\gamma}{2N})^{2} \Bigl\{
\pi_{\exp( - \beta_1 R)}\bigl[ R'(\cdot, \T) \bigr] +
\pi_{\exp( - \beta_2 R)} \bigl[ R'(\cdot, \T) \bigr] \Bigr\}
\\ + \frac{4N}{\gamma} \sinh(\tfrac{\gamma}{2N})^2 \varphi(x),
\end{multline*}
that
\begin{multline*}
C(\beta, \gamma) \leq
\log \biggl\{ \pi_{\exp( - \beta R)} \Bigl\{ \exp \Bigl[
2 N \sinh\bigl(\tfrac{\gamma}{2N}\bigr)^{2} \pi_{\exp( - \beta R)}(M') \Bigr] \Bigr\}
\biggr\} \\\shoveleft{\qquad\qquad\leq
\log \biggl\{ \pi_{\exp( - \beta R)} \Bigl\{ \exp \Bigl[
2 N \sinh\bigl(\tfrac{\gamma}{2N}\bigr)^{2} M'(\cdot, \T) \Bigr] \Bigr\}
\biggr\}} \\\shoveright{ + 2N\sinh(\tfrac{\gamma}{2N})^{2} \pi_{\exp( - \beta R)}
\bigl[ M'(\cdot, \T)\bigr]}\\
\shoveleft{\qquad\qquad \leq \log \biggl\{ \pi_{\exp( - \beta R)} \Bigl\{ \exp \Bigl[
2 x N \sinh(\tfrac{\gamma}{2N})^{2} R'( \cdot, \T) \Bigr] \Bigr\} \biggr\}}
\\\shoveright{+ 2 x N \sinh(\tfrac{\gamma}{2N})^{2} \pi_{\exp( - \beta R)}
\bigl[ R'(\cdot, \T) \bigr] + 4 N \sinh(\tfrac{\gamma}{2N})^{2}
\varphi(x)}\\
\shoveleft{\qquad\qquad = \int_{\beta - 2xN\sinh(\frac{\gamma}{2N})^2}^{\beta}
\pi_{\exp( - \alpha R)}\bigl[ R'(\cdot, \T) \bigr]  d \alpha}\\
\shoveright{+ 2 x N \sinh(\tfrac{\gamma}{2N})^{2} \pi_{\exp( - \beta R)}
\bigl[ R'(\cdot, \T) \bigr] + 4 N \sinh(\tfrac{\gamma}{2N})^{2}
\varphi(x)}\\
 \shoveleft{\qquad \qquad \leq 4xN\sinh(\tfrac{\gamma}{2N})^2 \pi_{\exp[ - (\beta - 2 x N
\sinh(\frac{\gamma}{2N})^2)R]}\bigl[ R'(\cdot, \T) \bigr]
}\\ + 4 N \sinh(\tfrac{\gamma}{2N})^2 \varphi(x),
\end{multline*}
and that
\begin{multline*}
\log \Bigl\{ \pi_{\exp( - \beta R)}^{\otimes 2} \Bigl[
\exp \Bigl( N \alpha \Phi_{- \alpha} \!\circ\!M' \Bigr) \Bigr] \Bigr\}
\\ \leq 2 \log \Bigl\{ \pi_{\exp( - \beta R)} \Bigl[ \exp \Bigl( N
\bigl[ \exp( \alpha) - 1 \bigr] M'(\cdot, \T) \Bigr) \Bigr] \Bigr\}
\\ \leq 2 x N \bigl[ \exp( \alpha) - 1\bigr]
\pi_{\exp[ - (\beta - x N [\exp(\alpha) - 1]) R]} \bigl[ R'(\cdot, \T) \bigr]
\\* + 2 x N \bigl[ \exp( \alpha) - 1 \bigr] \varphi(x).
\end{multline*}

Let us push further the investigation under the parametric
assumption that for some positive real constant $d$
\begin{equation}
\label{parametric}
\lim_{\beta \rightarrow + \infty} \beta \pi_{\exp( - \beta R)}\bigl[ R'( \cdot,
\T) \bigr] = d,
\end{equation}
This assumption will for instance hold true
with $d = \frac{n}{2}$ when $R: \Theta \rightarrow (0,1)$
is a smooth function defined on a compact subset $\Theta$ of $\RR^n$ that
reaches its minimum value on a finite number of non-degenerate (i.e. with
a positive definite Hessian) interior points of $\Theta$, and $\pi$
is absolutely continuous with respect to the
Lebesgue measure on $\Theta$ and has a smooth density.

In case of assumption \eqref{parametric}, if we restrict ourselves  to sufficiently large values of the
constants $\beta$, $\beta_1$, $\beta_2$, $\lambda_1$, $\lambda_2$ and $\gamma$
(the smaller of which is as a rule $\beta$, as we will see), we can
use the fact that for some (small) positive constant $\delta$, and
some (large) positive constant $A$,
\begin{equation}
\label{eq1.1.25}
\frac{d}{\alpha}(1 - \delta) \leq \pi_{\exp(- \alpha R)}\bigl[ R'(\cdot, \T)
\bigr] \leq
\frac{d}{\alpha}(1 + \delta), \qquad \alpha \geq A.
\end{equation}
Under this assumption,
\begin{multline*}
\pi_{\exp( - \beta_1 R)} \otimes \pi_{\exp( - \beta_2 R)}
\bigl[ \Psi_{- \frac{\gamma}{N}}(R', M') \bigr]
\\ \leq \tfrac{N}{\gamma} \sinh(\tfrac{\gamma}{N})
\bigl[ \tfrac{d}{\beta_1}(1 + \delta) - \tfrac{d}{\beta_2}(1 - \delta) \bigr]
\qquad \qquad  \\ \shoveright{+ \tfrac{2 x N}{\gamma}
\sinh(\tfrac{\gamma}{2N})^2 (1 + \delta)
\bigl[ \tfrac{d}{\beta_1}
+ \tfrac{d}{\beta_2} \bigr] + \tfrac{4N}{\gamma} \sinh(\tfrac{\gamma}{2N})^2
\varphi(x).}
\\
\shoveleft{C(\beta, \gamma) \leq d(1 + \delta) \log \Bigl( \tfrac{\beta}{\beta -
2xN\sinh(\frac{\gamma}{2N})^2} \Bigr)} \\
\shoveright{+ 2 x N \sinh(\tfrac{\gamma}{2N})^2
\tfrac{(1 + \delta)d}{\beta} + 4N \sinh(\tfrac{\gamma}{2N})^2 \varphi(x).}\\
\shoveleft{\log \Bigl\{ \pi_{\exp( - \beta R)}^{\otimes 2}
\Bigl[ \exp \Bigl( N \alpha \Phi_{- \alpha}\!\circ\!M' \Bigr) \Bigr] \Bigr\}
} \\ \leq 2xN\bigl[ \exp( \alpha) - 1 \bigr] \frac{d(1 + \delta)}{ \beta -
x N [\exp(\alpha) - 1]} + 2 N \bigl[ \exp( \alpha) - 1 \bigr] \varphi(x).
\end{multline*}
Thus with $\PP$ probability at least $1 - \eta$,
\begin{multline*}
B(\pi_{\exp( - \lambda_1 r)}, \beta, \gamma)
\leq - (\gamma - \lambda_1) \tfrac{N}{\gamma}
\sinh(\tfrac{\gamma}{N}) \tfrac{d}{\beta_2}( 1
- \delta)
\\ \shoveleft{+
(\gamma - \lambda_1) \biggl\{
\tfrac{N}{\gamma} \sinh(\tfrac{\gamma}{N}) \tfrac{(1+\delta)d}{\beta_1}
}\\*\shoveright{+ \tfrac{2xN}{\gamma} \sinh(\tfrac{\gamma}{2N})^2(1+\delta) \bigl[ \tfrac{d}{\beta_1}
+ \tfrac{d}{\beta_2} \bigr]
+ \tfrac{4N}{\gamma} \sinh(\tfrac{\gamma}{2N})^2 \varphi(x)
+ \frac{\log(\tfrac{7}{\eta})}{\gamma}}\\
+ \frac{4xN\sinh(\tfrac{\gamma}{2N})^2 \tfrac{(1+\delta)d}{\beta_1 -
2xN\sinh(\frac{\gamma}{2N})^2} + 4 N \sinh(\tfrac{\gamma}{2N})^2 \varphi(x)
+ \log(\frac{7}{\eta})}{\frac{N\lambda_1}{\beta_1}\sinh(\frac{\gamma}{N}) -
\gamma}\\
\shoveright{+ \frac{4xN\sinh(\tfrac{\gamma}{2N})^2 \tfrac{(1+\delta)d}{\beta_2 -
2xN\sinh(\frac{\gamma}{2N})^2} + 4 N \sinh(\tfrac{\gamma}{2N})^2 \varphi(x)
+ \log(\frac{7}{\eta})}{\gamma - \frac{N\lambda_2}{\beta_2}\sinh(\frac{\gamma}{N})}
\biggr\}}
\\ \shoveleft{+
\frac{2 \lambda_1}{\gamma}
\biggl\{ 4xN\sinh(\tfrac{\gamma}{2N})^2 \tfrac{(1+\delta)d}{\tfrac{N\lambda_1}{\gamma}
\sinh(\tfrac{\gamma}{N}) -
2xN\sinh(\frac{\gamma}{2N})^2}}\\
\shoveright{ + 4 N \sinh(\tfrac{\gamma}{2N})^2 \varphi(x)
+ \log(\tfrac{7}{\eta}) \biggr\}}\\
\shoveleft{+ \Bigl( 1 - \frac{\lambda_1}{\gamma} \Bigr) \Biggl\{
2 d(1+\delta) \Biggl( \tfrac{\lambda_1\sinh\bigl(\tfrac{\gamma}{N}\bigr)}{x \gamma
\Bigl[ \exp\Bigl(\frac{\log[\cosh(\frac{\gamma}{N})]}{1-\frac{\lambda_1}{\gamma}}
\Bigr)-1
\Bigr]}-1 \Biggr)^{-1}}\\\shoveright{ + 2N\Bigl[ \exp \Bigl( \tfrac{\log[\cosh(\frac{\gamma}{N})]}{1 -
\frac{\lambda_1}{\gamma}} \Bigr) - 1 \Bigr] \varphi(x)
\Biggr\}}\\
+ \Bigl(1 - \tfrac{\lambda_1}{\gamma} \Bigr)
\log(\tfrac{7}{\eta}) - \log\bigl[ \nu(\{\beta\}) \nu(\{\gamma\}) \epsilon\bigr]\\
\shoveleft{+ \frac{1 - \frac{\lambda_1}{\gamma}}{ \frac{N \lambda_2}{\beta \gamma}
\tanh(\frac{\gamma}{N}) - 1} \Biggl\{
\frac{2 \lambda_2}{\gamma}
\biggl\{ 4xN\sinh(\tfrac{\gamma}{2N})^2 \tfrac{(1+\delta)d}{\tfrac{N\lambda_2}{\gamma}
\sinh(\tfrac{\gamma}{N}) -
2xN\sinh(\frac{\gamma}{2N})^2}}\\
\shoveright{+ 4 N \sinh(\tfrac{\gamma}{2N})^2 \varphi(x)
+ \log(\tfrac{7}{\eta}) \biggr\}}\\
\shoveleft{+ \Bigl( 1 - \frac{\lambda_2}{\gamma} \Bigr) \Biggl[
2 d(1+\delta) \Biggl( \tfrac{\lambda_2\sinh\bigl(\tfrac{\gamma}{N}\bigr)}{x \gamma
\Bigl[ \exp\Bigl(\frac{\log[\cosh(\frac{\gamma}{N})]}{1-\frac{\lambda_2}{\gamma}}
\Bigr)-1
\Bigr]}-1 \Biggr)^{-1}} \\
\shoveright{+ 2N\Bigl[ \exp \Bigl( \tfrac{\log[\cosh(\frac{\gamma}{N})]}{1 -
\frac{\lambda_2}{\gamma}} \Bigr) - 1 \Bigr] \varphi(x)
\Biggr]\qquad\quad}\\
+ \Bigl(1 - \tfrac{\lambda_2}{\gamma} \Bigr)
\log(\tfrac{7}{\eta}) - \log\bigl[ \nu(\beta) \nu(\gamma) \epsilon\bigr]
\Biggr\}.
\end{multline*}

Now let us choose for simplicity
$\beta_2 = 2 \lambda_2 = 4 \beta$, $\beta_1 = \lambda_1 / 2 = \gamma / 4$,
and let us introduce the notation
\begin{align*}
C_1 & = \frac{N}{\gamma}\sinh(\frac{\gamma}{N}),\\
C_2 & = \frac{N}{\gamma} \tanh(\frac{\gamma}{N}),\\
C_3 & = \frac{N^2}{\gamma^2}
\bigl[ \exp( \frac{\gamma^2}{N^2} ) - 1 \bigr]\\
\text{and }\quad
C_4 & = \frac{2 N^2(1 - \frac{2 \beta}{\gamma})}{\gamma^2}
\Bigl[ \exp \Bigl( \frac{\gamma^2}{2 N^2 (1 - \frac{2 \beta}{\gamma})}
\Bigr) - 1 \Bigr],
\end{align*}
to obtain
\begin{multline*}
B(\pi_{\exp( - \lambda_1 r)}, \beta, \gamma) \leq
- \frac{C_1 \gamma}{8 \beta} (1 - \delta)d
\\ + \frac{C_1 \gamma}{2} \biggl\{
 \tfrac{4(1+\delta)d}{\gamma} + x \tfrac{\gamma}{2 N}(1+\delta)
\bigl[ \tfrac{4 d}{\gamma} + \tfrac{d}{4\beta} \bigr]
+ \tfrac{\gamma}{N} \varphi(x) \biggr\} +
\tfrac{1}{2} \log\bigl(\tfrac{7}{\eta}\bigr)\\*
\qquad + \frac{1}{2C_1-1}  \Bigl[(1+\delta) d \Bigl( \tfrac{N}{2xC_1\gamma} -1 \Bigr)^{-1}
+ C_1 \frac{\gamma^2}{2N} \varphi(x) + \tfrac{1}{2} \log(\tfrac{7}{\eta}) \Bigr]
\\*\hfill \hfill \hfill + \frac{1}{2 - C_1} \biggl[ 2 (1+\delta)d \Bigl( \tfrac{8 N \beta}{x C_1 \gamma^2}
- 1\Bigr)^{-1} + C_1 \frac{\gamma^2}{N} \varphi(x) + \log(\tfrac{7}{\eta}) \biggr]
\hfill \\*
\shoveright{+ \frac{2 x \gamma (1 + \delta) d}{N - x \gamma} + C_1 \tfrac{\gamma^2}{N} \varphi(x)
+ \log( \tfrac{7}{\eta})} \\*
\shoveright{+ d(1+\delta)\frac{x \gamma}{N} \biggl( \frac{C_1}{2
C_3 } - \frac{x \gamma}{N} \biggr)^{-1} +  \frac{\gamma^2}{N} C_3
\varphi(x) + \frac{\log(\frac{7}{\eta})}{2} -
\log\bigl[ \nu(\beta) \nu(\gamma) \epsilon\bigr]}\\*
\shoveleft{\qquad + \Bigl( 4 C_2  - 2\Bigr)^{-1}
\Biggl\{ \frac{4 \beta}{\gamma} \biggl\{
x \frac{\gamma^2}{N} C_1 (1 + \delta) d \Bigl(
2 \beta C_1 - x C_1 \frac{\gamma^2}{2N} \Bigr)^{-1}} \\\shoveright{
+ \tfrac{\gamma^2}{N} \varphi(x)
+ \log(\tfrac{7}{\eta})\biggr\}\quad }
\\* \shoveleft{\qquad + \Bigl(1 - \frac{2 \beta}{\gamma} \Bigr) \biggl\{
2 d (1 + \delta) \frac{x \gamma}{N}
\biggl[ \frac{4   \beta C_1}{
\gamma C_4}\biggl(1 - \frac{2 \beta}{\gamma}\biggr) - \frac{x \gamma}{N}
\biggr]^{-1}}\\ \shoveright{
+ \frac{\gamma^2}{N(1 - \frac{2 \beta}{\gamma})} C_4 \varphi(x)
\biggr\}\quad }
\\* + \Bigl( 1 - \tfrac{2 \beta}{\gamma} \Bigr) \log(\tfrac{7}{\eta}) - \log
\bigl[ \nu(\beta) \nu(\gamma) \epsilon \bigr]
\Biggr\}.
\end{multline*}
This simplifies to
\begin{multline*}
B( \pi_{\exp( - \lambda_1 r)}, \beta, \gamma) \leq
- \frac{C_1}{8}(1- \delta)d \frac{\gamma}{\beta}
\\ + 2 C_1(1 + \delta) d + \log(\tfrac{7}{\eta})
\biggl[ 2  +  \tfrac{3 C_1}{(4C_1-2)(2-C_1)}
+  \frac{ 1 + \frac{2 \beta}{\gamma}}{4C_2 - 2}
\biggr] \\ \hfill - \bigl( 1 + \tfrac{1}{4 C_2 - 2} \bigr)
\log\bigl[ \nu(\beta) \nu( \gamma) \epsilon\bigr]\qquad
\\\qquad  + \frac{(1 + \delta) d x \gamma}{N} \biggl\{
C_1 + \tfrac{1}{2 C_1 - 1} \Bigl(
\tfrac{1}{2C_1} - \tfrac{\gamma x}{N} \Bigr)^{-1}
\hfill \\\hfill + 2 \Bigl( 1 - \tfrac{\gamma x}{N} \Bigr)^{-1}
+ \Bigl( \tfrac{C_1}{2 C_3} -
\tfrac{\gamma x}{N} \Bigr)^{-1} + \tfrac{4C_1\beta}{\gamma(4C_2-2)}
\biggr\}\qquad \\
\qquad + \frac{(1 + \delta) d x \gamma^2}{N \beta} \biggl\{
\tfrac{C_1}{16} + \tfrac{2}{2-C_1} \Bigl( \tfrac{8}{C_1} -
\tfrac{x \gamma^2}{N \beta} \Bigr)^{-1} \hfill \\
\hfill +
\Bigl(1 - \tfrac{2 \beta}{\gamma} \Bigr) \tfrac{1}{2C_2 -1}
\Bigl[ \tfrac{4C_1}{C_4}\Bigl(1 - \tfrac{2 \beta}{\gamma}\Bigr)
- \tfrac{\gamma^2 x}{\beta N} \Bigr]^{-1}
\biggr\} \qquad
\\
+ \frac{\gamma^2}{N} \varphi(x) \biggl\{
\tfrac{3 C_1}{2} + \tfrac{C_1}{4C_1 - 2} + \tfrac{C_1}{2 - C_1} + C_3
+ \tfrac{4 \beta}{\gamma( 4 C_2 - 2)} + \tfrac{C_4}{4 C_2 - 2}
\biggr\}.
\end{multline*}

This shows that there exist universal positive real constants $A_1$, $A_2$, $B_1$, $B_2$, $B_3$,
and $B_4$
such that as soon as $\frac{\gamma \max\{x, 1\}}{N} \leq A_1 \frac{\beta}{\gamma}
\leq A_2$,
\begin{multline*}
B( \pi_{\exp( - \lambda_1 r) }, \beta, \gamma) \leq
- B_1 (1 - \delta) d \frac{\gamma}{\beta} + B_2 (1 + \delta) d \\
- B_3 \log\bigl[
\nu(\beta) \nu(\gamma) \epsilon\,\eta\bigr]
+ B_4 \frac{\gamma^2}{N} \varphi(x).
\end{multline*}
Thus $\pi_{\exp( - \lambda_1 r)}(R)
\leq \pi_{\exp( - \beta R)}(R) \leq \inf_{\Theta} R + \frac{ (1 + \delta) d}{\beta}$
as soon as
$$
\frac{\beta}{\gamma} \leq \frac{ B_1}{
B_2\frac{(1 + \delta)}{(1 - \delta)} + \frac{B_4 \frac{\gamma^2}{N} \varphi(x)
- B_3 \log[\nu(\beta) \nu(\gamma) \epsilon \eta]}{(1-\delta) d}}.
$$

Choosing some real ratio $\alpha > 1$,
we can now make the above result uniform for any
\begin{equation}
\label{eq1.1.27}
\beta, \gamma \in
\Lambda_{\alpha} \overset{\text{def}}{=}
\Bigl\{ \alpha^k ; k \in \NN, 0 \leq k < \tfrac{\log(N)}{\log(\alpha)} \Bigr\},
\end{equation}
by substituting $\nu(\beta)$ and $\nu(\gamma)$
with $\frac{\log(\alpha)}{\log(\alpha N)}$ and $- \log(\eta)$ with
$ - \log( \eta) + 2\times \log \left[ \frac{\log( \alpha N)}{\log(\alpha)} \right]$.

Taking  $\eta = \epsilon$ for simplicity,
we can summarize our  result in
\begin{thm}
\mypoint
\label{thm1.50}
There exist positive real universal constants
$A$, $B_1$, $B_2$, $B_3$ and $B_4$ such that
for any positive real constants $\alpha > 1$, $d$ and $\delta$, for any
prior distribution $\pi \in \C{M}_+^1(\Theta)$,
with
$\PP$ probability at least $1 - \epsilon$,
for any $\beta, \gamma
\in \Lambda_{\alpha}$ (where $\Lambda_{\alpha}$ is defined by equation
\eqref{eq1.1.27} above) such that
$$
\sup_{\beta' \in \RR, \beta' \geq \beta}
\biggl\lvert \frac{\beta'}{d} \bigl[
\pi_{\exp( - \beta' R)}(R) - \inf_{\Theta} R \bigr]  - 1 \biggr\rvert
\leq \delta
$$
and such that also for some positive real parameter $x$
$$
\frac{\gamma \max\{x, 1\}}{N} \leq \frac{A \beta}{\gamma} \text{ and }
\frac{\beta}{\gamma} \leq
\frac{B_1}{B_2 \frac{(1 + \delta)}{(1 - \delta)}
+ \frac{ B_4 \frac{\gamma^2}{N}\varphi(x) - 2 B_3 \log(\epsilon) + 4
B_3 \log \bigl[ \frac{\log(N)}{\log(\alpha)}\bigr]}{(1 - \delta) d}},
$$
the bound $B(\pi_{\exp( - \frac{\gamma}{2} r)}, \beta, \gamma)$
given by Theorem \ref{thm1.1.43} on page \pageref{thm1.1.43}
in the case where we have chosen $\nu$
to be the uniform probability measure on $\Lambda_{\alpha}$,
satisfies
$B(\pi_{\exp( - \frac{\gamma}{2} r)}, \beta,\break \gamma)
\leq 0,$ proving that $\w{\beta}(\pi_{\exp( - \frac{\gamma}{2} r)})
\geq \beta$ and therefore that
$$
\pi_{\exp( - \gamma \frac{r}{2} )}(R) \leq \pi_{\exp ( - \beta R)}(R)
\leq \inf_{\Theta} R + \frac{(1 + \delta) d}{\beta}.
$$
\end{thm}

What is important in this result is that we do not only bound
$\pi_{\exp( - \frac{\gamma}{2} r)}(R)$, but also
$B(\pi_{\exp( - \frac{\gamma}{2} r)}, \beta, \gamma)$,
and that we do it uniformly on a grid of values of $\beta$ and
$\gamma$, showing that we can indeed
set the constants $\beta$ and $\gamma$
adaptively using the empirical bound
$B( \pi_{\exp( - \frac{\gamma}{2} r)}, \beta, \gamma)$.

Let us see what we get under the margin assumption \myeq{eq1.1.17Bis}.
When $\kappa = 1$, we have $\varphi(c^{-1}) \leq 0$, leading to
\begin{cor}\mypoint
Assuming that the margin
assumption \myeq{eq1.1.17Bis} is
satisfied for $\kappa = 1$, that $R: \Theta \rightarrow (0,1)$
is independent of $N$ (which is the case for instance when
$\PP = P^{\otimes N}$), and is such that
$$
\lim_{\beta' \rightarrow + \infty} \beta'
\bigl[ \pi_{\exp( - \beta'
R)}(R) - \inf_{\Theta} R \bigr] = d,
$$
there are universal positive real constants
$B_5$ and $B_6$
and $N_1 \in \NN$
such that
for any $N \geq N_1$,
with $\PP$ probability at least $1 - \epsilon$
$$
\pi_{\exp( - \widehat{\gamma}\frac{r}{2} )}(R) \leq
\inf_{\Theta} R + \frac{ B_5  d}{c N}
\left[1 + \frac{B_6}{d}  \log \biggl( \frac{\log(N)}{
\epsilon } \biggr) \right]^2,
$$
where $\w{\gamma} \in \arg\max_{\gamma \in \Lambda_2} \max \bigl\{ \beta \in \Lambda_2
; B(\pi_{\exp( - \gamma \frac{r}{2})}, \beta, \gamma) \leq 0 \bigr\}$,
where $\Lambda_2$ is defined by equation \myeq{eq1.1.27}, and $B$ is
the bound of Theorem \thmref{thm1.1.43}.
\end{cor}

When $\kappa > 1$, $\varphi(x) \leq (1 - \kappa^{-1}) \bigl( \kappa c x \bigr)^{-
\frac{1}{\kappa -1}}$, and we can choose $\gamma$ and $x$ such that
$\frac{\gamma^2}{N} \varphi(x) \simeq d$ to prove
\begin{cor}\mypoint
\label{cor1.52}
Assuming that the margin assumption \myeq{eq1.1.17Bis} is satisfied
for some exponent $\kappa > 1$, that $R: \Theta \rightarrow (0,1)$
is independent of $N$ (which is for instance the case when
$\PP = P^{\otimes N}$), and is such that
$$
\lim_{\beta' \rightarrow + \infty} \beta'
\bigl[ \pi_{\exp ( - \beta' R)}(R) - \inf_{\Theta} R \bigr] = d,
$$
there are universal positive constants
$B_7$ and $B_8$
and $N_1 \in \NN$ such that for any $N \geq N_1$, with $\PP$
probability at least $1 - \epsilon$,
$$
\pi_{\exp( - \widehat{\gamma} \frac{r}{2} )}(R)
\leq \inf_{\Theta} R +  B_7
c^{ - \frac{1}{2 \kappa -1}}
\biggl[ 1 + \frac{B_8}{d} \log
\biggl( \frac{\log(N)}{\epsilon} \biggr)
\biggr]^{\frac{2 \kappa}{2 \kappa - 1}} \left(
\frac{d}{N} \right)^{ \frac{\kappa}{2 \kappa - 1}},
$$
where $\widehat{\gamma} \in \arg \max_{\gamma \in \Lambda_2}
\max \bigl\{ \beta \in \Lambda_2; B(\pi_{\exp( - \gamma \frac{r}{2})},
\beta, \gamma) \leq 0 \bigr\}$, $\Lambda_2$ being defined by equation
\myeq{eq1.1.27} and $B$ by Theorem \thmref{thm1.1.43}.
\end{cor}

We find the same rate of convergence as in Corollary
\thmref{cor1.1.23}, but this
time, we were able to provide an empirical posterior distribution
$\pi_{\exp( - \w{\gamma} \frac{r}{2})}$
which achieves this rate adaptively in all the parameters
(meaning in particular that we do not need to know $d$,
$c$ or $\kappa$). Moreover, as
already mentioned, the power
of $N$ in this rate of convergence is known to be optimal
in the worst case (see \cite{Mammen,Tsybakov,Tsybakov2}, and
more specifically in \cite{Audibert2} --- downloadable from
its author's web page --- Theorem 3.3, page 132).

\subsection{Estimating the divergence of a posterior
with respect to a Gibbs prior}
Another interesting question is to estimate
$\C{K} \bigl[ \rho, \pi_{\exp ( - \beta R)} \bigr]$
using relative deviation inequalities.
We follow here an idea to be found first
in \cite[page 93]{Audibert2}.
Indeed, combining equation \myeq{eq1.1.17} with
equation \myeq{eq1.26}, we see that
for any positive real parameters $\beta$ and $\lambda$,
with $\PP$ probability at least $1 - \epsilon$, for any
posterior distribution $\rho: \Omega \rightarrow \C{M}_+^1(\Theta)$,
\begin{multline*}
\C{K}\bigl[\rho, \pi_{\exp( - \beta R)}\bigr]
\leq \frac{\beta}{N \tanh(\frac{\gamma}{N})} \biggl\{
\gamma
\bigl[ \rho(r) - \pi_{\exp(- \beta R)}(r) \bigr]
\\ \hfill + N \log \bigl[ \cosh(\tfrac{\gamma}{N}) \bigr]
\rho \otimes \pi_{\exp( - \beta R)}
(m') \qquad \\\hfill  + \C{K}\bigl[ \rho, \pi_{\exp( - \beta R)}\bigr]
- \log(\epsilon) \biggr\} + \C{K}(\rho, \pi) - \C{K} \bigl[
\pi_{\exp( - \beta R)}, \pi \bigr]\quad
\\ \leq \C{K} \bigl[ \rho, \pi_{\exp [ - \frac{\beta \gamma}{
N \tanh( \frac{\gamma}{N})} r]} \bigr] + \frac{\beta}{N
\tanh(\frac{\gamma}{N})} \Bigl\{ \C{K}\bigl[
\rho, \pi_{\exp( - \beta R)}\bigr] -
\log(\epsilon)  \Bigr\} \\ +
\log \biggl[ \pi_{\exp [ - \frac{\beta \gamma}{N \tanh(
\frac{\gamma}{N})} r]}
\Bigl\{ \exp \Bigl[  \frac{\beta}{\tanh(\frac{\gamma}{N})}
\log \bigl[ \cosh( \tfrac{\gamma}{N}) \bigr]
\rho(m')\Bigr] \Bigr\} \biggr].
\end{multline*}
We thus obtain
\begin{thm}
\mypoint
\label{thm1.1.37}
For any positive real constants $\beta$ and $\gamma$ such
that $\beta < N \times\tanh ( \frac{\gamma}{N})$,
with $\PP$ probability at least $1 - \epsilon$, for any
posterior distribution $\rho: \Omega \rightarrow \C{M}_+^1(\Theta)$,
\begin{multline*}
\C{K}\bigl[ \rho, \pi_{\exp( - \beta R)}\bigr]
\leq \left( 1 - \frac{\beta}{N}\tanh\left(\frac{\gamma}{N}\right)^{-1}\right)^{-1}
\\ \times \Biggl\{ \C{K}\bigl[ \rho, \pi_{\exp [ - \frac{\beta\gamma}{N}
\tanh(\frac{\gamma}{N})^{-1}r]}
\bigr] - \frac{\beta}{N \tanh(\frac{\gamma}{N})} \log(\epsilon)
\\ + \log \Bigl\{ \pi_{\exp[ -
\frac{\beta \gamma}{N} \tanh(\frac{\gamma}{N})^{-1} r]} \Bigl[
\exp \bigl\{ \beta \tanh(\tfrac{\gamma}{N})^{-1} \log[\cosh(\tfrac{\gamma}{N})]
\rho(m') \bigr\} \Bigr] \Bigr\} \Biggr\}.
\end{multline*}
\end{thm}

This theorem provides another way of measuring over-fitting,
since it gives an upper bound for $\C{K}\bigl[
\pi_{\exp[ - \frac{\beta \gamma}{N}
\tanh(\frac{\gamma}{N})^{-1} r]}, \pi_{\exp( - \beta R)} \bigr]$.
It may be used in combination with Theorem \thmref{thm2.7}
as an alternative to Theorem
\thmref{thm1.1.17}.
It will also be used in the next section.

An alternative parametrization of the same result providing a simpler
right-hand side is also useful:
\begin{cor}\mypoint
\label{cor1.60}
For any positive real constants $\beta$ and $\gamma$ such that $
\beta < \gamma$, with $\PP$ probability at least $1 - \epsilon$, for any
posterior distribution $\rho: \Omega \rightarrow \C{M}_+^1(\Theta)$,
\begin{multline*}
\C{K}\bigl[ \rho, \pi_{\exp[ - N \frac{\beta}{\gamma} \tanh(\frac{\gamma}{N}) R]}
\bigr] \leq \biggl(1 - \frac{\beta}{\gamma} \biggr)^{-1}
\Biggl\{ \C{K}\bigl[ \rho, \pi_{\exp( - \beta r)}\bigr] - \frac{\beta}{\gamma}
\log( \epsilon) \\ +
\log \Bigl\{ \pi_{\exp( - \beta r)} \Bigl[ \exp \bigl\{
N \tfrac{\beta}{\gamma} \log \bigl[ \cosh(\tfrac{\gamma}{N})\bigr] \rho
(m') \bigr\} \Bigr] \Bigr\} \Biggr\}.
\end{multline*}
\end{cor}

\section[Playing with two posterior and two local prior distributions\hspace*{-30pt}]{Playing with two posterior and two local prior distributions}

\subsection{Comparing two posterior distributions}

Estimating the effective temperature of an estimator provides an efficient
way to tune parameters in a model with parametric behaviour. On the other
hand, it will not be fitted to choose between different models, especially
when they are nested, because as we already saw in the case
when $\Theta$ is a union of nested models, the prior distribution $\pi_{\exp
( - \beta R)}$ does not provide an efficient localization of the parameter
in this case, in the sense that $\pi_{\exp( - \beta R)}(R)$
does not go down to $\inf_{\Theta} R$ at the desired rate when
$\beta$ goes to $+ \infty$, requiring a  resort to partial localization.

Once some estimator (in the form of a posterior distribution) has been
chosen in each sub-model, these estimators can be compared between themselves
with the help of the relative bounds that we will establish in this section.
It is also possible to choose several estimators in each sub-model,
to tune parameters in the same time (like the inverse temperature parameter if
we decide to use Gibbs posterior distributions in each sub-model).

From equation (\ref{eq1.1.15}
page \pageref{eq1.1.15}) (slightly modified by replacing $\pi \otimes \pi$
with $\pi^1 \otimes \pi^2$), we easily obtain
\begin{thm}
\mypoint
\label{thm1.1.38}
For any positive real constant $\lambda$,
for any prior distributions $\pi^1, \pi^2 \in \C{M}_+^1(\Theta)$,
with $\PP$ probability at least $1 - \epsilon$,
for any posterior distributions $\rho_1$ and $\rho_2:
\Omega \rightarrow \C{M}_+^1(\Theta)$,
\begin{multline*}
- N \log \Bigl\{ 1 - \tanh\bigl( \tfrac{\lambda}{N} \bigr)
\Bigl[ \rho_2(R) - \rho_1(R) \Bigr] \Bigr\}
\leq \lambda \bigl[ \rho_2(r) - \rho_1(r) \bigr]
\\ + N \log \bigl[ \cosh \bigl( \tfrac{\lambda}{N} \bigr) \bigr]
\rho_1 \otimes \rho_2 (m') \\ + \C{K}\bigl( \rho_1, \pi^1 \bigr)
+ \C{K}\bigl( \rho_2, \pi^2\bigr) - \log(\epsilon).
\end{multline*}
\end{thm}

This is where the entropy bound
of the previous section enters into the game, providing a localized version
of Theorem \thmref{thm1.1.38}.
We will use the notation
\begin{equation}
\label{eq1.34Bis}
\Xi_{a} (q) = \tanh(a)^{-1} \bigl[ 1 -
\exp( - aq) \bigr] \leq \frac{a}{\tanh(a)}q, \qquad a, q \in \RR.
\end{equation}
\begin{thm}
\mypoint
\label{thm1.1.39}
For any $\epsilon \in )0,1($, any sequence of prior distributions $(\pi^i)_{i \in \NN } \in
\C{M}_+^1(\Theta)^{\NN}$,
any probability distribution $\mu$ on $\NN$,
any atomic probability distribution $\nu$ on $\RR_+$,
with $\PP$ probability at least $1 - \epsilon$, for any posterior distributions
$\rho_1, \rho_2: \Omega \rightarrow \C{M}_+^1(\Theta)$,
\begin{multline*}
\hfill \rho_2(R) - \rho_1(R) \leq B(\rho_1, \rho_2), \text{ where} \hfill
\\
\shoveleft{B(\rho_1, \rho_2) = \inf_{\lambda, \beta_1 < \gamma_1, \beta_2 <
\gamma_2 \in \RR_+, i, j \in \NN} \Xi_{\frac{\lambda}{N}}  \Biggl\{
\bigl[ \rho_2(r) - \rho_1(r) \bigr]}\\\shoveright{ + \tfrac{N}{\lambda} \log
\bigl[ \cosh(
\tfrac{\lambda}{N}) \bigr] \rho_1 \otimes \rho_2(m')
}\\\shoveleft{ + \frac{1}{\lambda \Bigl(1 - \frac{\beta_1}{\gamma_1}\Bigr)}
\biggl\{ \C{K} \bigl[ \rho_1, \pi^i_{\exp( - \beta_1 r)}\bigr]
}\\ + \log \Bigl\{ \pi^i_{\exp( - \beta_1 r)} \Bigl[ \exp \bigl\{
\beta_1 \tfrac{N}{\gamma_1}
\log \bigl[ \cosh(\tfrac{\gamma_1}{N})\bigr] \rho_1(m') \bigr\}
\Bigr] \Bigr\} \\ \shoveright{ - \frac{\beta_1}{\gamma_1} \log \bigl[
\nu(\gamma_1) \bigr] \biggr\} \quad}
\\ \shoveleft{+ \frac{1}{\lambda \Bigl( 1 - \frac{\beta_2}{\gamma_2} \Bigr)} \biggl\{
\C{K} \bigl[ \rho_2, \pi^j_{\exp( - \beta_2 r)}\bigr]
}\\ + \log \Bigl\{ \pi^j_{\exp( - \beta_2 r)} \Bigl[ \exp \bigl\{ \beta_2
\tfrac{N}{\gamma_2}
\log \bigl[ \cosh(\tfrac{\gamma_2}{N})\bigr] \rho_2(m') \bigr\}
\Bigr] \Bigr\} \\
\shoveright{ - \frac{\beta_2}{\gamma_2} \log \bigl[
\nu(\gamma_2) \bigr] \biggr\}\quad }
\\ \shoveleft{- \Bigl[ \bigl( \tfrac{\gamma_1}{\beta_1} - 1 \bigr)^{-1}
+ \bigl( \tfrac{\gamma_2}{\beta_2} - 1 \bigr)^{-1} + 1 \Bigr]
} \frac{
\log\bigl[3^{-1} \nu(\beta_1) \nu(\beta_2)
\nu(\lambda) \mu(i) \mu(j) \epsilon\bigr]}{\lambda}
\Biggr\}.
\end{multline*}
\end{thm}

The sequence of prior distributions $(\pi^i)_{i \in \NN}$
should be understood
to be typically supported by subsets of $\Theta$ corresponding to
parametric sub-models, that is sub-models for which it
is reasonable to expect that
$$
\lim_{\beta \rightarrow
+ \infty} \beta \bigl[ \pi^i_{\exp( - \beta R)}(R) -
\ess \inf_{\pi^i} R \bigr]
$$
exists and is positive and finite.
As there is no reason why the bound $B(\rho_1, \rho_2)$ provided by
the previous theorem should be sub-additive (in the sense that
$B(\rho_1, \rho_3) \leq B(\rho_1, \rho_2) + B(\rho_2, \rho_3)$),
it is adequate to
consider some workable subset $\C{P}$
of posterior distributions (for instance the distributions of
the form $\pi^i_{\exp( - \beta r)}$, $i \in \NN$, $\beta \in \RR_+$),
and to define the sub-additive chained bound
\newcommand{\TB}{\widetilde{B}}
\begin{multline}
\label{eq1.37Bis}
\TB (\rho, \rho') = \inf \Biggl\{
\sum_{k=0}^{n-1} B(\rho_k, \rho_{k+1});\, n \in \NN^*,
(\rho_k)_{k=0}^{n} \in \C{P}^{n+1},\\  \rho_0 = \rho,
\rho_n = \rho' \Biggr\}, \quad \rho, \rho' \in \C{P}.
\end{multline}
\begin{prop}\mypoint
\label{prop1.1.54}
With $\PP$ probability at least $1 - \epsilon$,
for any posterior distributions $\rho_1, \rho_2
\in \C{P}$,
$
\rho_2(R) - \rho_1(R) \leq \TB(\rho_1, \rho_2).
$
Moreover for any
posterior distribution $\rho_1 \in \C{P}$,
any posterior distribution $\rho_2 \in \C{P}$ such that
$\TB(\rho_1, \rho_2) = \inf_{\rho_3 \in \C{P}} \TB(\rho_1, \rho_3)$
is unimprovable with the help of $\TB$ in $\C{P}$
in the sense that $\inf_{\rho_3 \in \C{P}}
\TB(\rho_2, \rho_3) \geq 0$.
\end{prop}
\begin{proof} The first assertion is a direct consequence of the
previous theorem, so only the second assertion requires a proof: for
any $\rho_3 \in \C{P}$, we deduce from
the optimality of $\rho_2$ and the sub-additivity of $\TB$ that
$$
\TB(\rho_1,\rho_2) \leq \TB(\rho_1, \rho_3) \leq \TB(\rho_1, \rho_2) +
\TB(\rho_2, \rho_3).
$$
\end{proof}

This proposition provides a way to improve a posterior distribution
$\rho_1 \in \C{P}$ by choosing $\rho_2 \in \arg\min_{\rho \in \C{P}}
\TB(\rho_1, \rho)$ whenever $\TB(\rho_1, \rho_2) < 0$.
This improvement is proved by Proposition \ref{prop1.1.54}
to be one-step: the obtained improved posterior $\rho_2$
cannot be improved again using the same technique.

Let us give some examples of possible starting
distributions $\rho_1$ for this improvement scheme: $\rho_1$ may be chosen as
the best posterior Gibbs distribution
according to Proposition \thmref{prop1.1.37}.
More precisely, we may build
from the prior distributions $\pi^i$, $i \in \NN$,
a global prior $\pi = \sum_{i \in \NN} \mu(i) \pi^i$.
We can then define the estimator of the inverse effective
temperature as in Proposition \thmref{prop1.1.37}
and choose $\rho_1 \in \arg \min_{\rho \in \C{P}} \w{\beta}(\rho)$,
where $\C{P}$ is as suggested above the set of posterior
distributions
$$
\C{P} = \Bigl\{ \pi^i_{\exp( - \beta r)};\, i \in \NN, \beta \in \RR_+ \Bigr\}.
$$
This starting point $\rho_1$ should already be pretty good,
at least in an asymptotic perspective, the only
gain in the rate of convergence to be expected bearing
on spurious $\log(N)$ factors.

\subsection{Elaborate uses of relative bounds between posteriors}

More elaborate uses of relative bounds are described in
the third section of the second chapter of \cite{Audibert2}, where an algorithm
is proposed and analysed, which allows one to use relative bounds
between two posterior distributions as a stand-alone estimation
tool.

Let us give here some alternative way to address this issue.
We will assume for simplicity and without great loss of generality
that the working set of posterior distributions
$\C{P}$ is finite (so that among other things any ordering of it
has a first element).

It is natural to define the estimated complexity of any
given posterior distribution $\rho \in \C{P}$ in our
working set as the bound for $ \inf_{i \in \NN}
\C{K}(\rho, \pi^i)$
used in Theorem \thmref{thm1.1.38}.
This leads to set (given some confidence level $1 - \epsilon$)
\begin{multline*}
\C{C}(\rho) = \inf_{\beta < \gamma \in \RR_+, i \in \NN}
\biggl(1 - \frac{\beta}{\gamma}\biggr)^{-1} \biggl\{
\C{K}\bigl[\rho, \pi^i_{\exp( - \beta r)}\bigr]
\\ + \log \Bigl\{ \pi^i_{\exp( - \beta r)}\Bigl[ \exp \bigl\{
\beta \tfrac{N}{\gamma} \log \bigl[ \cosh ( \tfrac{\gamma}{N} )
\bigr] \rho(m') \bigr\} \Bigr] \Bigr\}
\\
- \frac{\beta}{\gamma} \log \bigl[
3^{-1} \nu(\gamma) \nu(\beta) \mu(i) \epsilon \bigr] \biggr\}.
\end{multline*}
Let us moreover call $\gamma(\rho)$, $\beta(\rho)$ and
$i(\rho)$ the values achieving this infimum, or nearly
achieving it, which requires a slight change of the
definition of $\C{C}(\rho)$ to take this modification
into account. For the sake of simplicity, we can assume
without substantial loss of generality that
the supports of $\nu$ and $\mu$ are large but
finite, and thus that the minimum is reached.

To understand how this notion of complexity comes
into play, it may be interesting to keep in mind that
for any posterior distributions $\rho$ and $\rho'$
we can write the bound in Theorem \thmref{thm1.1.39} as
\begin{equation}
\label{eq2.13}
B(\rho, \rho') = \inf_{\lambda \in \RR_+}
\Xi_{\frac{\lambda}{N}} \bigl[
\rho'(r) - \rho(r) + S_{\lambda}(\rho, \rho') \bigr],
\end{equation}
where
\begin{multline*}
S_{\lambda}(\rho, \rho') = S_{\lambda}(\rho', \rho)
\leq  \frac{N}{\lambda} \log \bigl[
\cosh(\tfrac{\lambda}{N}) \bigr]
\rho \otimes \rho'(m') +
\frac{\C{C}(\rho) + \C{C}(\rho')}{\lambda}
- \frac{\log (3^{-1} \epsilon )}{\lambda}
\\
-
\frac{\log \bigl\{ \nu \bigl[ \beta(\rho) \bigr] \mu
\bigl[ i (\rho) \bigr] \bigr\}}{\lambda\bigl( 1 - \tfrac{\beta(\rho')}{\gamma(\rho')} \bigr)}
-
\frac{\log \bigl\{ \nu \bigl[ \beta(\rho') \bigr] \mu
\bigl[ i (\rho') \bigr] \bigr\}}{\lambda\bigl( 1 - \tfrac{\beta(\rho)}{\gamma(\rho)} \bigr)}
\\
- \Bigl[ \bigl( \tfrac{\gamma(\rho)}{\beta(\rho)} - 1 \bigr)^{-1} + \bigl( \tfrac{\gamma(\rho')}{\beta(\rho')} -
1 \bigr)^{-1} + 1 \Bigr]
\frac{\log \bigl[
\nu(\lambda) \bigr]}{
\lambda}.
\end{multline*}
(Let us recall that the function $\Xi$ is defined by equation
\myeq{eq1.34Bis}.)
Thus for any $\rho, \rho'$ such that $B(\rho',\rho) > 0$,
we can deduce from the monotonicity of $\Xi_{\frac{\lambda}{N}}$
that
$$
\rho'(r) - \rho(r) \leq \inf_{\lambda \in \RR_+} S_{\lambda}(\rho, \rho'),
$$
proving that the left-hand side is small, and consequently
that $B(\rho, \rho')$ and
its chained counterpart defined
by equation \myeq{eq1.37Bis} are small:
$$
\TB(\rho, \rho')
\leq B(\rho, \rho') \leq \inf_{\lambda \in \RR_+} \Xi_{\frac{\lambda}{N}}
\bigl[ 2 S_{\lambda}(\rho, \rho') \bigr].
$$
It is also worth noticing that
$B(\rho, \rho')$ and $\TB(\rho, \rho')$ are
upper bounded in terms of variance and complexity
only.

The presence of the ratios $\tfrac{\gamma(\rho)}{\beta(\rho)}$ should
not be obnoxious, since their values should be automatically tamed
by the fact that $\beta(\rho)$ and $\gamma(\rho)$ should make
the estimate of the complexity of $\rho$ optimal.

As an alternative, it is possible to restrict to set
of parameter values $\beta$ and $\gamma$
such that, for some fixed constant $\zeta > 1$,
the ratio $\frac{\gamma}{\beta}$ is bounded
away from $1$ by the inequality $\frac{\gamma}{\beta}
\geq \zeta$. This leads to an alternative definition
of $\C{C}(\rho)$:
\begin{multline*}
\C{C}(\rho) = \inf_{\gamma \geq \zeta \beta \in \RR_+, i \in \NN}
\biggl( 1 - \frac{\beta}{\gamma}
\biggr)^{-1}
\biggl\{ \C{K} \bigl[ \rho, \pi^i_{\exp(- \beta r)} \bigr]
\\ + \log \Bigl\{ \pi^i_{\exp( - \beta r)}
\Bigl[ \exp \bigl\{
\beta \tfrac{N}{\gamma}
\log \bigl[ \cosh(\tfrac{\gamma}{N}) \bigr] \rho(m') \bigr\} \Bigr]
\Bigr\} \\
- \frac{\beta}{\gamma} \log \bigl[ 3^{-1} \nu(\gamma)
\nu(\beta) \mu(i) \epsilon \bigr] \biggr\}
- \frac{\log \bigl[ \nu(\beta) \mu(i) \bigr] }{(1 - \zeta^{-1})}
- \frac{\log(3^{-1}\epsilon)}{2}.
\end{multline*}
We can even push simplification a step further,
postponing the optimization of the ratio $\frac{\gamma}{\beta}$,
and setting it to the fixed value $\zeta$.
This leads us to adopt the definition
\begin{multline}
\label{eq1.34}
\C{C}(\rho) =
\inf_{\beta \in \RR_+,  i \in \NN}
\bigl(1 - \zeta^{-1} \bigr)^{-1} \biggl\{
\C{K}\bigl[ \rho, \pi^i_{\exp( - \beta r)} \bigr]
\\*
+ \log \Bigl\{ \pi^i_{\exp( - \beta r)} \Bigl[ \exp \bigl\{
\tfrac{N}{\zeta} \log \bigl[ \cosh( \tfrac{\zeta \beta}{N} ) \bigr]
\rho(m') \bigr\} \Bigr] \Bigr\} \biggr\} \\*
- \frac{\zeta + 1}{\zeta -1 }
\biggl\{ \log \bigl[ \nu(\beta) \mu(i) \bigr]
+ 2^{-1} \log(3^{-1}\epsilon) \biggr\}.
\end{multline}

With either of these modified definitions of the complexity
$\C{C}(\rho)$, we get the upper bound
\begin{multline}
\label{eq1.33}
S_{\lambda}(\rho, \rho') \leq
\wt{S}_{\lambda}(\rho, \rho') \overset{\text{def}}{=} \frac{N}{\lambda}
\log \bigl[ \cosh(\tfrac{\lambda}{N}) \bigr] \rho
\otimes \rho'(m')
\\ + \frac{1}{\lambda} \biggl\{ \C{C}(\rho) + \C{C}(\rho')
- \frac{\zeta + 1}{\zeta - 1} \log \bigl[ \nu(\lambda) \bigr] \biggr\}.
\end{multline}
With these definitions,
we have
for any posterior distributions $\rho$ and $\rho'$
$$
B(\rho, \rho') \leq \inf_{\lambda \in \RR_+} \Xi_{\frac{\lambda}{N}} \Bigl\{
\rho'(r) - \rho(r) + \wt{S}_{\lambda}(\rho, \rho') \Bigr\}.
$$
Consequently in the case when
$B(\rho', \rho) > 0$, we get
$$
\TB(\rho, \rho') \leq B(\rho, \rho') \leq \inf_{\lambda \in \RR_+}
\Xi_{\frac{\lambda}{N}}
\bigl[ 2 \wt{S}_{\lambda}( \rho, \rho') \bigr].
$$

To select some nearly optimal posterior distribution in
$\C{P}$, it is appropriate to order the posterior distributions
of $\C{P}$ according to increasing values
of their complexity $\C{C}(\rho)$ and consider some
indexation $\C{P} = \{ \rho_1, \dots, \rho_M \}$, where
$\C{C}(\rho_k) \leq \C{C}(\rho_{k+1})$, $1 \leq k < M$.

Let us now consider for each $\rho_k \in \C{P}$ the
first posterior distribution in $\C{P}$ which cannot be proved to
be worse than $\rho_k$ according to the bound $\wt{B}$:
\begin{equation}
\label{eq2.15}
t(k) = \min \Bigl\{
j \in \{1, \dots M\} \,:\,\TB(\rho_j, \rho_k) > 0 \Bigr\}.
\end{equation}
In this definition, which uses the chained bound
defined by equation \myeq{eq1.37Bis}, it is appropriate to assume by convention
that $\TB(\rho, \rho) = 0$, for any posterior distribution
$\rho$.
Let us now define our estimated best $\rho \in \C{P}$ as
$\rho_{\wh{k}}$, where
\begin{equation}
\label{eq2.16}
\wh{k} = \min ( \arg \max t ).
\end{equation}
Thus we take the posterior with smallest complexity which can be
proved to be better than the largest starting interval of $\C{P}$
in terms of estimated relative classification error.

The following theorem is a simple consequence of the
chosen optimisation scheme. It is valid for any arbitrary
choice of the complexity function $\rho \mapsto \C{C}(\rho)$.
\begin{thm}\mypoint
\label{thm1.58}
Let us put $\wh{t} = t( \wh{k})$, where $t$
is defined by equation \eqref{eq2.15} and $\wh{k}$
is defined by equation \eqref{eq2.16}.
With $\PP$ probability at least $1 - \epsilon$,
$$
\rho_{\wh{k}}(R) \leq \rho_j (R)
+
\begin{cases} 0, \quad 1 \leq j < \wh{t},\\
\wt{B}(\rho_j, \rho_{t(j)}), \quad  \wh{t} \leq j < \wh{k},\\
\TB(\rho_j, \rho_{\wh{t}}) + \TB(\rho_{\wh{t}},\rho_{\wh{k}}),
\quad j \in (\arg \max t),\\
\TB(\rho_j, \rho_{\wh{k}}), \quad
j \in \bigl\{\wh{k}+1, \dots, M \bigr\} \setminus
(\arg\max t),
\end{cases}
$$
where the chained bound $\wt{B}$ is defined from the bound
of Theorem \thmref{thm1.1.39} by equation
\myeq{eq1.37Bis}.
In the mean time, for any $j$ such that $\wh{t} \leq j < \wh{k}$,
$t(j) < \wh{t} = \max t$, because $j \not \in (\arg\max t)$.
Thus
\begin{align*}
\rho_{\wh{k}}(R) & \leq \rho_{t(j)} (R) \leq \rho_j(R) + \inf_{\lambda \in \RR_+}
\Xi_{\frac{\lambda}{N}}
\bigl[ 2 S_{\lambda}(\rho_j, \rho_{t(j)}) \bigr]\\
\text{while } \rho_{t(j)}(r) & \leq \rho_j(r) + \inf_{\lambda \in \RR_+}
S_{\lambda}(\rho_j, \rho_{t(j)}),
\end{align*}
where the function $\Xi$ is defined by equation \myeq{eq1.34Bis}
and $S_{\lambda}$ is defined by equation \myeq{eq2.13}.
For any $j \in (\arg \max t)$, (including notably
$\wh{k}$),
\begin{align*}
B(\rho_{\wh{t}}, \rho_j) & \geq \TB(\rho_{\wh{t}}, \rho_j) > 0,\\
B(\rho_j, \rho_{\wh{t}}) & \geq \TB(\rho_j, \rho_{\wh{t}}) > 0,
\end{align*}
so in this case
\begin{align*}
\rho_{\wh{k}}(R) & \leq \rho_j(R) + \inf_{\lambda \in \RR_+}
\Xi_{\frac{\lambda}{N}}
\Bigl[ S_{\lambda}(\rho_j, \rho_{\wh{t}})
+ S_{\lambda} (\rho_{\wh{t}}, \rho_{\wh{k}})
+ S_{\lambda} (\rho_j, \rho_{\wh{k}})  \Bigr], \\
\text{while }
\rho_{\wh{t}}(r) & \leq \rho_j(r) + \inf_{\lambda \in \RR_+}
S_{\lambda}(\rho_j, \rho_{\wh{t}}),\\
\rho_{\wh{k}}(r) & \leq \rho_{\wh{t}}(r) + \inf_{\lambda \in \RR_+}
S_{\lambda}(\rho_{\wh{t}}, \rho_{\wh{k}}),\\
\text{and } \rho_{\wh{t}}(R) & \leq
\rho_j(R) + \inf_{\lambda \in \RR_+} \Xi_{\frac{\lambda}{N}}
\bigl[ 2 S_{\lambda}(\rho_j, \rho_{\wh{t}}) \bigr].
\end{align*}
Finally in the case when $j \in \bigl\{ \wh{k}+1, \dots, M \bigr\}
\setminus (\arg \max t)$, due to the fact that in particular
$j \not\in (\arg \max t)$,
$$
B(\rho_{\wh{k}}, \rho_j) \geq \TB(\rho_{\wh{k}}, \rho_j) > 0.
$$
Thus in this last case
\begin{align*}
\rho_{\wh{k}}(R) & \leq \rho_j(R) +
\inf_{\lambda \in \RR_+} \Xi_{\frac{\lambda}{N}} \bigl[
2 S_{\lambda}(\rho_j, \rho_{\wh{k}}) \bigr],\\*
\text{while } \rho_{\wh{k}}(r)
& \leq \rho_j(r) + \inf_{\lambda \in \RR_+} S_{\lambda}(
\rho_j, \rho_{\wh{k}}).
\end{align*}

Thus for any $j = 1, \dots, M$,
$\rho_{\wh{k}}(R) - \rho_j(R)$
is bounded from above by an empirical quantity involving only variance
and entropy terms of posterior distributions $\rho_\ell$
such that $\ell \leq j$, and therefore such that $\C{C}(\rho_\ell) \leq \C{C}(\rho_j)$.
Moreover, these distributions $\rho_{\ell}$ are such that
$\rho_{\ell}(r) - \rho_j(r)$ and $\rho_{\ell}(R) - \rho_j(R)$
have an empirical upper bound of the same order
as the bound stated for $\rho_{\wh{k}}(R) - \rho_j(R)$
--- namely the bound for $\rho_{\ell}(r) - \rho_j(r)$
is in all circumstances not greater than $\Xi_{\frac{\lambda}{N}}^{-1}$
applied to the bound stated for $\rho_{\wh{k}}(R) - \rho_{j}(R)$,
whereas the bound for $\rho_{\ell}(R) - \rho_j(R)$
is always smaller than two times the bound stated for $\rho_{\wh{k}}(R)
- \rho_{j}(R)$.
This shows that variance terms are between posterior
distributions whose empirical as well as expected error
rates cannot be much larger
than those of $\rho_j$.
\end{thm}

Let us remark that the estimation scheme described in
this theorem is very general, the same method can
be used as soon as some \emph{confidence interval}
for the relative expected risks
$$
- B(\rho_2, \rho_1) \leq \rho_2(R) - \rho_1(R) \leq B(\rho_1, \rho_2)
\text{ with $\PP$ probability at least } 1 - \epsilon,
$$
is available. The definition of the complexity is arbitrary,
and could in an abstract context be chosen as
$$
\C{C}(\rho_1) = \inf_{\rho_2 \neq \rho_1} B(\rho_1, \rho_2) + B(\rho_2, \rho_1).
$$

\begin{proof}
The case when $1 \leq j < \wh{t}\,$
is straightforward from the definitions: when $j < \wh{t}$,
$\TB(\rho_j, \rho_{\wh{k}}) \leq 0$ and therefore
$\rho_{\wh{k}}(R) \leq \rho_j(R)$.

In the second case, that is
when $\wh{t} \leq j < \wh{k}$,
$j$ cannot be in $\arg \max t$, because
of the special choice of $\wh{k}$ in $\arg \max t$.
Thus $t(j) < \wh{t}$ and
we deduce from the first case that
$$
\rho_{\wh{k}}(R) \leq \rho_{t(j)}(R) \leq \rho_j(R) + \TB(\rho_j, \rho_{t(j)}).
$$
Moreover, we see from the defintion of $t$ that $\TB(\rho_{t(j)}, \rho_j) > 0$,
implying
$$\rho_{t(j)} (r) \leq \rho_j(r) + \inf_{\lambda \in \RR_+}
S_{\lambda}(\rho_j, \rho_{t(j)}),
$$
and therefore that
$$
\rho_{\wh{k}}(R) \leq \rho_j(R) + \inf_{\lambda} \Xi_{\frac{\lambda}{N}}
\bigl[ 2 S_{\lambda}(\rho_j, \rho_{t(j)}) \bigr].
$$

In the third case $j$ belongs to $\arg \max t$. In this case,
we are not sure that $\TB(\rho_{\wh{k}}, \rho_j) > 0$, and it
is appropriate to involve $\wh{t}$, which is the index of the
first posterior distribution which cannot be improved by
$\rho_{\wh{k}}$, implying notably that $\TB(\rho_{\wh{t}}, \rho_k) > 0$
for any $k \in \arg \max t$. On the other hand, $\rho_{\wh{t}}$
cannot either improve any posterior distribution $\rho_k$ with $ k \in (\arg \max t)$,
because this would imply for any $\ell < \wh{t}$
that $\TB(\rho_\ell, \rho_{\wh{t}}) \leq \TB(\rho_\ell,
\rho_k ) + \TB(\rho_k, \rho_{\wh{t}}) \leq 0$, and therefore
that $t(\wh{t}) \geq \wh{t}+1$, in contradiction of the
fact that $\wh{t} = \max t$.
Thus $\TB(\rho_k, \rho_{\wh{t}}) > 0$, and
these two remarks imply that
\begin{align*}
\rho_{\wh{t}}(r) & \leq \rho_j(r) + \inf_{\lambda \in \RR_+}
S_{\lambda}(\rho_j, \rho_{\wh{t}}), \\
\rho_{\wh{k}}(r) & \leq \rho_{\wh{t}}(r)
+ \inf_{\lambda \in \RR_+} S_{\lambda} (\rho_{\wh{t}}, \rho_{\wh{k}})\\
& \leq \rho_j(r) + \inf_{\lambda \in \RR_+}
S_{\lambda}(\rho_j, \rho_{\wh{t}}) +
\inf_{\lambda \in \RR_+} S_{\lambda}(\rho_{\wh{t}}, \rho_{\wh{k}}),
\end{align*}
and consequently also that
\begin{multline*}
\rho_{\wh{k}}(R) \leq \rho_j(R) + \TB(\rho_j, \rho_{\wh{k}})
\\ \leq \rho_j(R) + \inf_{\lambda \in \RR_+}
\Xi_{\frac{\lambda}{N}} \Bigl[
S_{\lambda}(\rho_j, \rho_{\wh{t}}) + S_{\lambda}(\rho_{\wh{t}},
\rho_{\wh{k}})+ S_{\lambda}(\rho_j, \rho_{\wh{k}}) \Bigr]
\end{multline*}
and that
$$
\rho_{\wh{t}}(R) \leq \rho_j(R) + \inf_{\lambda \in \RR_+}
\Xi_{\frac{\lambda}{N}} \bigl[ 2 S_{\lambda}(\rho_j, \rho_{\wh{t}}) \bigr]
\leq \rho_j(R) + 2 \inf_{\lambda \in \RR_+} 2 \Xi_{\frac{\lambda}{N}}
\bigl[ S_{\lambda}(\rho_j, \rho_{\wh{t}}) \bigr],
$$
the last inequality being due to the fact that $\Xi_{\frac{\lambda}{N}}$
is a concave function. Let us notice that it may be the case that
$\wh{k} < \wh{t}$, but that only the case when $j \geq \wh{t}$ is
to be considered, since otherwise we already know that $
\rho_{\wh{k}}(R) \leq \rho_j(R)$.

In the fourth case, $j$ is greater than $\wh{k}$, and the
complexity of $\rho_j$ is larger than the complexity of $\rho_{\wh{k}}$.
Moreover, $j$ is not in $\arg \max t$, and thus $\TB(\rho_{\wh{k}}, \rho_j)
> 0$, because otherwise, the sub-additivity of $\TB$ would imply
that $\TB(\rho_{\ell}, \rho_j) \leq 0$ for any $\ell \leq \wh{t}$
and therefore that $t(j) \geq \wh{t} = \max t$. Therefore
$$
\rho_{\wh{k}}(r) \leq \rho_j(r) + \inf_{\lambda \in \RR_+}
S_{\lambda}(\rho_j, \rho_{\wh{k}}),
$$
and
$$
\rho_{\wh{k}}(R) \leq \rho_j(R) + \TB(\rho_j, \rho_{\wh{k}})
\leq \rho_j(R) + \inf_{\lambda \in \RR_+} \Xi_{\frac{\lambda}{N}}
\bigl[ 2 S_{\lambda}(\rho_j, \rho_{\wh{k}}) \bigr].
$$
\end{proof}

\subsection{Analysis of relative bounds}

Let us start our investigation of the theoretical properties
of the algorithm described in Theorem \thmref{thm1.58}
by computing some non-random upper bounds for $B(\rho, \rho')$,
the bound of Theorem \thmref{thm1.1.39},
and $\C{C}(\rho)$, the complexity factor defined
by equation \myeq{eq1.34}, for any $\rho, \rho' \in \C{P}$.

This analysis will be done in the case when
$$
\C{P} = \Bigl\{ \pi^i_{\exp( - \beta r)}\,:\,\nu(\beta)>0, \mu(i) > 0
\Bigr\},
$$
in which it will be possible to get some control on the randomness
of any $\rho \in \C{P}$, in addition to controlling the other
random expressions appearing in the definition of
$B(\rho, \rho')$, $\rho, \rho' \in \C{P}$.
We will also use a simpler choice of complexity function,
removing from equation (\ref{eq1.34} page \pageref{eq1.34})
the optimization in $i$ and $\beta$ and using instead the
definition
\begin{multline}
\label{eq2.18}
\C{C}(\pi^i_{\exp( - \beta r)}) \overset{\text{\rm def}}{=}
\bigl( 1 - \zeta^{-1}\bigr)^{-1} \log \biggl\{ \pi^i_{\exp( - \beta r)}
\biggl[ \\ \exp \Bigl\{ \tfrac{N}{\zeta} \log \bigl[ \cosh\bigl( \tfrac{\zeta \beta}{N}
\bigr) \bigr] \pi^i_{\exp( - \beta r)}(m') \Bigr\} \biggr] \biggr\}\\
+ \frac{\zeta+1}{\zeta-1} \log \bigl[ \nu(\beta) \mu(i) \bigr].
\end{multline}
With this definition,
\begin{multline*}
S_{\lambda}(\pi^i_{\exp( - \beta r)}, \pi^j_{\exp( - \beta' r)} )
\leq \frac{N}{\lambda} \log \bigl[ \cosh(\tfrac{\lambda}{N})\bigr]
\pi^i_{\exp( - \beta r)} \otimes \pi^j_{\exp( - \beta' r)}
(m') \\ +
\frac{ \C{C}\bigl[ \pi^i_{\exp( - \beta r)}\bigr]  + \C{C}\bigl[
\pi^j_{\exp( - \beta' r)}\bigr]}{\lambda}
\\ + \frac{(\zeta + 1)}{(\zeta-1)\lambda}
\log \bigl[ 3^{-1} \nu(\lambda) \epsilon \bigr],
\end{multline*}
where $S_{\lambda}$ is defined by equation \myeq{eq2.13},
so that
\begin{multline*}
B\bigl[ \pi^i_{\exp( - \beta r)}, \pi^j_{\exp( - \beta' r)}\bigr]
= \inf_{\lambda \in \RR_+} \Xi_{\frac{\lambda}{N}}
\Bigl\{ \pi^j_{\exp( - \beta' r)}(r) - \pi^i_{\exp( - \beta r)}(r)
\\ + S_{\lambda} \bigl[ \pi^i_{\exp( - \beta r)},
\pi^j_{\exp( - \beta r)} \bigr] \Bigr\}.
\end{multline*}

Let us successively bound the various
random factors entering into the definition
of $B \bigl[ \pi^i_{\exp( - \beta r)}, \pi^j_{
\exp( - \beta' r)} \bigr]$.
The quantity $\pi^j_{\exp( - \beta' r)}(r) - \pi^i_{\exp(
- \beta r)}(r)$ can be bounded using a slight adaptation
of Proposition \ref{prop1.46} (page \pageref{prop1.46}).

\begin{prop}\mypoint
\label{prop1.59}
For any positive real constants $\lambda, \lambda'$ and $\gamma$,
with $\PP$ probability at least $1 - \eta$,
for any positive real constants $\beta$, $\beta'$
such that $\beta < \lambda \frac{\gamma}{N} \sinh(\tfrac{\gamma}{N})^{-1}$
and $\beta' > \lambda' \frac{\gamma}{N} \sinh(\tfrac{\gamma}{N})^{-1}$,
\begin{multline*}
\pi^j_{\exp(- \beta' r)}(r) - \pi^i_{\exp(- \beta r)}(r)
\\ \leq \pi^j_{\exp( - \lambda' R)} \otimes \pi^i_{\exp( - \lambda R)}
\bigl[ \Psi_{-\frac{\gamma}{N}}(R', M') \bigr] \\
+ \frac{\log\bigl(\tfrac{3}{\eta}\bigr)}{\gamma}
+ \frac{ C^j(\lambda', \gamma) + \log(\frac{3}{\eta})}{
\frac{N \beta'}{\lambda'} \sinh(\frac{\gamma}{N}) -
\gamma}
+ \frac{C^i(\lambda, \gamma) + \log(\frac{3}{\eta})}{
\gamma - \frac{N \beta}{\lambda} \sinh(\frac{\gamma}{N})},
\end{multline*}
where
\begin{multline*}
C^i(\lambda, \gamma) \overset{\text{\rm def}}{=}
\log \biggl\{
\int_{\Theta} \exp \biggl[ - \gamma
\int_{\Theta} \Bigl\{ \Psi_{\frac{\gamma}{N}} \bigl[
R'(\theta_1, \theta_2), M'(\theta_1, \theta_2) \bigr] \\
- \tfrac{N}{\gamma} \sinh(\tfrac{\gamma}{N}) R'(\theta_1,
\theta_2) \Bigr\} \pi^i_{\exp( -\lambda R)}( d \theta_2) \biggr]
\pi^i_{\exp( - \lambda R)}( d \theta_1) \biggr\}\\
\leq \log \biggl\{ \pi^i_{\exp( - \lambda R)} \biggl[
\exp \Bigl\{ 2 N \sinh \bigl( \tfrac{\gamma}{2N} \bigr)^2
\pi^i_{\exp( - \lambda R)} \bigl(M'\bigr) \Bigr\} \biggr] \biggr\}.
\end{multline*}
\end{prop}

As for $\pi^i_{\exp( - \beta r)} \otimes
\pi^j_{\exp( - \beta' r)}(m')$,
we can write with $\PP$ probability at least $1 - \eta$,
for any posterior distributions $\rho$ and $\rho'
: \Omega \rightarrow \C{M}_+^1(\Theta)$,
\begin{multline*}
\gamma \rho \otimes \rho' (m')
\leq \log \Bigl[ \pi^i_{\exp( - \lambda R)} \otimes \pi^j_{\exp(
- \lambda' R)} \bigl\{ \exp \bigl[ \gamma
\Phi_{- \frac{\gamma}{N}}(M') \bigr] \bigr\}   \Bigr]
\\ + \C{K}\bigl[\rho, \pi^i_{\exp( - \lambda R)} \bigr]
+ \C{K}\bigl[ \rho', \pi^j_{\exp( - \lambda' R)} \bigr] -
\log(\eta).
\end{multline*}
We can then replace $\lambda$ with
$\beta \frac{N}{\lambda} \sinh(\frac{\lambda}{N})$ and
use Theorem \thmref{thm2.1.12}
to get
\begin{prop}\mypoint
\label{prop1.60}
For any positive real constants $\gamma$,
$\lambda$, $\lambda'$, $\beta$ and $\beta'$,
with $\PP$ probability $1 - \eta$,
\begin{multline*}
\gamma \rho \otimes \rho'(m')
\\ \shoveleft{\qquad\leq \log \Bigl[
\pi^i_{\exp[ - \beta \frac{N}{\lambda} \sinh(
\frac{\lambda}{N}) R]}\otimes
\pi^j_{\exp[ - \beta' \frac{N}{\lambda'}
 \sinh(\frac{\lambda'}{N}) R]}
\bigl\{ \exp \bigl[
\gamma \Phi_{- \frac{\gamma}{N}}(M')
\bigr] \bigr\} \Bigr]} \\
+\frac{\C{K}\bigl[ \rho, \pi^i_{\exp( - \beta r)}\bigr]}{1 -
\frac{\beta}{\lambda}} + \frac{C^i\bigl[ \beta \frac{N}{\lambda}
\sinh(\frac{\lambda}{N}), \lambda\bigr] - \log(\frac{\eta}{3})}{
\frac{\lambda}{\beta} - 1} \\
+\frac{\C{K}\bigl[ \rho', \pi^j_{\exp( - \beta' r)}\bigr]}{1 -
\frac{\beta'}{\lambda'}} + \frac{C^j\bigl[ \beta' \frac{N}{\lambda'}
\sinh(\frac{\lambda'}{N}), \lambda' \bigr] - \log(\frac{\eta}{3})}{
\frac{\lambda}{\beta'} - 1} - \log(\tfrac{\eta}{3}).
\end{multline*}
\end{prop}

The last random factor in $B(\rho, \rho')$ that we need to upper bound is
$$
\log \Bigl\{ \pi^i_{\exp( - \beta r)} \Bigl[
\exp \bigl\{ \beta \tfrac{N}{\gamma} \log \bigl[ \cosh(\tfrac{\gamma}{N})
\bigr] \pi^i_{\exp( - \beta r)}(m') \bigr\} \Bigr] \Bigr\}.
$$
A slight adaptation of Proposition \ref{prop1.48} (page \pageref{prop1.48})
shows that
with $\PP$ probability at least $1 - \eta$,
\begin{multline*}
\log \Bigl\{ \pi^i_{\exp( - \beta r)} \Bigl[
\exp \bigl\{ \beta \tfrac{N}{\gamma} \log \bigl[ \cosh(\tfrac{\gamma}{N})
\bigr] \pi^i_{\exp( - \beta r)}(m') \bigr\} \Bigr] \Bigr\}
\\ \leq \frac{2 \beta}{\gamma}
C^i\bigl[ \tfrac{N \beta}{\gamma} \sinh(\tfrac{\gamma}{N}), \gamma
\bigr]
+ \bigl(1 - \tfrac{\beta}{\gamma} \bigr)
\log \biggl\{ \Bigl(\pi^i_{\exp[ - \frac{N \beta}{\gamma} \sinh(\frac{
\gamma}{N})R]} \Bigr)^{\otimes 2}
\biggl[ \\
\exp \biggl( \frac{N \log \bigl[ \cosh(\tfrac{\gamma}{N}
) \bigr] }{\frac{\gamma}{\beta} - 1}
\Phi_{- \frac{ \log[\cosh(\frac{\gamma}{N})]}{ \frac{\gamma}{\beta} -1}}
\circ M' \biggr) \biggr] \biggr\} \\
+ \bigl( 1 + \tfrac{\beta}{\gamma} \bigr) \log ( \tfrac{2}{\eta}),
\end{multline*}
where as usual $\Phi$ is the function defined by equation
(\ref{eq1.1}, page \pageref{eq1.1}).
This leads us to define for any $i, j \in \NN$,
any $\beta, \beta' \in \RR_+$,
\begin{multline}
\label{eq1.35Bis}
\ov{\C{C}}(i, \beta) \overset{\text{\rm def}}{=}
\frac{2 }{\zeta - 1} C^i\Bigl[ \tfrac{N}{\zeta}
\sinh(\tfrac{\zeta \beta}{N}), \zeta \beta \Bigr] \\* \shoveleft{\qquad +
\log \biggl\{ \bigl(\pi^i_{\exp[ - \frac{N}{\zeta}
\sinh(\frac{\zeta \beta}{N}) R]}\bigr)^{\otimes 2}
\biggl[} \\* \shoveright{\exp \biggl( \frac{N \log \bigl[
\cosh(\frac{\zeta \beta}{N})\bigr]}{
\zeta - 1}
\Phi_{- \frac{\log[\cosh(\frac{\zeta \beta}{N})]}{\zeta
- 1}}\circ M' \biggr) \biggr] \biggr\}
\qquad}\\*
- \frac{\zeta + 1}{\zeta -1} \biggl\{
2 \log \bigl[ \nu(\beta) \mu(i) \bigr] + \log \bigl(\tfrac{\eta}{2}\bigr)
\biggr\}.
\end{multline}
Recall that the definition of $C^i(\lambda, \gamma)$ is to be found
in Proposition \ref{prop1.59}, page \pageref{prop1.59}.
Let us remark that, since
\begin{multline*}
\exp \bigl[ N a \Phi_{-a} (p) \bigr] =
\exp \Bigl\{ N \log \Bigl[ 1 + \bigl[ \exp(a) - 1
\bigr] p \Bigr] \Bigr\}
\\ \leq \exp \Bigl\{ N \bigl[ \exp(a) - 1 \bigr] p \Bigr\}, \quad p \in (0,1),
a \in \RR,
\end{multline*}
we have
\begin{multline*}
\ov{\C{C}}(i, \beta) \leq
\frac{2}{\zeta - 1}
\log \biggl\{ \pi^i_{\exp[ - \frac{N}{\zeta} \sinh(\frac{\zeta \beta}{N}) R]}
\biggl[ \\ \shoveright{\exp \Bigl\{ 2 N \sinh\bigl(\tfrac{\zeta \beta}{2N}\bigr)^2
\pi^i_{\exp[ - \frac{N}{\zeta} \sinh(\frac{\zeta \beta}{N}) R]} \bigl(
M' \bigr) \Bigr\} \biggr]
\biggr\}\quad} \\
\shoveleft{\qquad + \log \biggl\{ \Bigl( \pi^i_{\exp[ - \frac{N}{\zeta} \sinh(\frac{\zeta \beta}{N}) R]}
\Bigr)^{\otimes 2} \biggl[} \\ \shoveright{
\exp \Bigl\{  N \Bigl[ \exp \bigl\{ (\zeta - 1)^{-1} \log \bigl[
\cosh \bigl( \tfrac{\zeta \beta}{N} \bigr) \bigr] \bigr\} - 1 \Bigr] M' \Bigr\}
\biggr] \biggr\}\quad}\\
- \frac{\zeta + 1}{\zeta -1} \biggl\{
2 \log \bigl[ \nu(\beta) \mu(i) \bigr] + \log \bigl( \tfrac{\eta}{2} \bigr)
\biggr\}.
\end{multline*}
Let us put
\begin{multline*}
\ov{S}_{\lambda}\bigl[ (i,\beta), (j, \beta') \bigr]
\overset{\text{\rm def}}{=} \frac{N}{\lambda} \log \bigl[
\cosh(\tfrac{\lambda}{N})\bigr]
\inf_{\gamma \in \RR_+} \gamma^{-1} \biggl\{ \\ \shoveleft{\qquad
\log \biggl[ \Bigl( \pi^i_{\exp[
 - \frac{N}{\zeta} \sinh(\frac{\zeta\beta}{N})R]}
\otimes \pi^j_{\exp[ - \frac{N}{\zeta}
\sinh(\frac{\zeta \beta'}{N})R]} \Bigr) \Bigl\{ \exp \bigl[ \gamma
\Phi_{- \frac{\gamma}{N}}
(M') \bigr] \Bigr\} \biggr]} \\
+ \frac{ C^i\bigl[ \frac{N}{\zeta} \sinh(\frac{\zeta \beta}{N}),
\zeta \beta \bigr] - \log(\frac{\ov{\eta}}{3})}{\zeta-1} \\
\shoveright{+ \frac{C^j\bigl[ \frac{N}{\zeta}
\sinh(\frac{\zeta \beta'}{N}), \zeta \beta' \bigr] -
\log(\frac{\ov{\eta}}{3})}{\zeta - 1} -
\log( \tfrac{\ov{\eta}}{3}) \biggr\}\qquad} \\
+ \frac{1}{\lambda} \biggl[
\ov{\C{C}}(i,\beta) + \ov{\C{C}}(j,\beta')
- \frac{\zeta+1}{\zeta-1} \log \bigl[ 3^{-1} \nu(\lambda) \epsilon \bigr] \biggr],
\end{multline*}
where
$$
\ov{\eta}  = \nu(\gamma)\nu(\beta) \nu(\beta')
\mu(i) \mu(j) \eta.
$$
Let us remark that
\begin{multline*}
\ov{S}_{\lambda}\bigl[ (i, \beta), (j, \beta') \bigr]
\leq \inf_{\gamma \in \RR_+} \\
\shoveleft{\quad \frac{\lambda}{2N \gamma}
\log \biggl[ \Bigl( \pi^i_{\exp[ - \frac{N}{\zeta} \sinh(\frac{\zeta \beta}{N})
R]} \otimes \pi^j_{\exp[ - \frac{N}{\zeta}\sinh(\frac{\zeta \beta'}{N}) R]}
\Bigr) \Bigl\{ } \\
\shoveright{\exp \Bigl[ N \bigl[ \exp \bigl( \tfrac{\gamma}{N} \bigr) - 1 \bigr] M'
\Bigr] \Bigr\} \biggr]\quad} \\
\shoveleft{\quad + \biggl(\frac{\lambda}{2N\gamma(\zeta-1)} + \frac{2}{\lambda(\zeta -1)}\biggr)
\log \biggl\{ \pi^i_{\exp[ - \frac{N}{\zeta} \sinh(\frac{\zeta \beta}{N}) R]}
\biggl[} \\
\shoveright{\exp \Bigl\{ 2 N \sinh \bigl( \tfrac{\zeta \beta}{2N} \bigr)^2 \pi^i_{\exp
[ - \frac{N}{\zeta} \sinh(\frac{\zeta \beta}{N}) R]} \bigl(M'\bigr) \Bigr\}
\biggr] \biggr\} \quad}\\
\shoveleft{\quad + \lambda^{-1} \log \biggl\{ \Bigl( \pi^i_{\exp[ - \frac{N}{\zeta}
\sinh(\frac{\zeta \beta}{N})R]}\Bigr)^{\otimes 2}
\biggl[} \\
\shoveright{\exp \Bigl\{ N \Bigl[ \exp \bigl\{(\zeta-1)^{-1} \log \bigl[ \cosh
\bigl( \tfrac{\zeta \beta}{N}\bigr) \bigr] \bigr\} - 1 \Bigr] M' \Bigr\}
\biggr] \biggr\} \quad}
\\
\shoveleft{\quad + \biggl(\frac{\lambda}{2N\gamma(\zeta-1)} + \frac{2}{\lambda(\zeta -1)}\biggr)
\log \biggl\{ \pi^j_{\exp[ - \frac{N}{\zeta} \sinh(\frac{\zeta \beta'}{N}) R]}
\biggl[} \\
\shoveright{\exp \Bigl\{ 2 N \sinh \bigl( \tfrac{\zeta \beta'}{2N} \bigr)^2 \pi^j_{\exp
[ - \frac{N}{\zeta} \sinh(\frac{\zeta \beta'}{N}) R]} \bigl(M'\bigr) \Bigr\}
\biggr] \biggr\} \quad}\\
\shoveleft{\quad + \lambda^{-1} \log \biggl\{ \Bigl( \pi^j_{\exp[ - \frac{N}{\zeta}
\sinh(\frac{\zeta \beta'}{N})R]}\Bigr)^{\otimes 2}
\biggl[} \\
\shoveright{\exp \Bigl\{ N \Bigl[ \exp \bigl\{(\zeta-1)^{-1} \log \bigl[ \cosh
\bigl( \tfrac{\zeta \beta'}{N}\bigr) \bigr] \bigr\} - 1 \Bigr] M' \Bigr\}
\biggr] \biggr\} \quad}\\
- \frac{(\zeta+1)\lambda}{2 N (\zeta -1)\gamma}
\log \bigl[ 3^{-1} \nu(\gamma) \nu(\beta) \nu(\beta') \mu(i) \mu(j) \eta \bigr]
\\
- \frac{(\zeta+1)}{(\zeta-1) \lambda} \biggl(
2 \log \bigl[ 2^{-1} \nu(\beta) \nu(\beta')
\mu(i) \mu(j) \eta \bigr]
+ \log\bigl[3^{-1} \nu(\lambda) \epsilon\bigr]
\biggr).
\end{multline*}
Let us define accordingly
\begin{multline*}
\ov{B}\bigl[(i, \beta), (j, \beta') \bigr]
\overset{\text{\rm def}}{=} \\
\shoveleft{\qquad \inf_{\lambda}
\Xi_{\frac{\lambda}{N}} \Biggl\{ \inf_{\alpha, \gamma, \alpha', \gamma'}
\biggl[
\pi^j_{\exp( - \alpha' R)} \otimes \pi^i_{\exp( - \alpha R)}
\bigl[ \Psi_{- \frac{\lambda}{N}}(R',M') \bigr]
}\\ - \frac{
\log\bigl(\frac{\wt{\eta}}{3} \bigr)}{\lambda}
+ \frac{C^j(\alpha', \gamma') - \log \bigl( \frac{\wt{\eta}}{3}\bigr)}{
\frac{N \beta'}{\alpha'} \sinh(\frac{\gamma'}{N}) - \gamma'}
+ \frac{C^i(\alpha, \gamma) - \log\bigl(\frac{\wt{\eta}}{3}\bigr)}{\gamma
- \frac{N\beta}{\alpha} \sinh(\frac{\gamma}{N})} \biggr]\\
+ \ov{S}_{\lambda}\bigl[ (i,\beta), (j, \beta') \bigr] \Biggr\},
\end{multline*}
where
$$
\wt{\eta} = \nu(\lambda)\nu(\alpha)\nu(\gamma)\nu(\beta)
\nu(\alpha')\nu(\gamma')\nu(\beta')\mu(i)\mu(j) \eta.
$$

\begin{prop}\mypoint
\begin{itemize}
\item With $\PP$ probability at least $1 - \eta$, for any
$\beta \in \RR_+$ and $i \in \NN$,\\
$\C{C}(\pi^i_{\exp( - \beta r)})
\leq \ov{\C{C}}(i, \beta)$;
\item With $\PP$ probability at least $1 - 3 \eta$,
for any $\lambda, \beta, \beta' \in \RR_+$, any $i, j \in \NN$,
$S_{\lambda}\bigl[ (i, \beta), (j, \beta') \bigr]
\leq \ov{S}_{\lambda} \bigl[ (i, \beta), (j, \beta') \bigr]$;
\item With $\PP$ probability at least $1 - 4 \eta$, for any $i, j \in \NN$,
any $\beta, \beta' \in \RR_+$,\break
$B(\pi^i_{\exp( - \beta r)}, \pi^j_{\exp( - \beta' r)} ) \leq
\ov{B} \bigl[ (i, \beta), (j, \beta') \bigr]$.
\end{itemize}
\end{prop}

It is also interesting to find a non-random lower bound
for $\C{C}(\pi^i_{\exp( - \beta r)})$.
Let us start from the fact that with $\PP$ probability
at least $1 - \eta$,
\begin{eqnarray*}
&&\pi^i_{\exp( - \alpha R)} \otimes \pi^i_{\exp( - \alpha R)}
\bigl[ \Phi_{\frac{\gamma'}{N}}(M') \bigr]\\
&&\qquad\leq
\pi^i_{\exp( - \alpha R)} \otimes \pi^i_{\exp( - \alpha R)}
(m') - \frac{\log(\eta)}{\gamma'}.
\end{eqnarray*}
On the other hand, we already proved that with $\PP$ probability
at least $1 - \eta$,
\begin{multline*}
0 \leq \biggl(1 - \frac{\alpha}{N \tanh(\frac{\lambda}{N})}
\biggr) \C{K} \bigl[ \rho, \pi^i_{\exp( - \alpha R)} \bigr]  \\
\shoveleft{\quad \leq \frac{ \alpha}{N \tanh(\frac{\lambda}{N})}
\biggl\{ \lambda \bigl[ \rho(r) - \pi^i_{\exp( \alpha R)}(r) \bigr]}
\\ + N \log \bigl[ \cosh(\tfrac{\lambda}{N}) \bigr] \rho \otimes
\pi^i_{\exp( - \alpha R)}(m') - \log(\eta) \biggr\} \\
+ \C{K}\bigl( \rho, \pi^i \bigr) -
\C{K}\bigl( \pi^i_{\exp( - \alpha R)}, \pi^i \bigr).
\end{multline*}
Thus for any $\xi > 0$, putting $\beta = \frac{\alpha \lambda}{
N \tanh(\frac{\lambda}{N})}$,
with $\PP$ probability at least $1 - \eta$,
\begin{multline*}
\xi \pi^i_{\exp( - \alpha R)} \otimes \pi^i_{\exp( - \alpha R)}
\bigl[ \Phi_{\frac{\gamma'}{N}}(M') \bigr] \\ \shoveleft{ \quad \leq
\pi^i_{\exp( - \alpha R)} \biggl\{
\log \biggl[ \pi^i_{\exp( - \beta r)} \Bigl\{} \\ \exp \Bigl[
\beta \tfrac{N}{\lambda} \log \bigl[ \cosh(\tfrac{\lambda}{N}) \bigr]
\pi^i_{\exp( - \beta r)}(m') + \xi m' \Bigr] \Bigr\} \biggr] \biggr\}
\\ \shoveright{- \biggl( \frac{\beta}{\lambda} + \frac{\xi}{\gamma'} \biggr)
\log \biggl( \frac{\eta}{2} \biggr)} \\
\shoveleft{\quad \leq \log \biggl\{ \pi^i_{\exp( - \beta r)} \biggl[
\exp \Bigl\{ \beta \tfrac{N}{\lambda} \log \bigl[ \cosh(\tfrac{\lambda}{N}) \bigr]
\pi^i_{\exp( - \beta r)}(m') \Bigr\}
}\\ \times
\pi^i_{\exp( - \beta r)} \Bigl\{ \exp \Bigl[
\beta \tfrac{N}{\lambda} \log \bigl[ \cosh(\tfrac{\lambda}{N}) \bigr]
\pi^i_{\exp( - \beta r)}(m') + \xi m' \Bigr] \Bigr\} \biggr] \biggr\}
\\ \shoveright{- \biggl( 2 \frac{\beta}{\lambda} +
\frac{\xi}{\gamma'} \biggr) \log \biggl( \frac{\eta}{2} \biggr)}
\\ \leq 2 \log \biggl\{ \pi^i_{\exp( - \beta r)} \biggl[ \exp \Bigl\{
\Bigl[ \xi + \beta \tfrac{N}{\lambda} \log \bigl[ \cosh(\tfrac{\lambda}{N})\bigr]
\Bigr] \pi^i_{\exp( -
\beta r)}(m')\Bigr\} \biggr] \biggr\}
\\ \shoveright{- \biggl( 2 \frac{\beta}{\lambda} + \frac{\xi}{\gamma'} \biggr)
\log \biggl( \frac{\eta}{2} \biggr)}\\
\leq 2 \log \biggl\{ \pi^i_{\exp( -\beta r)} \biggl[
\exp \Bigl\{ \Bigl[ \xi + \tfrac{\beta\lambda}{2N} \Bigr] \pi^i_{\exp( -
\beta r)}(m') \Bigr\} \biggr] \biggr\}\\
- \biggl( \frac{2 \beta}{\lambda} + \frac{\xi}{\gamma'} \biggr) \log
\biggl( \frac{\eta}{2} \biggr).
\end{multline*}
Taking $\xi = \frac{\beta \lambda}{2N}$,
we get with $\PP$ probability at least $1 - \eta$
\begin{multline*}
\frac{\beta \lambda}{4N} \Bigl(\pi^i_{\exp[ - \beta \frac{N}{\lambda}
\tanh(\frac{\lambda}{N}) R]} \Bigr)^{\otimes 2} \Bigl[ \Phi_{\frac{\gamma'}{N}}
\bigl(M'\bigr)\Bigr] \\
\leq \log \biggl\{ \pi^i_{\exp( - \beta r)} \biggl[
\exp \Bigl\{ \frac{\beta \lambda}{N} \pi^i_{\exp( - \beta r)}(m') \Bigr\}
\biggr] \biggr\}\\
- \biggl( \frac{2 \beta}{\lambda} + \frac{\beta \lambda}{2 N \gamma'} \biggr)
\log \biggl( \frac{\eta}{2} \biggr).
\end{multline*}
Putting
\begin{align*}
\lambda & = \frac{N^2}{\gamma}\log \bigl[ \cosh(\tfrac{\gamma}{N} ) \bigr]\\
\text{and }\Upsilon(\gamma) & \overset{\text{\rm def}}{=} \frac{\gamma \tanh\bigl\{ \frac{N}{\gamma}\log\bigl[
\cosh(\frac{\gamma}{N})
\bigr] \bigr\}}{N \log \bigl[ \cosh(\frac{\gamma}{N})\bigr]}
\underset{\gamma \rightarrow 0}{\sim} 1,
\end{align*}
this can be rewritten as
\begin{multline*}
\frac{\beta N}{4 \gamma} \log \bigl[ \cosh(\tfrac{\gamma}{N})\bigr]
\Bigl( \pi^i_{\exp( - \beta \Upsilon(\gamma) R)} \Bigr)^{\otimes 2}
\Bigl[ \Phi_{\frac{\gamma'}{N}}\bigl(M'\bigr) \Bigr] \\
\leq \log \biggl\{ \pi^i_{\exp( - \beta r)}
\biggl[ \exp \Bigl\{ \beta \tfrac{N}{\gamma} \log \bigl[ \cosh(\tfrac{\gamma}{N})\bigr]
\pi^i_{\exp( - \beta r)}(m') \Bigr\} \biggr] \biggr\} \\
- \biggl( \frac{2 \beta \gamma}{ N^2 \log \bigl[ \cosh(\frac{\gamma}{N})\bigr]}
+ \frac{\beta N \log \bigl[ \cosh(\frac{\gamma}{N}) \bigr]}{2 \gamma \gamma'}
\biggr) \log \biggl( \frac{\eta}{2} \biggr).
\end{multline*}
It is now tempting to simplify the picture  a little bit by setting $\gamma' = \gamma$,
leading to
\newcommand{\un}[1]{\underline{#1}}
\begin{prop}\mypoint
With $\PP$ probability at least $1 - \eta$, for any $i \in \NN$,
any $\beta \in \RR_+$,
\begin{multline*}
\C{C}\bigl[ \pi^i_{\exp( - \beta r)}\bigr] \geq
\un{\C{C}} (i, \beta) \\*
\shoveleft{\qquad \overset{\text{\rm def}}{=}
\frac{1}{\zeta - 1}
\Biggl\{
\frac{N}{4} \log \bigl[ \cosh(\tfrac{\zeta \beta}{N})\bigr]
\Bigl(\pi^i_{\exp( - \beta \Upsilon(\zeta \beta) R)} \Bigr)^{\otimes 2} \Bigl[
\Phi_{\frac{\zeta \beta}{N}}\bigl(M'\bigr) \Bigr]}\\* +
\Biggl( \frac{2 \zeta^2 \beta^2 }{N^2 \log \bigl[ \cosh(\frac{\zeta \beta}{N})\bigr]}
+ \frac{N \log \bigl[ \cosh(\frac{\zeta \beta}{N})\bigr]}{2 \zeta \beta}
\Biggr)
\log \bigl[ 2^{-1} \nu(\beta) \mu(i) \eta \bigr] \\*
- (\zeta + 1) \Bigl\{
\log \bigl[ \nu(\beta) \mu(i) \bigr] + 2^{-1} \log \bigl( 3^{-1}
\epsilon \bigr) \Bigr\} \Biggr\},
\end{multline*}
where $\C{C}\bigl[ \pi^i_{\exp( - \beta r)}\bigr]$ is defined by equation \myeq{eq2.18}.
\end{prop}

We are now going to analyse Theorem \ref{thm1.58} (page \pageref{thm1.58}).
For this, we will also need an upper bound for $S_{\lambda}(\rho, \rho')$,
defined by equation \myeq{eq2.13},
using $M'$ and empirical complexities, because of the special relations
between empirical complexities induced by the selection
algorithm.
To this purpose, a useful alternative to Proposition
\ref{prop1.60} (page \pageref{prop1.60}) is to write,
with $\PP$ probability at least
$1 - \eta$,
\begin{multline*}
\gamma \rho \otimes \rho' (m') \leq \gamma \rho \otimes \rho'
\bigl[ \Phi_{- \frac{\gamma}{N}} \bigl(M'\bigr) \bigr]
\\ + \C{K}\bigl[ \rho, \pi^i_{\exp( - \lambda R)} \bigr] +
\C{K} \bigl[ \rho' , \pi^j_{\exp( - \lambda' R)} \bigr] - \log(\eta),
\end{multline*}
and thus at least with $\PP$ probability $1 - 3 \eta$,
\begin{multline*}
\gamma \rho \otimes \rho' (m') \leq \gamma \rho \otimes \rho'
\bigl[ \Phi_{- \frac{\gamma}{N}} \bigl(M' \bigr) \bigr] \\
+ (1 - \zeta^{-1})^{-1}
\biggl\{ \C{K}\bigl[ \rho, \pi^i_{\exp( - \beta r)} \bigr]
\\ + \log \Bigl\{ \pi^i_{\exp( - \beta r)} \Bigl[ \exp \bigl\{
\tfrac{N}{\zeta} \log \bigl[ \cosh \bigl( \tfrac{\zeta \beta}{N} \bigr)
\bigr] \rho( m' ) \bigr\} \Bigr] \Bigr\} - \zeta^{-1} \log(\eta) \biggr\}
\\ + (1 - \zeta^{-1})^{-1}
\biggl\{ \C{K}\bigl[ \rho, \pi^j_{\exp( - \beta' r)} \bigr]
\\ + \log \Bigl\{ \pi^j_{\exp( - \beta' r)} \Bigl[ \exp \bigl\{
\tfrac{N}{\zeta} \log \bigl[ \cosh \bigl( \tfrac{\zeta \beta'}{N} \bigr)
\bigr] \rho( m' ) \bigr\} \Bigr] \Bigr\} - \zeta^{-1} \log(\eta) \biggr\}\\  - \log(\eta).
\end{multline*}
When $\rho = \pi^i_{\exp( - \beta r)}$ and $\rho' = \pi^j_{\exp( - \beta' r)}$,
we get with $\PP$ probability at least $1 - \eta$, for any $\beta$, $\beta'$,
$\gamma \in \RR_+$, any $i$, $j \in \NN$,
\begin{multline*}
\gamma \rho \otimes \rho'
(m') \leq \gamma \rho \otimes \rho' \bigl[ \Phi_{- \frac{\gamma}{N}}
\bigl[ \bigl( M' \bigr) \bigr]  \\
+ \C{C}(\rho) + \C{C}(\rho') - \frac{\zeta +1 }{\zeta - 1}
\biggl[ \log \bigl[ 3^{-1} \nu(\gamma) \eta \bigr]
\biggr].
\end{multline*}
\begin{prop}\mypoint
With $\PP$ probability at least $1 - \eta$, for any $\rho = \pi^i_{\exp( - \beta
r)}$, any $\rho' = \pi^j_{\exp( - \beta' r)} \in \C{P}$,
\begin{multline*}
S_{\lambda}(\rho, \rho') \leq \frac{N}{\lambda} \log \bigl[ \cosh(\tfrac{\lambda}{N})
\bigr] \rho \otimes \rho' \bigl[ \Phi_{- \frac{\gamma}{N}} \bigl( M'
\bigr) \bigr] \\
+ \frac{1 + \frac{N}{\gamma} \log \bigl[ \cosh(\frac{\lambda}{N})\bigr]}{\lambda}
\bigl[ \C{C}(\rho) + \C{C}(\rho') \bigr] \\
- \frac{(\zeta+1)}{(\zeta-1)\lambda} \biggl\{
\log \bigl[ 3^{-1} \nu(\lambda)  \epsilon \bigr] +
\tfrac{N}{\gamma} \log \bigl[ \cosh\bigl( \tfrac{\lambda}{N} \bigr)
\bigr] \log \bigl[ 3^{-1} \nu(\gamma) \eta \bigr] \biggr\}.
\end{multline*}
\end{prop}

In order to analyse Theorem \ref{thm1.58} (page \pageref{thm1.58}),
we need to index $\C{P} = \bigl\{ \rho_1, \dots, \rho_M \bigr\}$
in order of increasing empirical
complexity $\C{C}(\rho)$. To deal in a convenient way
with this indexation, we will write $\ov{\C{C}}(i,\beta)$
as $\ov{\C{C}} \bigl[ \pi^i_{\exp( - \beta r)} \bigr]$,
$\un{\C{C}}(i, \beta)$ as $\un{\C{C}}\bigl[ \pi^i_{\exp ( - \beta r)} \bigr]$,
and $\ov{S} \bigl[ (i, \beta), (j, \beta') \bigr]$
as $\ov{S} \bigl[ \pi^i_{\exp( - \beta r)}, \pi^j_{\exp( - \beta' r)}
\bigr]$.
\vspace{6pt}

With $\PP$ probability at least $1 - \epsilon$,
when $\wh{t} \leq j < \wh{k}$, as we already saw,
$$
\rho_{\wh{k}}(R) \leq \rho_i(R) \leq \rho_j(R) + \inf_{\lambda \in \RR_+}
\Xi_{\frac{\lambda}{N}} \bigl[ 2 S_{\lambda}(\rho_j, \rho_i) \bigr],
$$
where $i=t(j) < \wh{t}$. Therefore, with $\PP$ probability at least $1 - \epsilon
- \eta$,
\begin{multline*}
\rho_i(R) \leq \rho_j(R) + \inf_{\lambda \in \RR_+} \Xi_{\frac{\lambda}{N}} \Biggl\{
2 \frac{N}{\lambda} \log \bigl[ \cosh\bigl( \tfrac{\lambda}{N} \bigr)
\bigr] \rho_j \otimes \rho_i \bigl[ \Phi_{- \frac{\gamma}{N}}
\bigl( M' \bigr) \bigr] \\
+ 4 \frac{1 + \frac{N}{\gamma} \log \bigl[ \cosh \bigl(\tfrac{\lambda}{N}
\bigr) \bigr]}{\lambda} \C{C}(\rho_j)
\\ - \frac{(\zeta +1)}{(\zeta - 1)\lambda} \biggl\{
\log \bigl[ 3^{-1} \nu(\lambda) \epsilon \bigr]
+ \tfrac{N}{\gamma} \log \bigl[ \cosh \bigl( \tfrac{\lambda}{N}
\bigr) \bigr] \log \bigl[ 3^{-1}
\nu(\gamma) \eta \bigr] \biggr\} \Biggr\}.
\end{multline*}
We can now remark that
$$
\Xi_a(p+q) \leq \Xi_a(p) + q \Xi'_a(p) q
\leq \Xi_a(p) + \Xi'_a(0) q = \Xi_a(p) + \frac{a}{\tanh(a)} q
$$
and that
$$
\Phi_{-a}(p+q) \leq \Phi_{-a}(p) + \Phi_{-a}'(0) q =
\Phi_{-a}(p) + \frac{\exp(a)-1}{a} q.
$$
Moreover, assuming as usual without substantial loss of\vspace*{2pt}
generality that there exists $\wt{\theta} \in \arg \min_{\Theta} R$,
we can split $M'(\theta, \theta') \leq M'(\theta, \wt{\theta}) +
M'(\wt{\theta}, \theta')$.  Let us then consider the \emph{expected
margin function} defined by
$$
\varphi(y) = \sup_{\theta \in \Theta} M'(\theta, \wt{\theta}) - y
R'(\theta, \wt{\theta}), \quad y \in \RR_+,
$$
and let us write for any $y \in \RR_+$,
\begin{multline*}
\rho_j \otimes \rho_i \bigl[ \Phi_{-\frac{\lambda}{N}}\bigl( M' \bigr) \bigr]
\leq \rho_j \otimes \rho_i \bigl\{ \Phi_{-\frac{\gamma}{N}}\bigl[
M'(.,\wt{\theta}) + y R'(., \wt{\theta}) + \varphi(y) \bigr] \bigr\} \\
\leq \rho_j \bigl\{ \Phi_{-\frac{\lambda}{N}} \bigl[ M'(., \wt{\theta})
+ \varphi(y) \bigr] \bigr\}
+ \frac{N y \bigl[ \exp( \frac{\gamma}{N} ) - 1\bigr] }{\gamma}
\bigl[ \rho_i(R) - R(\wt{\theta}) \bigr]
\end{multline*}
and
\begin{multline*}
\Biggl( 1 -
\frac{2 y N \bigl[ \exp( \frac{\gamma}{N}) - 1\bigr]
\log \bigl[ \cosh\bigl( \tfrac{\lambda}{N}
\bigr)\bigr]}{\gamma
\tanh \bigl( \tfrac{\lambda}{N} \bigr)} \Biggr) \bigl[ \rho_i(R) - R(\wt{\theta}) \bigr]
\\ \leq
\bigl[ \rho_j(R) - R(\wt{\theta})\bigr] + \Xi_{\frac{\lambda}{N}} \Biggl\{
\frac{2 N}{\lambda} \log \bigl[ \cosh \bigl( \tfrac{\lambda}{N} \bigr)
\bigr] \rho_j \bigl\{ \Phi_{- \frac{\gamma}{N}} \bigl[
M'(., \wt{\theta}) + \varphi(y) \bigr] \bigr\}
\\
+ 4 \frac{1 + \frac{N}{\gamma} \log \bigl[ \cosh \bigl( \frac{\lambda}{N} \bigr) \bigr] }
{\lambda}  \C{C}(\rho_j) \\
- \frac{2(\zeta + 1)}{(\zeta - 1)\lambda} \biggl\{
\log \bigl[ 3^{-1} \nu(\lambda) \epsilon \bigr]
+ \tfrac{N}{\gamma} \log \bigl[ \cosh \bigl(
\tfrac{\lambda}{N} \bigr) \bigr] \log \bigl[ 3^{-1} \nu(\gamma)
\eta \bigr] \biggr\} \Biggr\}.
\end{multline*}

With $\PP$ probability at least $1 - \epsilon - \eta$,
for any $\lambda$, $\gamma$, $x$, $y \in \RR_+$, any
$j \in \bigl\{ \wh{t}, \dots, \wh{k}-1 \bigr\}$,
\begin{multline*}
\rho_{\wh{k}}(R) - R(\wt{\theta}) \leq \rho_i(R) - R(\wt{\theta})
\\ \shoveleft{\quad \leq \Biggl( 1 -
\frac{2 y N \bigl[ \exp( \frac{\gamma}{N}) - 1\bigr]
\log \bigl[ \cosh\bigl( \tfrac{\lambda}{N}
\bigr)\bigr]}{\gamma
\tanh \bigl( \tfrac{\lambda}{N} \bigr)} \Biggr)^{-1}
\Biggl\{} \\
\shoveright{ \Biggl( 1 +
\frac{2 x N \bigl[ \exp( \frac{\gamma}{N}) - 1\bigr]
\log \bigl[ \cosh\bigl( \tfrac{\lambda}{N}
\bigr)\bigr]}{\gamma
\tanh \bigl( \tfrac{\lambda}{N} \bigr)} \Biggr) \bigl[ \rho_j(R) -
R(\wt{\theta}) \bigr] \quad} \\ \shoveleft{+
\Xi_{\frac{\lambda}{N}} \biggl\{
\frac{2 N}{\lambda} \log \bigl[ \cosh \bigl( \tfrac{\lambda}{N} \bigr) \bigr]
\Phi_{- \frac{\gamma}{N}} \bigl[
\varphi(x) + \varphi(y) \bigr]
}\\ + 4 \frac{1 + \frac{N}{\gamma} \log \bigl[ \cosh \bigl( \frac{\lambda}{N}
\bigr) \bigr]}{\lambda} \ov{\C{C}}(\rho_j) \\
\qquad - \frac{2 (\zeta + 1)}{(\zeta - 1) \lambda} \Bigl\{
\log \bigl[ 3^{-1} \nu(\lambda) \epsilon \bigr]
+ \tfrac{N}{\gamma} \log \bigl[ \cosh \bigl( \tfrac{\lambda}{N} \bigr)
\bigr] \log \bigl[ 3^{-1} \nu(\gamma) \eta \bigr] \Bigr\} \biggr\}
\Biggr\}.
\end{multline*}
Now we have to get an upper bound for $\rho_j(R)$.
We can write $\rho_j = \pi^\ell_{\exp( - \beta' r)}$, as we
assumed that all the posterior distributions in $\C{P}$ are
of this special form.
Moreover, we already know from Theorem \ref{thm1.43} (page \pageref{thm1.43})
that with $\PP$ probability at least $1 - \eta$,
\begin{multline*}
\bigl[N \sinh\bigl( \tfrac{\beta'}{N} \bigr) -  \beta' \zeta^{-1} \bigr]
\bigl[ \pi^{\ell}_{\exp( - \beta' r)}(R) - \pi^{\ell}_{\exp( - \beta' \zeta^{-1} R)}(R) \bigr]
\\ \leq C^{\ell}(\beta' \zeta^{-1}, \beta') - \log\bigl[ \nu(\beta') \mu(\ell)
\eta \bigr].
\end{multline*}
This proves that with $\PP$ probability at least $1 - \epsilon - 2 \eta$,
\begin{multline*}
\rho_{\wh{k}}(R) \leq R(\wt{\theta})
\\ \shoveleft{\quad + \biggl( 1 - \frac{2 y N \bigl[ \exp\bigl(\frac{\gamma}{N}\bigr) -1 \bigr]
\log \bigl[ \cosh\bigl(\frac{\lambda}{N}\bigr) \bigr]}{\gamma
\tanh\bigl(\frac{\lambda}{N}\bigr)} \biggr)^{-1}
\Biggl\{}\\ \shoveleft{\quad \biggl( 1 + \frac{2 x N \bigl[
\exp\bigl(\frac{\gamma}{N}\bigr) -1 \bigr]
\log \bigl[ \cosh\bigl(\frac{\lambda}{N}\bigr) \bigr]}{\gamma
\tanh\bigl(\frac{\lambda}{N}\bigr)} \biggr)
}\\ \shoveright{\times \Biggl( \pi^{\ell}_{\exp( - \zeta^{-1} \beta' R)}(R)
- R(\wt{\theta})
+ \frac{C^{\ell}(\zeta^{-1}\beta', \beta') - \log \bigl[
\nu(\beta') \mu(\ell) \eta \bigr]}{
N \sinh (\frac{\beta'}{N}) - \zeta^{-1} \beta'} \Biggr)\quad}\\
+ \Xi_{\frac{\lambda}{N}} \biggl\{
\frac{2N}{\lambda} \log \bigl[ \cosh \bigl(\tfrac{\lambda}{N} \bigr)
\bigr] \Phi_{- \frac{\gamma}{N}} \bigl[
\varphi(x) + \varphi(y) \bigr]
\\ + 4 \frac{ 1 + \frac{N}{\gamma} \log \bigl[ \cosh \bigl( \frac{\lambda}{N}
\bigr) \bigr]}{\lambda} \ov{\C{C}} (\ell, \beta') \\
- \frac{2 (\zeta + 1)}{(\zeta - 1) \lambda} \Bigl\{
\log \bigl[ 3^{-1} \nu(\lambda) \epsilon \bigr]
+ \tfrac{N}{\gamma}
\log \bigl[ \cosh \bigl( \tfrac{\lambda}{N} \bigr) \bigr]
\log \bigl[ 3^{-1} \nu(\gamma) \eta \bigr] \Bigr\}
\biggr\} \Biggr\}.
\end{multline*}

The case when $j \in \bigl\{ \wh{k} + 1, \dots, M \bigr\} \setminus (\arg \max t)$
is dealt with exactly in the same way, with $i = t(j)$ replaced directly with $\wh{k}$
itself, leading to the same inequality.

The case when $j \in (\arg \max t)$ is dealt with bounding first
$\rho_{\wh{k}}(R) - R( \wt{\theta})$ in terms of $\rho_{\wh{t}}(R)
- R(\wt{\theta})$, and this latter in terms of $\rho_j(R) - R(\wt{\theta})$.
Let us put
\begin{align}
\nonumber A(\lambda, \gamma) & = \Biggl( 1 - \frac{2 x N \bigl[
\exp \bigl( \frac{\gamma}{N} \bigr) - 1 \bigr]
\log \bigl[ \cosh\bigl( \frac{\lambda}{N} \bigr) \bigr]}{
\gamma \tanh\bigl( \frac{\lambda}{N}\bigr)} \Biggr),\\
\nonumber
B(\lambda, \gamma) & = 1 + \frac{2 y N \bigl[ \exp\bigl( \frac{\gamma}{N}
\bigr) - 1 \bigr]
\log \bigl[ \cosh \bigl( \frac{\lambda}{N} \bigr) \bigr]}{ \gamma
\tanh \bigl (\frac{\lambda}{N} \bigr) },\\
\label{eq1.35}
D(\lambda, \gamma, \rho_j) & =
\begin{aligned}[t]
\Xi_{\frac{\lambda}{N}}
\biggl\{ \frac{2N}{\lambda} \log \bigl[ \cosh\bigl( \tfrac{\lambda}{N} \bigr)
\bigr] & \Phi_{- \frac{\gamma}{N}}
\bigl[ \varphi(x) + \varphi(y) \bigr] \\
+ 4 & \frac{1 + \frac{N}{\gamma} \log \bigl[ \cosh \bigl( \frac{\lambda}{N}
\bigr) \bigr]}{\lambda}
\ov{\C{C}}(\rho_j)\\
- \frac{2(\zeta+1)}{(\zeta-1)\lambda} \Bigl\{ \log \bigl[
3^{-1} & \nu(\lambda) \epsilon \bigr]
\\ + \frac{N}{\gamma} \log \bigl[ \cosh \bigl( & \tfrac{\lambda}{N} \bigr)
\bigr] \log
\bigl[ 3^{-1} \nu(\gamma) \eta \bigr] \Bigr\} \biggr\},
\end{aligned}
\end{align}
where $\ov{\C{C}}(\rho_j) = \ov{\C{C}}(\ell, \beta')$
is defined, when $\rho_j = \pi^{\ell}_{\exp( - \beta'r)}$,
by equation \myeq{eq1.35Bis}.
We obtain, still with $\PP$ probability $1 - \epsilon - 2 \eta$,
\begin{align*}
\rho_{\wh{k}}(R) - R(\wt{\theta}) & \leq
\frac{B(\lambda, \gamma)}{A(\lambda, \gamma)} \bigl[ \rho_{\wh{t}}(R)
- R(\wt{\theta}) \bigr] + \frac{D(\lambda, \gamma, \rho_j)}{A(
\lambda, \gamma)},\\
\rho_{\wh{t}}(R) - R(\wt{\theta}) & \leq
\frac{B(\lambda, \gamma)}{A(\lambda, \gamma)} \bigl[ \rho_j(R)
- R(\wt{\theta}) \bigr] + \frac{D(\lambda, \gamma, \rho_j)}{A(
\lambda, \gamma)}.
\end{align*}
The use of the factor $D(\lambda, \gamma, \rho_j)$ in the first of
these two inequalities, instead of
$D(\lambda, \gamma, \rho_{\wh{t}})$,
is justified by the fact that $\C{C}(\rho_{\wh{t}})
\leq \C{C}(\rho_j)$. Combining the two we get
$$
\rho_{\wh{k}}(R) \leq R(\wt{\theta})
+ \frac{B( \lambda, \gamma)^2}{A(\lambda, \gamma)^2}  \bigl[ \rho_j(R)
- R(\wt{\theta}) \bigr] + \biggl[
\frac{B(\lambda, \gamma)}{A(\lambda, \gamma)} +1 \biggr]
\frac{D(\lambda, \gamma, \rho_j)}{A(\lambda, \gamma)}.
$$
Since it is the worst bound of all cases, it holds for any value of $j$,
proving
\begin{thm}\mypoint
\label{thm1.64}
With $\PP$ probability at least $1 - \epsilon - 2 \eta$,
\begin{multline*}
\rho_{\wh{k}}(R) \leq R(\wt{\theta}\,) + \inf_{i, \beta, \lambda,
\gamma, x, y} \Biggl\{ \\
\frac{B(\lambda, \gamma)^2}{A(\lambda, \gamma)^2} \Bigl[
\pi^i_{\exp( - \beta r)}(R) - R(\wt{\theta}) \Bigr] + \biggl[ \frac{B(\lambda, \gamma)}{A(\lambda,
\gamma)} + 1 \biggr] \frac{D(\lambda, \gamma, \pi^i_{\exp( - \beta r)})}{
A(\lambda, \gamma)} \Biggr\}
\\ \leq R(\wt{\theta}) + \inf_{i, \beta, \lambda, \gamma, x, y}
\Biggl\{ \\*
\frac{B(\lambda, \gamma)^2}{A(\lambda, \gamma)^2}
\Biggl( \pi^i_{\exp( - \zeta^{-1} \beta R)}(R) - R(\wt{\theta})
+ \frac{C^i(\zeta^{-1} \beta, \beta) - \log \bigl[ \nu(\beta) \mu(i) \eta \bigr]}{
N \sinh\bigl( \frac{\beta}{N} \bigr) - \zeta^{-1} \beta} \Biggr)\\*
+ \biggl[ \frac{B(\lambda, \gamma)}{A(\lambda, \gamma)} + 1 \biggr]
\frac{D(\lambda, \gamma, \pi^i_{\exp( - \beta r)})}{A(\lambda, \gamma)} \Biggr\},
\end{multline*}
where the notation $A(\lambda, \gamma)$, $B(\lambda, \gamma)$
and $D(\lambda, \gamma, \rho)$ is defined by equation (\ref{eq1.35}
page \pageref{eq1.35}) and where the notation $C^i(\beta, \gamma)$
is defined in Proposition \ref{prop1.59} (page \pageref{prop1.59}).
\end{thm}

The bound is a little involved, but
as we will prove next, it gives the same rate as
Theorem \ref{thm1.50} (page
\pageref{thm1.50}) and its corollaries, when we work with
a single model (meaning that the support of $\mu$ is reduced to one
point) and the goal is to choose adaptively the temperature of the
Gibbs posterior, except for the appearance of the union bound
factor $-\log \bigl[ \nu(\beta) \bigr]$ which can be made of order
$\log \bigl[ \log(N) \bigr]$ without spoiling the order of
magnitude of the bound.

We will encompass the case when one must choose
between possibly several parametric models.
Let us assume that each $\pi^i$
is supported by some measurable parameter subset $\Theta_i$ (
meaning that $\pi^i(\Theta_i) = 1$), let us also assume
that the behaviour of $\pi^i$ is parametric in the sense that
there exists a dimension $d_i \in \RR_+$ such that
\begin{equation}
\label{parametric2}
\sup_{\beta \in \RR_+} \beta \bigl[ \pi^i_{\exp( - \beta R)}(R) - \inf_{\Theta_i} R
\bigr] \leq d_i.
\end{equation}
Then
\begin{multline*}
C^i(\lambda, \gamma) \leq \log \biggl\{ \pi^i_{\exp( - \lambda R)}
\biggl[ \exp \Bigl\{ 2N\sinh\bigl( \tfrac{\gamma}{2N} \bigr)^2
M'(., \wt{\theta}) \Bigr\} \biggr] \biggr\} \\
+ 2N \sinh \bigl( \tfrac{\gamma}{2N} \bigr)^2 \pi^i_{\exp( - \lambda R)}
\bigl[ M'(., \wt{\theta}) \bigr] \\
\leq \log \biggl\{ \pi^i_{\exp ( - \lambda R) } \biggl[
\exp 2 x N \sinh \bigl( \tfrac{\gamma}{2N} \bigr)^2 \bigl[
R - R(\wt{\theta}) \bigr] \Bigr\} \biggr] \biggr\} \\
+ 2 x N \sinh\bigl( \tfrac{\gamma}{2N} \bigr)^2 \pi^i_{\exp( - \lambda R)}
\bigl[ R - R(\wt{\theta}) \bigr] \\
+ 4 N \sinh \bigl( \tfrac{\gamma}{2N} \bigr)^2 \varphi(x) \\
\leq 2x N \sinh \bigl( \tfrac{\gamma}{2N} \bigr)^2
\pi^i_{\exp \{ - [ \lambda - 2xN\sinh(\frac{\gamma}{2N})^2]R\}}
\bigl[ R - R(\wt{\theta}) \bigr]  \\
+ 2 x N \sinh\bigl( \tfrac{\gamma}{2N} \bigr)^2 \pi^i_{\exp( - \lambda R)}
\bigl[ R - R(\wt{\theta}) \bigr] \\
+ 4 N \sinh \bigl( \tfrac{\gamma}{2N} \bigr)^2 \varphi(x).
\end{multline*}
Thus
\begin{multline*}
C^i(\lambda, \gamma) \leq
4 N \sinh \bigl( \tfrac{\gamma}{2N} \bigr)^2
\Biggl( x \bigl[ \inf_{\Theta_i} R - R(\wt{\theta}) \bigr]  + \varphi(x)
\\* + \frac{x d_i}{ 2 \lambda} + \frac{x d_i}{2 \lambda - 4 x N \sinh \bigl( \frac{\gamma}{2N} \bigr)^2}
\Biggr).
\end{multline*}
In the same way,
\begin{multline*}
\ov{\C{C}}(i, \beta) \leq \tfrac{8 N }{\zeta - 1}
\sinh \bigl( \tfrac{\zeta \beta}{2N} \bigr)^2
\Biggl[ x \bigl[ \inf_{\Theta_i} R - R(\wt{\theta}) \bigr] + \varphi(x) \\
+ \frac{\zeta x d_i}{2 N \sinh\bigl(\frac{\zeta \beta}{N}\bigr)} \biggl(
1 + \frac{1}{ 1
- x \zeta \tanh \bigl( \frac{\zeta \beta}{2N} \bigr)} \biggr) \Biggr]\\
+ 2 N \Bigl[ \exp \Bigl( \tfrac{ \zeta^2 \beta^2}{2 N^2(\zeta -1)} \Bigr)
- 1 \Bigr]
\Biggl( \varphi(x) + x \bigl[ \inf_{\Theta_i} R - R(\wt{\theta}) \bigr]
\\ + \frac{x \zeta d_i}{ N\sinh \bigl(\frac{\zeta \beta}{N} \bigr)
- x \zeta N \bigl[
\exp \bigl( \frac{\zeta^2 \beta^2}{2 N^2(\zeta-1)} \bigr)
- 1 \bigr]} \Biggr) \\
- \frac{(\zeta + 1)}{(\zeta-1)} \biggl[
2 \log \bigl[ \nu(\beta) \mu(i) \bigr] + \log \bigl(
\tfrac{\eta}{2} \bigr) \biggr].
\end{multline*}
In order to keep the right order of magnitude while simplifying
the bound, let us consider
\begin{multline}
\label{eq1.36}
C_1 = \max \biggl\{ \zeta - 1,
\Bigl( \tfrac{2N}{\zeta \beta_{\max}} \Bigr)^2 \sinh \Bigl( \tfrac{\zeta \beta_{\max}}{2N} \Bigr)^2,
\\ \tfrac{2N^2(\zeta-1)}{\zeta^2 \beta_{\max}^2} \Bigl[ \exp \Bigl( \tfrac{\zeta^2 \beta_{\max}^2}{2 N^2(\zeta -1)} \Bigr) -
1 \Bigr]  \biggr\}.
\end{multline}
Then, for any $\beta \in (0, \beta_{\max})$,
\begin{multline*}
\ov{\C{C}}(i, \beta) \leq \inf_{y \in \RR_+}
\frac{3 C_1 \zeta^2 \beta^2}{(\zeta - 1) N}
\Biggl[ y \bigl[ \inf_{\Theta_i} R - R(\wt{\theta}) \bigr] + \varphi(y)
+ \frac{y d_i}{ \beta
\bigl[ 1 - \frac{ y C_1 \zeta^2 \beta}{2 (\zeta-1)N} \bigr]}
\Biggr]
\\ - \frac{(\zeta+1)}{(\zeta -1)} \biggl[ 2
\log \bigl[ \nu(\beta) \mu(i) \bigr] + \log \bigl( \tfrac{\eta}{2} \bigr) \biggr].
\end{multline*}
Thus
\begin{multline*}
D\bigl[\lambda, \gamma, \pi^i_{\exp( - \beta r)} \bigr]
\leq \frac{\lambda}{N \tanh \bigl( \frac{\lambda}{N} \bigr)}
\Biggl\{ \frac{\lambda \bigl[ \exp \bigl( \tfrac{\gamma}{N} \bigr) - 1
\bigr]}{\gamma} \bigl[ \varphi(x) + \varphi(y) \bigr] \\*
+ 4 \frac{1 + \frac{\lambda^2}{2 N \gamma}}{\lambda}
\Biggl[ \frac{3 C_1 \zeta^2 \beta^2}{(\zeta -1) N}
\Biggl( z \bigl[ \inf_{\Theta_i} R - R(\wt{\theta}) \bigr]
+ \varphi(z) + \frac{z d_i}{ \beta \bigl[ 1 - \frac{z C_1 \zeta^2 \beta}{2
(\zeta -1) N} \bigr]} \Biggr)
\\* - \frac{(\zeta + 1)}{(\zeta - 1)} \biggl[
2 \log \bigl[ \nu(\beta) \mu(i) \bigr]
+ \log \bigl( \tfrac{\eta}{2} \bigr) \biggr] \Biggr]\\
- \frac{2(\zeta+1)}{(\zeta-1)\lambda} \biggl[
\log \bigl[3^{-1} \nu(\lambda) \epsilon \bigr]
+ \frac{\lambda^2}{2 N \gamma} \log \bigl[ 3^{-1} \nu(\gamma) \eta \bigr] \biggr]
\Biggr\}
\end{multline*}

If we are not seeking tight constants, we can take
for the sake of simplicity $\lambda = \gamma = \beta$,
$x = y$ and $\zeta = 2$.

Let us put
\begin{multline}
\label{eq1.37}
C_2 = \max \biggl\{ C_1,
\frac{N \bigl[ \exp \bigl(\frac{\beta_{\max}}{N}\bigr) - 1
\bigr]}{\beta_{\max}}, \\ \frac{
2 N \log \bigl[ \cosh \bigl(\frac{\beta_{\max}}{N} \bigr)
\bigr]}{\beta_{\max} \tanh \bigl( \tfrac{\beta_{\max}}{N} \bigr)},
\frac{\beta_{\max}}{N \tanh \bigl( \frac{\beta_{\max}}{N}
\bigr)} \biggr\},
\end{multline}
so that
\begin{align*}
A(\beta, \beta)^{-1} & \leq \biggl( 1 - \frac{C_2 x \beta }{N} \biggr)^{-1},\\
B(\beta, \beta) & \leq 1 + \frac{C_2 x \beta }{N},
\end{align*}
\begin{multline*}
D \bigl[ \beta, \beta, \pi^i_{\exp( - \beta r)} \bigr]
\leq C_2^2 \frac{2\beta}{N} \varphi(x) \\* +
\Bigl( 4 + \tfrac{2 \beta}{N} \Bigr)\frac{C_2}{\beta} \Biggl[
\frac{12 C_1 \beta^2}{N} \Biggl( z
\bigl[ \inf_{\Theta_i} R - R(\wt{\theta}) \bigr] + \varphi(z) + \frac{zd_i}{
\beta \bigl[ 1 - \frac{2 zC_1 \beta}{N} \bigr]} \Biggr)\\*
- 6 \log \bigl[ \nu(\beta) \mu(i) \bigr] - 3 \log\bigl(\tfrac{\eta}{2} \bigr)
\Biggr] \\*
- \frac{6 C_2 }{\beta} \biggl[
\log\bigl[3^{-1} \nu(\beta) \epsilon \bigr] +
\frac{\beta}{2N} \log \bigl[ 3^{-1} \nu(\beta) \eta \bigr] \biggr]
\end{multline*}
and
$$
C^i(\zeta^{-1} \beta, \beta) \leq \frac{C_1 \beta^2}{N}
\biggl( x \bigl[ \inf_{\Theta_i}R - R(\wt{\theta}) \bigr]
+ \varphi(x) \\
+ \frac{2 x d_i}{\beta \bigl[ 1 - \frac{x \beta}{N}
\bigr]} \biggr).
$$
This leads to
\begin{multline*}
\rho_{\wh{k}}(R)
\leq R(\wt{\theta})
+ \inf_{i, \beta}
\Biggl( \frac{1 + \frac{C_2x\beta}{N}}{1 - \frac{C_2x \beta}{N}} \Biggr)^2
\Biggl\{ \frac{2 d_i}{\beta} + \inf_{\Theta_i} R
- R(\wt{\theta}) \\ +
\frac{2}{\beta} \Biggl[\frac{C_1 \beta^2}{N}
\biggl( x \bigl[ \inf_{\Theta_i} R - R(\wt{\theta} \bigr]
+ \varphi(x) + \frac{2xd_i}{ \beta \bigl( 1 -
\frac{x \beta}{N} \bigr)} \biggr) \\
- \log \bigl[ \nu(\beta) \mu(i) \eta \bigr]
\Biggr] \Biggr\} \\
+ \frac{2}{ \Bigl(1 - \frac{C_2 x \beta}{N}
\Bigr)^2} \Biggl\{
C_2^2 \frac{2 \beta}{N} \varphi(x) \\ +
\Bigl(4 + \tfrac{2\beta}{N} \Bigr) \frac{C_2}{\beta}
\biggl[ \frac{12 C_1 \beta^2}{N}
\biggl( x \bigl[ \inf_{\Theta_i} R - R(\wt{\theta}) \bigr]
+ \varphi(x) +
\frac{xd_i}{\beta \bigl[ 1 - \frac{2 x C_1 \beta}{N} \bigr]}
\biggr) \\
- 6 \log \bigl[ \nu(\beta) \mu(i) \bigr] -
3 \log \bigl( \tfrac{\eta}{2} \bigr) \biggr] \\
- \frac{6 C_2}{\beta} \biggl[
\log \bigl[ 3^{-1} \nu(\beta) \epsilon \bigr]
+ \frac{\beta}{2N} \log \bigl[ 3^{-1} \nu(\beta) \eta
\bigr] \biggr] \Biggr\}.
\end{multline*}
We see in this expression that, in order to balance the various
factors depending on $x$ it is advisable to choose $x$
such that
$$
\inf_{\Theta_i} R - R(\wt{\theta}) = \frac{\varphi(x)}{x},
$$
as long as $x \leq \frac{N}{4C_2 \beta}$.

Following Mammen and Tsybakov, let us assume that the usual margin
assumption holds:
for some real constants $c > 0$ and $\kappa \geq 1$,
$$
R(\theta) - R(\wt{\theta}) \geq c \bigl[ D(\theta, \wt{\theta}) \bigr]^{\kappa}.$$
As $D(\theta, \wt{\theta}\,) \geq M'(\theta, \wt{\theta}\,)$,
this also implies the weaker assumption
$$
R(\theta) - R(\wt{\theta}\,) \geq c \bigl[
M'(\theta, \wt{\theta})\bigr]^{\kappa},\qquad \theta \in \Theta,
$$
which we will really need and use.
Let us take $\beta_{\max} = N$ and
$$
\nu = \frac{1}{\lceil \log_2(N) \rceil}
\sum_{k=1}^{\lceil \log_2(N) \rceil} \delta_{2^k}.
$$
Then, as we have already seen, $\varphi(x) \leq (1 - \kappa^{-1}
) \bigl( \kappa c x \bigr)^{- \frac{1}{\kappa-1}}$.
Thus ${\varphi(x)}/{x} \leq b x^{- \frac{\kappa}{\kappa-1}}$,
where $b = (1 - \kappa^{-1}) \bigl(\kappa c \bigr)^{- \frac{1}{\kappa-1}}$.
Let us choose accordingly
$$
x = \min \biggl\{ x_1 \overset{\text{\rm def}}{=}
\biggl(\frac{\inf_{\Theta_i} R - R(\wt{\theta})}{b}
\biggr)^{- \frac{\kappa-1}{\kappa}},
x_2 \overset{\text{\rm def}}{=} \frac{N}{4 C_2 \beta} \biggr\}.
$$
Using the
fact that when $r \in (0,\frac{1}{2})$,
$ \bigl( \frac{1+r}{1-r} \bigr)^2 \leq 1 + 16 r \leq 9$,
we get with $\PP$ probability at least $1 - \epsilon$, for
any $\beta \in \supp \nu$, in the case when $x = x_1 \leq x_2$,
\begin{multline*}
\rho_{\wh{k}}(R) \leq \inf_{\Theta_i} R
+ 538\, C_2^2 \frac{\beta}{N} b^{\frac{\kappa-1}{\kappa}}
\bigl[ \inf_{\Theta_i} R - R(\wt{\theta}) \bigr]^{\frac{1}{\kappa}}\\
+ \frac{C_2}{\beta} \biggl[ 138\, d_i + 166 \log \bigl[ 1 + \log_2(N) \bigr]
- 134 \log \bigl[ \mu(i) \bigr] - 102 \log(\epsilon) + 724\biggr],
\end{multline*}
and in the case when $x = x_2 \leq x_1$,
\begin{multline*}
\rho_{\wh{k}}(R) \leq \inf_{\Theta_i} R +
68 C_1 \bigl[ \inf_{\Theta_i} R - R(\wt{\theta}) \bigr]
+ 269\,C_2^2\frac{\beta}{N}\varphi(x) \\
+ \frac{C_2}{\beta} \biggl[ 138\, d_i + 166 \log \bigl[ 1 + \log_2(N) \bigr]
- 134 \log \bigl[ \mu(i) \bigr] - 102 \log(\epsilon) + 724\biggr]
\\ \leq \inf_{\Theta_i} R + 541 C_2^2 \frac{\beta}{N} \varphi(x) \\
+ \frac{C_2}{\beta} \biggl[ 138\, d_i + 166 \log \bigl[ 1 + \log_2(N) \bigr]
- 134 \log \bigl[ \mu(i) \bigr] - 102 \log(\epsilon) + 724\biggr].
\end{multline*}
Thus with $\PP$ probability at least $1 - \epsilon$,
\begin{multline*}
\rho_{\wh{k}}(R) \leq \inf_{\Theta_i} R +
\inf_{\beta \in (1,N) } 1082 \, C_2^2 \frac{\beta}{N} \max \biggl\{
b^{\frac{\kappa-1}{\kappa}}
\bigl[ \inf_{\Theta_i} R - R(\wt{\theta}) \bigr]^{\frac{1}{\kappa}},
\\* \shoveright{b \biggl( \frac{4 C_2 \beta}{N}
\biggr)^{\frac{1}{\kappa-1}} \biggr\}\quad}\\*
+ \frac{C_2}{\beta} \biggl[ 138\, d_i + 166 \log \bigl[ 1 + \log_2(N) \bigr]
\\* - 134 \log \bigl[ \mu(i) \bigr] - 102 \log(\epsilon) + 724\biggr].
\end{multline*}
\begin{thm}\mypoint
\label{thm1.71}
With probability at least $1 - \epsilon$, for any $i \in \NN$,
\begin{multline*}
\rho_{\wh{k}}(R) \leq \inf_{\Theta_i} R \\ + \max \left\{
\rule{0pt}{7ex} \right. 847 C_2^{\frac{3}{2}}
\sqrt{ \frac{b^{\frac{\kappa-1}{\kappa}} \bigl[
\inf_{\Theta_i} R - R(\wt{\theta}) \bigr]^{\frac{1}{\kappa}}
\Bigl\{ d_i + \log \Bigl( \frac{1 + \log_2(N)}{\epsilon \mu(i)} \Bigr)
+  5 \Bigr\}}{N}
},\\
2 C_2 \bigl[1082\,b\bigr]^{\frac{\kappa-1}{2\kappa-1}} 4^{\frac{1}{2\kappa-1}}
\left\{\frac{166 C_2 \Bigl[
d_i + \log \Bigl( \frac{1 + \log_2(N)}{\epsilon \mu(i)}\Bigr) + 5 \Bigr]}{N}
\right\}^{\frac{\kappa}{2\kappa -1}} \left. \rule{0pt}{7ex} \right\},
\end{multline*}
where $C_2$, given by equation (\ref{eq1.37} page \pageref{eq1.37}),
will in most cases be close to $1$, and in any case less
than $3.2$.
\end{thm}

This result gives a bound of the same
form as that given in Theorem \ref{thm1.50} (page \pageref{thm1.50})
in the special case when there is only one model --- that is when
$\mu$ is a Dirac mass, for instance $\mu(1) = 1$, implying that
$R(\wt{\theta}_1) - R(\wt{\theta}) = 0$. Morover
 the parametric complexity assumption we made for this
theorem, given by equation (\ref{parametric2}
page \pageref{parametric2}), is
weaker than the one used in Theorem \ref{thm1.50}
and described by equation (\ref{parametric}, page \pageref{parametric}).
When there is more than one model, the bound shows that
the estimator makes a trade-off between model accuracy,
represented by $\inf_{\Theta_i} R - R(\wt{\theta})$,
and dimension, represented by $d_i$, and that for optimal parametric
sub-models, meaning those for which $\inf_{\Theta_i} R = \inf_{\Theta}
R$,
the estimator does at least as well as the minimax optimal convergence
speed in the best of these.

Another point is that we
obtain more explicit constants than in Theorem \ref{thm1.50}.
It is also clear that a more
careful choice of parameters could have brought some improvement
in the value of these constants.

These results show that the selection scheme described in this section
is a good candidate to perform temperature selection of a Gibbs posterior
distribution built within a single parametric model in a rate optimal way,
as well as a proposal with proven performance bound for model
selection.

\section{Two step localization}

\subsection{Two step localization of bounds relative to a Gibbs prior}

  Let us reconsider the case where we want to choose adaptively among a family
of parametric models. Let us thus assume that the parameter
set is a disjoint union of measurable sub-models, so that we can write
$\Theta = \sqcup_{m \in M} \Theta_m$, where $M$ is some measurable
index set. Let us choose some prior probability distribution
on the index set $\mu \in \C{M}_+^1(M)$, and some regular conditional
prior distribution $\pi: M \rightarrow \C{M}_+^1(\Theta)$,
such that $\pi(i, \Theta_i) = 1$, $i \in M$. Let us then study some
arbitrary posterior distributions $\nu: \Omega \rightarrow \C{M}_+^1(M)$
and $\rho: \Omega \times M: \rightarrow \C{M}_+^1(\Theta)$, such
that $\rho(\omega, i, \Theta_i) = 1$, $\omega \in \Omega$, $i \in M$.
We would like to compare $\nu \rho(R)$ with some doubly localized
prior distribution $\mu_{\exp[ - \frac{\beta}{1 + \zeta_2} \pi_{
\exp( - \beta R)}(R)]} \bigl[ \pi_{\exp( - \beta R)} \bigr](R)$
(where $\zeta_2$ is a positive parameter to be set as needed later on).
To ease notation we will define two prior distributions (one
being more precisely a conditional distribution) depending on
the positive real parameters $\beta$ and $\zeta_2$, putting
\begin{equation}
\label{eqprior}
\ov{\pi} = \pi_{\exp( - \beta R)}
\text{ and }\ov{\mu} = \mu_{\exp[ - \frac{\beta}{1 + \zeta_2}
\ov{\pi}(R)]}.
\end{equation}

Similarly to Theorem \ref{thm2.2.18} on page \pageref{thm2.2.18}
we can write for any positive real constants $\beta$ and $\gamma$
\begin{multline*}
\PP \biggl\{ (\ov{\mu}\,\ov{\pi}) \otimes (\ov{\mu}\,\ov{\pi})
\biggl[ \exp \Bigl[ - N \log \bigl[  1 - \tanh(\tfrac{\gamma}{N})R' \bigr]
\\ - \gamma r' - N \log \bigl[
\cosh(\tfrac{\gamma}{N})\bigr] m' \Bigr] \biggr] \biggr\}
\leq 1,
\end{multline*}
and deduce, using Lemma \ref{lemma1.3} on page \pageref{lemma1.3}, that
\begin{multline}
\label{eq1.31}
\PP \biggl\{ \exp \biggl[
\sup_{\nu \in \C{M}_+^1(M)} \sup_{\rho: M \rightarrow \C{M}_+^1(\Theta)}
\Bigl\{ - N
\log \bigl[ 1 - \tanh(\tfrac{\gamma}{N})
(\nu \rho - \ov{\mu}\,\ov{\pi}) (R) \bigr]\\* - \gamma (\nu \rho - \ov{\mu}
\,\ov{\pi})(r)
- N \log \bigl[ \cosh(\tfrac{\gamma}{N}) \bigr] (\nu \rho) \otimes
(\ov{\mu}\,\ov{\pi}) (m') \\* - \C{K}(\nu, \ov{\mu}) - \nu
\bigl[ \C{K}(\rho, \ov{\pi}) \bigr] \Bigr\} \biggr] \biggr\} \leq 1.
\end{multline}
This will be our starting point in comparing
$\nu \rho(R)$ with $\ov{\mu}\,\ov{\pi}(R)$.
However, obtaining an empirical bound will require some supplementary efforts.
For each index of the model index set $M$, we can write
in the same way
$$
\PP \biggl\{ \ov{\pi} \otimes \ov{\pi}
\biggl[ \exp \Bigl[ - N \log \bigl[  1 - \tanh(\tfrac{\gamma}{N})R' \bigr]
- \gamma r' - N \log \bigl[ \cosh(\tfrac{\gamma}{N})\bigr] m' \Bigr] \biggr] \biggr\}
\leq 1.
$$
Integrating this inequality with respect to $\ov{\mu}$ and using Fubini's lemma
for positive functions, we get
$$
\PP \biggl\{ \ov{\mu}(\ov{\pi} \otimes \ov{\pi})
\biggl[ \exp \Bigl[ - N \log \bigl[  1 - \tanh(\tfrac{\gamma}{N})R' \bigr]
- \gamma r' - N \log \bigl[ \cosh(\tfrac{\gamma}{N})\bigr] m' \Bigr] \biggr] \biggr\}
\leq 1.
$$
Note that $\ov{\mu}(\ov{\pi} \otimes \ov{\pi})$ is a probability
measure on $M \times \Theta \times \Theta$, whereas $(\ov{\mu}\,\ov{\pi})
\otimes (\ov{\mu}\,\ov{\pi})$ considered previously is a probability measure
on  $(M\times \Theta) \times (M \times \Theta)$.
We get as previously
\begin{multline}
\label{eq1.31bis}
\PP \biggl\{ \exp \biggl[
\sup_{\nu \in \C{M}_+^1(M)}
\sup_{\rho: M \rightarrow \C{M}_+^1(\Theta)} \Bigl\{
- N
\log \bigl[ 1 - \tanh(\tfrac{\gamma}{N})
\nu (\rho - \ov{\pi}) (R) \bigr]
\\ - \gamma \nu (\rho - \ov{\pi})(r) - N \log
\bigl[\cosh(\tfrac{\gamma}{N})\bigr]
\nu ( \rho \otimes \ov{\pi} ) (m') \\ - \C{K}(\nu, \ov{\mu})
- \nu \bigl[ \C{K}(\rho, \ov{\pi}) \bigr]
\Bigr\} \biggr] \biggr\} \leq 1.
\end{multline}
Let us finally recall that
\begin{align}
\C{K}(\nu, \ov{\mu}) & = \tfrac{\beta}{1 + \zeta_2} (\nu - \ov{\mu})\ov{\pi}(R) + \C{K}(\nu, \mu)
- \C{K}(\ov{\mu}, \mu),\\
\label{eq1.31ter}
\C{K}(\rho, \ov{\pi}) & = \beta (\rho - \ov{\pi})(R) + \C{K}(\rho, \pi)
- \C{K}(\ov{\pi}, \pi).
\end{align}
From equations \eqref{eq1.31}, \eqref{eq1.31bis} and \eqref{eq1.31ter} we deduce
\begin{prop}\mypoint
\label{prop1.58}
For any positive real constants $\beta$, $\gamma$ and $\zeta_2$,
with $\PP$ probability at least $1 - \epsilon$, for any posterior
distribution $\nu: \Omega \rightarrow \C{M}_+^1(M)$ and any conditional posterior
distribution $\rho: \Omega \times M \rightarrow \C{M}_+^1(\Theta)$,
\begin{multline*}
- N \log \bigl[ 1 - \tanh(\tfrac{\gamma}{N})(\nu \rho - \ov{\mu}\,\ov{\pi})(R)
\bigr] - \beta \nu(\rho - \ov{\pi})(R) \\ \leq \gamma (\nu \rho - \ov{\mu}\,\ov{\pi}) (r)
+ N \log \bigl[ \cosh(\tfrac{\gamma}{N}) \bigr] (\nu \rho) \otimes
(\ov{\mu}\,\ov{\pi}) (m') \\ + \C{K}(\nu, \ov{\mu}) + \nu \bigl[ \C{K}(\rho, \pi) \bigr]
- \nu \bigl[ \C{K}( \ov{\pi}, \pi) \bigr] + \log \bigl( \tfrac{2}{\epsilon} \bigr).
\end{multline*}
and
\begin{multline*}
- N \log \bigl[ 1 - \tanh(\tfrac{\gamma}{N}) \nu(\rho - \ov{\pi})(R) \bigr]
\\\leq \gamma \nu(\rho - \ov{\pi})(r)
+ N \log \bigl[ \cosh(\tfrac{\gamma}{N}) \bigr]
\nu( \rho\otimes \ov{\pi})(m') \\ + \C{K}(\nu, \ov{\mu}) + \nu\bigl[ \C{K}(\rho,
\ov{\pi}) \bigr] +
\log\bigl(\tfrac{2}{\epsilon}\bigr),
\end{multline*}
where the prior distribution $\ov{\mu}\,\ov{\pi}$ is defined by equation
\eqref{eqprior} on page \pageref{eqprior} and depends on $\beta$ and $\zeta_2$.
\end{prop}

Let us put for short
$$
T = \tanh(\tfrac{\gamma}{N}) \text{ and } C = N \log \bigl[ \cosh(\tfrac{\gamma}{N})
\bigr].
$$

\newcommand{\omu}{\ov{\mu}}
\newcommand{\opi}{\ov{\pi}}
We will use an entropy compensation strategy for which we need a couple
of entropy bounds.
We have according to Proposition \ref{prop1.58},
with $\PP$ probability at least $1 - \epsilon$,
\begin{multline*}
\nu \bigl[ \C{K}(\rho, \opi) \bigr]
= \beta \nu(\rho - \opi)(R) + \nu \bigl[ \C{K}(\rho, \pi) -
\C{K}(\opi, \pi) \bigr] \\\shoveleft{\qquad
\leq \frac{\beta}{NT} \biggl[ \gamma \nu(\rho - \opi) (r)
+ C \nu(\rho \otimes \opi)(m')} \\ + \C{K}(\nu, \omu)
+  \nu \bigl[ \C{K}( \rho, \opi) \bigr]
+ \log( \tfrac{2}{\epsilon} ) \biggr] \\ + \nu \bigl[ \C{K}(\rho, \pi)
- \C{K}(\opi, \pi) \bigr].
\end{multline*}
Similarly
\begin{multline*}
\C{K}(\nu, \omu) = \frac{\beta}{1 + \zeta_2} (\nu - \omu) \opi(R)
+ \C{K}(\nu, \mu) - \C{K}(\omu, \mu) \\
\leq \frac{\beta}{(1 + \zeta_2) NT} \biggl[
\gamma (\nu - \omu) \opi(r) + C (\nu \opi) \otimes ( \omu\,\opi) (m')
\\ + \C{K}(\nu, \omu) + \log (\tfrac{2}{\epsilon}) \biggr]
+ \C{K}(\nu, \mu) - \C{K}(\omu, \mu).
\end{multline*}
Thus, for any positive real constants $\beta$, $\gamma$ and $\zeta_i$,
$i = 1, \dots, 5$, with $\PP$ probability at least $1 - \epsilon$,
for any posterior distributions $\nu, \nu_3
: \Omega \rightarrow \C{M}_+^1(\Theta)$, any posterior conditional distributions
$\rho, \rho_1, \rho_2, \rho_4, \rho_5
: \Omega \times M \rightarrow \C{M}_+^1(\Theta)$,
\begin{multline*}
- N \log \bigl[ 1 - T (\nu \rho - \omu\,\opi)(R) \bigr]
- \beta \nu (\rho - \opi)(R) \\ \leq
\gamma (\nu \rho - \omu\,\opi)(r) + C (\nu \rho) \otimes (\omu\,\opi)(m')
\\
\hfill + \C{K}(\nu, \omu) + \nu \bigl[ \C{K}(\rho, \pi)
- \C{K}(\opi, \pi) \bigr] + \log(\tfrac{2}{\epsilon}),
\quad\\\quad
\zeta_1 \frac{NT}{\beta} \omu \bigl[ \C{K}(\rho_1, \opi) \bigr]
\leq \zeta_1 \gamma \omu(\rho_1 - \opi)(r) + \zeta_1 C \omu(\rho_1 \otimes \opi)(m')
\hfill \\ \hfill + \zeta_1 \omu \bigl[ \C{K}(\rho_1, \opi) \bigr] +
\zeta_1 \log( \tfrac{2}{\epsilon})
+ \zeta_1 \frac{NT}{\beta} \omu \bigl[ \C{K}(\rho_1, \pi)
- \C{K}(\opi, \pi) \bigr],\quad\\\quad
\zeta_2 \frac{NT}{\beta} \nu \bigl[ \C{K}(\rho_2, \opi) \bigr]
\leq \zeta_2 \gamma \nu(\rho_2- \opi)(r) + \zeta_2 C \nu(
\rho_2 \otimes \opi)(m') \hfill \\
+ \zeta_2 \C{K}(\nu, \omu) + \zeta_2 \nu \bigl[ \C{K}(\rho_2, \opi) \bigr]
+ \zeta_2 \log( \tfrac{2}{\epsilon}) \\ \hfill
+ \zeta_2 \frac{NT}{\beta} \nu \bigl[ \C{K}(\rho_2, \pi) - \C{K}(\opi, \pi)
\bigr],\quad\\\quad
\zeta_3 (1 + \zeta_2)\frac{ N T}{\beta} \C{K}(\nu_3, \omu)
\leq \zeta_3 \gamma( \nu_3 - \omu) \opi(r)
\hfill \\ +
\zeta_3 C \bigl[ (\nu_3 \opi) \otimes (\nu_3 \rho_1) + (\nu_3 \rho_1)
\otimes ( \omu \, \opi) \bigr] (m')
+ \zeta_3 \C{K}(\nu_3, \omu) + \zeta_3 \log(\tfrac{2}{\epsilon})
\\ \hfill + \zeta_3 (1 + \zeta_2)\frac{NT}{ \beta}
 \bigl[ \C{K}(\nu_3, \mu) - \C{K}(\ov{\mu}, \mu) \bigr],\quad\\\quad
\zeta_4 \frac{NT}{\beta} \nu_3 \bigl[ \C{K}(\rho_4, \opi) \bigr]
\leq \zeta_4 \gamma \nu_3(\rho_4 - \opi)(r) \hfill \\
+ \zeta_4 C \nu_3(\rho_4 \otimes \opi)
(m') + \zeta_4 \C{K}(\nu_3, \omu) + \zeta_4 \nu_3 \bigl[ \C{K}(\rho_4, \opi) \bigr]
+ \zeta_4 \log( \tfrac{2}{\epsilon}) \\
\hfill + \zeta_4 \frac{NT}{\beta} \nu_3 \bigl[ \C{K}(\rho_4,
\pi) - \C{K}( \opi, \pi) \bigr],
\quad\\\quad
\zeta_5 \frac{NT}{\beta} \omu \bigl[ \C{K}(\rho_5, \opi) \bigr]
\leq \zeta_5 \gamma \omu(\rho_5 - \opi)(r) + \zeta_5 C \omu(\rho_5 \otimes \opi)(m')
\hfill \\ \hfill + \zeta_5 \omu \bigl[ \C{K}(\rho_5, \opi) \bigr] +
\zeta_5 \log( \tfrac{2}{\epsilon})
+ \zeta_5 \frac{NT}{\beta} \omu \bigl[ \C{K}(\rho_5, \pi)
- \C{K}(\opi, \pi) \bigr].
\end{multline*}
Adding these six inequalities and assuming that
\begin{equation}
\label{cond1}
\zeta_4 \leq \zeta_3 \bigl[
( 1 + \zeta_2) \tfrac{NT}{\beta} - 1 \bigr],
\end{equation}
we find
\begin{multline*}
- N \log \bigl[ 1 - T (\nu \rho - \omu\,\opi)(R) \bigr]
- \beta (\nu \rho - \omu \, \opi)(R) \\\qquad \leq
- N \log \bigl[ 1 - T (\nu \rho - \omu\,\opi)(R) \bigr]
- \beta (\nu \rho - \omu \, \opi)(R)\hfill\\+
\zeta_1 \bigl( \tfrac{NT}{\beta} - 1\bigr)
\omu \bigl[ \C{K}(\rho_1, \opi)\bigr]
+ \zeta_2 \bigl( \tfrac{NT}{\beta} - 1 \bigr)
\nu \bigl[ \C{K}(\rho_2, \opi) \bigr] \\ +
\bigl[ \zeta_3(1 + \zeta_2) \tfrac{NT}{\beta} - \zeta_3
- \zeta_4 \bigr] \C{K}(\nu_3, \omu)\\\hfill
+ \zeta_4 \bigl( \tfrac{NT}{\beta} - 1 \bigr)
\nu_3 \bigl[ \C{K}(\rho_4, \opi) \bigr] +
\zeta_5 \bigl( \tfrac{NT}{\beta} - 1 \bigr)
\omu \bigl[ \C{K}(\rho_5, \opi) \bigr] \quad\\\qquad
\leq \gamma (\nu \rho - \omu\,\opi)(r)
+ \zeta_1 \gamma \omu(\rho_1 - \opi) (r) +
\zeta_2 \gamma \nu(\rho_2 - \opi) (r)
\hfill \\ + \zeta_3 \gamma(\nu_3 - \omu) \opi(r) +
\zeta_4 \gamma \nu_3(\rho_4 - \opi)(r) + \zeta_5 \gamma \omu(\rho_5 - \opi)
(r) \qquad\\ \hfill
+ C \bigl[ (\nu \rho) \otimes (\omu\,\opi)+ \zeta_1
\omu(\rho_1 \otimes \opi) + \zeta_2 \nu( \rho_2 \otimes \opi)\qquad\\
\quad + \zeta_3 (\nu_3 \opi) \otimes (\nu_3 \rho_1) +
\zeta_3 (\nu_3 \rho_1) \otimes ( \omu \, \opi)\hfill \\
\hfill + \zeta_4
\nu_3 ( \rho_4 \otimes \opi) + \zeta_5 \omu(\rho_5\otimes \opi) \bigr] (m')\qquad\\
\quad + (1 + \zeta_2) \bigl[\C{K}(\nu, \mu) - \C{K}(\omu, \mu)\bigr]
+ \nu \bigl[ \C{K}(\rho, \pi) - \C{K}(\opi, \pi) \bigr]\hfill\\
\hfill + \zeta_1 \tfrac{NT}{\beta} \omu \bigl[ \C{K}(\rho_1, \pi)
- \C{K}(\opi, \pi) \bigr] + \zeta_2 \tfrac{NT}{\beta}
\nu \bigl[ \C{K}(\rho_2, \pi) - \C{K}(\opi, \pi) \bigr] \qquad
\\\quad + \zeta_3 (1 + \zeta_2) \tfrac{NT}{\beta} \bigl[ \C{K}(\nu_3, \mu)
- \C{K}(\omu, \mu) \bigr]
+ \zeta_4 \tfrac{NT}{\beta} \nu_3 \bigl[ \C{K}( \rho_4, \pi)
- \C{K}(\opi, \pi) \bigr] \hfill \\
+ \zeta_5 \tfrac{NT}{\beta} \omu \bigl[
\C{K}(\rho_5, \pi) - \C{K}(\opi, \pi) \bigr]
+ (1 + \zeta_1 + \zeta_2 + \zeta_3 + \zeta_4 + \zeta_5 ) \log( \tfrac{2}{\epsilon}),
\end{multline*}
where we have also used the fact (concerning the $11$th
line of the preceding inequalities) that
\begin{multline*}
- \beta (\nu \rho - \ov{\mu}\,\ov{\pi} ) (R) + \C{K}(\nu, \ov{\mu}) +
\nu \bigl[ \C{K}( \rho, \ov{\pi}) \bigr]
\\ \leq - \beta(\nu \rho - \ov{\mu}\,\ov{\pi})(R) +
(1 + \zeta_2) \C{K}(\nu, \ov{\mu}) + \nu\bigl[
\C{K}(\rho, \ov{\pi}) \bigr] \\
= (1 + \zeta_2) \bigl[ \C{K}(\nu, \mu) - \C{K}(\ov{\mu}, \mu) \bigr]
+ \nu \bigl[ \C{K}(\rho, \pi) - \C{K}(\ov{\pi}, \pi) \bigr].
\end{multline*}
Let us now apply to $\opi$ (we shall later do the same with $\omu$)
the following inequalities, holding for any random
functions of the sample and the parameters $h: \Omega \times \Theta \rightarrow
\RR$ and $g: \Omega \times \Theta \rightarrow \RR$,
\begin{multline*}
\opi(g-h) - \C{K}(\opi, \pi) \leq
\sup_{\rho: \Omega \times M \rightarrow \C{M}_+^1(\Theta)} \rho( g - h) - \C{K}(\rho, \pi) \\
\shoveleft{\qquad = \log \bigl\{ \pi \bigl[ \exp (g - h)  \bigr] \bigr\}} \\
\shoveleft{\qquad \qquad =
\log \bigl\{ \pi \bigl[ \exp ( - h ) \bigr] \bigr\}
+ \log \bigl\{ \pi_{\exp( - h)} \bigl[ \exp (g) \bigr] \bigr\}}
\\ = - \pi_{\exp( - h)}(h) - \C{K}(\pi_{\exp( - h)}, \pi)
+ \log \bigl\{ \pi_{\exp( - h)} \bigl[ \exp (g) \bigr] \bigr\}.
\end{multline*}
When $h$ and $g$ are observable, and $h$ is not too far from
$\beta r \simeq \beta R$, this gives a way to replace $\opi$ with
a satisfactory empirical approximation.
We will apply this method, choosing $\rho_1$ and $\rho_5$ such that
$\omu\,\opi$ is replaced either with $\omu \rho_1$,
when it comes from the first two inequalities or
with $\omu \rho_5$ otherwise,
choosing $\rho_2$ such that $\nu \opi$ is replaced with $\nu \rho_2$
and $\rho_4$ such that $\nu_3 \opi$ is replaced with $\nu_3 \rho_4$. We will do
so because it leads to a lot of helpful cancellations.
For those to happen, we need to choose $\rho_i = \pi_{\exp( - \lambda_i r)}$,
$i=1,2,4$, where $\lambda_1$, $\lambda_2$ and $\lambda_4$ are such that
\begin{align}
\label{cond2}
(1 + \zeta_1) \gamma & = \zeta_1 \tfrac{NT}{\beta} \lambda_1,\\
\label{cond3}
\zeta_2 \gamma & = \bigl(1 + \zeta_2 \tfrac{NT}{\beta} \bigr) \lambda_2,\\
\label{cond4}
(\zeta_4 - \zeta_3) \gamma & = \zeta_4 \frac{NT}{\beta} \lambda_4,\\
\label{cond5}
\zeta_3 \gamma & = \zeta_5 \tfrac{NT}{\beta} \lambda_5,
\end{align}
and to assume that
\begin{equation}
\label{cond6}
\zeta_4 > \zeta_3.
\end{equation}
We obtain that with $\PP$ probability at least $1 - \epsilon$,
\begin{multline*}
- N \log \bigl[ 1 - T(\mu \rho - \omu\,\opi)(R) \bigr]
- \beta (\nu \rho - \omu\,\opi)(R)\\
\leq \gamma(\nu \rho - \omu\,\rho_1)(r) +
\zeta_3 \gamma(\nu_3 \rho_4 - \omu \rho_5)(r)
\\
+ \zeta_1 \tfrac{NT}{\beta} \omu \Biggl\{
\log \Biggl[ \rho_1 \biggl\{ \exp \biggl[ C \tfrac{\beta}{NT \zeta_1}
\bigl[ \nu \rho + \zeta_1 \rho_1 \bigr](m') \biggr]
\biggr\} \Biggr] \Biggr\}\\
+ \bigl( 1 + \zeta_2 \tfrac{NT}{\beta}\bigr) \nu \Biggl\{
\log \Biggl\{ \rho_2 \biggl\{ \exp \biggl[ \tfrac{C}{1 + \zeta_2
\frac{NT}{\beta}} \zeta_2 \rho_2 (m') \biggr] \biggr\} \Biggr] \Biggr\}\\
+ \zeta_4 \tfrac{NT}{\beta} \nu_3 \Biggl\{ \log \Biggl[
\rho_4 \biggl\{ \exp \biggl[ C \tfrac{\beta}{NT \zeta_4}
\bigl[ \zeta_3 \nu_3 \rho_1 + \zeta_4
\rho_4 \bigr] (m') \biggr] \biggr\} \Biggr] \Biggr\}\\
+ \zeta_5 \tfrac{NT}{\beta} \omu \Biggl\{
\log \Biggl[ \rho_5 \biggl\{ \exp \biggl[ C \tfrac{\beta}{NT \zeta_5}
\bigl[ \zeta_3 \nu_3 \rho_1 + \zeta_5 \rho_5 \bigr] (m') \biggr]
\biggr\} \Biggr] \Biggr\}\\
+ (1 + \zeta_2) \bigl[ \C{K}(\nu, \mu) - \C{K}(\omu, \mu) \bigr]
+ \nu \bigl[ \C{K}(\rho, \pi) - \C{K}(\rho_2, \pi) \bigr]
\\ + \zeta_3(1 + \zeta_2) \tfrac{NT}{\beta} \bigl[
\C{K}(\nu_3, \mu) - \C{K}(\omu, \mu) \bigr] \\
+
\biggl(1 + \sum_{i=1}^5 \zeta_i\biggr) \log \bigl( \tfrac{2}{\epsilon} \bigr).
\end{multline*}
In order to obtain more cancellations while replacing $\omu$ by
some posterior distribution, we will choose the constants such that
$\lambda_5 = \lambda_4$, which can be done by choosing
\begin{equation}
\label{cond7}
\zeta_5 = \frac{\zeta_3 \zeta_4}{\zeta_4 - \zeta_3}.
\end{equation}
We can now replace $\omu$ with
$\mu_{\exp - \xi_1 \rho_1(r) - \xi_4 \rho_4(r)}$,
where
\begin{align}
\label{cond8}
\xi_1 & = \frac{\gamma}{(1 + \zeta_2)\bigl(1 + \tfrac{NT}{\beta} \zeta_3 \bigr)},\\
\label{cond9}
\xi_4 & = \frac{\gamma\zeta_3}{(1 + \zeta_2)\bigl(1 + \tfrac{NT}{\beta} \zeta_3 \bigr)}.
\end{align}
Choosing moreover $\nu_3 = \mu_{\exp - \xi_1 \rho_1(r) - \xi_4 \rho_4(r)}$,
to induce some more cancellations,
we get
\begin{thm}\mypoint
\label{thm1.59}
Let us use the notation introduced above.
For any positive real constants satisfying equations
\myeq{cond1}, \myeq{cond2}, \myeq{cond3}, \myeq{cond4},
\myeq{cond5}, \myeq{cond6}, \myeq{cond7}, \myeq{cond8}, \myeq{cond9},
with $\PP$ probability at least $1 - \epsilon$, for any posterior distribution
$\nu: \Omega \rightarrow \C{M}_+^1(M)$ and any conditional posterior
distribution $\rho: \Omega \times M \rightarrow \C{M}_+^1(\Theta)$,
\begin{multline*}
- N \log \bigl[ 1 - T(\nu \rho - \omu\,\opi)(R) \bigr]
- \beta (\nu \rho - \omu\,\opi)(R) \leq B(\nu, \rho, \beta),\\
\shoveleft{\text{where }
B(\nu, \rho, \beta) \overset{\text{\rm def}}{=} \gamma ( \nu \rho -
\nu_3 \rho_1)(r)} \\*
\shoveleft{\qquad + (1 + \zeta_2) \bigl( 1 + \tfrac{NT}{\beta} \zeta_3 \bigr) }
\\ \times
\log \Biggl\{ \nu_3 \Biggl[ \rho_1 \biggl\{
\exp \biggl[ C \tfrac{\beta}{NT \zeta_1} \bigl[ \nu \rho
+ \zeta_1 \rho_1 \bigr] (m') \biggr] \biggr\}^{\frac{\zeta_1 N T}{\beta
(1 + \zeta_2)(1 + \frac{NT}{\beta}\zeta_3)}} \\
\shoveright{\times \rho_4 \biggl\{ \exp \biggl[
C \tfrac{\beta}{NT \zeta_5} \bigl[
\zeta_3 \nu_3 \rho_1 + \zeta_5 \rho_4 \bigr] (m')
\biggr] \biggr\}^{\frac{\zeta_5 N T}{\beta(1 + \zeta_2)(1 + \frac{NT}{\beta}
\zeta_3)}} \Biggr] \Biggr\}}\\
+ \bigl( 1 + \zeta_2 \tfrac{NT}{\beta}\bigr) \nu \Biggl\{
\log \Biggl\{ \rho_2 \biggl\{ \exp \biggl[ \tfrac{C}{1 + \zeta_2
\frac{NT}{\beta}} \zeta_2 \rho_2 (m') \biggr] \biggr\} \Biggr] \Biggr\}\\
+ \zeta_4 \tfrac{NT}{\beta} \nu_3 \Biggl\{ \log \Biggl[
\rho_4 \biggl\{ \exp \biggl[ C \tfrac{\beta}{NT \zeta_4}
\bigl[ \zeta_3 \nu_3 \rho_1 + \zeta_4
\rho_4 \bigr] (m') \biggr] \biggr\} \Biggr] \Biggr\}\\
\shoveleft{\qquad + (1 + \zeta_2) \bigl[ \C{K}(\nu, \mu) - \C{K}(\nu_3, \mu) \bigr]
} \\ + \nu \bigl[ \C{K}(\rho, \pi) - \C{K}(\rho_2, \pi) \bigr]
+ \biggl( 1 + \sum_{i=1}^5 \zeta_i \biggr)
\log \bigl( \tfrac{2}{\epsilon} \bigr).
\end{multline*}
\end{thm}

This theorem can be used to find the largest value $\w{\beta}(\nu \rho)$ of
$\beta$ such that
$ B( \nu, \rho,\break \beta) \leq 0$, thus providing an estimator for
$\beta(\nu \rho)$ defined as $\nu \rho(R) = \ov{\mu}_{\beta(\nu \rho)}
\ov{\pi}_{\beta(\nu \rho)}(R)$, where we have mentioned explicitly
the dependence of $\ov{\mu}$ and $\ov{\pi}$ in $\beta$, the constant
$\zeta_2$ staying fixed. The posterior distribution $\nu \rho$ may
then be chosen to maximize $\w{\beta}(\nu \rho)$ within some manageable
subset of posterior distributions $\C{P}$, thus gaining the assurance
that $\nu \rho(R) \leq \ov{\mu}_{\w{\beta}(\nu \rho)}\ov{\pi}_{\w{\beta}(\nu \rho)}
(R)$, with the largest parameter $\w{\beta}(\nu \rho)$ that this
approach can provide. Maximizing $\w{\beta}(\nu \rho)$ is supported by the
fact that $\lim_{\beta \rightarrow + \infty} \ov{\mu}_{\beta}\ov{\pi}_{\beta}(R)
= \ess \inf_{\mu \pi} R$. Anyhow, there is no assurance (to our knowledge) that
$\beta \mapsto \ov{\mu}_{\beta} \ov{\pi}_{\beta}(R)$ will be a decreasing
function of $\beta$ all the way, although this may be expected to be the case
in many practical situations.

We can make the bound more explicit in several ways. One point
of view is to put forward the optimal values of $\rho$ and $\nu$.
We can thus remark that
\begin{multline*}
\nu \bigl[ \gamma \rho(r) + \C{K}(\rho, \pi) -
\C{K}(\rho_2, \pi) \bigr] + (1 + \zeta_2) \C{K}(\nu, \mu)
\\ =
\nu \biggl[ \C{K}\bigl[ \rho, \pi_{\exp( - \gamma r)} \bigr]
+ \lambda_2 \rho_2(r)
+ \int_{\lambda^2}^{\gamma}
\pi_{\exp( - \alpha r)}(r) d \alpha \biggr]
+ (1 + \zeta_2) \C{K}( \nu, \mu)
\\ = \nu \bigl\{ \C{K}\bigl[ \rho, \pi_{\exp( - \gamma r)} \bigr]
\bigr\} + (1 + \zeta_2)
\C{K}\bigl[ \nu, \mu_{ \exp
\bigl( - \frac{\lambda_2 \rho_2(r)}{1 + \zeta_2}
- \frac{1}{1 + \zeta_2} \int_{\lambda_2}^{\gamma}
\pi_{\exp( - \alpha r)}(r) d \alpha \bigr)} \bigr]
\\ - (1 + \zeta_2) \log \Biggl\{ \mu \Biggl[ \exp \biggl\{
- \frac{\lambda_2}{1 + \zeta_2} \rho_2(r)
- \frac{1}{1 + \zeta_2} \int_{\lambda_2}^{\gamma}
\pi_{\exp( - \alpha r )}(r) d \alpha \biggr\} \Biggr] \Biggr\}.
\end{multline*}
Thus
\begin{multline*}
B(\nu, \rho, \beta) =
(1 + \zeta_2) \Bigl[ \xi_1 \nu_3 \rho_1(r) + \xi_4
\nu_3 \rho_4(r) \\ + \log \bigl\{ \mu \bigl[ \exp
\bigl( - \xi_1 \rho_1(r) - \xi_4 \rho_4(r) \bigr) \bigr] \bigr\}
\Bigr] \\ - (1 + \zeta_2) \log \Biggl\{ \mu \Biggl[ \exp \biggl\{
- \frac{\lambda_2}{1 + \zeta_2} \rho_2(r)
- \frac{1}{1 + \zeta_2} \int_{\lambda_2}^{\gamma}
\pi_{\exp( - \alpha r )}(r) d \alpha \biggr\} \Biggr] \Biggr\} \\ \shoveleft{\quad
- \gamma \nu_3 \rho_1 (r)
+ (1 + \zeta_2) \bigl( 1 + \tfrac{NT}{\beta} \zeta_3 \bigr) }
\\ \times
\log \Biggl\{ \nu_3 \Biggl[ \rho_1 \biggl\{
\exp \biggl[ C \tfrac{\beta}{NT \zeta_1} \bigl[ \nu \rho
+ \zeta_1 \rho_1 \bigr] (m') \biggr] \biggr\}^{\frac{\zeta_1 N T}{\beta
(1 + \zeta_2)(1 + \frac{NT}{\beta}\zeta_3)}} \\
\shoveright{\times \rho_4 \biggl\{ \exp \biggl[
C \tfrac{\beta}{NT \zeta_5} \bigl[
\zeta_3 \nu_3 \rho_1 + \zeta_5 \rho_4 \bigr] (m')
\biggr] \biggr\}^{\frac{\zeta_5 N T}{\beta(1 + \zeta_2)(1 + \frac{NT}{\beta}
\zeta_3)}} \Biggr] \Biggr\}}\\
+ \bigl( 1 + \zeta_2 \tfrac{NT}{\beta}\bigr) \nu \Biggl\{
\log \Biggl\{ \rho_2 \biggl\{ \exp \biggl[ \tfrac{C}{1 + \zeta_2
\frac{NT}{\beta}} \zeta_2 \rho_2 (m') \biggr] \biggr\} \Biggr] \Biggr\}\\
+ \zeta_4 \tfrac{NT}{\beta} \nu_3 \Biggl\{ \log \Biggl[
\rho_4 \biggl\{ \exp \biggl[ C \tfrac{\beta}{NT \zeta_4}
\bigl[ \zeta_3 \nu_3 \rho_1 + \zeta_4
\rho_4 \bigr] (m') \biggr] \biggr\} \Biggr] \Biggr\}\\
\shoveleft{\quad + \nu \bigl\{ \C{K}\bigl[ \rho, \pi_{\exp( - \gamma r)} \bigr]
\bigr\}} \\  + (1 + \zeta_2)
\C{K}\bigl[ \nu, \mu_{ \exp
\bigl( - \frac{\lambda_2 \rho_2(r)}{1 + \zeta_2}
- \frac{1}{1 + \zeta_2} \int_{\lambda_2}^{\gamma}
\pi_{\exp( - \alpha r)}(r) d \alpha \bigr)} \bigr]\\
+ \biggl(1 + \sum_{i=1}^5 \zeta_i \biggr) \log\bigl(\tfrac{2}{\epsilon}
\bigr).
\end{multline*}
This formula is better understood when thinking about
the following upper bound for the two first lines
in the expression of $B(\nu, \rho, \beta)$:
\begin{multline*}
(1 + \zeta_2) \Bigl[ \xi_1 \nu_3 \rho_1(r) + \xi_4
\nu_3 \rho_4(r) + \log \bigl\{ \mu \bigl[ \exp
\bigl( - \xi_1 \rho_1(r) - \xi_4 \rho_4(r) \bigr) \bigr] \bigr\}
\Bigr] \\ \shoveleft{\qquad - (1 + \zeta_2) \log \Biggl\{ \mu \Biggl[ \exp \biggl\{
- \frac{\lambda_2}{1 + \zeta_2} \rho_2(r) }
\\ \shoveright{ - \frac{1}{1 + \zeta_2} \int_{\lambda_2}^{\gamma}
\pi_{\exp( - \alpha r )}(r) d \alpha \biggr\} \Biggr] \Biggr\} -
\gamma \nu_3 \rho_1 (r)\qquad}\\
\leq \nu_3 \biggl[ \lambda_2 \rho_2(r) + \int_{\lambda_2}^{\gamma}
\pi_{\exp( - \alpha r)}(r) d \alpha - \gamma \rho_1(r) \biggr].
\end{multline*}
Another approach to understanding Theorem \ref{thm1.59} is
to put forward $\rho_0 =\break \pi_{\exp(- \lambda_0 r)}$,
for some positive real constant $\lambda_0 < \gamma$,
noticing that
$$
\nu \bigl[ \C{K}(\rho_0, \pi) - \C{K}(\rho_2, \pi) \bigr]
= \lambda_0 \nu (\rho_2 - \rho_0)(r) - \nu \bigl[
\C{K}(\rho_2, \rho_0) \bigr].
$$
Thus
\begin{multline*}
B(\nu, \rho_0, \beta) \leq
\nu_3 \bigl[ (\gamma - \lambda_0) (\rho_0 - \rho_1)(r) + \lambda_0
(\rho_2 - \rho_1)(r) \bigr]  \\
\shoveleft{\quad + (1 + \zeta_2) \bigl( 1 + \tfrac{NT}{\beta} \zeta_3 \bigr)
} \\ \times \log \Biggl\{ \nu_3 \Biggl[ \rho_1 \biggl\{
\exp \biggl[ C \tfrac{\beta}{NT \zeta_1} \bigl[ \nu \rho_0
+ \zeta_1 \rho_1 \bigr] (m') \biggr] \biggr\}^{\frac{\zeta_1 N T}{\beta
(1 + \zeta_2)(1 + \frac{NT}{\beta}\zeta_3)}} \\
\shoveright{ \times \rho_4 \biggl\{ \exp \biggl[
C \tfrac{\beta}{NT \zeta_5} \bigl[
\zeta_3 \nu_3 \rho_1 + \zeta_5 \rho_4 \bigr] (m')
\biggr] \biggr\}^{\frac{\zeta_5 N T}{\beta(1 + \zeta_2)(1 + \frac{NT}{\beta}
\zeta_3)}} \Biggr] \Biggr\}\quad}\\
+ \bigl( 1 + \zeta_2 \tfrac{NT}{\beta}\bigr) \nu \Biggl\{
\log \Biggl\{ \rho_2 \biggl\{ \exp \biggl[ \tfrac{C}{1 + \zeta_2
\frac{NT}{\beta}} \zeta_2 \rho_2 (m') \biggr] \biggr\} \Biggr] \Biggr\}\\
+ \zeta_4 \tfrac{NT}{\beta} \nu_3 \Biggl\{ \log \Biggl[
\rho_4 \biggl\{ \exp \biggl[ C \tfrac{\beta}{NT \zeta_4}
\bigl[ \zeta_3 \nu_3 \rho_1 + \zeta_4
\rho_4 \bigr] (m') \biggr] \biggr\} \Biggr] \Biggr\}\\
\shoveleft{\quad + (1 + \zeta_2) \C{K}\Bigl[
\nu, \mu_{\exp \bigl( - \frac{(\gamma - \lambda_0) \rho_0(r) + \lambda_0 \rho_2(r)}{
1 + \zeta_2} \bigr)} \Bigr] }\\
- \nu \bigl[ \C{K}(\rho_2, \rho_0) \bigr]
+ \biggl( 1 + \sum_{i=1}^5 \zeta_i \biggr)
\log \bigl( \tfrac{2}{\epsilon} \bigr).
\end{multline*}

In the case when we want to select a single model $\wm(\omega)$,
and therefore to set $\nu = \delta_{\wm}$, the previous
inequality engages us to take \\
\mbox{} \hfill $\ds \wm \in \arg \min_{m \in M}
(\gamma - \lambda_0) \rho_0(m, r) + \lambda_0 \rho_2(m, r)$.
\hfill \mbox{}\\
In parametric situations where $$\pi_{\exp( - \lambda r)}(r)
\simeq \sr(m) + \frac{d_e(m)}{\lambda},$$
we get
$$(\gamma - \lambda_0) \rho_0(m, r) - \lambda_0 \rho_2(m, r)
\simeq \gamma \bigl[ \sr(m) + d_e(m) \bigl( \tfrac{1}{\lambda_0}
+ \tfrac{\lambda_0 - \lambda_2}{\gamma \lambda_2} \bigr)\bigr],$$
resulting in a linear penalization of the empirical dimension of the
models.
\eject

\subsection{Analysis of two step bounds relative to a Gibbs prior}
We will not state a formal result, but will nevertheless give some
hints about how to establish one. This is a rather technical
section, which can be skipped at a first reading , since it will not be used below.
We should start from Theorem \thmref{thm4.1}, which gives a deterministic variance
term. From Theorem \ref{thm4.1}, after a
change of prior distribution, we obtain
for any positive constants $\alpha_1$ and $\alpha_2$,
any prior distributions $\wt{\mu}_1$ and $\wt{\mu}_2
\in \C{M}_+^1(M)$,
for any prior conditional distributions $\wt{\pi}_1$
and $\wt{\pi}_2: M \rightarrow \C{M}_+^1(\Theta)$,
with $\PP$ probability at least $1 - \eta$,
for any posterior distributions $\nu_1 \rho_1$ and
$\nu_2 \rho_2$,
\begin{multline*}
\alpha_1(\nu_1 \rho_1 - \nu_2 \rho_2)(R) \leq
\alpha_2(\nu_1 \rho_1 - \nu_2 \rho_2)(r) \\ +
\C{K}\bigl[ (\nu_1 \rho_1) \otimes (\nu_2 \rho_2),
(\wt{\mu}_1\,\wt{\pi}_1)\otimes(\wt{\mu}_2\,\wt{\pi}_2)
\bigr] \\
+ \log \Bigl\{ (\wt{\mu}_1\,\wt{\pi}_1)\otimes (\wt{\mu}_2\,\wt{\pi}_2) \Bigl[
\exp \bigl\{ - \alpha_2 \Psi_{\frac{\alpha_2}{N}}(R',M') + \alpha_1 R' \bigr\}
\Bigr] \Bigr\} - \log(\eta).
\end{multline*}
Applying this to $\alpha_1 = 0$, we get that
\begin{multline*}
(\nu \rho - \nu_3 \rho_1)(r)
\leq \frac{1}{\alpha_2} \biggl[ \C{K}\bigl[
(\nu \rho) \otimes (\nu_3 \rho_1), (\wt{\mu}\,\wt{\pi})\otimes (
\wt{\mu}_3\,\wt{\pi}_1) \bigr]
\\ + \log \Bigl\{  (\wt{\mu}\,\wt{\nu})\otimes(\wt{\mu}_3\,\wt{\pi}_1)
\Bigl[ \exp \bigl\{
\alpha_2 \Psi_{-\frac{\alpha_2}{N}} (R', M') \bigr\} \Bigr] \Bigr\}
- \log(\eta) \biggr].
\end{multline*}
In the same way, to bound quantities of the form
\begin{multline*}
\log \Biggl\{ \nu_3 \Biggl[ \rho_1 \biggl\{
\exp \biggl[ C_1 (\nu \rho + \zeta_1 \rho_1)(m') \biggr] \biggr\}^{p_1}
\\ \times \rho_4 \biggl\{ \exp \biggl[ C_2 \bigl[
\zeta_3 \nu_3 \rho_1 + \zeta_5 \rho_4 \bigr] (m') \biggr]
\biggr\}^{p_2} \Biggr] \Biggr\}
\\ = \sup_{\nu_5} \biggl\{ p_1 \sup_{\rho_5} \Bigl\{
C_1 \bigl[ (\nu \rho) \otimes (\nu_5 \rho_5) + \zeta_1 \nu_5(\rho_1
\otimes \rho_5) \bigr](m') - \C{K}(\rho_5, \rho_1) \Bigr\}
\\\qquad \qquad + p_2 \sup_{\rho_6} \Bigl\{  C_2 \bigl[ \zeta_3
(\nu_3 \rho_1) \otimes (\nu_5 \rho_6) \hfill \\ + \zeta_5 \nu_5(\rho_4
\otimes \rho_6) \bigr] (m') - \C{K}(\rho_6, \rho_4) \Bigr\}
- \C{K}(\nu_5, \nu_3) \biggr\},
\end{multline*}
where $C_1$, $C_2$, $p_1$ and $p_2$ are positive constants,
and similar terms,
we need to use inequalities of the type: for any prior distributions
$\wt{\mu}_i\,\wt{\pi}_i$, $i = 1, 2$, with $\PP$ probability
at least $1 - \eta$, for any posterior distributions
$\nu_i \rho_i$, $i = 1,2$,
\begin{multline*}
\alpha_3 (\nu_1 \rho_1) \otimes (\nu_2 \rho_2)(m')
\leq
\log \Bigl\{ (\wt{\mu}_1\,\wt{\pi}_1) \otimes
(\wt{\mu}_2\,\wt{\pi}_2) \exp \Bigl[ \alpha_3 \Phi_{\frac{- \alpha_3}{N}}
(M') \Bigr] \Bigr\} \\ + \C{K}\bigl[
(\nu_1 \rho_1) \otimes (\nu_2 \rho_2), (\wt{\mu}_1\,\wt{\pi}_1)
\otimes (\wt{\mu}_2\,\wt{\pi}_2) \bigr] - \log(\eta).
\end{multline*}
We need also the variant: with $\PP$ probability at least $1 - \eta$,
for any posterior distribution $\nu_1: \Omega \rightarrow \C{M}_+^1(M)$
and any conditional posterior distributions $\rho_1, \rho_2:
\Omega \times M \rightarrow \C{M}_+^1(\Theta)$,
\begin{multline*}
\alpha_3 \nu_1 (\rho_1 \otimes \rho_2)(m')
\leq
\log \Bigl\{ \wt{\mu}_1\bigl(\wt{\pi}_1 \otimes \wt{\pi}_2 \bigr)
\exp \Bigl[ \alpha_3 \Phi_{- \frac{\alpha_3}{N}}(M') \Bigr] \Bigr\}
\\ + \C{K}(\nu_1, \wt{\mu}_1) + \nu_1 \bigl\{
\C{K}\bigl[
\rho_1 \otimes \rho_2, \wt{\pi}_1
\otimes \wt{\pi}_2 \bigr] \bigr\} - \log(\eta).
\end{multline*}
We deduce that
\begin{multline*}
\log \Biggl\{ \nu_3 \Biggl[
\rho_1 \biggl\{ \exp \biggl[
C_1 (\nu \rho + \zeta_1 \rho_1)(m') \biggr]
\biggr\}^{p_1}
\\ \shoveright{ \times \rho_4 \biggl\{ \exp
\biggl[
C_2 \bigl[ \zeta_3 \nu_3 \rho_1 + \zeta_5
\rho_4 \bigr] (m') \biggr] \biggr\}^{p_2} \Biggr] \Biggr\} \quad } \\
\leq \sup_{\nu_5} \Biggl\{ p_1
\sup_{\rho_5} \Biggl[
\frac{C_1}{\alpha_3} \biggl\{ \log \Bigl\{ (\wt{\mu} \, \wt{\pi})
\otimes (\wt{\mu}_5\,\wt{\pi}_5) \exp \Bigl[
\alpha_3 \Phi_{- \frac{\alpha_3}{N}}(M') \Bigr] \Bigr\}
\\ + \C{K}\bigl[ (\nu \rho) \otimes (\nu_5 \rho_5),
(\wt{\mu}\,\wt{\pi} \otimes (\wt{\mu}_5\,\wt{\pi}_5) \bigr]
+ \log(\tfrac{2}{\eta}) \\
+ \zeta_1 \biggl[
\log \Bigl\{ \wt{\mu}_5 \bigl(
\wt{\pi}_1 \otimes \wt{\pi}_5 \bigr)
\exp \Bigl[ \alpha_3 \Phi_{- \frac{\alpha_3}{N}}
(M') \Bigr] \Bigr\}
\\ + \C{K}(\nu_5, \wt{\mu}_5)
+ \nu_5 \bigl\{ \C{K} \bigl[
\rho_1 \otimes \rho_5,
\wt{\pi}_1 \otimes \wt{\pi}_5 \bigr] \bigr\}
+ \log \bigl(  \tfrac{2}{\eta} \bigr)
\biggr] \biggr\} - \C{K}(\rho_5, \rho_1) \Biggr] \\
+ p_2 \sup_{\rho_6} \Biggl[
\frac{C_1}{\alpha_3} \biggl\{ \log \Bigl\{ (\wt{\mu}_3 \, \wt{\pi}_1)
\otimes (\wt{\mu}_5\,\wt{\pi}_6) \exp \Bigl[
\alpha_3 \Phi_{- \frac{\alpha_3}{N}}(M') \Bigr] \Bigr\}
\\ + \C{K}\bigl[ (\nu_3 \rho_1) \otimes (\nu_5 \rho_6),
(\wt{\mu}_3\,\wt{\pi}_1 \otimes (\wt{\mu}_5\,\wt{\pi}_6) \bigr]
+ \log(\tfrac{2}{\eta}) \\
+ \zeta_1 \biggl[
\log \Bigl\{ \wt{\mu}_5 \bigl(
\wt{\pi}_4 \otimes \wt{\pi}_6 \bigr)
\exp \Bigl[ \alpha_3 \Phi_{- \frac{\alpha_3}{N}}
(M') \Bigr] \Bigr\}
\\ \hfill + \C{K}(\nu_5, \wt{\mu}_5)
+ \nu_5 \bigl\{ \C{K} \bigl[
\rho_4 \otimes \rho_6,
\wt{\pi}_4 \otimes \wt{\pi}_6 \bigr] \bigr\}
+ \log \bigl(  \tfrac{2}{\eta} \bigr)
\biggr] \biggr\}\qquad \\ - \C{K}(\rho_6, \rho_4) \Biggr]
- \C{K}(\nu_5, \nu_3) \Biggr\}.
\end{multline*}

We are then left with the need to bound entropy terms like
$\C{K}(\nu_3 \rho_1, \wt{\mu}_3\wt{\pi}_1)$, where we have the choice of
$\wt{\mu}_3$ and $\wt{\pi}_1$, to obtain a useful bound.
As could be expected, we decompose it into
$$
\C{K}(\nu_3 \rho_1, \wt{\mu}_3\wt{\pi}_1) =
\C{K}(\nu_3, \wt{\mu}_3) + \nu_3 \bigl[ \C{K}(\rho_1, \wt{\pi}_1) \bigr].
$$
Let us look after the second term first, choosing $\wt{\pi}_1 = \pi_{\exp
( - \beta_1 R)}$:
\begin{multline*}
\nu_3 \bigl[ \C{K}(\rho_1, \wt{\pi}_1) \bigr]
= \nu_3 \bigl[ \beta_1 (\rho_1 - \wt{\pi}_1)(R) + \C{K}(\rho_1, \pi)
- \C{K}(\wt{\pi}_1, \pi) \bigr]
\\ \leq \frac{\beta_1}{\alpha_1}  \biggl[ \alpha_2 \nu_3(\rho_1 - \wt{\pi}_1)(r)
+ \C{K}(\nu_3, \wt{\mu}_3) + \nu_3 \bigl[ \C{K}(\rho_1, \wt{\pi}_1) \bigr]
\\+ \log \Bigl\{ \wt{\mu}_3 \bigl( \wt{\pi}_1^{\otimes 2}
\bigr) \Bigl[
\exp \bigl\{ - \alpha_2 \Psi_{\frac{\alpha_2}{N}}
(R', M') + \alpha_1 R' \bigr\} \Bigr] \Bigr\} - \log(\eta) \biggr]
\\ \shoveright{+ \nu_3 \bigl[ \C{K}(\rho_1, \pi) - \C{K}(\wt{\pi}_1, \pi) \bigr]
\qquad}
\\ \quad \leq \frac{\beta_1}{\alpha_1} \biggl[
\C{K}(\nu_3, \wt{\mu}_3) + \nu_3 \bigl[ \C{K}(\rho_1, \wt{\pi}_1) \bigr]
\hfill \\ + \log \Bigl\{
\wt{\mu}_3 \bigl( \wt{\pi}_1^{\otimes 2} \bigr)
\Bigl[ \exp \bigl\{
- \alpha_2 \Psi_{\frac{\alpha_2}{N}}(R', M') + \alpha_1 R' \bigr\}
\Bigr] \Bigr\} - \log(\eta) \biggr]
\\ + \nu_3
\bigl\{ \C{K}\bigl[ \rho_1 , \pi_{\exp ( -
\frac{\beta_1 \alpha_2}{\alpha_1} r)} \bigr] \bigr\}.
\end{multline*}
Thus, when the constraint $\lambda_1 = \frac{\beta_1 \alpha_2}{\alpha_1}$
is satisfied,
\begin{multline*}
\nu_3 \bigl[ \C{K}(\rho_1, \wt{\pi}_1) \bigr]
\leq \Bigl( 1 - \frac{\beta_1}{\alpha_1} \Bigr)^{-1} \frac{\beta_1}{\alpha_1} \biggl[
\C{K}(\nu_3, \wt{\mu}_3) \\ + \log \Bigl\{
\wt{\mu}_3 \bigl(\wt{\pi}_1^{\otimes 2} \bigr)
\Bigl[ \exp \bigl\{ - \alpha_2 \Psi_{\frac{\alpha_2}{N}}(R', M') + \alpha_1
R' \bigr\} \Bigr] \Bigr\}
- \log(\eta) \biggr].
\end{multline*}
We can further specialize the constants, choosing $\alpha_1
= N \sinh(\frac{\alpha_2}{N})$, so that
$$
- \alpha_2 \Psi_{\frac{\alpha_2}{N}}(R', M') + \alpha_1 R'
\leq 2 N \sinh\Bigl(\frac{\alpha_2}{2 N}\Bigr)^2 M'.
$$
We can for instance choose $\alpha_2 = \gamma$, $\alpha_1 = N \sinh(\frac{\gamma}{N})$
and $\beta_1 = \lambda_1 \frac{N}{\gamma} \sinh(\frac{\gamma}{N})$,
leading to
\begin{prop}\mypoint
With the notation of Theorem \ref{thm1.59}, the constants being
set as explained above, putting $
\wt{\pi}_1  = \pi_{\exp( - \lambda_1 \frac{N}{\gamma}\sinh(\frac{\gamma}{N}) R)}$,
with $\PP$ probability at least $1 - \eta$,
\begin{multline*}
\nu_3 \bigl[ \C{K}(\rho_1, \wt{\pi}_1) \bigr]
\leq \Bigl( 1 - \frac{\lambda_1}{\gamma} \Bigr)^{-1}
\frac{\lambda_1}{\gamma} \biggl[ \C{K}(\nu_3, \wt{\mu}_3)
\\ + \log \Bigl\{
\wt{\mu}_3 \bigl( \wt{\pi}_1^{\otimes 2} \bigr)\Bigl[
\exp \bigl\{ 2 N \sinh(\tfrac{\gamma}{2N})^2 M' \bigr\} \Bigr] \Bigr\}
- \log(\eta) \biggr].
\end{multline*}
More generally
\begin{multline*}
\nu_3 \bigl[ \C{K}(\rho, \wt{\pi}_1) \bigr]
\leq \Bigl( 1 - \frac{\lambda_1}{\gamma} \Bigr)^{-1}
\frac{\lambda_1}{\gamma} \biggl[ \C{K}(\nu_3, \wt{\mu}_3)
\\ + \log \Bigl\{
\wt{\mu}_3 \bigl( \wt{\pi}_1^{\otimes 2} \bigr)\Bigl[
\exp \bigl\{ 2 N \sinh(\tfrac{\gamma}{2N})^2 M' \bigr\}
\Bigr] \Bigr\} - \log(\eta) \biggr]
\\ + \Bigl( 1 - \frac{\lambda_1}{\gamma} \Bigr)^{-1} \nu_3 \bigl[ \C{K}(
\rho, \rho_1) \bigr].
\end{multline*}
\end{prop}

In a similar way, let us now choose $\wt{\mu}_3 = \mu_{\exp[ - \alpha_3 \opi(R)]}$.
We can write
\begin{multline*}
\C{K}(\nu, \wt{\mu}_3) = \alpha_3 (\nu - \wt{\mu}_3)\opi(R)
+ \C{K}(\nu, \mu) - \C{K}(\wt{\mu}_3, \mu)
\\ \leq \frac{\alpha_3}{\alpha_1} \biggl[ \alpha_2 (\nu - \wt{\mu}_3)\opi(r)
+ \C{K}(\nu, \wt{\mu}_3) \\ + \log \Bigl\{ (\wt{\mu}_3 \opi) \otimes
(\wt{\mu}_3 \opi) \Bigl[ \exp \bigl\{
- \alpha_2 \Psi_{\frac{\alpha_2}{N}}(R',M') + \alpha_1 R' \bigr\} \Bigr] \Bigr\}
- \log(\eta) \biggr] \\
+ \C{K}(\nu, \mu) - \C{K}(\wt{\mu}_3, \mu).
\end{multline*}
Let us choose $\alpha_2 = \gamma$, $\alpha_1 = N \sinh(\frac{\gamma}{N})$, and
let us add some other entropy inequalities to get
rid of $\opi$ in a suitable way, the approach of entropy
compensation being the same as that used
to obtain the empirical bound of Theorem \thmref{thm1.59}.
This results with $\PP$ probability
at least $1 - \eta$ in
\begin{multline*}
\Bigl( 1 - \frac{\alpha_3}{\alpha_1} \Bigr)
\C{K}(\nu, \wt{\mu}_3) \leq \frac{\alpha_3}{\alpha_1}  \biggl[
\gamma (\nu - \wt{\mu}_3)\opi(r)
\\+ \log \Bigl\{ ( \wt{\mu}_3 \opi) \otimes ( \wt{\mu}_3 \opi)
\Bigl[ \exp \bigl\{ - \gamma \Psi_{\frac{\gamma}{N}}(R', M') + \alpha_1 R' \bigr\}
\Bigr] \Bigr\} + \log(\tfrac{2}{\eta}) \biggr]
\\ \hfill + \C{K}(\nu, \mu) - \C{K}(\wt{\mu}_3, \mu),\quad\\\quad
\zeta_6 \Bigl(1 - \frac{\beta}{\alpha_1} \Bigr)
\wt{\mu}_3 \bigl[ \C{K}(\rho_6, \opi) \bigr]
\leq \zeta_6 \frac{\beta}{\alpha_1} \biggl[
\gamma \wt{\mu}_3 (\rho_6 - \opi)(r)\hfill\\
+ \log \Bigl\{ \wt{\mu}_3\bigl(\opi^{\otimes 2}\bigr)
\Bigl[ \exp \bigl\{ - \gamma \Psi_{\frac{\gamma}{N}}(R', M')
+ \alpha_1 R' \bigr\} \Bigr] \Bigr\} + \log(\tfrac{2}{\eta}) \biggr]
\\ \hfill + \zeta_6 \wt{\mu}_3 \bigl[
\C{K}(\rho_6, \pi) - \C{K}(\opi, \pi) \bigr],\quad\\\quad
\zeta_7 \Bigl(1 - \frac{\beta}{\alpha_1} \Bigr)
\wt{\mu}_3 \bigl[ \C{K}(\rho_7, \opi) \bigr]
\leq \zeta_7 \frac{\beta}{\alpha_1} \biggl[
\gamma \wt{\mu}_3 (\rho_7 - \opi)(r)\hfill \\
+ \log \Bigl\{ \wt{\mu}_3\bigl(\opi^{\otimes 2}\bigr)
\Bigl[ \exp \bigl\{ - \gamma \Psi_{\frac{\gamma}{N}}(R', M')
+ \alpha_1 R' \bigr\} \Bigr] \Bigr\} + \log(\tfrac{2}{\eta}) \biggr]
\\ \hfill + \zeta_7 \wt{\mu}_3 \bigl[
\C{K}(\rho_7, \pi) - \C{K}(\opi, \pi) \bigr],\quad\\\quad
\zeta_8 \Bigl( 1 - \frac{\beta}{\alpha_1} \Bigr) \nu \bigl[ \C{K}(\rho_8, \opi) \bigr]
\leq \zeta_8 \frac{\beta}{\alpha_1} \biggl[ \gamma \nu ( \rho_8 - \opi) (r)
+ \C{K}(\nu, \wt{\mu}_3) \hfill\\ +
\log \Bigl\{ \wt{\mu}_3\bigl(\opi^{\otimes 2}\bigr)
\Bigl[ \exp \bigl\{ - \gamma \Psi_{\frac{\gamma}{N}}(R', M') + \alpha_1 R' \bigr\}
\Bigr] \Bigr\} + \log(\tfrac{2}{\eta}) \biggr]
\\ \hfill + \zeta_8 \nu \bigl[ \C{K}(\rho_8, \pi)
- \C{K}(\opi, \pi) \bigr],\quad\\\quad
\zeta_9 \Bigl( 1 - \frac{\beta}{\alpha_1} \Bigr) \nu \bigl[ \C{K}(\rho_9, \opi) \bigr]
\leq \zeta_9 \frac{\beta}{\alpha_1} \biggl[ \gamma \nu ( \rho_9 - \opi) (r)
+ \C{K}(\nu, \wt{\mu}_3) \hfill\\ +
\log \Bigl\{ \wt{\mu}_3\bigl(\opi^{\otimes 2}\bigr)
\Bigl[ \exp \bigl\{ - \gamma \Psi_{\frac{\gamma}{N}}(R', M') + \alpha_1 R' \bigr\}
\Bigr] \Bigr\} + \log(\tfrac{2}{\eta}) \biggr]
\\ \hfill + \zeta_9 \nu \bigl[ \C{K}(\rho_9, \pi)
- \C{K}(\opi, \pi) \bigr],
\end{multline*}
where we have introduced a bunch of constants, assumed to be positive,
that we will more precisely set to
\begin{align*}
x_8 + x_9 & = 1,\\
( \zeta_6 \beta + x_8 \alpha_3) \frac{\gamma}{\alpha_1} & = \lambda_6,\\
( \zeta_7 \beta + x_9 \alpha_3) \frac{\gamma}{\alpha_1} & = \lambda_7,\\
( \zeta_8 \beta - x_8 \alpha_3) \frac{\gamma}{\alpha_1} & = \lambda_8,\\
( \zeta_9 \beta - x_9 \alpha_3) \frac{\gamma}{\alpha_1} & = \lambda_9.
\end{align*}
We get with $\PP$ probability at least $1 - \eta$,
\begin{multline*}
\Bigl( 1 - \frac{\alpha_3}{\alpha_1} -
(\zeta_8 + \zeta_9)  \frac{\beta}{\alpha_1} \Bigr)
\C{K}(\nu, \wt{\mu}_3) \leq
\\ \frac{\alpha_3}{\alpha_1} \biggl[ \gamma \bigl[ \nu (
x_8 \rho_8 + x_9 \rho_9)(r) - \wt{\mu}_3 (x_8 \rho_6 + x_9 \rho_7) (r) \bigr]
\\ + \frac{\alpha_3}{\alpha_1} \log
\Bigl\{ (\wt{\mu}_3 \opi) \otimes (\wt{\mu}_3 \opi)
\Bigl[ \exp \bigl\{ - \gamma \Psi_{\frac{\gamma}{N}}(R', M')
+ \alpha_1 R' \bigr\} \Bigr] \Bigr\} \\
+ (\zeta_6 + \zeta_7 + \zeta_8 + \zeta_9) \frac{\beta}{\alpha_1}
\log \Bigl\{ \wt{\mu}_3 \bigl(
\opi^{\otimes 2} \bigr)
\Bigl[ \exp \bigl\{ - \gamma
\Psi_{\frac{\gamma}{N}}(R', M') + \alpha_1 R' \bigr\} \Bigr] \Bigr\}\\
+ \C{K}(\nu, \mu) - \C{K}(\wt{\mu}_3, \mu)
+ \Bigl( \frac{\alpha_3}{\alpha_1} + (\zeta_6 + \zeta_7 + \zeta_8 +
\zeta_9) \frac{\beta}{\alpha_1} \Bigr) \log\bigl( \tfrac{2}{\eta} \bigr).
\end{multline*}
Let us choose the constants so that
$\lambda_1 = \lambda_7 = \lambda_9$, $\lambda_4 = \lambda_6 = \lambda_8$,
$\alpha_3 x_9 \frac{\gamma}{\alpha_1} = \xi_1$ and $ \alpha_3 x_8
\frac{\gamma}{\alpha_1} = \xi_4$.
This is done by setting
\begin{align*}
x_8 & = \frac{\xi_4}{\xi_1 + \xi_4},\\
x_9 & = \frac{\xi_1}{\xi_1 + \xi_4},\\
\alpha_3 & = \tfrac{N}{\gamma} \sinh(\tfrac{\gamma}{N}) ( \xi_1 + \xi_4),\\
\zeta_6 & = \tfrac{N}{\gamma}\sinh(\tfrac{\gamma}{N}) \frac{(\lambda_4 - \xi_4)}{\beta},\\
\zeta_7 & = \tfrac{N}{\gamma}\sinh(\tfrac{\gamma}{N})
\frac{(\lambda_1 - \xi_1)}{\beta},\\
\zeta_8 & = \tfrac{N}{\gamma} \sinh(\tfrac{\gamma}{N}) \frac{(\lambda_4 +
\xi_4)}{\beta},\\
\zeta_9 & = \tfrac{N}{\gamma} \sinh(\tfrac{\gamma}{N}) \frac{(\lambda_1 + \xi_1)}{
\beta}.
\end{align*}
The inequality $\lambda_1 > \xi_1$ is always satisfied. The inequality
$\lambda_4 > \xi_4$ is required for the above choice of constants, and
will be satisfied for a suitable choice of $\zeta_3$ and $\zeta_4$.

Under these assumptions, we obtain with $\PP$ probability at least $1 - \eta$
\begin{multline*}
\Bigl( 1 - \frac{\alpha_3}{\alpha_1} -
(\zeta_8 + \zeta_9)  \frac{\beta}{\alpha_1} \Bigr)
\C{K}(\nu, \wt{\mu}_3) \leq
(\nu - \wt{\mu}_3) (\xi_1 \rho_1 + \xi_4 \rho_4)(r)
\\ + \frac{\alpha_3}{\alpha_1} \log
\Bigl\{ (\wt{\mu}_3 \opi) \otimes (\wt{\mu}_3 \opi)
\Bigl[ \exp \bigl\{ - \gamma \Psi_{\frac{\gamma}{N}}(R', M')
+ \alpha_1 R' \bigr\} \Bigr] \Bigr\} \\
+ (\zeta_6 + \zeta_7 + \zeta_8 + \zeta_9) \frac{\beta}{\alpha_1}
\log \Bigl\{ \wt{\mu}_3 \bigl(
\opi^{\otimes 2} \bigr)
\Bigl[ \exp \bigl\{ - \gamma
\Psi_{\frac{\gamma}{N}}(R', M') + \alpha_1 R' \bigr\} \Bigr] \Bigr\}\\
+ \C{K}(\nu, \mu) - \C{K}(\wt{\mu}_3, \mu)
+ \Bigl( \frac{\alpha_3}{\alpha_1} + (\zeta_6 + \zeta_7 + \zeta_8 +
\zeta_9) \frac{\beta}{\alpha_1} \Bigr) \log\bigl( \tfrac{2}{\eta} \bigr).
\end{multline*}
This proves
\begin{prop}
\mypoint
The constants being set as explained above,
with $\PP$ probability at least $1 - \eta$,
for any posterior distribution $\nu: \Omega \rightarrow \C{M}_+^1(M)$,
\begin{multline*}
\C{K}(\nu, \wt{\mu}_3) \leq \Bigl( 1 - \frac{\alpha_3}{\alpha_1} -
(\zeta_8 + \zeta_9)  \frac{\beta}{\alpha_1} \Bigr)^{-1}
\biggl[ \C{K}(\nu, \nu_3)
\\ + \frac{\alpha_3}{\alpha_1} \log
\Bigl\{ (\wt{\mu}_3 \opi) \otimes (\wt{\mu}_3 \opi)
\Bigl[ \exp \bigl\{ - \gamma \Psi_{\frac{\gamma}{N}}(R', M')
+ \alpha_1 R' \bigr\} \Bigr] \Bigr\} \\
+ (\zeta_6 + \zeta_7 + \zeta_8 + \zeta_9) \frac{\beta}{\alpha_1}
\log \Bigl\{ \wt{\mu}_3 \bigl(
\opi^{\otimes 2} \bigr)
\Bigl[ \exp \bigl\{ - \gamma
\Psi_{\frac{\gamma}{N}}(R', M') + \alpha_1 R' \bigr\} \Bigr] \Bigr\}\\
+ \Bigl( \frac{\alpha_3}{\alpha_1} + (\zeta_6 + \zeta_7 + \zeta_8 +
\zeta_9) \frac{\beta}{\alpha_1} \Bigr) \log\bigl( \tfrac{2}{\eta} \bigr)\biggr] .
\end{multline*}
\end{prop}

Thus
\begin{multline*}
\C{K}(\nu_3 \rho_1, \wt{\mu}_3\,\wt{\pi}_1) \leq
\frac{1 + \bigl(1 - \frac{\lambda_1}{\gamma}\bigr)^{-1} \frac{\lambda_1}{\gamma}}{
1 - \frac{\alpha_3}{\alpha_1} - (\zeta_8+\zeta_9)\frac{\beta}{\alpha_1}} \\ \times
\biggl[ \frac{\alpha_3}{\alpha_1} \log \Bigl\{
(\wt{\mu}_3 \ov{\pi} \otimes (\wt{\mu}_3 \ov{\pi}) \Bigl[
\exp \bigl\{ - \gamma \Psi_{\frac{\gamma}{N}}
(R',M') + \alpha_1 R' \bigr\} \Bigr] \Bigr\}
\\ + (\zeta_6 + \zeta_7 + \zeta_8 + \zeta_9) \frac{\beta}{\alpha_1}
\log \Bigl\{ \wt{\mu}_3 \bigl( \ov{\pi}^{\otimes 2} \bigr) \Bigl[
\exp \bigl\{ - \gamma \Psi_{\frac{\gamma}{N}}(R', M') + \alpha_1 R' \bigr\} \Bigr]
\Bigr\} \\
+ \Bigl( \frac{\alpha_3}{\alpha_1} + (
\zeta_6 + \zeta_7 + \zeta_8 + \zeta_9) \frac{\beta}{\alpha_1} \Bigr)
\log \bigl( \tfrac{2}{\eta} \bigr) \biggr] \\
+ \Bigl( 1 - \frac{\lambda_1}{\gamma} \Bigr)^{-1} \frac{\lambda_1}{\gamma} \biggl[
\log \Bigl\{ \wt{\mu}_3 \bigl( \wt{\pi}_1^{\otimes 2} \bigr)
\Bigl[ \exp \bigl\{ 2 N \sinh\bigl(\tfrac{\gamma}{2N} \bigr)^2
M' \bigr\} \Bigr] \Bigr\} - \log( \tfrac{2}{\eta} ) \biggr].
\end{multline*}
We will not go further, lest it may become tedious, but we hope we have
given sufficient hints to state informally that the bound $B(\nu, \rho, \beta)$
of Theorem \ref{thm1.59} (page \pageref{thm1.59}) is upper bounded
with $\PP$ probability close to one by a
bound of the same flavour where the empirical quantities $r$ and $m'$
have been replaced with their expectations $R$ and $M'$.

\subsection{Two step localization between posterior distributions}

Here we work with a family of prior distributions
described by a regular conditional prior distribution
$\pi = M \rightarrow \C{M}_+^1(\Theta)$, where $M$ is some
measurable index set. This family may typically describe
a countable family of parametric models. In this case $M = \NN$,
and each of the prior distributions $\pi(i, .)$, $i \in \NN$
satisfies some parametric complexity assumption of the type
$$
\limsup_{\beta \rightarrow + \infty}
\beta \bigl[ \pi_{\exp( - \beta R)}(i,.)(R) - \essinf_{ \pi(i,.)}
R \bigr] = d_i < + \infty, \qquad i \in M.
$$
Let us consider also a prior distribution $\mu \in \C{M}_+^1(M)$
defined on the index set $M$.

Our aim here will be to
compare the performance of two given posterior distributions
$\nu_1 \rho_1$ and $\nu_2 \rho_2$, where $\nu_1, \nu_2:
\Omega \rightarrow \C{M}_+^1(M)$, and where
$\rho_1, \rho_2: \Omega \times M \rightarrow \C{M}_+^1(\Theta)$.
More precisely, we would like to establish a bound for
$(\nu_1 \rho_1 - \nu_2 \rho_2) (R)$ which could be a starting
point to implement a selection method similar to the one
described in Theorem \thmref{thm1.58}.
To this purpose, we can start with Theorem \thmref{thm1.1.38},
which says that with $\PP$
probability at least $1 - \epsilon$,
\begin{multline*}
- N \log \Bigl\{ 1 - \tanh(\tfrac{\lambda}{N})
\bigl( \nu_1 \rho_1 - \nu_2 \rho_2 \bigr) (R) \Bigr\}
\leq \lambda (\nu_1 \rho_1 - \nu_2 \rho_2)(r)
\\ + N \log \bigl[ \cosh(\tfrac{\lambda}{N})\bigr]
(\nu_1 \rho_1) \otimes (\nu_2 \rho_2) (m')
+ \C{K}(\nu_1, \wt{\mu}) + \C{K}(\nu_2, \wt{\mu})
\\ + \nu_1 \bigl[ \C{K}(\rho_1, \wt{\pi}) \bigr]
+ \nu_2 \bigl[ \C{K}(\rho_2, \wt{\pi}) \bigr] - \log(\epsilon),
\end{multline*}
where $\wt{\mu} \in \C{M}_+^1(M)$ and $\wt{\pi}:
M \rightarrow \C{M}_+^1(\Theta)$ are suitably
localized prior distributions to be chosen later
on.
To use these localized prior distributions,
we need empirical bounds for the entropy
terms $\C{K}(\nu_i, \wt{\mu})$ and $\nu_i \bigl[ \C{K}
( \rho_i, \wt{\pi}) \bigr]$, $i = 1, 2$.

Bounding $\nu \bigl[ \C{K}(\rho, \wt{\pi})\bigr]$ can be done
using the following generalization of Corollary \ref{cor1.60} page
\pageref{cor1.60}:
\begin{cor} \mypoint
\label{cor1.76}
For any positive real constants $\gamma$ and $\lambda$
such that $\gamma < \lambda$, for any prior distribution
$\ov{\mu} \in \C{M}_+^1(M)$ and any conditional
prior distribution $\pi: M \rightarrow \C{M}_+^1(\Theta)$,
with $\PP$ probability
at least $1 - \epsilon$, for any posterior distribution
$\nu: \Omega \rightarrow \C{M}_+^1(M)$, and any conditional
posterior distribution $\rho: \Omega \times M \rightarrow
\C{M}_+^1(\Theta)$,
$$
\nu \Bigl\{ \C{K} \bigl[ \rho,
\pi_{\exp[ - N \frac{\gamma}{\lambda} \tanh(\frac{\lambda}{N}) R]}
\bigr] \Bigr\}
\leq K'(\nu, \rho, \gamma, \lambda, \epsilon )  +
\frac{1}{\frac{\lambda}{\gamma} - 1} \C{K}(\nu, \ov{\mu}),
$$
where
\begin{multline*}
K'(\nu, \rho, \gamma, \lambda, \epsilon)
\overset{\text{\rm def}}{=} \Bigl( 1 - \tfrac{\gamma}{\lambda} \Bigr)^{-1}
\biggl\{
\nu \bigl[ \C{K}(\rho, \pi_{\exp( - \gamma r)}\bigr]
\\
- \frac{\gamma}{\lambda} \log(\epsilon)
+ \nu \Bigl\{ \log \Bigl[ \pi_{\exp( - \gamma r)} \Bigl(
\exp \bigl\{ N \tfrac{\gamma}{\lambda} \log \bigl[
\cosh(\tfrac{\lambda}{N}) \bigr] \rho(m') \bigr\} \Bigr) \Bigr] \Bigr\}
\biggr\}.
\end{multline*}
\end{cor}

To apply this corollary to our case, we have to set
$$
\wt{\pi} = \pi_{\exp[ - N \frac{\gamma}{\lambda}
\tanh(\frac{\lambda}{N}) R]}.
$$ Let us also consider
for some positive real constant $\beta$ the conditional
prior distribution
$$
\ov{\pi} = \pi_{\exp( - \beta R)}
$$
and the prior distribution
$$
\ov{\mu} = \mu_{\exp[ - \alpha \ov{\pi}(R) ]}.
$$

Let us see how we can bound, given any posterior distribution
$\nu: \Omega \rightarrow \C{M}_+^1(M)$,
the divergence $\C{K}(\nu, \ov{\mu})$.
We can see that
$$
\C{K}(\nu, \ov{\mu}) = \alpha (\nu - \ov{\mu}) \ov{\pi}(R)
+ \C{K}(\nu, \mu) - \C{K}(\ov{\mu}, \mu).
$$

Now, let us introduce the conditional posterior distribution
$$
\wh{\pi} = \pi_{\exp( - \gamma r)}
$$
and let us decompose
$$
(\nu - \ov{\mu}) \bigl[ \ov{\pi}(R) \bigr]
= \nu \bigl[ \ov{\pi}(R) - \wh{\pi}(R) \bigr] +
(\nu - \ov{\mu}) \bigl[ \wh{\pi}(R) \bigr] +
\ov{\mu} \bigl[ \wh{\pi}(R) - \ov{\pi}(R) \bigr].
$$

Starting from the exponential inequality
$$
\PP \biggl[ \ov{\mu} \bigl[
\ov{\pi} \otimes \ov{\pi} \bigr] \exp \Bigl\{
-N \log \bigl[ 1 - \tanh( \tfrac{\gamma}{N}) R'\bigr]
 - \gamma r' - N \log \bigl[ \cosh(\tfrac{\gamma}{N}) \bigr] m'
\Bigr\} \biggr] \leq 1,
$$
and reasoning in the same way that led to Theorem \ref{thm1.1.41Bis}
(page \pageref{thm1.1.41Bis}) in the simple case when
we take in this theorem $\lambda = \gamma$, we get with $\PP$
probability at least $1 - \epsilon$, that
\begin{multline*}
- N \log \bigl\{ 1 - \tanh(\tfrac{\gamma}{N}) \nu
(\ov{\pi} - \wh{\pi}) (R) \bigr\} +
\beta \nu (\ov{\pi} - \wh{\pi})(R) \\
\leq \nu \biggl[ \log \Bigl\{ \wh{\pi} \Bigl[
\exp \bigl\{ N \log \bigl[ \cosh(\tfrac{\gamma}{N}) \wh{\pi}(m')
\bigr\} \Bigr] \Bigr\} \biggr] + \C{K}(\nu, \ov{\mu}) - \log(\epsilon).
\end{multline*}
\begin{multline*}
- N \log \bigl\{ 1 - \tanh(\tfrac{\gamma}{N}) \ov{\mu}(\wh{\pi} - \ov{\pi})(R)
\bigr\} - \beta \ov{\mu}(\wh{\pi}-\ov{\pi})(R) \\
\leq \ov{\mu} \biggl[ \log \Bigl\{ \wh{\pi} \Bigl[ \exp \bigl\{
N \log \bigl[ \cosh(\tfrac{\gamma}{N}) \wh{\pi}(m') \bigr\} \Bigr]
\Bigr\} \biggr] - \log(\epsilon).
\end{multline*}

In the meantime, using Theorem \ref{thm1.1.38} (page \pageref{thm1.1.38})
and Corollary \ref{cor1.76} above, we see that with $\PP$ probability at least $1 - 2 \epsilon$,
for any conditional posterior distribution $\rho: \Omega \times M
\rightarrow \C{M}_+^1(\Theta)$,
\begin{multline*}
-N \log \Bigl\{ 1 - \tanh(\tfrac{\lambda}{N}) (\nu - \ov{\mu})\rho (R) \Bigr\}
\leq \lambda (\nu - \ov{\mu}) \rho (r) \\ +
N \log \bigl[ \cosh(\tfrac{\lambda}{N}) \bigr] (\nu \rho) \otimes
(\ov{\mu} \rho) (m')
+ (\nu + \ov{\mu}) \C{K}(\rho, \wt{\pi}) +
\C{K}(\nu, \ov{\mu}) - \log(\epsilon)\\
\leq \lambda  (\nu - \ov{\mu})\rho(r) + N \log \bigl[ \cosh(\tfrac{\lambda}{N})
\bigr] (\nu \rho) \otimes (\ov{\mu} \rho) (m') +
\C{K}(\nu, \ov{\mu}) - \log(\epsilon) \\
+ \Bigl( 1 - \tfrac{\gamma}{\lambda} \Bigr)^{-1} (\nu + \ov{\mu}) \biggl\{
\C{K} \bigl( \rho, \wh{\pi} \bigr)
+
\log \Bigl\{ \wh{\pi} \Bigl[
\exp \bigl\{ N \tfrac{\gamma}{\lambda} \log \bigl[
\cosh (\tfrac{\lambda}{N}) \bigr] \rho(m') \bigr\} \Bigr] \Bigr\} \biggr\}
\\ + \Bigl( \tfrac{\lambda}{\gamma} - 1\Bigr)^{-1} \bigl[
\C{K}(\nu, \ov{\mu})
- 2 \log(\epsilon) \bigr].
\end{multline*}

Putting all this together,
we see that with $\PP$ probability at least $1 - 3 \epsilon$,
for any posterior distribution $\nu \in \C{M}_+^1(M)$,
\begin{multline*}
\biggl[ 1 - \frac{\alpha}{N \tanh(\frac{\gamma}{N}) + \beta}
- \frac{\alpha}{N \tanh(\frac{\lambda}{N})\bigl(1 - \frac{\gamma}{\lambda}
\bigr)} \biggr] \C{K}(\nu, \ov{\mu}) \leq \\
\alpha \Bigl[ N \tanh(\tfrac{\gamma}{N}) + \beta \Bigr]^{-1}
\biggl\{ \nu \biggl[ \log \Bigl\{ \wh{\pi} \Bigl[
\exp \bigl\{ N \log \bigl[ \cosh(\tfrac{\gamma}{N}) \bigr] \wh{\pi}(m')
\bigr\} \Bigr] \Bigr\} \biggr]
- \log(\epsilon) \biggr\}  \\
+ \alpha \Bigl[ N \tanh(\tfrac{\gamma}{N}) - \beta \Bigr]^{-1}
\biggl\{ \ov{\mu} \biggl[ \log \Bigl\{ \wh{\pi} \Bigl[ \exp
\bigl\{ N \log \bigl[ \cosh(\tfrac{\gamma}{N}) \bigr]  \wh{\pi}(m')
\bigr\} \Bigr] \Bigr\} \biggr] - \log(\epsilon) \biggr\}  \\
+ \alpha \bigl[ N \tanh(\tfrac{\lambda}{N}) \bigr]^{-1} \Biggl\{
\\
\lambda (\nu - \ov{\mu}) \wh{\pi}(r) +
N \log \bigl[ \cosh(\tfrac{\lambda}{N})\bigr] (\nu \wh{\pi})
\otimes (\ov{\mu} \wh{\pi})(m')
\\
+ \Bigl( 1 - \tfrac{\gamma}{\lambda}\Bigr)^{-1}
(\nu + \ov{\mu}) \biggl[
\log \Bigl\{ \wh{\pi} \Bigl[ \exp \bigl\{ N \tfrac{\gamma}{\lambda}
\log \bigl[ \cosh(\tfrac{\lambda}{N}) \bigr]
\wh{\pi}(m') \bigr\} \Bigr] \Bigr\} \biggr]
\\
- \frac{1 + \frac{\gamma}{\lambda}}{1 - \frac{\gamma}{\lambda}}
\log(\epsilon) \Biggr\} + \C{K}(\nu, \mu) - \C{K}(\ov{\mu}, \mu).
\end{multline*}
Replacing in the right-hand side of this inequality the
unobserved prior distribution $\ov{\mu}$ with the worst
possible posterior distribution, we obtain
\begin{thm}\mypoint
For any positive real constants $\alpha$, $\beta$, $\gamma$
and $\lambda$, using the notation,
\begin{align*}
\ov{\pi} & = \pi_{\exp( - \beta R)},\\
\ov{\mu} & = \mu_{\exp[ - \alpha \ov{\pi}(R) ]},\\
\wh{\pi} & = \pi_{\exp( - \gamma r)},\\
\wh{\mu} & = \mu_{\exp[ - \alpha \frac{\lambda}{N} \tanh(
\frac{\lambda}{N})^{-1} \wh{\pi}(r)]},
\end{align*}
with $\PP$ probability at least $1 -
\epsilon$, for any posterior distribution $\nu: \Omega \rightarrow
\C{M}_+^1(M)$,
\begin{multline*}
\biggl[ 1 - \frac{\alpha}{N \tanh(\frac{\gamma}{N}) + \beta}
- \frac{\alpha}{N \tanh(\frac{\lambda}{N})
\bigl( 1 - \frac{\gamma}{\lambda} \bigr) } \biggr]
\C{K}(\nu, \ov{\mu}) \leq \C{K}(\nu, \wh{\mu})
\\ + \frac{\alpha}{N \tanh(\frac{\gamma}{N}) + \beta}
\biggl\{ \nu \biggl[ \log \Bigl\{
\wh{\pi} \Bigl[ \exp \bigl\{ N \log \bigl[ \cosh(\tfrac{\gamma}{N})\bigr]
\wh{\pi}(m') \bigr\} \Bigr] \Bigr\} \biggr] \biggr\} \\
+ \frac{\alpha}{ N \tanh(\frac{\lambda}{N}) (1 - \frac{\gamma}{\lambda})}
\biggl\{ \nu \biggl[ \log \Bigl\{ \wh{\pi} \Bigl[
\exp \bigl\{ N \tfrac{\gamma}{\lambda}\log \bigl[ \cosh(\tfrac{\lambda}{N}) \bigr]
\wh{\pi}(m') \bigr\} \Bigr] \Bigr\} \biggr]
\biggr\} \\
+ \log \Biggl\{ \wh{\mu} \Biggl[
\biggl[ \wh{\pi} \Bigl\{ \exp \Bigl[ N \log \bigl[ \cosh(\tfrac{\gamma}{N})
\bigr] \wh{\pi}(m') \Bigr] \Bigr\} \biggr]^{\frac{\alpha}{N \tanh(\frac{\gamma}{N})
- \beta}} \\
\times \biggl[ \wh{\pi} \Bigl\{ \exp \Bigl[
N \tfrac{\gamma}{\lambda} \log \bigl[ \cosh(\tfrac{\lambda}{N}) \bigr]
\wh{\pi}(m') \Bigr] \Bigr\} \biggr]^{
\frac{\alpha}{N \tanh(\frac{\lambda}{N})(1 - \frac{\gamma}{\lambda})}} \\
\times \exp \biggl[ \frac{\alpha \log [ \cosh(\frac{\lambda}{N})]
}{\tanh(\frac{\lambda}{N})} (\nu \wh{\pi}) \otimes \wh{\pi}(m') \biggr] \Biggr]
\Biggr\} \\
+ \biggl[
\frac{1}{N \tanh(\frac{\gamma}{N}) + \beta} +
\frac{1}{N \tanh(\frac{\gamma}{N}) - \beta}
+ \frac{1 + \frac{\gamma}{\lambda}}{N \tanh(\frac{\lambda}{N})
\bigl( 1 - \frac{\gamma}{\lambda} \bigr)} \biggr]  \log\bigl(
\tfrac{3}{\epsilon}\bigr).
\end{multline*}
\end{thm}

This result is satisfactory, but in the same time hints
at some possible improvement in the choice of the localized
prior $\ov{\mu}$, which is here somewhat lacking a variance
term. We will consider in the remainder of this section the use
of
\begin{equation}
\label{eq1.50}
\ov{\mu} = \mu_{\exp[ - \alpha \ov{\pi}(R)
- \xi \wt{\pi} \otimes \wt{\pi} (M')},
\end{equation}
where $\xi$ is some positive real constant and $\wt{\pi} =
\pi_{\exp( - \wt{\beta} R)}$ is some appropriate conditional prior distribution
with positive real parameter $\wt{\beta}$.
With this new choice
$$
\C{K}(\nu, \ov{\mu}) = \alpha (\nu - \ov{\mu})\ov{\pi}(R)
+ \xi (\nu - \ov{\mu}) (\wt{\pi} \otimes \wt{\pi})(M')
+ \C{K}(\nu, \mu) - \C{K}(\ov{\mu}, \mu).
$$
We already know how to deal with the first factor
$\alpha(\nu - \ov{\mu}) \ov{\pi}(R)$, since the computations
we made to give it an empirical upper bound were valid
for any choice of the localized prior distribution $\ov{\mu}$.
Let us now deal with $\xi (\nu - \ov{\mu})(\wt{\pi}\otimes
\wt{\pi})(M')$. Since $m'(\theta, \theta')$ is a sum
of independent Bernoulli random variables, we can easily
generalize the result of Theorem \ref{thm2.3} (page
\pageref{thm2.3}) to prove that with $\PP$ probability
at least $1 - \epsilon$
\begin{multline*}
N \bigl[ 1 - \exp( - \tfrac{\zeta}{N}) \bigr] \nu(\wt{\pi} \otimes
\wt{\pi}) (M') \\ \leq \zeta \Phi_{\frac{\zeta}{N}}\bigl[ \nu(\wt{\pi}\otimes \wt{\pi})(M')
\bigr] \leq \zeta \nu(\wt{\pi}\otimes \wt{\pi})(m')
+ \C{K}(\nu, \ov{\mu}) - \log(\epsilon).
\end{multline*}
In the same way, with $\PP$ probability at least $1 - \epsilon$,
\begin{multline*}
- N\bigl[ \exp(\tfrac{\zeta}{N}) - 1 \bigr]
\ov{\mu}(\wt{\pi}\otimes\wt{\pi})(M') \\ \leq
- \zeta \Phi_{- \frac{\zeta}{N}}\bigl[ \ov{\mu}(\wt{\pi}\otimes
\wt{\pi})(M') \bigr] \leq - \zeta \ov{\mu}(\wt{\pi}\otimes
\wt{\pi})(m') - \log(\epsilon).
\end{multline*}
We would like now to replace $(\wt{\pi} \otimes \wt{\pi})(m')$
with an empirical quantity. In order to do this, we will
 use an entropy bound. Indeed for any conditional posterior
distribution $\rho: \Omega \times M \rightarrow \C{M}_+^1(\Theta)$,
\begin{multline*}
\nu \bigl[ \C{K}(\rho, \wt{\pi}) \bigr]  = \wt{\beta}\nu(\rho - \wt{\pi})(R)
+ \nu \bigl[ \C{K}(\rho, \pi)
- \C{K}(\wt{\pi},\pi) \bigr] \\
\leq \frac{\wt{\beta}}{N \tanh(\frac{\gamma}{N})} \biggl\{
\gamma \nu(\rho - \wt{\pi})(r) + N \log \bigl[ \cosh(\tfrac{\gamma}{N})
\bigr] \nu(\rho \otimes \wt{\pi})(m') \\ + \C{K}(\nu, \ov{\mu})
+ \nu \bigl[ \C{K}(\rho, \wt{\pi}) \bigr] - \log(\epsilon) \biggr\}
+ \nu \bigl[ \C{K}(\rho, \pi) - \C{K}(\wt{\pi}, \pi) \bigr].
\end{multline*}
Thus choosing $\wt{\beta} = N \tanh(\tfrac{\gamma}{N})$,
\begin{multline*}
\gamma \nu(\wt{\pi} - \rho)(r) + \nu \bigl[ \C{K}(\wt{\pi}, \pi)
- \C{K}(\rho, \pi) \bigr] \\ \leq N \log\bigl[ \cosh(\tfrac{\gamma}{N})
\bigr] \nu(\rho \otimes \wt{\pi})(m')
+ \C{K}(\nu, \ov{\mu}) - \log(\epsilon).
\end{multline*}
Choosing $\rho = \wh{\pi}$, we get
$$
\nu \bigl[ \C{K}(\wt{\pi}, \wh{\pi}) \bigr]
\leq N \log \bigl[ \cosh(\tfrac{\gamma}{N}) \bigr] \nu ( \wh{\pi}
\otimes \wt{\pi})(m') + \C{K}(\nu, \ov{\mu}) - \log(\epsilon).
$$
This implies that
\begin{multline*}
\xi \nu ( \wh{\pi}\otimes \wt{\pi})(m') =
\nu \Bigl\{ \wt{\pi} \bigl[
\xi \wh{\pi}(m') \bigr] - \C{K}(\wt{\pi}, \wh{\pi}) \Bigr\}
+ \nu \bigl[ \C{K}(\wt{\pi}, \wh{\pi}) \bigr]
\\ \leq \nu \Bigl\{ \log \Bigl[ \wh{\pi} \bigl\{  \exp
\bigl[ \xi \wh{\pi}(m') \bigr] \bigr\} \Bigr] \Bigr\} \\ +
 N \log \bigl[ \cosh ( \tfrac{\gamma}{N}) \bigr]
\nu ( \wh{\pi} \otimes \wt{\pi} ) (m') +
\C{K}(\nu, \ov{\mu}) - \log(\epsilon).
\end{multline*}
Thus
\begin{multline*}
\bigl\{ \xi - N \log \bigl[ \cosh(\tfrac{\gamma}{N})\bigr]
\bigr\} \nu (\wh{\pi} \otimes \wt{\pi}) (m')
\\ \leq \nu \Bigl\{ \log \Bigl[ \wh{\pi} \bigl\{ \exp \bigl[
\xi \wh{\pi}(m') \bigr] \bigr\} \Bigr] \Bigr\} +
\C{K}(\nu, \ov{\mu}) - \log(\epsilon)
\end{multline*}
and
\begin{multline*}
\nu\bigl[ \C{K}(\wt{\pi}, \wh{\pi}) \bigr] \leq
\biggl( \frac{\xi}{N \log [ \cosh(\frac{\gamma}{N})]} -1 \biggr)^{-1}
\biggl[ \nu \Bigl\{ \log \Bigl[ \wh{\pi} \bigl\{
\exp \bigl[ \xi \wh{\pi}(m') \bigr] \bigr\} \Bigr] \Bigr\}
\\ + \C{K}(\nu, \ov{\mu}) - \log(\epsilon) \biggr] + \C{K}(\nu, \ov{\mu})
- \log(\epsilon).
\end{multline*}
Taking for simplicity $\xi = 2 N \log \bigl[ \cosh(\tfrac{\gamma}{N})
\bigr]$ and noticing that
$$
2 N \log \bigl[ \cosh(\tfrac{\gamma}{N}) \bigr] = - N
\log \bigl( 1 - \tfrac{\wt{\beta}^2}{N^2} \bigr),
$$ we get
\begin{thm}\mypoint
Let us put
$\wt{\pi} = \pi_{\exp( - \wt{\beta} R)}$ and $\wh{\pi} =
\pi_{\exp( - \gamma r)}$, where $\gamma$ is some arbitrary
positive real constant and $\wt{\beta} = N \tanh(\tfrac{\gamma}{N})$,
so that $\gamma = \frac{N}{2} \log \Bigl(
\frac{1 + \frac{\wt{\beta}}{N}}{1 - \frac{
\wt{\beta}}{N}} \Bigr) $.
With $\PP$ probability at least $1 - \epsilon$,
$$
\nu \bigl[ \C{K}(\wt{\pi}, \wh{\pi})\bigr]
\leq \nu \biggl[ \log \Bigl\{ \wh{\pi} \Bigl[
\exp \bigl\{ 2 N \log \bigl[ \cosh(\tfrac{\gamma}{N}) \bigr]
\wh{\pi}(m') \bigr\} \Bigr]
\Bigr\} \biggr]  + 2 \bigl[ \C{K}(\nu, \ov{\mu}) - \log(\epsilon) \bigr].
$$
\end{thm}

As a consequence
\begin{multline*}
\zeta \nu(\wt{\pi}\otimes \wt{\pi})(m') =
\zeta \nu(\wt{\pi}\otimes \wt{\pi})(m') - \nu \bigl[
\C{K}(\wt{\pi}\otimes\wt{\pi},
\wh{\pi}\otimes \wh{\pi})\bigr] + 2 \nu \bigl[ \C{K}(\wt{\pi}, \wh{\pi})
\bigr] \\
\leq \nu \Bigl\{ \log \Bigl[ \wh{\pi} \otimes \wh{\pi}
\bigl[ \exp ( \zeta m' ) \bigr]   \Bigr] \Bigr\} \\ + 2
\nu \biggl[ \log \Bigl\{ \wh{\pi} \Bigl[ \exp \bigl\{
2 N \log \bigl[\cosh(\tfrac{\gamma}{N}) \bigr] \wh{\pi}(m')
\bigr\} \Bigr] \Bigr\} \biggr] + 4 \bigl[
\C{K}(\nu, \ov{\mu}) - \log(\epsilon) \bigr].
\end{multline*}
Let us take for the sake of simplicity
$\zeta = 2 N \log \bigl[ \cosh (\tfrac{\gamma}{N}) \bigr]$,
to get
$$
\zeta \nu (\wt{\pi} \otimes \wt{\pi}) (m')
\leq 3 \nu \Bigl\{ \log \Bigl[ \wh{\pi} \otimes \wh{\pi}
\bigl[ \exp ( \zeta m') \bigr] \Bigr] \Bigr\} +
4 \bigl[ \C{K}(\nu, \ov{\mu}) - \log(\epsilon) \bigr].
$$
This proves
\begin{prop}\mypoint
Let us consider some arbitrary prior distribution
$\ov{\mu} \in \C{M}_+^1(M)$ and some arbitrary
conditional prior distribution $\pi: M \rightarrow
\C{M}_+^1(\Theta)$. Let $\wt{\beta} < N$ be some positive
real constant.
Let us put $\wt{\pi} = \pi_{\exp( - \wt{\beta} R)}$ and
$\wh{\pi} = \pi_{\exp( - \gamma r)}$,
with $\wt{\beta} = N \tanh(\tfrac{\gamma}{N})$.
Moreover let us put
$\zeta = 2 N \log \bigl[ \cosh ( \tfrac{\gamma}{N}) \bigr]$.
With $\PP$ probability at least $1 - 2 \epsilon$,
for any posterior distribution $\nu \in \C{M}_+^1(M)$,
\begin{multline*}
\nu(\wt{\pi} \otimes \wt{\pi})(M')
\leq \frac{
3 \nu \Bigl\{ \log \Bigl[ \wh{\pi} \otimes \wh{\pi}
\bigl[ \exp( \zeta m') \bigr] \Bigr] \Bigr\}
+  5 \bigl[ \C{K}(\nu, \ov{\mu})
- \log(\epsilon) \bigr]}{N \bigl[ 1 - \exp ( - \frac{\zeta}{N}) \bigr]}
\\ =
\frac{1}{N \tanh(\frac{\gamma}{N})^2} \biggl\{ 3 \nu \biggl[
\log \Bigl\{ \wh{\pi} \otimes \wh{\pi}
\Bigl[ \exp \bigl\{ 2 N \log \bigl[ \cosh ( \tfrac{\gamma}{N}) \bigr]
m' \bigr\} \Bigr] \Bigr\} \biggr] \\
+ 5 \bigl[ \C{K}(\nu, \ov{\mu}) - \log(\epsilon) \bigr] \biggr\}.
\end{multline*}
\end{prop}

In the same way,
\begin{multline*}
-\zeta \ov{\mu}(\wt{\pi} \otimes \wt{\pi}) (m')
\leq \ov{\mu} \Bigl\{ \log \Bigl[
\wh{\pi} \otimes \wh{\pi} \bigl[ \exp (
- \zeta m' ) \bigr] \Bigr] \Bigr\} \\
+ 2 \ov{\mu} \biggl[ \log \Bigl\{ \wh{\pi} \Bigl[ \exp
\bigl\{ 2 N \log \bigl[ \cosh(\tfrac{\gamma}{N}) \bigr]  \wh{\pi}(m')
\bigr\} \Bigr] \Bigr\} \biggr]
- 4 \log(\epsilon)
\end{multline*}
and thus
\begin{multline*}
- \ov{\mu}(\wt{\pi} \otimes \wt{\pi})(M')
\leq \frac{1}{N \bigl[ \exp( \frac{\zeta}{N}) - 1 \bigr] }
\biggl\{ \ov{\mu} \Bigl\{ \log \Bigl[ \wh{\pi} \otimes \wh{\pi}
\bigl[ \exp( - \zeta m' ) \bigr] \Bigr] \Bigr\} \\ +
2 \ov{\mu} \biggl[ \log \Bigl\{ \wh{\pi} \Bigl[ \exp
\bigl\{ 2 N \log \bigl[ \cosh(\tfrac{\gamma}{N}) \bigr]
\wh{\pi}(m') \bigr\} \Bigr] \Bigr\} \biggr] - 5 \log(\epsilon) \biggr\}.
\end{multline*}
Here we have purposely kept $\zeta$ as an arbitrary positive real
constant, to be tuned later (in order to be able to strengthen
more or less the compensation of variance terms).

We are now properly equipped to estimate the divergence with
respect to $\ov{\mu}$, the choice of prior distribution
made in equation (\ref{eq1.50}, page \pageref{eq1.50}).
Indeed we can now write
\begin{multline*}
\biggl[ 1 - \frac{\alpha}{N \tanh(\frac{\gamma}{N})
+ \beta} - \frac{\alpha}{N \tanh(\frac{\lambda}{N})
\bigl( 1 - \frac{\gamma}{\lambda} \bigr)}
- \frac{5 \xi}{N \tanh(\frac{\gamma}{N})^2} \biggr] \C{K}(\nu, \ov{\mu})
\\
\leq \frac{\alpha}{N \tanh(\frac{\gamma}{N}) + \beta}
\biggl\{ \nu \biggl[ \log \Bigl\{ \wh{\pi}
\Bigl[ \exp \bigl\{ N \log \bigl[ \cosh(\tfrac{\gamma}{N})\bigr]
\wh{\pi}(m')\bigr\} \Bigr] \Bigr\} \biggr] - \log(\epsilon) \biggr\} \\
+ \frac{\alpha}{N \tanh(\frac{\gamma}{N}) - \beta} \biggl\{
\ov{\mu} \biggl[ \log \Bigl\{ \wh{\pi} \Bigl[ \exp
\bigl\{ N \log \bigl[ \cosh(\tfrac{\gamma}{N}) \bigr] \wh{\pi}(m')
\bigr\} \Bigr] \Bigr\} \biggr] - \log(\epsilon) \biggr\} \\
+ \frac{\alpha}{N \tanh(\tfrac{\lambda}{N})}
\Biggl\{ \\ \lambda(\nu - \ov{\mu}) \wh{\pi}(r)
+ N \log\bigl[ \cosh(\tfrac{\lambda}{N})\bigr] (\nu \wh{\pi})
\otimes (\ov{\mu} \wh{\pi}) (m') \\
+ \Bigl( 1 - \tfrac{\gamma}{\lambda} \Bigr)^{-1}
(\nu + \ov{\mu}) \biggl[
\log \Bigl\{ \wh{\pi} \Bigl[ \exp \bigl\{
N \tfrac{\gamma}{\lambda} \log \bigl[ \cosh(\tfrac{\lambda}{N}
) \bigr] \wh{\pi}(m') \bigr\} \Bigr] \Bigr\} \biggr] \\
- \frac{1 + \frac{\gamma}{N}}{1 - \frac{\gamma}{N}} \log(\epsilon)
\Biggr\} \\
+ \frac{\xi}{N \tanh(\frac{\gamma}{N})^2} \biggl\{
3 \nu \biggl[ \log \Bigl\{ \wh{\pi} \otimes \wh{\pi}
\Bigl[ \exp \bigl\{ 2 N \log \bigl[ \cosh (\tfrac{\gamma}{N})
\bigr] m' \bigr\} \Bigr] \Bigr\} \biggr] - 5 \log(\epsilon) \biggr\}
\\
+ \frac{\xi}{ N \bigl[ \exp( \tfrac{\zeta}{N}) - 1 \bigr]}
\biggl\{ \ov{\mu} \Bigl\{
\log \Bigl[ \wh{\pi} \otimes \wh{\pi} \bigl[
\exp( - \zeta m') \bigr] \Bigr] \Bigr\} \\
+ 2 \ov{\mu} \biggl[ \log \Bigl\{ \wh{\pi} \Bigl[
\exp \bigl\{ 2 N \log \bigl[ \cosh ( \tfrac{\gamma}{N}) \bigr] \wh{\pi}
(m') \bigr\} \Bigr] \Bigr\} \biggr] - 5 \log(\epsilon) \biggr\}.
\\+ \C{K}(\nu, \mu)
- \C{K}(\ov{\mu}, \mu).
\end{multline*}

It remains now only to replace in the right-hand side of this inequality
$\ov{\mu}$ with the worst possible posterior distribution to obtain
\begin{thm}\mypoint
\label{thm1.80}
Let $\lambda > \gamma > \beta$, $\zeta$, $\alpha$ and $\xi$ be arbitrary
positive real constants.
Let us use the notation
$\ov{\pi} = \pi_{\exp( - \beta R)}$,
$\wt{\pi} = \pi_{\exp( - N \tanh(\frac{\gamma}{N}) R)}$,
$\wh{\pi} = \pi_{\exp( - \gamma r)}$,
$ \ov{\mu} = \mu_{\exp  [
- \alpha \ov{\pi}(R)
- \xi \wt{\pi} \otimes \wt{\pi} (M')]}$
and let us define the posterior distribution $\wh{\mu}: \Omega \rightarrow \C{M}_+^1(M)$ by
\begin{multline*}
\frac{d \wh{\mu}}{d \mu} \sim \exp \biggl\{ - \frac{\alpha \lambda}{
N \tanh(\frac{\lambda}{N})} \wh{\pi}(r) \\
+ \frac{\xi}{N \bigl[ \exp(
\frac{\zeta}{N}) - 1\bigr]} \log \Bigl\{
\wh{\pi}\otimes \wh{\pi} \bigl[ \exp( - \zeta m') \bigr] \Bigr\}
\biggr\}.
\end{multline*}
Let us assume moreover that
$$
\frac{\alpha}{N \tanh (\frac{\gamma}{N}) + \beta}
+ \frac{\alpha}{N \tanh(\frac{\lambda}{N}) (1 - \frac{\gamma}{\lambda})}
+ \frac{5 \xi}{N \tanh(\frac{\gamma}{N})^2} < 1.
$$
With $\PP$ probability at least $1 - \epsilon$, for any
posterior distribution $\nu: \Omega \rightarrow \C{M}_+^1(M)$,
\begin{multline*}
\C{K}(\nu, \ov{\mu})
\leq
\biggl[ 1 - \frac{\alpha}{N \tanh(\tfrac{\gamma}{N}) + \beta}
\\ - \frac{\alpha}{N \tanh(\tfrac{\lambda}{N}) \bigl(
1 - \frac{\gamma}{\lambda} \bigr)}
- \frac{5 \xi}{N \tanh(\tfrac{\gamma}{N})^2} \biggr]^{-1}
\Biggl\{ \C{K}(\nu, \wh{\mu}) \\ +
\frac{\alpha}{N \tanh(\frac{\gamma}{N}) + \beta} \biggl\{ \nu \biggl[ \log
\Bigl\{ \wh{\pi} \Bigl[ \exp \bigl\{ N \log \bigl[ \cosh(\tfrac{\gamma}{N})
\bigr] \wh{\pi}(m') \bigr\} \Bigr] \Bigr\} \biggr]
\biggr\}
\\
+ \frac{\alpha}{N \tanh(\frac{\lambda}{N})
\bigl( 1 - \frac{\gamma}{\lambda} \bigr)} \biggl\{ \nu \biggl[ \log \Bigl\{
\wh{\pi} \Bigl[ \exp \bigl\{ N \tfrac{\gamma}{\lambda} \log
\bigl[ \cosh(\tfrac{\lambda}{N})\bigr] \wh{\pi}(m') \bigr\}
\Bigr] \Bigr\} \biggr]
\biggr\} \\
+ \frac{\xi}{N \tanh(\frac{\gamma}{N})^2} \biggl\{
3 \nu \biggl[ \log \Bigl\{ \wh{\pi} \otimes \wh{\pi}\Bigl[
\exp \bigl\{ 2 N \log \bigl[ \cosh(\tfrac{\gamma}{N}) \bigr] m'
\bigr\} \Bigr] \Bigr\} \biggr] \biggr\} \\
+ \frac{\xi}{N \bigl[ \exp( \frac{\zeta}{N}) - 1 \bigr] }
\biggl\{ \nu \Bigl\{ \log \Bigl[ \wh{\pi} \otimes \wh{\pi} \bigl[
\exp( - \zeta m')\bigr] \Bigr] \Bigr\}
\biggr\}
\\
+ \log \biggl\{ \wh{\mu} \biggl[
\Bigl\{ \wh{\pi} \Bigl[ \exp \bigl\{ N \log \bigl[
\cosh(\tfrac{\gamma}{N}) \bigr] \wh{\pi}(m')\bigr\} \Bigr] \Bigr\}^{
\frac{\alpha}{N \tanh(\frac{\gamma}{N}) - \beta}}
\\ \times \Bigl\{ \wh{\pi} \Bigl[ \exp \bigl\{
N \tfrac{\gamma}{\lambda} \log \bigl[ \cosh(\tfrac{\lambda}{N})
\bigr] \wh{\pi}(m')\bigr\} \Bigr] \Bigr\}^{
\frac{\alpha}{N \tanh(\frac{\lambda}{N}) \bigl( 1 - \frac{\gamma}{\lambda}
\bigr)}} \\
\times \Bigl\{ \wh{\pi}
\Bigl[ \exp \bigl\{ 2 N \log \bigl[
\cosh(\tfrac{\gamma}{N})\bigr] \wh{\pi}(m') \bigr\}
\Bigr] \Bigr\}^{\frac{2 \xi}{N
\bigl[ \exp ( \frac{\zeta}{N}) - 1\bigr]}}\\
\times \exp \Bigl\{ N
\log \bigl[ \cosh(\tfrac{\lambda}{N}) \bigr] \bigl[ (\nu \wh{\pi})
\otimes \wh{\pi}\bigr] (m') \Bigr\} \biggr] \biggr\} \\
+ \Biggl[ \frac{\alpha}{N \tanh(\frac{\gamma}{N}) + \beta}
+ \frac{\alpha}{N \tanh(\frac{\gamma}{N}) - \beta}
+ \frac{2 \alpha \bigl( 1 + \frac{\gamma}{N} \bigr)}{
N \tanh(\frac{\lambda}{N}) \bigl( 1 - \frac{\gamma}{\lambda} \bigr)}
\\ + \frac{5 \xi}{N \tanh(\frac{\gamma}{N})^2} +
\frac{5 \xi}{N \bigl[ \exp(\frac{\zeta}{N}) - 1 \bigr]} \Biggr]
\log\bigl(\tfrac{5}{\epsilon} \bigr)
 \Biggr\}.
\end{multline*}
\end{thm}

The interest of this theorem lies in the presence of
a variance term in the localized posterior distribution
$\wh{\mu}$, which with a suitable choice of parameters
seems to be an interesting option in the case when
there are nested models: in this situation there may
be a need to prevent integration with respect to $\wh{\mu}$
in the right-hand side to put weight on wild oversized models
with large variance terms. Moreover, the right-hand side
being empirical, parameters can be, as usual, optimized from data
using a union bound on a grid of candidate values.

If one is only interested in the general shape
of the result, a simplified inequality as the one
below may suffice:
\begin{cor}\mypoint
\label{cor1.82}
For any positive real constants $\lambda > \gamma > \beta$, $\zeta$,
$\alpha$ and
$\xi$, let us use the same notation as in Theorem \thmref{thm1.80}.
Let us put moreover
\begin{align*}
A_1 & = \frac{\alpha}{N \tanh(\frac{\gamma}{N}) + \beta}
+ \frac{\alpha}{N \tanh(\frac{\lambda}{N})\bigl(
1 - \frac{\gamma}{\lambda}\bigr)}
+ \frac{5 \xi}{N \tanh( \frac{\gamma}{N})^2},\\
A_2 & = \frac{\alpha}{N \tanh(\frac{\gamma}{N}) + \beta}
+ \frac{\alpha}{N \tanh(\frac{\lambda}{N}) \bigl( 1 -
\frac{\gamma}{\lambda}\bigr)} + \frac{3 \xi}{N \tanh(\frac{\gamma}{N}
)^2} \\
A_3 & = \frac{\xi}{N \bigl[ \exp \bigl( \frac{\zeta}{N} \bigr)
- 1 \bigr]} \\
A_4 & = \frac{\alpha}{N \tanh(\frac{\gamma}{N}) - \beta}
+ \frac{\alpha}{N \tanh(\frac{\lambda}{N}) ( 1 - \frac{\gamma}{
\lambda})} + \frac{2 \xi}{N [ \exp( \frac{\zeta}{N}) - 1 ]}, \\
A_5 & = \frac{\alpha}{N \tanh(\frac{\gamma}{N}) + \beta}
+ \frac{\alpha}{N \tanh(\frac{\gamma}{N}) - \beta}
+ \frac{2 \alpha \bigl( 1 + \frac{\gamma}{N} \bigr)}{
N \tanh(\frac{\lambda}{N}) \bigl( 1 - \frac{\gamma}{\lambda} \bigr)}
\\ & \qquad + \frac{5 \xi}{N \tanh(\frac{\gamma}{N})^2} +
\frac{5 \xi}{N \bigl[ \exp(\frac{\zeta}{N}) - 1 \bigr]} , \\
C_1 & = 2 N \log \bigl[ \cosh \bigl( \tfrac{\lambda}{N} \bigr)
\bigr],\\
C_2 & = N \log \bigl[ \cosh \bigl( \tfrac{\lambda}{N} \bigr) \bigr].
\end{align*}
Let us assume that $A_1 < 1$.
With $\PP$ probability at least $1 - \epsilon$,
for any posterior distribution $\nu: \Omega \rightarrow
\C{M}_+^1(M)$,
\begin{multline*}
\C{K}(\nu, \ov{\mu}) \leq K(\nu, \alpha, \beta, \gamma, \lambda,
\xi, \zeta, \epsilon) \overset{\text{\rm def}}{=} \bigl( 1 - A_1 \bigr)^{-1}
\Biggl\{ \C{K}(\nu, \wh{\mu}) \\
+ A_2 \nu \Bigl[ \log \Bigl( \wh{\pi} \otimes \wh{\pi}
\bigl[ \exp \bigl( C_1 m' \bigr) \bigr] \Bigr) \Bigr]
+ A_3 \nu \Bigl[ \log \Bigl( \wh{\pi} \otimes
\wh{\pi} \bigl[ \exp \bigl(  - \zeta m'\,\bigr) \bigr] \Bigr) \Bigr]
\\ \shoveright{+ \log \biggl\{ \wh{\mu} \biggl[
\Bigl[ \wh{\pi} \Bigl( \exp \bigl[
C_1 \wh{\pi}(m') \bigr] \Bigr) \Bigr]^{A_4}
\exp \Bigl( C_2 \bigl[ (\nu \wh{\pi}) \otimes \wh{\pi} \bigr]
(m') \Bigr) \biggr] \biggr\}\quad}
\\ + A_5 \log\bigl(\tfrac{5}{\epsilon}\bigr) \Biggr\}.
\end{multline*}
\end{cor}

Putting this corollary together with Corollary \thmref{cor1.76},
we obtain
\begin{thm}\mypoint
Let us consider the notation introduced in Corollary \thmref{cor1.76}
and in Theorem \thmref{thm1.80} and its Corollary
\thmref{cor1.82}.
Let us consider real positive parameters $\lambda$, $\gamma_1' <
\lambda_1'$
and $\gamma_2' < \lambda_2'$. Let us consider also
two sets of parameters $\alpha_i, \beta_i, \gamma_i, \lambda_i, \xi_i, \zeta_i$,where $i = 1, 2$, both satisfying the conditions stated in Corollary
\thmref{cor1.82}.
With $\PP$ probability at least $1 - \epsilon$,
for any posterior distributions $\nu_1, \nu_2: \Omega
\rightarrow \C{M}_+^1(M)$, any conditional posterior
distributions $\rho_1, \rho_2: \Omega \times M \rightarrow
\C{M}_+^1 \bigl( \Theta \bigr) $,
\begin{multline*}
- N \log \Bigl\{ 1 - \tanh\bigl( \tfrac{\lambda}{N}\bigr)
\bigl( \nu_1 \rho_1 - \nu_2 \rho_2 \bigr) (R) \Bigr\}
\leq \lambda \bigl( \nu_1 \rho_1 - \nu_2 \rho_2 \bigr) (r)
\\
+ N \log \bigl[ \cosh \bigl( \tfrac{\lambda}{N} \bigr) \bigr]
 \bigl( \nu_1 \rho_1 \bigr)
\otimes \bigl( \nu_2 \rho_2 \bigr)
\bigl( m' \, \bigr)
\\
+ K'\bigl(\nu_1, \rho_1, \gamma_1', \lambda_1', \tfrac{\epsilon}{5}\bigr)
+ K'\bigl(\nu_2, \rho_2, \gamma_2', \lambda_2', \tfrac{\epsilon}{5}\bigr)
\\ + \frac{1}{1 - \frac{\gamma_1'}{\lambda_1'}}
K\bigl(\nu_1, \alpha_1, \beta_1, \gamma_1,\lambda_1,\xi_1, \zeta_1, \tfrac{\epsilon}{5} \bigr)
\\ + \frac{1}{1 - \frac{\gamma_2'}{\lambda_2'}}
K\bigl(\nu_2, \alpha_2, \beta_2, \gamma_2,\lambda_2,\xi_2, \zeta_2, \tfrac{\epsilon}{5} \bigr) - \log\bigl(\tfrac{\epsilon}{5} \bigr).
\end{multline*}
\end{thm}

This theorem provides, using a union bound argument to further optimize
the parameters, an empirical bound for $\nu_1\rho_1 (R) - \nu_2 \rho_2(R)$,
which can serve to build a selection algorithm exactly in the same way
as what was done in Theorem \thmref{thm1.58}. This represents the
highest degree of sophistication that we will achieve in
this monograph, as far as model selection is concerned: this theorem
shows that it is indeed possible to derive a selection scheme in which
localization is performed in two steps and in which the localization
of the model selection itself, as opposed to the localization of the
estimation in each model, includes a variance term as well as a
bias term, so that it should be possible to localize the choice of
nested models, something that would not have been feasible with
the localization techniques exposed in the previous sections of
this study. We should point out however that \emph{more sophisticated}
does not necessarily mean \emph{more efficient}: as the reader may
have noticed, sophistication comes at a price, in terms of
the complexity of the estimation schemes, with some possible
loss of accuracy in the constants that can mar the benefits of
using an asymptotically more efficient method for small sample
sizes.

We will do the hurried reader a favour: we will not launch into a study
of the theoretical properties of this selection algorithm,
although it is clear that all the tools needed are at hand!

We would like as a conclusion to this chapter, to put forward
a simple idea: this approach of model selection revolves around
entropy estimates concerned with the divergence of posterior
distributions with respect to localized prior distributions.
Moreover, this localization of the prior distribution is
more effectively done in several steps in some situations,
and it is worth mentioning that these situations include the typical case
of selection from a family of parametric models.
Finally, the whole story relies
upon estimating the relative generalization error rate of one posterior
distribution with respect to some local prior distribution
as well as with respect to another posterior distribution,
because these relative rates can be estimated more accurately than
absolute generalization error rates, at least as soon as no classification
model of reasonable size provides a good match to the training sample,
meaning that the classification problem is either difficult or
noisy.

\chapter{Transductive PAC-Bayesian learning}

\section{Basic inequalities}
\subsection{The transductive setting}
In this chapter the \emph{observed} sample $(X_i, Y_i)_{i=1}^N$
will be supplemented with a \emph{test} or \emph{shadow} sample
$(X_i,Y_i)_{i=N+1}^{(k+1)N}$.
This point of view, called \emph{transductive classification},
has been introduced by V.~Vapnik. It may be justified in different
ways.

On the practical side,
one interest of the transductive setting is that it is
often a lot easier to collect examples than it is to label them,
so that it is not unrealistic to assume that we indeed have
two training samples, one labelled and one unlabelled.
It also covers the case when a batch of patterns
is to be classified and we are allowed to observe
the whole batch before issuing the classification.

On the mathematical side, considering a shadow sample
proves technically fruitful. Indeed, when introducing
the Vapnik--Cervonenkis entropy and Vapnik--Cervo\-nenkis dimension concepts, as well as when
dealing with compression
schemes, albeit the \emph{inductive} setting is our
final concern, the transductive setting is a
useful detour.
In this second scenario, intermediate technical results
involving the shadow sample are integrated with respect
to unobserved random variables in a second stage of the proofs.

Let us describe now the changes to be made to previous
notation to adapt them to the transductive setting.
The distribution $\PP$ will be a probability measure on the
canonical space $\Omega = (\C{X} \times \C{Y})^{(k+1)N}$,
and $(X_i,Y_i)_{i=1}^{(k+1)N}$
will be the canonical process on this space
(that is the coordinate process).
Unless explicitly mentioned, the parameter $k$ indicating the
size of the shadow sample will remain fixed.
Assuming the shadow sample size is a multiple of the
training sample size is convenient without significantly
restricting generality.
For a while, we will use a weaker assumption than independence,
assuming that $\PP$ is \emph{partially exchangeable},
since this is all we need in the proofs.
\begin{dfn}
\mypoint For $i = 1, \dots, N$,
let $\tau_i: \Omega \rightarrow \Omega$ be defined
for any \linebreak $\omega = (\omega_j)_{j=1}^{(k+1)N} \in \Omega$ by
$$
\begin{cases}
\tau_i(\omega)_{i + jN} = \omega_{i + (j-1)N}, & j=1, \dots, k,\\
\tau_i(\omega)_{i} = \omega_{i+kN}, & \\
\text{and } \tau_i(\omega)_{m + j N} = \omega_{m + j N}, &
m\neq i, m = 1, \dots, N, j=0, \dots k.
\end{cases}
$$
Clearly, if we arrange the $(k+1)N$ samples in a $N \times (k+1)$ array,
$\tau_i$ performs a circular permutation of $k+1$ entries
on the $i$th row, leaving the
other rows unchanged.
Moreover, all the circular permutations of the $i$th
row have the form $\tau_i^j$, $j$ ranging from $0$ to $k$.

The probability distribution $\PP$ is said to be partially exchangeable if
for any $i = 1, \dots, N$, $\PP \circ \tau_i^{-1} = \PP$.

This means equivalently that for any
bounded measurable function $h: \Omega \rightarrow \RR$,  $\PP ( h \circ \tau_i) = \PP (h)$.

In the same way a function $h$ defined on $\Omega$ will be said to
be partially exchangeable if $h \circ \tau_i = h$ for
any $i=1, \dots, N$.
Accordingly a posterior distribution
$\rho: \Omega \rightarrow \C{M}_+^1(\Theta, \C{T})$ will be said to
be partially exchangeable when $\rho(\omega, A) = \rho \bigl[\tau_i(\omega), A
\bigr]$, for any $\omega \in \Omega$, any $i = 1, \dots, N$
and any $A \in \C{T}$.
\end{dfn}

For any bounded measurable function $h$, let us define
$T_i(h) = \frac{1}{k+1} \sum_{j=0}^k h \circ \tau_i^j$.
Let $T(h) = T_N \circ \dots \circ T_1(h)$.
For any partially exchangeable probability distribution $\PP$, and for
any bounded measurable function $h$, $\PP \bigl[ T(h) \bigr] = \PP(h)$.
Let us put
\renewcommand{\rr}{\overline{r}}
\begin{align*}
\sigma_i(\theta)  & = \B{1} \bigl[ f_{\theta}(X_i) \neq Y_i \bigr],
\quad \begin{tabular}[t]{l}indicating the success or failure of $f_{\theta}$\\
to predict $Y_i$ from $X_i$,\end{tabular}\\
r_1(\theta) & = \frac{1}{N} \sum_{i=1}^N \sigma_i(\theta),
\quad \begin{tabular}[t]{l} the empirical error rate of $f_{\theta}$ \\
on the observed sample,\end{tabular}\\
r_2(\theta) & = \frac{1}{kN} \sum_{i=N+1}^{(k+1)N}
\sigma_i(\theta),\quad \text{the error rate of $f_{\theta}$
on the shadow sample,}\\
\rr(\theta) & = \frac{r_1(\theta) + k r_2(\theta)}{k+1}
= \frac{1}{(k+1)N} \sum_{i=1}^{(k+1)N}
\sigma_i(\theta), \quad \begin{tabular}[t]{l}the global error \\
rate of $f_{\theta}$,\end{tabular}\\
R_i(\theta) & = \PP \bigl[ f_{\theta}(X_i) \neq Y_i \bigr],\quad
\begin{tabular}[t]{l}the expected error \\ rate of $f_{\theta}$ on the $i$th
input,\end{tabular}\\
R(\theta) & = \frac{1}{N} \sum_{i=1}^N R_i(\theta) =
\PP \bigl[ r_1(\theta) \bigr] = \PP \bigl[ r_2(\theta) \bigr],
\quad \text{the average expected} \\*  \text{error} & \text{ rate of $f_{\theta}$
on all inputs.}
\end{align*}
We will allow for posterior
distributions $\rho: \Omega \rightarrow \C{M}_+^1(\Theta)$
depending on the shadow sample. The most interesting ones will anyhow
be independent of the shadow labels $Y_{N+1}, \dots, Y_{(k+1)N}$.
We will be interested in the conditional expected
error rate of the randomized classification
rule described by $\rho$ on the shadow sample, given the observed
sample, that is,
$\PP \bigl[ \rho(r_2) \lvert (X_i,Y_i)_{i=1}^N\bigr]$.
This is a natural extension of the notion of \emph{generalization
error rate}: this is indeed the error rate to be expected when
the randomized classification rule described by the posterior
distribution $\rho$ is applied to the shadow sample (which should in this
case more purposefully be called the test sample).

To see the connection with the previously defined
generalization error rate, let us comment on the case when $\PP$ is invariant
by any permutation of any row, meaning that
\\ \mbox{} \hfill $\PP
\bigl[ h(\omega \circ s) \bigr] = \PP \bigl[ h(\omega) \bigr]$
for all $s \in \mathfrak{S}(\{i+jN ; j=0, \dots, k \})$
\hfill\mbox{}\\ and all $i=1,
\dots, N$, where $\mathfrak{S}(A)$ is the set of permutations of $A$,
extended to $\{1, \dots, (k+1)N \}$ so as to be the identity outside
of $A$. In other words, $\PP$ is assumed to be invariant under
any permutation which keeps the rows unchanged.
In this case, if $\rho$ is invariant by any permutation of any row of
the shadow sample, meaning that $\rho(\omega \circ s) = \rho(\omega)
\in \C{M}_+^1(\Theta)$, $s \in \mathfrak{S}(\{i+jN; j=1, \dots, k \})$,
$i = 1, \dots, N$, then $\PP \bigl[ \rho(r_2) \lvert (X_i,Y_i)_{i=1}^N \bigr] =
\frac{1}{N} \sum_{i=1}^N \PP \bigl[ \rho(\sigma_{i+N})
\lvert (X_i,Y_i)_{i=1}^N \bigr]$, meaning that
the expectation can be taken on a restricted shadow sample
of the same size as the observed sample.
If moreover the rows are equidistributed, meaning that their marginal distributions
are equal, then
\\\mbox{}\hfill $\PP \bigl[ \rho(r_2)
\lvert (X_i,Y_i)_{i=1}^N \bigr] = \PP \bigl[ \rho(\sigma_{N+1})
\lvert (X_i,Y_i)_{i=1}^N \bigr]$. \hfill \mbox{}\\
This means that under these quite commonly fulfilled assumptions,
the expectation can be taken on a single
new object to be classified,
our study thus covers the case when only one of the
patterns from the shadow sample is to be labelled and one is interested
in the expected error rate of this single labelling.
Of course, in the case when
$\PP$ is i.i.d. and $\rho$ depends only on the
training sample $(X_i,Y_i)_{i=1}^N$, we fall back on
the usual criterion of performance
$\PP \bigl[ \rho(r_2) \lvert (Z_i)_{i=1}^N \bigr] = \rho(R)
= \rho(R_1)$.

\subsection{Absolute bound}
Using an obvious factorization, and considering for the moment
a fixed value of $\theta$ and any partially exchangeable positive real measurable
function $\lambda: \Omega \rightarrow \RR_+$, we can compute the
$\log$-Laplace transform of $r_1$ under $T$, which acts like a
conditional probability distribution:
\begin{multline*}
\log \Bigl\{ T \bigl[ \exp ( - \lambda r_1 ) \bigr] \Bigr\}
= \sum_{i=1}^N \log \Bigl\{ T_i \bigl[ \exp ( - \tfrac{\lambda}{N} \sigma_i ) \bigr]
\Bigr\}  \\
\leq N \log \biggl\{ \frac{1}{N} \sum_{i=1}^N T_i \Bigl[
\exp \bigl( - \tfrac{\lambda}{N} \sigma_i \bigr) \Bigr] \biggr\}
= - \lambda \Phi_{\frac{\lambda}{N}}(\rr),
\end{multline*}
where the function $\Phi_{\frac{\lambda}{N}}$ was defined by equation
\myeq{eq1.1}.
Remarking that $T \Bigl\{ \exp \Bigl[
\lambda \bigl[ \Phi_{\frac{\lambda}{N}}(\rr) - r_1 \bigr] \Bigr] \Bigr\}
= \exp \bigl[ \lambda \Phi_{\frac{\lambda}{N}}(\rr) \bigr]  T \bigl[
\exp ( - \lambda r_1) \bigr]$ we obtain
\begin{lemma}
\mypoint For any $\theta \in \Theta$ and any partially
exchangeable positive real
measurable function $\lambda: \Omega \rightarrow \RR_+$,
$$
T \Bigl\{ \exp \Bigl[ \lambda \bigl\{ \Phi_{\frac{\lambda}{N}}
\bigl[ \rr(\theta) \bigr]  - r_1(\theta) \bigr\} \Bigr]
\Bigr\} \leq 1.
$$
\end{lemma}

We deduce from this lemma a result analogous to the inductive case:
\begin{thm}
\label{thm1.2}
\mypoint For any partially exchangeable positive real measurable
function $\lambda: \Omega \times \Theta \rightarrow \RR_+$,
for any partially exchangeable posterior distribution
$\pi: \Omega \rightarrow \C{M}_+^1(\Theta)$,
$$
\PP \biggl\{ \exp \biggl[ \sup_{\rho \in \C{M}_+^1(\Theta)}
\rho \Bigl[ \lambda \bigl[ \Phi_{\frac{\lambda}{N}}(\rr) - r_1 \bigr] \Bigr]
- \C{K}(\rho, \pi) \biggr] \biggr\} \leq 1.
$$
\end{thm}

The proof is deduced from the previous lemma, using the
fact that $\pi$ is partially exchangeable:
\begin{multline*}
\PP \biggl\{ \exp \biggl[ \sup_{\rho \in \C{M}_+^1(\Theta)}
\rho \Bigl[ \lambda \bigl[ \Phi_{\frac{\lambda}{N}}(\rr) - r_1 \bigr] \Bigr]
- \C{K}(\rho, \pi) \biggr] \biggr\} \\ =
\PP \biggl\{ \pi \Bigl\{ \exp \Bigl[ \lambda \bigl[ \Phi_{\frac{\lambda}{N}}(\rr) -
r_1 \bigr] \Bigr] \Bigr\} \biggr\} =
\PP \biggl\{ T \pi \Bigl\{ \exp \Bigl[ \lambda \bigl[ \Phi_{\frac{\lambda}{N}}(\rr) -
r_1 \bigr] \Bigr] \Bigr\} \biggr\} \\ =
\PP \biggl\{  \pi \Bigl\{ T \exp \Bigl[ \lambda \bigl[ \Phi_{\frac{\lambda}{N}}(\rr) -
r_1 \bigr] \Bigr] \Bigr\} \biggr\} \leq 1.
\end{multline*}

\subsection{Relative bounds}
Introducing in the same way
\newcommand{\Bm}{\overline{m}}
\begin{align*}
m'(\theta, \theta') & = \frac{1}{N}
\sum_{i=1}^{N} \Bigl\lvert \B{1} \bigl[ f_{\theta}(X_i) \neq Y_i \bigr]
- \B{1}\bigl[ f_{\theta'}(X_i) \neq Y_i \bigr] \Bigr\rvert\\
\text{and } \quad \Bm(\theta, \theta') & = \frac{1}{(k+1)N}
\sum_{i=1}^{(k+1)N} \Bigl\lvert \B{1} \bigl[ f_{\theta}(X_i) \neq Y_i \bigr]
- \B{1}\bigl[ f_{\theta'}(X_i) \neq Y_i \bigr] \Bigr\rvert,
\end{align*}
we could prove along the same line of reasoning
\begin{thm}\mypoint
For any real parameter $\lambda$, any $\T \in \Theta$, any partially exchangeable
posterior distribution $\pi: \Omega \rightarrow \C{M}_+^1(\Theta)$,
\begin{multline*}
\PP \biggl\{ \exp \biggl[ \sup_{\rho \in \C{M}_+^1(\Theta)}
\lambda \Bigl[ \rho \bigl\{
\Psi_{\frac{\lambda}{N}} \bigl[ \rr(\cdot) - \rr(\T), \Bm(\cdot, \T)\bigr]
\bigr\} \\* -
\bigl[ \rho(r_1) - r_1(\T) \bigr] \Bigr] - \C{K}(\rho, \pi) \biggr] \biggr\}
\leq 1,
\end{multline*}
where the function $\Psi_{\frac{\lambda}{N}}$ was defined by equation
\myeq{eq1.19}.
\end{thm}
\begin{thm}\mypoint
For any real constant $\gamma$, for any $\T \in \Theta$,
for any partially exchangeable posterior distribution $\pi: \Omega
\rightarrow \C{M}_+^1(\Theta)$,
\begin{multline*}
\PP \Biggl\{ \exp \Biggl[ \sup_{\rho \in \C{M}_+^1(\Theta)}
\biggl\{ - N \rho \Bigl\{ \log \Bigl[ 1 - \tanh\bigl(\tfrac{\gamma}{N}\bigr) \bigl[ \rr(\cdot) - \rr(\T) \bigr]
\Bigr] \Bigr\} \\
- \gamma
\bigl[\rho(r_1) - r_1(\T) \bigr] -
N \log \bigl[ \cosh \bigl( \tfrac{\gamma}{N} \bigr) \bigr] \rho \bigl[ m'( \cdot, \T) \bigr] -
\C{K}(\rho, \pi) \biggr\} \Biggr] \Biggr\} \leq 1.
\end{multline*}
\end{thm}

This last theorem can be generalized to give
\begin{thm}\mypoint
For any real constant $\gamma$, for any partially
exchangeable posterior distributions $\pi^1, \pi^2: \Omega
\rightarrow \C{M}_+^1(\Theta)$,
\begin{multline*}
\PP \Biggl\{ \exp \Biggl[
\sup_{\rho_1, \rho_2 \in \C{M}_+^1(\Theta)}
\biggl\{
- N \log \Bigl\{ 1 - \tanh\bigl( \tfrac{\gamma}{N} \bigr)
\bigl[ \rho_1(\rr) - \rho_2(\rr) \bigr] \Bigr\} \\
- \gamma \bigl[ \rho_1(r_1) - \rho_2(r_1) \bigr]
- N \log \bigl[ \cosh \bigl( \tfrac{\gamma}{N}
\bigr) \bigr]
\rho_1 \otimes \rho_2 (m') \\ - \C{K}(\rho_1, \pi^1) -
\C{K}(\rho_2, \pi^2) \biggr\} \Biggr] \Biggr\} \leq 1.
\end{multline*}
\end{thm}

To conclude this section, we see that the basic theorems of transductive PAC-Bayesian
classification have exactly the same form as the basic inequalities of inductive
classification, Theorems \thmref{thm2.3}, \thmref{thm4.1}
and \thmref{thm2.2.18}
\emph{with $R(\theta)$ replaced with $\rr(\theta)$}, $r(\theta)$ replaced
with $r_1(\theta)$ and $M'(\theta, \T)$
replaced with $\Bm(\theta, \T)$.
\label{page97}

\emph{Thus all the results of the first two chapters
remain true under the hypotheses
of transductive classification, with $R(\theta)$ replaced with $\rr(\theta)$,
$r(\theta)$ replaced with $r_1(\theta)$
and $M'(\theta, \T\,)$ replaced with $\Bm(\theta, \T)$.}

\emph{Consequently, in the case when the unlabelled shadow sample is observed,
it is possible
to improve on the Vapnik bounds to be discussed hereafter by using
an explicit partially exchangeable posterior distribution $\pi$ and
resorting to localized or to relative bounds (in the case at least of
unlimited computing resources, which of course may still be unrealistic
in many real world situations, and with the caveat, to be recalled in
the conclusion of this study, that for small sample sizes and comparatively
complex classification models, the improvement may not be so decisive).}

Let us notice also that the transductive setting when experimentally available,
has the advantage that
\newcommand{\Bd}{\overline{d}}
\begin{multline*}
\Bd(\theta, \theta') = \frac{1}{(k+1)N}
\sum_{i=1}^{(k+1)N} \B{1} \bigl[ f_{\theta'}(X_i) \neq f_{\theta}(X_i) \bigr]
\\ \geq \Bm(\theta, \theta') \geq \rr(\theta) - \rr(\theta'), \qquad
\theta, \theta' \in \Theta,
\end{multline*}
is observable in this context, providing an empirical upper bound for
the difference
$\rr(\wtheta) - \rho(\rr)$ for any non-randomized estimator
$\wtheta$ and any posterior distribution $\rho$, namely
$$
\rr(\wtheta) \leq \rho(\rr) + \rho\bigl[\,\Bd( \cdot, \wtheta)\bigr].
$$
Thus in the setting of transductive statistical experiments,
the PAC-Bayesian framework provides fully empirical bounds
for the error rate of non-randomized estimators $\wtheta:
\Omega \rightarrow \Theta$, even when using a non-atomic
prior $\pi$ (or more generally a non-atomic partially exchangeable
posterior distribution $\pi$), even when $\Theta$
is not a vector space and even when $\theta \mapsto R(\theta)$
cannot be proved to be convex on the support of some useful
posterior distribution $\rho$.

\section{Vapnik bounds for transductive classification}
In this section, we will stick to plain unlocalized non-relative
bounds. As we have already mentioned, (and as it was put forward
by Vapnik himself in his seminal works), these bounds are not always
superseded by the asymptotically better ones when the sample
is of small size: they deserve all our attention for this reason.
We will start with the general case of a shadow sample of arbitrary size.
We will then discuss the case of a shadow sample of equal size to
the training set and the case of a fully exchangeable sample distribution,
showing how they can be taken advantage of to sharpen inequalities.
\eject

\subsection{With a shadow sample of arbitrary size}
The great thing with the transductive setting is that we are manipulating
only $r_1$ and $\rr$ which can take only a finite number of values
and therefore are piecewise constant on $\Theta$. This makes it
possible to derive inequalities that will hold uniformly for
any value of the parameter $\theta \in \Theta$. To this purpose,
let us consider for any value $\theta \in \Theta$ of the parameter
the subset $\Delta(\theta) \subset \Theta$ of parameters $\theta'$ such
that the classification rule $f_{\theta'}$ answers the same on the
extended sample $(X_i)_{i=1}^{(k+1)N}$ as $f_{\theta}$. Namely, let us put
for any $\theta \in \Theta$
$$
\Delta(\theta) = \bigl\{ \theta' \in \Theta ; f_{\theta'}(X_i) = f_{\theta}(X_i),
i = 1, \dots, (k+1)N \bigr\}.
$$
We see immediately that $\Delta(\theta)$ is an exchangeable parameter subset on
which $r_1$ and $r_2$ and therefore also $\rr$ take constant values.
Thus for any $\theta \in \Theta$ we may consider the posterior $\rho_{\theta}$
defined by
$$
\frac{d\rho_{\theta}}{d \pi}(\theta') = \B{1} \bigl[ \theta' \in \Delta(\theta) \bigr]\pi
\bigl[ \Delta(\theta) \bigr]^{-1},
$$
and use the fact that $\rho_{\theta}(r_1) = r_1(\theta)$ and $\rho_{\theta}(\rr) = \rr(\theta)$,
to prove that
\begin{lemma}
\mypoint For any partially exchangeable positive real measurable function
$\lambda: \Omega \times \Theta \rightarrow \RR$ such that
\begin{equation}
\label{eq2.2.1}
\lambda(\omega, \theta') = \lambda(\omega, \theta), \quad \theta \in \Theta, \theta'
\in \Delta(\theta), \omega \in \Omega,
\end{equation}
and any partially exchangeable posterior distribution
$\pi: \Omega \rightarrow \C{M}_+^1(\Theta)$,
with $\PP$ probability at least $1 - \epsilon$, for any $\theta \in \Theta$,
$$
\Phi_{\frac{\lambda}{N}}\bigl[ \rr(\theta) \bigr] + \frac{\log \bigl\{ \epsilon \pi \bigl[
\Delta(\theta) \bigr] \bigr\}}{\lambda(\theta)} \leq r_1(\theta).
$$
\end{lemma}

We can then remark that for any value of $\lambda$ independent of $\omega$,
the left-hand side of the previous inequality is a partially exchangeable function of
$\omega \in \Omega$. Thus this left-hand side is maximized by some
partially exchangeable function $\lambda$, namely $$
\arg\max_{\lambda}
\Biggl\{ \Phi_{\frac{\lambda}{N}} \bigl[ \rr(\theta) \bigr]
+ \frac{\log \bigl\{ \epsilon \pi
\bigl[ \Delta(\theta) \bigr] \bigr\}}{\lambda} \Biggr\}
$$
is partially exchangeable as depending only on partially exchangeable quantities.
Moreover this choice of $\lambda(\omega, \theta)$ satisfies also condition
\eqref{eq2.2.1}
stated in the previous lemma of being constant on $\Delta(\theta)$,
proving
\begin{lemma}
\mypoint For any partially exchangeable posterior distribution $\pi: \Omega
\rightarrow\break
\C{M}_+^1(\Theta)$, with $\PP$ probability at least $1 - \epsilon$,
for any $\theta \in \Theta$ and any $\lambda \in \RR_+$,
$$
\Phi_{\frac{\lambda}{N}} \bigl[ \rr(\theta) \bigr] + \frac{\log \bigl\{
\epsilon \pi \bigl[ \Delta(\theta) \bigr] \bigr\}}{\lambda} \leq r_1(\theta).
$$
\end{lemma}

Writing $\rr = \frac{r_1 + k r_2}{k+1}$ and rearranging terms we obtain
\begin{thm}
\label{thm2.1.5}
\mypoint For any partially exchangeable posterior
distribution $\pi: \Omega \rightarrow
\C{M}_+^1(\Theta)$, with $\PP$ probability at least $1 - \epsilon$,
for any $\theta \in \Theta$,
$$
r_2(\theta) \leq \frac{k+1}{k} \inf_{\lambda \in \RR_+}
\frac{\ds 1 - \exp \left( - \frac{\lambda}{N} r_1(\theta) + \frac{ \log \bigl\{
\epsilon \pi \bigl[ \Delta(\theta) \bigr] \bigr\}}{N} \right)}{\ds 1
- \exp \bigl( - \tfrac{\lambda}{N}\bigr)} - \frac{r_1(\theta)}{k}.
$$
\end{thm}

If we have a set of binary
classification rules $\{ f_{\theta}; \theta \in \Theta \}$ whose
Vapnik--Cervo\-nenkis dimension is not greater than $h$, we can choose $\pi$ such
that $\pi \bigl[ \Delta(\theta) \bigr]$ is independent of $\theta$
and not less than $\ds \left(\frac{h}{e(k+1)N}\right)^h$, as will be proved
further on in Theorem \thmref{th1}.

Another important setting where the complexity term $- \log \bigl\{
\pi \bigl[ \Delta(\theta) \bigr] \bigr\}$ can easily be controlled
is the case of \emph{compression schemes},
introduced by \citet{Little}.
It goes as follows: we are given for each labelled sub-sample
$(X_i, Y_i)_{i \in J}$, $J \subset \{1, \dots, N\}$,
an estimator of the parameter
$$
\wtheta\bigl[ (X_i, Y_i)_{i \in J} \bigr]
= \wtheta_J, \quad J \subset \{ 1, \dots, N \}, \lvert J \rvert \leq h,
$$
\label{compression} where
$$
\wtheta: \bigsqcup_{k=1}^N \bigl( \C{X} \times \C{Y} \bigr)^k \rightarrow \Theta
$$
is an exchangeable function providing estimators for
sub-samples of arbitrary size.
Let us assume that $\w{\theta}$
is exchangeable, meaning that for any $k = 1, \dots, N$ and
any permutation $\sigma$ of $\{1, \dots, k\}$
$$
\w{\theta} \bigl[ (x_i, y_i)_{i=1}^k \bigr]
= \w{\theta} \bigl[ (x_{\sigma(i)}, y_{\sigma(i)})_{i=1}^k
\bigr], \qquad
(x_i, y_i)_{i=1}^k \in \bigl( \C{X} \times \C{Y} \bigr)^k.
$$
In this situation, we can introduce the exchangeable subset
$$
\Bigl\{ \wtheta_J ; J \subset \{1, \dots, (k+1)N\}, \lvert J
\rvert \leq h \Bigr\} \subset \Theta,
$$
which is seen to contain at most $$\ds \sum_{j=0}^h \binom{(k+1)N}{j}
\leq \left( \frac{e(k+1)N}{h} \right)^h$$ classification rules
--- as will be proved later on in Theorem \thmref{th2}.
Note that we had to extend the range of $J$ to all the subsets
of the extended sample, although we will use for estimation
only those of the training sample, on which the labels
are observed.
Thus in this case also we can find a partially exchangeable posterior
distribution $\pi$ such that $$\ds \pi \bigl[ \Delta(\wtheta_J) \bigr]
\geq \left( \frac{h}{e(k+1)N} \right)^h.$$ We see that the size of
the compression scheme plays the same role in this complexity bound
as the Vapnik--Cervonenkis dimension for Vapnik--Cervonenkis classes.

In these two cases of binary classification with Vapnik--Cervonenkis dimension
not greater than $h$ and compression schemes depending on a
compression set with at most $h$ points, we get a bound of
\begin{multline*}
r_2(\theta) \leq \frac{k+1}{k} \inf_{\lambda \in \RR_+}
\frac{\ds 1 - \exp \left( - \frac{\lambda}{N} r_1(\theta) - \frac{ h
\log \left( \frac{e(k+1)N}{h} \right) - \log(\epsilon)}{N} \right)}{\ds 1
- \exp \bigl( - \tfrac{\lambda}{N}\bigr)} \\ - \frac{r_1(\theta)}{k}.
\end{multline*}
Let us make some numerical application: when $N = 1000, h = 10, \epsilon = 0.01$,
and $\inf_{\Theta} r_1 = r_1(\w{\theta}) = 0.2$,
we find that $r_2(\w{\theta}) \leq 0.4093$, for $k$ between
$15$ and $17$, and values of $\lambda$ equal respectively to $965$,
$968$ and $971$. For $k=1$, we find only $r_2(\w{\theta}) \leq 0.539$, showing
the interest of allowing $k$ to be larger than $1$.

\subsection{When the shadow sample has the same size as the training sample}
In the case when $k = 1$, we can improve Theorem \ref{thm1.2} by taking advantage
of the fact that $T_i(\sigma_i)$ can take only $3$ values, namely $0$, $0.5$
and $1$. We see thus that $T_i(\sigma_i) - \Phi_{\frac{\lambda}{N}}\bigl[
T_i(\sigma_i) \bigr]$ can take only two values, $0$ and $\frac{1}{2} - \Phi_{\frac{
\lambda}{N}}(\frac{1}{2})$, because $\Phi_{\frac{\lambda}{N}}(0) = 0$ and
$\Phi_{\frac{\lambda}{N}}(1) = 1$. Thus
$$
T_i(\sigma_i) - \Phi_{\frac{\lambda}{N}} \bigl[ T_i(\sigma_i) \bigr]
= \bigl[ 1 - \lvert 1 - 2 T_i(\sigma_i) \rvert \bigr] \bigl[
\tfrac{1}{2} - \Phi_{\frac{\lambda}{N}}(\tfrac{1}{2}) \bigr].
$$
This shows that in the case when $k=1$,
\begin{multline*}
\log \Bigl\{ T \bigl[ \exp ( - \lambda r_1) \bigr] \Bigr\}
= - \lambda \rr
+ \frac{\lambda}{N} \sum_{i=1}^N T_i(\sigma_i) - \Phi_{\frac{\lambda}{N}}
\bigl[ T_i(\sigma_i) \bigr]\\
= - \lambda \rr + \frac{\lambda}{N} \sum_{i=1}^N \bigl[ 1 - \lvert 1 - 2 T_i(\sigma_i) \rvert
\bigr] \bigl[ \tfrac{1}{2} - \Phi_{\frac{\lambda}{N}}(\tfrac{1}{2}) \bigr]
\\ \leq - \lambda \rr + \lambda \bigl[ \tfrac{1}{2} - \Phi_{\frac{\lambda}{N}}(\tfrac{1}{2}) \bigr] \bigl[ 1 - \lvert 1 - 2 \rr \rvert \bigr].
\end{multline*}
Noticing that $\frac{1}{2} - \Phi_{\frac{\lambda}{N}}(\frac{1}{2}) =
\frac{N}{\lambda} \log \bigl[ \cosh(\frac{\lambda}{2N}) \bigr]$,
we obtain
\begin{thm}
\mypoint For any partially exchangeable function $\lambda: \Omega \times \Theta
\rightarrow \RR_+$, for any partially exchangeable posterior distribution
$\pi: \Omega \rightarrow \C{M}_+^1(\Theta)$,
\begin{multline*}
\PP \biggl\{ \exp \biggl[
\sup_{\rho \in \C{M}_+^1(\Theta)}
\rho \Bigl[ \lambda ( \rr - r_1) \\ -
N \log \bigl[ \cosh(\tfrac{\lambda}{2N}) \bigr]
\bigl( 1 - \lvert 1 - 2 \rr \rvert \bigr) \Bigr] - \C{K}(\rho, \pi) \biggr]
\biggr\} \leq 1.
\end{multline*}
\end{thm}
As a consequence, reasoning as previously, we deduce
\begin{thm}
\label{thm2.2.5}
\mypoint In the case when $k=1$,
for any partially exchangeable posterior distribution $\pi: \Omega
\rightarrow \C{M}_+^1(\Theta)$, with $\PP$ probability at least
$1 - \epsilon$, for any $\theta \in \Theta$ and any
$\lambda \in \RR_+$,
$$
\rr(\theta) - \tfrac{N}{\lambda} \log \bigl[
\cosh(\tfrac{\lambda}{2N}) \bigr] \bigl( 1 - \lvert 1
- 2 \rr(\theta) \rvert \bigr) + \frac{ \log \bigl\{ \epsilon
\pi\bigl[\Delta(\theta)\bigr] \bigr\}}{\lambda} \leq r_1(\theta);
$$
and consequently for any $\theta \in \Theta$,
$$
r_2(\theta) \leq 2 \inf_{\lambda \in \RR_+} \frac{\ds r_1(\theta) - \frac{\log \bigl\{
\epsilon \pi \bigl[ \Delta(\theta) \bigr] \bigr\}}{\lambda}}{
1 - \frac{2N}{\lambda} \log \bigl[ \cosh(\frac{\lambda}{2N})
\bigr]} - r_1(\theta).
$$
\end{thm}

In the case of binary classification using a Vapnik--Cervonenkis class
of\break Vapnik--Cervonenkis dimension not greater than $h$, we can choose $\pi$ such that
$- \log \bigl\{ \pi \bigl[ \Delta(\theta) \bigr] \bigr\}
\leq h \log ( \frac{2eN}{h})$ and obtain the following
numerical illustration of this theorem: for $N = 1000$, $h = 10$,
$\epsilon = 0.01$ and $\inf_{\Theta} r_1 = r_1(\w{\theta}) = 0.2$,
we find an upper bound $r_2(\w{\theta})
\leq 0.5033$, which improves on Theorem \ref{thm2.1.5} but still
is not under the significance level $\frac{1}{2}$ (achieved by
blind random classification). This indicates that considering
shadow samples of arbitrary sizes some noisy situations yields
a significant improvement on bounds obtained with a shadow sample
of the same size as the training sample.

\subsection{When moreover the distribution of the augmented sample
is exchangeable} When $k=1$ and $\PP$ is exchangeable meaning that for
any bounded measurable function $h: \Omega \rightarrow \RR$
and any permutation $s \in \mathfrak{S} \bigl(
\{1, \dots, 2N \} \bigr)$ $\PP \bigl[ h( \omega \circ s ) \bigr]
= \PP \bigl[ h(\omega) \bigr]$, then we can still improve the bound
as follows. Let
$$
T' (h) = \frac{1}{N!} \sum_{s \in \mathfrak{S}
\bigl( \{ N+1, \dots, 2N \} \bigr)} h(\omega \circ s).
$$
Then we can write
$$
1 - \lvert 1 - 2 T_i(\sigma_i) \rvert = (\sigma_i - \sigma_{i+N})^2
= \sigma_i + \sigma_{i+N} - 2 \sigma_i \sigma_{i+N}.
$$
Using this identity, we get for any exchangeable function
$\lambda: \Omega \times \Theta \rightarrow \RR_+$,
$$
T \biggl\{ \exp \biggl[ \lambda (\rr - r_1) - \log \bigl[ \cosh(\tfrac{\lambda}{2N}
) \bigr] \sum_{i=1}^N \bigl( \sigma_i + \sigma_{i+N} - 2 \sigma_i \sigma_{i+N}
\bigr) \biggr] \biggl\} \leq 1.
$$
Let us put
\label{page39}
\begin{align}
\label{eq2.2}
A(\lambda) & = \tfrac{2N}{\lambda} \log \bigl[ \cosh(\tfrac{\lambda}{2N}
) \bigr],\\
v(\theta) & = \frac{1}{2N} \sum_{i=1}^N (\sigma_i + \sigma_{i+N}
- 2 \sigma_i \sigma_{i+N}).
\end{align}
With this notation
$$
T \Bigl\{ \exp \bigl\{ \lambda \bigl[ \rr - r_1 - A(\lambda) v \bigr] \bigr\}
\Bigr\} \leq 1.
$$
Let us notice now that
$$
T'\bigl[ v(\theta) \bigr] = \rr(\theta) - r_1(\theta) r_2(\theta).
$$
Let $\pi: \Omega \rightarrow \C{M}_+^1(\Theta)$ be any given
exchangeable posterior distribution. Using the exchangeability
of $\PP$ and $\pi$ and the exchangeability of the exponential
function, we get
\begin{align*}
\PP & \Bigl\{ \pi \Bigl[ \exp \bigl\{ \lambda \bigl[
\rr - r_1 - A(\rr - r_1 r_2) \bigr] \bigr\} \Bigr] \Bigr\}
 = \PP \Bigl\{ \pi \Bigl[ \exp \bigl\{ \lambda \bigl[
\rr - r_1 - AT'(v) \bigr] \bigr\} \Bigr] \Bigr\}
\\ & \leq
\PP \Bigl\{ \pi \Bigl[ T' \exp \bigl\{ \lambda \bigl[
\rr - r_1 - Av \bigr] \bigr\} \Bigr] \Bigr\}
 =
\PP \Bigl\{ T' \pi \Bigl[ \exp \bigl\{ \lambda \bigl[
\rr - r_1 - Av \bigr] \bigr\} \Bigr] \Bigr\}
\\ & =
\PP \Bigl\{ \pi \Bigl[ \exp \bigl\{ \lambda \bigl[
\rr - r_1 - Av \bigr] \bigr\} \Bigr] \Bigr\}
 =
\PP \Bigl\{ T \pi \Bigl[ \exp \bigl\{ \lambda \bigl[
\rr - r_1 - Av \bigr] \bigr\} \Bigr] \Bigr\}
\\  & =
\PP \Bigl\{ \pi \Bigl[ T \exp \bigl\{ \lambda \bigl[
\rr - r_1 - Av \bigr] \bigr\} \Bigr] \Bigr\}
\leq 1.
\end{align*}
We are thus ready to state
\begin{thm}
\label{thm3.3.8}
\mypoint
In the case when $k = 1$, for any exchangeable probability distribution $\PP$,
for any exchangeable posterior distribution $\pi: \Omega \rightarrow
\C{M}_+^1(\Theta)$, for any exchangeable function
$\lambda: \Omega \times \Theta \rightarrow \RR_+$,
$$
\PP \biggl\{ \exp \biggl[ \sup_{\rho \in \C{M}_+^1(\Theta)}
\rho \Bigl\{ \lambda \bigl[ \rr - r_1 - A(\lambda)(\rr - r_1 r_2)\bigr] \Bigr\}
- \C{K}(\rho, \pi) \biggr] \biggr\} \leq 1,
$$
where $A(\lambda)$ is defined by equation \myeq{eq2.2}.
\end{thm}

We then deduce as previously
\begin{cor}
\label{thm2.2.6}
\mypoint For any exchangeable posterior distribution $\pi:
\Omega \rightarrow \C{M}_+^1(\Theta)$, for any
exchangeable probability measure $\PP \in \C{M}_+^1(\Omega)$,
for any measurable exchangeable function $\lambda: \Omega \times \Theta
\rightarrow \RR_+$,
with $\PP$ probability at least $1 - \epsilon$, for any $\theta \in \Theta$,
$$
\rr(\theta) \leq r_1(\theta) + A(\lambda) \bigl[ \rr(\theta) - r_1( \theta)
r_2(\theta) \bigr] - \frac{ \log \bigl\{ \epsilon \pi\bigl[
\Delta(\theta) \bigr] \bigr\}}{\lambda},
$$
where $A(\lambda)$ is defined by equation \myeq{eq2.2}.
\end{cor}

In order to deduce an empirical bound from this theorem, we have
to make some choice for $\lambda(\omega, \theta)$.
Fortunately, it is easy to show that the bound  holds uniformly
in $\lambda$, because the inequality can
be rewritten as a function of only one non-exchangeable quantity,
namely $r_1(\theta)$. Indeed, since
$r_2 = 2 \rr - r_1$, we see that the
inequality can be written as
$$
\rr(\theta) \leq r_1(\theta) + A(\lambda) \bigl[
\rr(\theta) - 2 \rr(\theta) r_1(\theta) + r_1(\theta)^2 \bigr]
- \frac{\log \bigl\{ \epsilon \pi \bigl[ \Delta(\theta)\bigr]}{\lambda}.
$$
It can be solved in $r_1(\theta)$, to get
$$
r_1(\theta) \geq f \Bigl(\lambda, \rr(\theta), -\log \bigl\{ \epsilon
\pi\bigl[ \Delta(\theta) \bigr] \bigr\} \Bigr),
$$
where
\begin{multline*}
f(\lambda, \rr, d) = \bigl[2 A(\lambda)\bigr]^{-1}
\biggl\{ 2 \rr A(\lambda) - 1 \\ + \sqrt{\bigl[1 - 2 \rr A(\lambda)\bigr]^2
+ 4 A(\lambda) \Bigl\{ \rr\bigl[ 1 - A(\lambda) \bigr] - \tfrac{d}{\lambda}
\Bigr\}} \biggr\}.
\end{multline*}
Thus we can find some exchangeable function $\lambda(\omega, \theta)$,
such that
$$
f\Bigl( \lambda(\omega, \theta), \rr(\theta), -
\log \bigl\{ \epsilon \pi \bigl[ \Delta(\theta) \bigr] \bigr\} \Bigr)
= \sup_{\beta \in \RR_+} f \Bigl( \beta, \rr(\theta), - \log\bigl\{
\epsilon \pi \bigl[ \Delta(\theta) \bigr]\bigr\} \Bigr).
$$
Applying Corollary \thmref{thm2.2.6} to that choice of $\lambda$, we
see that
\begin{thm}
\mypoint For any exchangeable probability measure
$\PP \in \C{M}_+^1(\Omega)$, for any exchangeable posterior
probability distribution $\pi: \Omega \rightarrow \C{M}_+^1(\Theta)$,
with $\PP$ probability at least $1 - \epsilon$, for any $\theta \in \Theta$,
for any $\lambda \in \RR_+$,
$$
\rr(\theta) \leq  r_1(\theta) + A(\lambda) \bigl[
\rr(\theta) - r_1(\theta) r_2(\theta) \bigr] - \frac{
\log \bigl\{ \epsilon \pi \bigl[ \Delta(\theta) \bigr] \bigr\}}{\lambda},
$$
where $A(\lambda)$ is defined by equation \myeq{eq2.2}.
\end{thm}

Solving the previous inequality in $r_2(\theta)$, we get
\begin{cor}
\mypoint Under the same assumptions as in the
previous theorem, with
$\PP$ probability at least $1 - \epsilon$, for any
$\theta \in \Theta$,
$$
r_2(\theta) \leq \inf_{\lambda \in \RR_+}
\frac{\ds r_1(\theta) \Bigl\{ 1 + \tfrac{2N}{\lambda}\log \bigl[
\cosh(\tfrac{\lambda}{2N})\bigr] \Bigr\} - \frac{ 2 \log \bigl\{ \epsilon \pi
\bigl[ \Delta(\theta) \bigr] \bigr\}}{\lambda}}{\ds 1 - \tfrac{2N}{\lambda}
\log \bigl[ \cosh(\tfrac{\lambda}{2N})\bigr] \bigl[
1 - 2 r_1(\theta) \bigr]}.
$$
\end{cor}

Applying this to our usual numerical example of a binary classification
model with Vapnik--Cervonenkis dimension not greater than $h = 10$, when $N=1000$, $
\inf_{\Theta} r_1 = r_1(\w{\theta}) = 10$ and
$\epsilon = 0.01$, we obtain that $r_2(\w{\theta}) \leq 0.4450$.

\section{Vapnik bounds for inductive classification}
\subsection{Arbitrary shadow sample size}
\newcommand{\F}[1]{\mathfrak{#1}}
We assume in this section that
$$
\PP = \biggl( \bigotimes_{i=1}^N P_i
\biggr)^{\otimes \, \infty} \in \C{M}_+^1 \Bigl\{ \bigl[
\bigl( \C{X} \times \C{Y} \bigr)^N \bigr]^{\NN} \Bigr\},
$$
where
$P_i \in \C{M}_+^1\bigl( \C{X} \times \C{Y} \bigr)$:
we consider an infinite i.i.d. sequence of independent
\emph{non}-identically distributed samples of size $N$,
the first one only being observed. More precisely, under $\PP$ each sample
$(X_{i+jN}, Y_{i+jN})_{i=1}^N$ is distributed according
to $\bigotimes_{i=1}^N P_i$, and they are all independent from
each other. Only the first sample $(X_i,Y_i)_{i=1}^N$ is assumed
to be observed. The shadow samples will only appear
in the proofs. The aim of this section is to prove better Vapnik
bounds, generalizing them in the same time to the independent
non-i.i.d. setting, which to our knowledge has not been done before.

Let us introduce the notation $\PP'\bigl[h(\omega) \bigr]  =
\PP \bigl[ h(\omega) \,\lvert\, (X_i,Y_i)_{i=1}^N \bigr]$,
where $h$ may be any suitable (e.g. bounded)
random variable, let us also put
$\Omega = \bigl[(\C{X} \times \C{Y})^N \bigr]^{\NN}$.
\begin{dfn}
\mypoint For any subset $A \subset \NN$ of
integers, let $\F{C}(A)$ be the set of circular permutations of the
totally ordered set $A$, extended to a permutation of $\NN$ by
taking it to be the identity on the complement $\NN \setminus A$
of $A$.
We will say that a random function $h: \Omega \rightarrow \RR$ is $k$-partially
exchangeable if
$$
h( \omega \circ s ) = h( \omega ), \quad s \in \F{C}\bigl(
\{i + j N\,;\,j=0, \dots, k \} \bigr), i=1, \dots, N.
$$

In the same way, we will say that a posterior distribution
$\pi: \Omega \rightarrow \C{M}_+^1(\Theta)$ is $k$-partially
exchangeable if
$$
\pi( \omega \circ s ) = \pi ( \omega ) \in \C{M}_+^1(\Theta), \quad s \in \F{C}\bigl(
\{i + j N\,;\,j=0, \dots, k \} \bigr), i=1, \dots, N.
$$
\end{dfn}
Note that $\PP$ itself is $k$-partially exchangeable for any $k$ in the
sense that for any bounded measurable function $h: \Omega \rightarrow \RR$
$$
\PP \bigl[ h( \omega \circ s ) \bigr]  =  \PP \bigl[ h( \omega ) \bigr] , \quad s \in \F{C}\bigl(
\{i + j N\,;\,j=0, \dots, k \} \bigr), i=1, \dots, N.
$$
Let $\ds
\Delta_k(\theta) = \Bigl\{ \theta' \in \Theta \,;\,
\bigl[ f_{\theta'}(X_i) \bigr]_{i=1}^{(k+1)N} =
\bigl[ f_{\theta}(X_i) \bigr]_{i=1}^{(k+1)N} \Bigr\},$ $\theta \in \Theta,
k \in \NN^*$,
and let also $\ds \rr_k(\theta) = \frac{1}{(k+1)N} \sum_{i=1}^{(k+1) N}
\B{1} \bigl[ f_{\theta}(X_i) \neq Y_i \bigr]$.
Theorem \ref{thm1.2} shows that for any positive real parameter
$\lambda$
and any $k$-partially exchangeable posterior distribution $\pi_k: \Omega
\rightarrow \C{M}_+^1(\Theta)$,
$$
\PP \biggl\{ \exp \biggl[ \sup_{\theta \in \Theta}
\lambda \bigl[ \Phi_{\frac{\lambda}{N}}(\rr_k) - r_1 \bigr]
+ \log \bigl\{ \epsilon \pi_k \bigl[ \Delta_k (\theta) \bigr] \bigr\} \biggr] \biggr\}
\leq \epsilon.
$$
Using the general fact that
$$
\PP \bigl[ \exp( h ) \bigr] =
\PP \Bigl\{ \PP' \bigl[ \exp( h) \bigr] \Bigr\} \geq \PP \Bigl\{
\exp \bigl[ \PP' (h) \bigr] \Bigr\},
$$
and the fact that the expectation of a supremum is larger than the
supremum of an expectation, we see that with $\PP$ probability
at most $1 - \epsilon$, for any $\theta \in \Theta$,
$$
\PP'\Bigl\{ \Phi_{\frac{\lambda}{N}} \bigl[ \rr_k(\theta) \bigr]
\Bigr\} \leq r_1(\theta) - \frac{
\PP' \Bigl\{ \log \bigl\{ \epsilon \pi_k \bigl[ \Delta_k(\theta) \bigr] \bigr\}
\Bigr\}}{\lambda}.
$$
For short let us put
\newcommand{\dd}{\Bar{d}}
\begin{align*}
\dd_k(\theta)  & = - \log \bigl\{ \epsilon \pi_k \bigl[ \Delta_k(\theta) \bigr] \bigr\},\\
d'_k(\theta) & = - \PP' \Bigl\{ \log \bigl\{ \epsilon \pi_k \bigl[ \Delta_k(\theta) \bigr] \bigr\}
\Bigr\},\\
d_k(\theta) & = - \PP \Bigl\{ \log \bigl\{ \epsilon \pi_k \bigl[ \Delta_k(\theta) \bigr] \bigr\}
\Bigr\}.
\end{align*}

We can use the convexity of $\Phi_{\frac{\lambda}{N}}$ and the fact
that $\PP'(\rr_k) = \frac{r_1 + k R}{k+1}$, to establish that
$$
\PP' \Bigl\{ \Phi_{\frac{\lambda}{N}} \bigl[ \rr_k(\theta) \bigr]
\Bigr\} \geq \Phi_{\frac{\lambda}{N}}
\left[ \frac{r_1(\theta) + k R(\theta)}{k+1} \right].
$$
We have proved
\begin{thm}
\mypoint Using the above hypotheses and notation,
for any sequence
$\pi_k: \Omega \rightarrow \C{M}_+^1(\Theta)$, where $\pi_k$
is a $k$-partially exchangeable posterior distribution,
for any positive real constant $\lambda$, any positive integer $k$,
with $\PP$ probability
at least $1 - \epsilon$, for any $\theta \in \Theta$,
$$
\Phi_{\frac{\lambda}{N}} \left[
\frac{ r_1(\theta) + k R(\theta)}{k+1} \right]
\leq r_1(\theta) + \frac{d'_k(\theta)}{\lambda}.
$$
\end{thm}
We can make
as we did with Theorem \thmref{thm2.7} the
result of this theorem uniform in $\lambda \in \{ \alpha^j\,;\,
j \in \NN^* \}$ and $k \in \NN^*$ (considering
on $k$ the prior $\frac{1}{k(k+1)}$ and on $j$ the prior
$\frac{1}{j(j+1)}$), and obtain

\begin{thm}
\mypoint For any real parameter
$\alpha > 1$, with $\PP$ probability at least $1 - \epsilon$,
for any $\theta \in \Theta$,
\begin{multline*}
R(\theta) \leq  \\* \inf_{k \in \NN^*, j \in \NN^*}
\frac{1 - \exp \biggl\{ - \frac{\alpha^j}{N} r_1(\theta) - \frac{1}{N}
\Bigl\{ d'_k(\theta) + \log \bigl[ k (k+1) j (j+1)\bigr]
\Bigr\} \biggr\}}{\frac{k}{k+1} \left[ 1 -
\exp \left( - \frac{\alpha^j}{N}\right) \right] } \\* - \frac{r_1(\theta)}{k}.
\end{multline*}
\end{thm}

As a special case we can choose $\pi_k$ such that $
\log \bigl\{ \pi_k\bigl[ \Delta_k(\theta) \bigr] \bigr\}$ is independent of
$\theta$ and equal to $\log (\F{N}_k)$, where
$$
\F{N}_k = \bigl\lvert  \bigl\{
\bigl[ f_{\theta}(X_i) \bigr]_{i=1}^{(k+1)N} \,;\,
\theta \in \Theta \bigr\} \bigr\rvert$$ is the size of the trace of the
classification model on the extended sample
of size $(k+1)N$.
With this choice, we obtain a bound involving a new flavour
of conditional Vapnik entropy, namely
$$
d'_k(\theta) = \PP \bigl[ \log (\F{N}_k) \,\lvert (Z_i)_{i=1}^N \bigr] - \log(\epsilon).
$$

In the case of binary classification using a Vapnik--Cervonenkis class of Vapnik--Cervonenkis dimension not
greater than $h = 10$, when $N = 1000$, $\inf_{\Theta}
r_1 = r_1(\w{\theta}) = 0.2$ and $\epsilon = 0.01$,
choosing $\alpha = 1.1$, we obtain $R(\w{\theta}) \leq 0.4271$
(for an optimal value of $\lambda = 1071.8$, and an optimal
value of $k = 16$).

\subsection{A better minimization with respect to the exponential parameter}If we are not pleased with optimizing $\lambda$ on a discrete
subset of the real line, we can use a slightly different approach.
From Theorem \thmref{thm1.2}, we see that for any positive integer
$k$, for any $k$-partially exchangeable
positive real measurable function $\lambda: \Omega \times \Theta
\rightarrow \RR_+$ satisfying equation \myeq{eq2.2.1}
--- with $\Delta(\theta)$ replaced
with $\Delta_k(\theta)$ ---
for any $\epsilon \in )0,1)$ and $\eta \in )0,1)$,
$$
\PP \biggl\{ \PP' \biggl[ \exp \Bigl[ \sup_{\theta}
\lambda \bigl[ \Phi_{\frac{\lambda}{N}}(\rr_k) - r_1 \bigr] +
\log \bigl\{ \epsilon \eta \pi_k \bigl[ \Delta_k(\theta) \bigr] \bigr\}
\biggr] \biggr\}
\leq \epsilon \eta,
$$
therefore with $\PP$ probability at least $1 - \epsilon$,
$$
\PP' \biggl\{ \exp \Bigl[ \sup_{\theta}
\lambda \bigl[ \Phi_{\frac{\lambda}{N}}(\rr_k) - r_1 \bigr] +
\log \bigl\{ \epsilon \eta \pi_k \bigl[ \Delta_k(\theta) \bigr] \bigr\}
\Bigr]
\biggr\}
\leq \eta,
$$
and consequently, with $\PP$ probability at least $1 - \epsilon$,
with $\PP'$ probability at least $1 - \eta$, for any $\theta \in \Theta$,
$$
\Phi_{\frac{\lambda}{N}}(\rr_k) +
\frac{\log \bigl\{ \epsilon \eta \pi_{k} \bigl[ \Delta_k(\theta)
\bigr] \bigr\}}{\lambda}
\leq r_1.
$$
Now we are entitled to choose $$
\lambda(\omega, \theta)
\in \arg \max_{\lambda' \in \RR_+} \Phi_{\frac{\lambda'}{N}}(\rr_k)
+ \frac{\log \bigl\{ \epsilon \eta \pi_{k} \bigl[ \Delta_k(\theta)
\bigr] \bigr\}}{\lambda'}.
$$
This shows that with $\PP$ probability
at least $1 - \epsilon$, with $\PP'$ probability at least $1 - \eta$,
for any $\theta \in \Theta$,
$$
\sup_{\lambda \in \RR_+} \Phi_{\frac{\lambda}{N}}(\rr_k) -
\frac{\dd_k(\theta) - \log(\eta)}{\lambda}
\leq r_1,
$$
which can also be written
$$
\Phi_{\frac{\lambda}{N}}(\rr_k) - r_1 - \frac{
\dd_k(\theta)}{\lambda} \leq - \frac{\log(\eta)}{\lambda}, \quad \lambda \in \RR_+.
$$
Thus with $\PP$ probability at least $1 - \epsilon$,
for any $\theta \in \Theta$, any $\lambda \in \RR_+$,
$$
\PP'\biggl[ \Phi_{\frac{\lambda}{N}}(\rr_k) - r_1 -
\frac{\dd_k(\theta)}{\lambda} \biggr] \leq - \frac{
\log(\eta)}{\lambda} + \biggl[1 - r_1 + \frac{\log(\eta)}{\lambda}
\biggr] \eta.
$$
On the other hand, $\Phi_{\frac{\lambda}{N}}$ being a convex function,
\begin{align*}
\PP'\biggl[ \Phi_{\frac{\lambda}{N}}(\rr_k) - r_1 -
\frac{\dd_k(\theta)}{\lambda} \biggr]
& \geq \Phi_{\frac{\lambda}{N}}\bigl[ \PP'(\rr_k) \bigr] - r_1
- \frac{d'_k}{\lambda} \\ & = \Phi_{\frac{\lambda}{N}}
\biggl( \frac{kR+r_1}{k+1} \biggr) - r_1 - \frac{d'_k}{\lambda}.
\end{align*}
Thus with $\PP$ probability at least $1 - \epsilon$, for any $\theta \in \Theta$,
$$
\frac{kR+r_1}{k+1} \leq \inf_{\lambda \in \RR_+}
\Phi_{\frac{\lambda}{N}}^{-1} \biggl[ r_1(1 - \eta) + \eta +
\frac{d'_k - \log(\eta) (1 - \eta)}{\lambda} \biggr].
$$
We can generalize this approach by considering a finite decreasing sequence
$\eta_0=1 > \eta_1 > \eta_2 > \dots > \eta_J > \eta_{J+1} = 0$, and
the corresponding sequence of levels
\begin{align*}
L_j & = - \frac{\log(\eta_j)}{\lambda}, 0 \leq j \leq J,\\
L_{J+1} & = 1 - r_1 - \frac{\log(J) - \log(\epsilon)}{\lambda}.
\end{align*}
Taking a union bound in $j$, we see that with $\PP$ probability at least $1 - \epsilon$,
for any $\theta \in \Theta$, for any $\lambda \in \RR_+$,
$$
\PP' \biggl[ \Phi_{\frac{\lambda}{N}}(\rr_k) - r_1
- \frac{\dd_k + \log(J)}{\lambda} \geq L_j \biggr] \leq \eta_j, \quad j=0, \dots, J+1,
$$
and consequently
\begin{align*}
\PP' & \biggl[ \Phi_{\frac{\lambda}{N}}(\rr_k) - r_1
- \frac{\dd_k + \log(J)}{\lambda} \biggr] \\
& \leq \int_{0}^{L_{J+1}}
\PP' \biggl[ \Phi_{\frac{\lambda}{N}}(\rr_k) - r_1
- \frac{\dd_k+ \log(J)}{\lambda} \geq \alpha \biggr] d \alpha
\quad \leq \sum_{j=1}^{J+1} \eta_{j-1}(L_j - L_{j-1})
\\ & = \eta_J \biggl[ 1 - r_1 - \frac{\log(J) -
\log(\epsilon) - \log(\eta_J)}{\lambda}
\biggr] - \frac{\log(\eta_1)}{\lambda} + \sum_{j=1}^{J-1}
\frac{\eta_{j}}{\lambda} \log \biggl(
\frac{\eta_{j}}{\eta_{j+1}}\biggr).
\end{align*}
Let us put
\begin{multline*}
d''_k\bigl[\theta, (\eta_j)_{j=1}^J \bigr]
= d'_k(\theta) +
\log(J) - \log(\eta_1)
\\ + \sum_{j=1}^{J-1}
\eta_j \log \left( \frac{\eta_j}{\eta_{j+1}} \right)
+ \log\left(\frac{\epsilon \eta_J}{J} \right) \eta_J.
\end{multline*}

We have proved that for any decreasing sequence $(\eta_j)_{j=1}^J$,
with $\PP$ probability at least $1 - \epsilon$,
for any $\theta \in \Theta$,
$$
\frac{k R + r_1}{k+1}
\leq \inf_{\lambda \in \RR_+}
\Phi_{\frac{\lambda}{N}}^{-1} \biggl[
r_1(1 - \eta_J) + \eta_J +
\frac{ d''_k \bigl[ \theta, (\eta_j)_{j=1}^J \bigr]}{\lambda} \biggr].
$$

\begin{rmk}
\mypoint We can for instance choose
$J=2$, $\eta_2 = \frac{1}{10N}$, $\eta_1 =
\frac{1}{\log(10 N)}$,
resulting in
$$
d''_k = d'_k + \log(2) + \log\log(10 N) + 1 -
\frac{\log\log(10N)}{\log(10N)} - \frac{\log \left( \frac{20N}{\epsilon} \right)}{10N}.
$$
In the case where $N = 1000$ and for any $\epsilon \in )0,1)$,
we get $d''_k \leq d'_k + 3.7$, in the case where $N = 10^6$,
we get $d''_k \leq d'_k + 4.4$, and in the case $N = 10^9$,
we get $d''_k \leq d'_k + 4.7$.

Therefore, for any practical
purpose we could take $d''_k = d'_k + 4.7$ and $\eta_J = \frac{1}{10N}$
in the above inequality.
\end{rmk}

Taking moreover a weighted union bound in $k$, we get
\begin{thm}
\label{thm2.3.3}
\mypoint For any $\epsilon \in )0,1)$, any sequence
$1 > \eta_1 > \dots > \eta_J > 0$,
any sequence $\pi_k: \Omega \rightarrow \C{M}_+^1(\Theta)$,
where $\pi_k$ is a $k$-partially exchangeable posterior distribution,
with $\PP$ probability at least $1 - \epsilon$, for any $\theta
\in \Theta$,
\begin{multline*}
R(\theta) \leq \inf_{k \in \NN^*} \frac{k+1}{k} \inf_{\lambda \in \RR_+}
\Phi_{\frac{\lambda}{N}}^{-1}
\biggl[ r_1(\theta) + \eta_J \bigl[1 - r_1(\theta) \bigr]
\\ + \frac{d''_k\bigl[\theta, (\eta_j)_{j=1}^J \bigr] + \log\bigl[k(k+1)\bigr]}{\lambda}
\biggr] - \frac{r_1(\theta)}{k}.
\end{multline*}
\end{thm}
\begin{cor}
\label{cor3.3.14}
\mypoint For any $\epsilon \in )0,1)$, for any $N \leq 10^9$, with $\PP$ probability
at least $1 - \epsilon$, for any $\theta \in \Theta$,
\begin{multline*}
R(\theta) \leq
\inf_{k \in \NN^*} \inf_{\lambda \in \RR_+}
\frac{k+1}{k} \bigl[ 1 - \exp( - \tfrac{\lambda}{N}) \bigr]^{-1}
\biggl\{ 1 - \exp \biggl[ - \tfrac{\lambda}{N} \bigl[ r_1(\theta) +
\tfrac{1}{10N} \bigr]
\\ - \frac{ \PP' \bigl[ \log(\F{N}_k)\,\lvert\,(Z_i)_{i=1}^N
\bigr]
- \log(\epsilon) + \log\bigl[k(k+1)\bigr] + 4.7}{N} \biggr]
\biggr\}
- \frac{r_1(\theta)}{k}.
\end{multline*}
\end{cor}

Let us end this section with a numerical example: in the case of binary classification
with a Vapnik--Cervonenkis class of dimension not greater than $10$, when $N=1000$,
$\inf_{\Theta} r_1 = r_1(\w{\theta}) = 0.2$
and $\epsilon = 0.01$, we get a bound $R(\w{\theta}) \leq 0.4211$ (for optimal
values of $k = 15$ and of $\lambda = 1010$).

\subsection{Equal shadow and training sample sizes}In the case
when $k=1$, we can use Theorem \thmref{thm2.2.5} and replace
$\Phi_{\frac{\lambda}{N}}^{-1}(q)$ with $\bigl\{ 1 -
\frac{2N}{\lambda}\times
\log \bigl[ \cosh(\frac{\lambda}{2N}) \bigr] \bigr\}^{-1}q$,
resulting in
\begin{thm}
\mypoint For any $\epsilon \in )0,1)$, any $N \leq 10^9$, any one-partially exchangeable
posterior distribution
$\pi_1: \Omega \rightarrow \C{M}_+^1(\Theta)$,
with $\PP$ probability at least $1 - \epsilon$,
for any $\theta \in \Theta$,
$$
R(\theta) \leq
\inf_{\lambda \in \RR_+} \frac{\ds
\Bigl\{ 1 + \tfrac{2N}{\lambda} \log \bigl[ \cosh(\tfrac{\lambda}{2N}) \bigr] \Bigr\} r_1(\theta)
+ \frac{1}{5N} + 2 \frac{d_1'(\theta) + 4.7}{\lambda}}{\ds
1 - \tfrac{2N}{\lambda} \log \bigl[ \cosh(\tfrac{\lambda}{2N}
) \bigr]}.
$$
\end{thm}

\subsection{Improvement on the equal sample size bound in the i.i.d.~case}
Finally, in the case when $\PP$ is i.i.d., meaning that all the
$P_i$ are equal, we can improve the previous bound. For any
partially exchangeable function $\lambda: \Omega \times \Theta
\rightarrow \RR_+$, we saw in the discussion preceding Theorem
\thmref{thm3.3.8} that
$$
T \Bigl[ \exp \bigl[ \lambda (\rr_k - r_1) - A(\lambda) v \bigr] \Bigr]
\leq 1,
$$
with the notation introduced therein.
Thus for any partially exchangeable positive real measurable function
$\lambda: \Omega \times \Theta \rightarrow \RR_+$ satisfying equation
\myeq{eq2.2.1}, any one-partially exchangeable
posterior distribution $\pi_1: \Omega \rightarrow \C{M}_+^1(\Theta)$,
$$
\PP \Bigl\{ \exp \Bigl[  \sup_{\theta \in \Theta}
\lambda \bigl[ \rr_k(\theta) - r_1(\theta) - A(\lambda)v(\theta) \bigr] + \log \bigl[
\epsilon \pi_1 \bigl[ \Delta(\theta) \bigr] \Bigr] \Bigr\} \leq 1.
$$
Therefore with $\PP$ probability at least $1 - \epsilon$, with $\PP'$
probability $1 - \eta$,
$$
\rr_k(\theta) \leq r_1(\theta) + A(\lambda) v(\theta) + \frac{1}{\lambda} \bigl[
\dd_1(\theta) - \log(\eta) \bigr].
$$

We can then choose $\ds \lambda(\omega, \theta) \in
\arg\min_{\lambda' \in \RR_+} A(\lambda') v(\theta) + \frac{\dd_1(\theta)
- \log(\eta) \bigr]}{\lambda'}$, which satisfies the required
conditions, to show that with $\PP$ probability at least $1 - \epsilon$,
for any $\theta \in \Theta$, with $\PP'$ probability at least $1 - \eta$,
for any $\lambda \in \RR_+$,
$$
\rr_k(\theta) \leq r_1(\theta) +
A(\lambda)v(\theta) + \frac{\dd_1(\theta) - \log(\eta)}{\lambda}.
$$

We can then take a union bound on a decreasing sequence of $J$
values $\eta_1 \geq \dots \geq \eta_J$ of $\eta$.
Weakening the order of quantifiers a little,
we then obtain the following statement:
with $\PP$ probability at least $1 - \epsilon$, for any $\theta \in \Theta$,
for any $\lambda \in \RR_+$, for any $j=1, \dots, J$
$$
\PP' \biggl[ \rr_k(\theta) - r_1(\theta) -
A(\lambda) v(\theta) - \frac{\dd_1(\theta) + \log(J)}{\lambda}
\geq - \frac{\log(\eta_j)}{\lambda}  \biggr] \leq \eta_j.
$$
Consequently for any $\lambda \in \RR_+$,
\begin{multline*}
\PP' \biggl[ \rr_k(\theta) - r_1(\theta) -
A(\lambda) v(\theta) - \frac{\dd_1(\theta) + \log(J)}{\lambda} \biggr]
\\ \leq - \frac{  \log(\eta_1)}{\lambda} +
\eta_J \biggl[1 - r_1(\theta) - \frac{\log(J) - \log(\epsilon) - \log(\eta_J)}{\lambda}
\biggr]
\\ + \sum_{j=1}^{J-1} \frac{\eta_{j}}{\lambda} \log \left( \frac{\eta_j}{\eta_{j+1}}
\right).
\end{multline*}
Moreover $\PP' \bigl[ v(\theta) \bigr] = \frac{r_1 + R}{2} - r_1 R$,
(this is where we need equidistribution) thus proving that
$$
\frac{R - r_1}{2} \leq
\frac{A(\lambda)}{2} \Bigl[ R+r_1 - 2 r_1 R \Bigr]
+ \frac{
d''_1\bigl[\theta, (\eta_j)_{j=1}^J\bigr]
}{\lambda} + \eta_J\bigl[1 - r_1(\theta)\bigr].
$$
Keeping track of quantifiers, we obtain
\begin{thm}
\label{thm2.3.9}
\mypoint For any decreasing sequence $(\eta_j)_{j=1}^J$, any
$\epsilon \in )0,1)$, any one-partially exchangeable posterior
distribution $\pi: \Omega \rightarrow \C{M}_+^1(\Theta)$,
with $\PP$ probability at least $1 - \epsilon$, for any $\theta \in \Theta$,
\begin{multline*}
R(\theta) \leq \inf_{\lambda \in \RR_+} \\
\frac{\ds \Bigl\{ 1  + \tfrac{2N}{\lambda}\log \bigl[ \cosh(\tfrac{\lambda}{2N})
\bigr] \Bigr\} r_1(\theta) + \frac{2 d''_1\bigl[ \theta, (\eta_j)_{j=1}^J
\bigr] }{\lambda} + 2 \eta_J
\bigl[ 1 - r_1(\theta) \bigr]}{\ds
1 - \tfrac{2N}{\lambda}\log\bigl[ \cosh(\tfrac{\lambda}{2N})
\bigr] \bigl[ 1 - 2 r_1(\theta) \bigr] }.
\end{multline*}
\end{thm}

\section{Gaussian approximation in Vapnik bounds}
\subsection{Gaussian upper bounds of variance terms}
To obtain formulas which could be easily compared with original Vapnik bounds,
we may replace $p - \Phi_a(p)$ with a Gaussian upper bound:
\begin{lemma}
\mypoint For any $p \in (0,\frac{1}{2})$, any $a \in \RR_+$,
$$
p - \Phi_a(p) \leq \frac{a}{2} p(1-p).
$$
For any $p \in (\frac{1}{2}, 1)$,
$$
p - \Phi_a(p) \leq \frac{a}{8} .
$$

\end{lemma}
\begin{proof}
Let us notice that for any $p \in (0,1)$,
\begin{align*}
\frac{\partial}{\partial a} \bigl[ - a \Phi_a(p) \bigr]
& = - \frac{p \exp(-a) }{1 - p + p \exp( - a)},\\
\frac{\partial^2}{\partial^2 a} \bigl[ - a \Phi_a(p) \bigr]
& =
\frac{p \exp(-a) }{1 - p + p \exp( - a)}
\left( 1 - \frac{p \exp( - a)}{1 - p + p\exp( - a)} \right) \\
& \leq
\begin{cases}
p(1-p) & p \in (0, \frac{1}{2}),\\
\frac{1}{4} & p \in (\frac{1}{2}, 1).
\end{cases}
\end{align*}
Thus taking a Taylor expansion of order one with integral remainder:
$$
-a \Phi(a) \leq
\begin{cases}
\begin{aligned}[b]-a p + \int_0^a p (1-p) & (a-b) db \\
& = -a p + \frac{a^2}{2}p(1-p),\end{aligned} & p \in
(0,\frac{1}{2}),\\
\ds -a p + \int_0^a \frac{1}{4}(a -b) db = -a p + \frac{a^2}{8}, & p \in
(\frac{1}{2}, 1).
\end{cases}
$$
This ends the proof of our lemma. \end{proof}
\begin{lemma}
\mypoint
\label{lemma2.22}
Let us consider the bound
$$
B(q,d) = \left(1 + \frac{2 d}{N} \right)^{-1}
\biggl[ q + \frac{d}{N} + \sqrt{ \frac{2 d q(1-q)}{N}
+ \frac{d^2}{N^2}} \biggr], \quad q \in \RR_+, d \in \RR_+.
$$
Let us also put
$$
\Bar{B}(q,d) =
\begin{cases}
B(q,d) & B(q,d) \leq \frac{1}{2},\\
q + \sqrt{\frac{d}{2N}} & \text{ otherwise}.
\end{cases}
$$
For any positive real parameters $q$ and $d$
$$
\inf_{\lambda \in \RR_+} \Phi_{\frac{\lambda}{N}}^{-1}
\biggl( q + \frac{d}{\lambda} \biggr) \leq \Bar{B}(q,d).
$$
\end{lemma}
\begin{proof}
Let $\ds p = \inf_{\lambda} \Phi_{\frac{\lambda}{N}}^{-1} \biggl(
q + \frac{d}{\lambda}\,\biggr)$. For any $\lambda \in \RR_+$,
$$
p - \frac{\lambda}{2N} (p \wedge \tfrac{1}{2})\bigl[1 -
(p \wedge \tfrac{1}{2}) \bigr] \leq \Phi_{\frac{\lambda}{N}}(p)
\leq q + \frac{d}{\lambda}.
$$
Thus
\begin{multline*}
p \leq q + \inf_{\lambda \in \RR_+} \frac{\lambda}{2N}
(p \wedge \tfrac{1}{2}) \bigl[ 1 - ( p \wedge \tfrac{1}{2}) \bigr]
+ \frac{d}{\lambda} \\ = q + \sqrt{\frac{2 d
(p \wedge \tfrac{1}{2}) \bigl[ 1 - ( p \wedge \tfrac{1}{2}) \bigr]}{N}}
\leq q + \sqrt{\frac{d}{2N}}.
\end{multline*}
Then let us remark that
$\ds
B(q,d) = \sup \left\{ p' \in \RR_+ \,;\, p' \leq q + \sqrt{\frac{2dp'(1-p')}{N}}
\right\}.$
If moreover $\tfrac{1}{2} \geq B(q,d)$, then according
to this remark $\tfrac{1}{2} \geq q + \sqrt{\frac{d}{2N}} \geq p$.
Therefore $p \leq \tfrac{1}{2}$, and consequently $p \leq q + \sqrt{\frac{2dp(1-p)}{N}}$,
implying that $p \leq B(q,d)$.
\end{proof}

\subsection{Arbitrary shadow sample size}
The previous lemma combined with Corollary \thmref{cor3.3.14}
implies
\begin{cor}
\label{cor2.3.7}
\mypoint Let us use the notation introduced in Lemma \thmref{lemma2.22}.
For any $\epsilon \in )0,1)$, any integer $N \leq 10^9$,
with $\PP$ probability at least $1 - \epsilon$,
for any $\theta \in \Theta$,
$$
R(\theta) \leq \inf_{k \in \NN^*}
\frac{k+1}{k} \Bigl\{
\Bar{B}\Bigl[r_1(\theta) + \frac{1}{10N}, d'_k(\theta) + \log \bigl[
k(k+1)\bigr] + 4.7 \Bigr] \Bigr\} - \frac{r_1(\theta)}{k}.
$$
\end{cor}

\subsection{Equal sample sizes in the i.i.d.~case}
To make a link with Vapnik's result, it is useful to state
the Gaussian approximation to Theorem \thmref{thm2.3.9}.
Indeed, using the upper bound $A(\lambda) \leq \frac{\lambda}{4N}$,
where $A(\lambda)$ is defined by equation \eqref{eq2.2}
on page \pageref{eq2.2}, we
get with $\PP$ probability at least $1 - \epsilon$
$$
R  - r_1 - 2 \eta_J \leq \inf_{\lambda \in \RR_+}
\frac{\lambda}{4N} \bigl[ R + r_1 - 2 r_1 R \bigr]
+ \frac{2 d''_1}{\lambda}
= \sqrt{\frac{2 d''_1 (R + r_1 - 2 r_1 R)}{N}},
$$
which can be solved in $R$ to obtain
\begin{cor}
\label{cor2.3.10}
\mypoint With $\PP$ probability at least
$1 - \epsilon$, for any $\theta \in \Theta$,
\begin{multline*}
R(\theta) \leq r_1(\theta) + \frac{d''_1(\theta)}{N}
\bigl[ 1 - 2 r_1(\theta) \bigr]
+ 2 \eta_J
\\ + \sqrt{ \frac{4 d''_1(\theta) \bigl[ 1 - r_1(\theta) \bigr] r_1(\theta)}{N}
+ \frac{{d''_1}(\theta)^2}{N^2} \bigl[ 1 - 2 r_1(\theta) \bigr]^2
+ \frac{4 d''_1(\theta)}{N} \bigl[ 1 - 2 r_1(\theta) \bigr] \eta_J}.
\end{multline*}
\end{cor}

This is to be compared with Vapnik's result, as proved in \citet[page 138]{Vapnik}:
\begin{thm}[Vapnik]
\label{thmVapnik}
\mypoint For any i.i.d. probability distribution $\PP$,
with $\PP$ probability at least $1 - \epsilon$, for any $\theta \in \Theta$,
putting
$$
d_V = \log \bigl[ \PP (\F{N}_1) \bigr] + \log(4/\epsilon),
$$
$$
R(\theta) \leq r_1(\theta) + \frac{2 d_V}{N} +
\sqrt{ \frac{4 d_V r_1(\theta)}{N} + \frac{4 d_V^2}{N^2}}.
$$
\end{thm}

Recalling that we can choose $(\eta_j)_{j=1}^2$ such that
$\eta_J = \eta_2 = \frac{1}{10N}$ (which brings a negligible
contribution to the bound) and
such that for any $N \leq
10^9$,
$$
d''_1( \theta) \leq \PP \bigl[ \log ( \F{N}_1 ) \,\lvert\,
(Z_i)_{i=1}^N\bigr]
- \log(\epsilon) + 4.7,
$$
we see that our complexity term is somehow more satisfactory than Vapnik's,
since it is integrated outside the logarithm, with a slightly larger additional
constant (remember that $\log 4 \simeq 1.4$, which is better than our $4.7$,
which could presumably be improved by working out a better sequence $\eta_j$,
but not down to $\log(4)$). Our variance term is better, since we get
$r_1(1-r_1)$, instead of $r_1$.
We also have $\ds \frac{d''_1}{N}$ instead of
$\ds 2 \frac{d_V}{N}$, because we use no symmetrization
trick.

Let us illustrate these bounds on a numerical example, corresponding to
a situation where the sample is noisy or the classification model is
weak. Let us assume that $N = 1000$, $
\inf_{\Theta} r_1 = r_1(\w{\theta}) = 0.2$, that we
are performing binary classification with a model with Vapnik--Cervonenkis dimension
not greater than $h = 10$, and that we work at confidence level
$\epsilon = 0.01$. Vapnik's theorem provides an upper bound for
$R(\w{\theta})$ not smaller than
$0.610$, whereas Corollary \ref{cor2.3.10} gives
$R(\w{\theta}) \leq 0.461$ (using the bound $d''_1 \leq d'_1 + 3.7$ when $N = 1000$).
Now if we go for Theorem
\ref{thm2.3.9} and do not make a Gaussian approximation,
we get $R(\w{\theta}) \leq 0.453$.  It is interesting to
remark that this bound is achieved for $\lambda = 1195 > N = 1000$.
This explains why the Gaussian approximation in Vapnik's bound
can be improved: for such a large value of $\lambda$, $\lambda r_1(\theta)$
does not behave like a Gaussian random variable.

Let us recall in conclusion that the best bound is provided by
Theorem \thmref{thm2.3.3}, giving $R(\w{\theta}) \leq 0.4211$,
(that is approximately $2/3$
of Vapnik's bound), for optimal values
of $k = 15$, and of $\lambda = 1010$. This bound can be seen to
take advantage of the fact that Bernoulli random variables
are not Gaussian (its Gaussian approximation, Corollary \ref{cor2.3.7},
gives a bound $R(\theta) \simeq 0.4325$, still with an optimal $k = 15$),
and of the fact that the optimal size of
the shadow sample is significantly larger than the size
of the observed sample. Moreover, Theorem \ref{thm2.3.3} does not
assume that the sample is i.i.d., but only that it is
independent, thus generalizing Vapnik's bounds to inhomogeneous
data (this will presumably be the case when data are collected
from different places where the experimental conditions may
not be the same, although they may reasonably
be assumed to be independent).

Our little numerical example was chosen to illustrate the
case when it is non-trivial to decide whether the chosen
classifier does better than the 0.5 error rate of blind
random classification. This case is of interest to choose
``weak learners'' to be aggregated or combined in some
appropriate way in a second stage to reach a better classification
rate. This stage of feature selection is unavoidable in many real world classification tasks.
Our little computations are meant to exemplify the fact
that Vapnik's bounds, although asymptotically suboptimal,
as is obvious by comparison with the first two chapters,
can do the job when dealing with moderate sample sizes.

\chapter{Support Vector Machines}
\section{How to build them}
\subsection{The canonical hyperplane}
\label{chapSVM}

Support
Vector Machines, of wide use and renown,
were conceived by V.~Vapkik \citep{Vapnik}.
Before introducing them,
we will study as a prerequisite the separation of points by hyperplanes
in a finite dimensional Euclidean space.
Support Vector Machines perform the same kind of linear
separation after
an implicit change of pattern space.
The preceding PAC-Bayesian results provide a
fit framework to analyse their generalization properties.

In this section we deal with the classification
of points in $\RR^d$ in two classes.
Let $Z = (x_i, y_i)_{i=1}^N \in \bigl(\RR^d \times \{-1,+1\}
\bigr)^N$ be some set of labelled examples (called
the training set hereafter). Let us split the set of
indices $I = \{1, \dots, N\}$
according to the labels into two subsets
\begin{align*}
I_+ & = \{ i \in I\,: y_i = + 1 \},\\
I_- & = \{ i \in I\,: y_i = - 1 \}.
\end{align*}
Let us then consider the set of admissible separating directions
$$
A_Z = \bigl\{ w \in \RR^d \,: \sup_{b \in \RR} \inf_{i \in I}
( \langle w, x_i \rangle - b ) y_i \geq 1 \bigr\},
$$
which can also be written as
$$
A_Z = \bigl\{ w \in \RR^d\,:
\max_{i \in I_-} \langle w, x_i
\rangle + 2 \leq \min_{i \in I_+} \langle w, x_i \rangle \bigr\}.
$$
As it is easily seen, the optimal value of $b$ for a fixed value of $w$, in other
words the value of $b$ which maximizes $\inf_{i \in I}
(\langle w, x_i \rangle - b)y_i$, is equal to
$$
b_w = \frac{1}{2} \Bigl[ \max_{i \in I_-} \langle w, x_i \rangle +
\min_{i \in I_+} \langle w, x_i \rangle \Bigr].
$$
\begin{lemma}\mypoint
When $A_Z \neq \varnothing$, $\inf \{ \lVert w \rVert^2 \,: w
\in A_Z \}$ is reached for only one value $w_Z$ of $w$.
\end{lemma}
\begin{proof}
Let $w_0 \in A_Z$. The set $A_Z \cap \{ w \in \RR^d:
\lVert w \rVert \leq \lVert w_0 \rVert \}$ is a compact convex set and $w \mapsto \lVert w \rVert^2$ is strictly
convex and therefore has a unique minimum on this set, which
is also obviously its minimum on $A_Z$.
\end{proof}
\begin{dfn}\mypoint
When $A_Z \neq \varnothing$, the training set $Z$ is said
to be linearly separable. The hyperplane
$$
H = \{ x \in \RR^d \,: \langle w_Z, x \rangle - b_Z = 0 \},
$$
where
\begin{align*}
w_Z & = \arg\min \{ \lVert w \rVert \,: w \in A_Z \},\\
b_Z & = b_{w_Z},
\end{align*}
is called the canonical separating hyperplane of the training set $Z$.
The quantity $\lVert w_Z \rVert^{-1}$ is called the margin of the
canonical hyperplane.
\end{dfn}

As $\min_{i \in I_+} \langle w_Z, x_i \rangle -
\max_{i \in I_-} \langle w_Z, x_i \rangle = 2$, the margin is
also equal to half the distance between the projections
on the direction $w_Z$ of the positive and negative patterns.

\subsection{Computation of the canonical hyperplane}

Let us consider the convex hulls $X_+$ and $X_-$ of the positive
and negative patterns:
\begin{align*}
\C{X}_+ & = \Bigl\{ \sum_{i \in I_+} \lambda_i x_i\,:\bigl( \lambda_i
\bigr)_{i \in I_+} \in \RR_+^{I_+}, \sum_{i \in I_+} \lambda_i
= 1 \Bigr\},\\
\C{X}_- & = \Bigl\{ \sum_{i \in I_-} \lambda_i x_i\,:\bigl( \lambda_i
\bigr)_{i \in I_-} \in \RR_+^{I_-}, \sum_{i \in I_-} \lambda_i
= 1 \Bigr\}.
\end{align*}
Let us introduce the closed convex set
$$
\C{V} = \C{X}_+ - \C{X}_- = \bigl\{ x_+ - x_-\,: x_+ \in \C{X}_+, x_- \in
\C{X}_- \bigr\}.
$$
As $v \mapsto \lVert v \rVert^{2}$ is strictly convex,
with compact lower level sets, there is a unique
vector $v^*$ such that
$$
\lVert v^* \rVert^2 = \inf_{v \in \C{V}} \bigl\{ \lVert v \rVert^2\,: v \in \C{V} \bigr\}.
$$
\begin{lemma}\mypoint
The set $A_Z$ is non-empty (i.e. the training set $Z$
is linearly separable) if and only if $v^* \neq 0$. In this case
$$
w_Z = \frac{2}{\lVert v^* \rVert^{2}} v^*,
$$
and the margin of the canonical hyperplane is equal to $\frac{1}{2}
\lVert v^* \rVert$.
\end{lemma}

This lemma proves that the distance between the convex hulls
of the positive and negative patterns is equal to twice the
margin of the canonical hyperplane.

\begin{proof}
Let us assume first that $v^* = 0$, or equivalently that
$\C{X}_+ \cap \C{X}_- \neq \varnothing$. For any vector $w \in \RR^d$,
\begin{align*}
\min_{i \in I_+} \langle w, x_i \rangle & = \min_{x \in \C{X}_+}
\langle w, x \rangle,\\
\max_{i \in I_-} \langle w, x_i \rangle & = \max_{x \in \C{X}_-}
\langle w, x \rangle,
\end{align*}
so $ \min_{i \in I_+}
\langle w, x_i \rangle - \max_{i \in I_-}
\langle w, x_i \rangle \leq 0$, which shows that
$w$ cannot be in $A_Z$ and therefore that $A_Z$
is empty.

Let us assume now that $v^* \neq 0$, or equivalently that
$\C{X}_+ \cap \C{X}_- = \varnothing$. Let us put
$w^* = {2}v^*/{\lVert v^* \rVert^2} $.
Let us remark first that
\begin{align*}
\min_{i \in I_+} \langle w^*, x_i \rangle -
\max_{i \in I_-} \langle w^*, x_i \rangle & =
\inf_{x \in \C{X}_+} \langle w^*, x \rangle -
\sup_{x \in \C{X}_-} \langle w^*, x \rangle
\\ & = \inf_{x_+ \in \C{X}_+, x_- \in \C{X}_-}
\langle w^*, x_+ - x_- \rangle \\ & =
\frac{2}{\lVert v^* \rVert^2}
\inf_{v \in \C{V}} \langle v^*, v \rangle.
\end{align*}
Let us now prove that $\inf_{v \in \C{V}}
\langle v^*, v \rangle = \lVert v^* \rVert^2$.
Some arbitrary $v \in \C{V}$ being fixed,
consider the function $$\beta \mapsto \lVert
\beta v + (1 - \beta) v^* \rVert^2: [0,1]
\rightarrow \RR.$$ By definition of $v^*$,
it reaches its minimum value for $\beta = 0$,
and therefore has a non-negative derivative at
this point. Computing this derivative, we find
that $\langle v - v^*, v^* \rangle \geq 0$,
as claimed. We have proved that
$$
\min_{i \in I_+} \langle w^*, x_i \rangle
- \max_{i \in I_-} \langle w^*, x_i \rangle
= 2,
$$
and therefore that $w^* \in A_Z$. On the other hand,
any $w \in A_Z$ is such that
$$
2 \leq \min_{i \in I_+} \langle w, x_i \rangle
- \max_{i \in I_-} \langle w, x_i \rangle
= \inf_{v \in \C{V}} \langle w, v \rangle \leq \lVert w \rVert
\inf_{v \in \C{V}} \lVert v \rVert = \lVert w \rVert
\,\lVert v^* \rVert.
$$
This proves that $\lVert w^* \rVert = \inf \bigl\{ \lVert w \rVert\,:
w \in A_Z \bigr\}$, and therefore that $w^* = w_Z$ as claimed.
\end{proof}

One way to compute $w_Z$ would therefore be to compute $v^*$ by minimizing
$$
\Biggl\{ \Biggl\lVert \sum_{i \in I} \lambda_i y_i x_i \Biggr\rVert^2\,:
(\lambda_i)_{i \in I} \in \RR_+^I, \sum_{i \in I} \lambda_i = 2,
\sum_{i \in I} y_i \lambda_i = 0 \Biggr\}.
$$
Although this is a tractable quadratic programming problem, a
direct computation of $w_Z$ through the following proposition
is usually preferred.
\begin{prop}\mypoint
\label{wComp}
The canonical direction $w_Z$ can be expressed as
$$
w_Z = \sum_{i=1}^N \alpha_i^* y_i x_i,
$$
where $(\alpha_i^*)_{i=1}^N$ is obtained by minimizing
$$
\inf \bigl\{ F(\alpha)\,: \alpha \in \C{A} \bigr\}
$$
where
$$
\C{A} = \Bigl\{ (\alpha_i)_{i \in I}
\in \RR_+^{I}, \sum_{i \in I} \alpha_i y_i = 0 \Bigr\},
$$
and
$$
F(\alpha) = \Bigl\lVert \sum_{i \in I} \alpha_i y_i x_i \Bigr\rVert^2
- 2 \sum_{i \in I} \alpha_i.
$$
\end{prop}
\begin{proof}
Let $w(\alpha) = \sum_{i \in I} \alpha_i y_i x_i$ and
let $S(\alpha) = \frac{1}{2} \sum_{i\in I}\alpha_i$.
We can express the function $F(\alpha)$ as
$F(\alpha) = \lVert w(\alpha) \rVert^2 - 4 S(\alpha)$.
Moreover it is important to notice that for any $s \in \RR_+$,
$\{ w(\alpha)\,: \alpha \in \C{A}, S(\alpha) = s\} = s \C{V}$.
This shows that for any $s \in \RR_+$, $\inf \{ F(\alpha)
: \alpha \in \C{A}, S(\alpha) = s \}$ is reached and that for any
\linebreak $\alpha_s \in \{ \alpha \in \C{A}\,: S(\alpha)  = s \}$ reaching this infimum,
$w(\alpha_s) = s v^*$. As \linebreak $s \mapsto s^2 \lVert v^* \rVert^2 - 4 s:
\RR_+ \rightarrow \RR$ reaches its infimum for only one value
$s^*$ of $s$, namely at $s^* = \frac{2}{\lVert v^* \rVert^2}$,
this shows that $F(\alpha)$ reaches its infimum on $\C{A}$,
and that for any $\alpha^* \in \C{A}$ such that $F(\alpha^*) =
\inf \{ F(\alpha)\,: \alpha \in \C{A} \}$, $w(\alpha^*)
= \frac{2}{\lVert v^* \rVert^2} v^* = w_Z$.
\end{proof}

\subsection{Support vectors}
\begin{dfn}\mypoint
The set of support vectors $\C{S}$ is defined by
$$
\C{S} = \{ x_i \,: \langle w_Z , x_i \rangle - b_Z = y_i \}.
$$
\end{dfn}

\begin{prop}\mypoint
\label{chap4Prop3.1}
Any $\alpha^*$ minimizing $F(\alpha)$ on $\C{A}$
is such that
$$
\{ x_i\,: \alpha_i^* > 0 \} \subset \C{S}.
$$
This implies that the representation $w_Z = w(\alpha^*)$
involves in general only a limited number of non-zero
coefficients and that $w_Z = w_{Z'}$, where $Z' =
\{ (x_i,y_i)\,: x_i \in \C{S} \}$.
\end{prop}
\begin{proof}
Let us consider any given $i \in I_+$ and $j \in I_-$, such that
$\alpha_i^* > 0$ and $\alpha_j^* > 0$. There exists at least
one such index in each set $I_-$ and $I_+$, since the sum of the
components of $\alpha^*$ on each of these sets are equal and
since $\sum_{k \in I} \alpha^*_k > 0$.
For any $t \in \RR$, consider
$$
\alpha_k(t) = \alpha_k^* + t \B{1}(k \in \{i,j\}), \quad k \in I.
$$
The vector $\alpha(t)$ is in $\C{A}$
for any value of $t$ in some neighbourhood of $0$,
therefore $\frac{\partial}{\partial t}_{|t = 0} F\bigl[\alpha(t) \bigr] = 0$.
Computing this derivative, we find that
$$
y_i \langle w(\alpha^*), x_i \rangle +
y_j \langle w(\alpha^*) , x_j \rangle = 2.
$$
As $y_i = - y_j$, this can also be written as
$$
y_i \bigl[ \langle w(\alpha^*), x_i \rangle - b_Z \bigr] +
y_j \bigl[ \langle w(\alpha^*) , x_j \rangle -b_Z \bigr] = 2.
$$
As $w(\alpha^*)\in A_Z$,
$$
y_k \bigl[ \langle w(\alpha^*), x_k \rangle - b_Z \bigr] \geq 1,
\qquad k \in I,
$$
which implies necessarily as claimed that
$$
y_i \bigl[ \langle w(\alpha^*), x_i \rangle - b_Z \bigr]
= y_j \bigl[ \langle w(\alpha^*) , x_j \rangle -b_Z \bigr] = 1.
$$
\end{proof}
\subsection{The non-separable case}
In the case when the training set $Z = (x_i, y_i)_{i=1}^N$
is not linearly separable, we can define a noisy canonical
hyperplane as follows: we can choose $w \in \RR^d$ and
$b \in \RR$ to minimize
\begin{equation}
C(w,b) =
\sum_{i=1}^N \bigl[ 1 - \bigl( \langle w, x_i \rangle - b \bigr)
y_i \bigr]_+ + \tfrac{1}{2} \lVert w \rVert^2,
\end{equation}
where for any real number $r$, $r_+ = \max \{r, 0\}$ is
the positive part of $r$.
\newcommand{\Bw}{\overline{w}}
\begin{thm}\mypoint
Let us introduce the dual criterion
$$
F(\alpha) = \sum_{i=1}^N \alpha_i - \frac{1}{2}
\biggl\lVert \sum_{i=1}^N y_i \alpha_i x_i \biggr\rVert^2
$$
and the domain
$\ds
\C{A}' = \biggl\{ \alpha \in \RR_+^N: \alpha_i \leq 1, i = 1, \dots, N,
\sum_{i=1}^N y_i \alpha_i = 0 \biggr\}.
$
Let $\alpha^* \in \C{A}'$ be such that $ F(\alpha^*) = \sup_{\alpha \in
\C{A}'} F(\alpha)$.
Let $w^* = \sum_{i=1}^N y_i \alpha^*_i x_i$. There is
a threshold $b^*$ (whose construction will be detailed
in the proof), such that
$$
C(w^*, b^*) = \inf_{w \in \RR^d, b \in \RR}
C(w, b).
$$
\end{thm}
\begin{cor}\mypoint \!\!{\sc(scaled criterion)}
For any positive real parameter $\lambda$
let us consider the criterion
$$
C_{\lambda}(w,b) = \lambda^2
\sum_{i=1}^N \bigl[ 1 - (\langle w, x_i \rangle - b ) y_i
\bigr]_+ + \tfrac{1}{2} \lVert w \rVert^2
$$
and the domain
\[
\C{A}'_{\lambda} = \biggl\{
\alpha \in \RR_+^N: \alpha_i \leq \lambda^2, i = 1, \dots, N,
\sum_{i=1}^N y_i \alpha_i = 0 \biggr\}.
\]
For any solution $\alpha^*$ of the minimization problem
$ F(\alpha^*) = \sup_{\alpha \in \C{A}'_{\lambda}} F(\alpha)$,
the vector $w^* = \sum_{i=1}^N y_i \alpha^*_i x_i$
is such that
$$
\inf_{b \in \RR} C_{\lambda}(w^*, b)
= \inf_{w \in \RR^d, b \in \RR} C_{\lambda}(w, b).
$$
\end{cor}

In the separable case, the scaled criterion is
minimized by the canonical hyperplane for $\lambda$ large enough.
This extension of the canonical hyperplane computation
in dual space is often called \emph{the box constraint},
for obvious reasons.

\begin{proof}
The corollary is a straightforward consequence of
the scale property $C_{\lambda}(w, b, x) = \lambda^2 C(\lambda^{-1}
w, b, \lambda x)$, where we have made the dependence
of the criterion in $x \in \RR^{d N}$ explicit.
Let us come now to the proof of the theorem.

The minimization of $C(w, b)$ can be performed in dual
space extending the couple of parameters $(w, b)$
to $\Bw = (w, b, \gamma) \in \RR^d \times \RR \times \RR_+^N$
and introducing the dual multipliers $\alpha \in \RR_+^N$
and the criterion
$$
G( \alpha, \Bw ) =
\sum_{i = 1}^N \gamma_i + \sum_{i=1}^N \alpha_i
\bigl\{ \bigl[ 1 - (\langle w, x_i \rangle - b ) y_i \bigr] - \gamma_i
\bigr\} + \tfrac{1}{2} \lVert w \rVert^2.
$$
We see that
$$
C(w, b) = \inf_{\gamma \in \RR_+^N} \sup_{\alpha \in \RR_+^N}
G\bigl[ \alpha, (w, b, \gamma) \bigr],
$$
and therefore, putting $\ov{\C{W}} = \{ (w, b, \gamma):
w \in \RR^d, b \in \RR, \gamma \in \RR_+^N \bigr \}$,
we are led to solve the minimization problem
$$
G(\alpha_*, \Bw_*) = \inf_{\Bw \in \ov{\C{W}}} \sup_{\alpha \in \RR_+^N}
G(\alpha, \Bw),
$$
whose solution $\Bw_* = (w_*, b_*, \gamma_*)$ is such that
$C(\Bw_*, b_*) = \inf_{(w, b) \in \RR^{d+1}} C(w, b)$,
according to the preceding identity.
As for any value of $\alpha' \in \RR_+^N$,
$$
\inf_{\Bw \in \ov{\C{W}}} \sup_{\alpha \in \RR_+^N}
G(\alpha, \ov{w}) \geq
\inf_{\Bw \in \ov{\C{W}}} G(\alpha', \ov{w}),
$$
it is immediately seen that
$$
\inf_{\Bw \in \ov{\C{W}}} \sup_{\alpha \in \RR_+^N}
G(\alpha, \ov{w}) \geq
\sup_{\alpha \in \RR_+^N} \inf_{\Bw \in \ov{\C{W}}}
G(\alpha, \ov{w}).
$$
We are going to show that there is no duality gap,
meaning that this inequality is indeed an equality.
More importantly, we will do so by exhibiting
a saddle point, which, solving the dual minimization
problem will also solve the original one.

Let us first make explicit the solution of the
dual problem (the interest of this dual problem
precisely lies in the fact that it can more easily
be solved explicitly).
Introducing the admissible set of values
of $\alpha$,
$$
\C{A}' =  \bigl\{ \alpha \in \RR^N: 0 \leq \alpha_i \leq
1, i = 1, \dots, N, \sum_{i=1}^N y_i \alpha_i = 0 \bigr\},
$$
it is elementary to check that
$$
\inf_{\Bw \in \ov{\C{W}}} G(\alpha, \Bw) =
\begin{cases}\ds
\inf_{w \in \RR^d} G \bigl[ \alpha, (w,0,0) \bigr],
& \alpha \in \C{A}',\\
- \infty, & \text{otherwise}.
\end{cases}
$$
As
$$
G \bigl[ \alpha, (w, 0, 0) \bigr]
= \tfrac{1}{2} \lVert w \rVert^2 + \sum_{i=1}^N \alpha_i \bigl(
1 -  \langle w, x_i \rangle y_i \bigr),
$$
we see that $\inf_{w \in \RR^d} G\bigl[ \alpha, (w,0,0) \bigr]$
is reached at
$$
w_{\alpha} = \sum_{i=1}^N y_i \alpha_i x_i.
$$
This proves that
\newcommand{\BW}{\ov{\C{W}}}
$$
\inf_{\Bw \in \BW} G(\alpha, \Bw) = F(\alpha).
$$
The continuous map $\alpha \mapsto \inf_{\Bw \in \ov{\C{W}}}
G(\alpha, \Bw)$ reaches a maximum $\alpha^*$, not necessarily unique,
on the compact convex set $\C{A}'$.
We are now going to exhibit a choice of $\Bw^* \in \BW$
such that $(\alpha^*, \Bw^*)$ is a \emph{saddle point}.
This means that we are going to show that
$$
G(\alpha^*, \Bw^*) =
\inf_{\Bw \in \BW} G(\alpha^*, \Bw) =
\sup_{\alpha \in \RR_+^N} G(\alpha, \Bw^*).
$$
It will imply that
$$
\inf_{\Bw \in \BW} \sup_{\alpha \in \RR_+^d} G(\alpha, \Bw)
\leq \sup_{\alpha \in \RR_+^N} G(\alpha, \Bw^*) = G(\alpha^*, \Bw^*)
$$
on the one hand and that
$$
\inf_{\Bw \in \BW} \sup_{\alpha \in \RR_+^d} G(\alpha, \Bw)
\geq \inf_{\Bw \in \BW} G(\alpha^*, \Bw) = G(\alpha^*, \Bw^*)
$$
on the other hand, proving that
$$
G(\alpha^*, \Bw^*) = \inf_{\Bw \in \BW} \sup_{\alpha \in \RR_+^N}
G(\alpha, \Bw)
$$
as required.

\noindent{\sc Construction of $\Bw^*$.}
\begin{itemize}
\item Let us put $w^* = w_{\alpha^*}$.
\item If there is $j \in \{1, \dots, N \}$
such that $0 < \alpha^*_j < 1$,
let us put
$$
b^* = \langle x_j , w^* \rangle - y_j.
$$
Otherwise, let us put
$$
b^* = \sup \{ \langle x_i , w^* \rangle - 1: \alpha^*_i > 0 , y_i = + 1,
i = 1, \dots, N\}.
$$
\item Let us then put
$$
\gamma^*_i =
\begin{cases}
0, & \alpha^*_i < 1,\\
1 - (\langle w^*, x_i \rangle - b^*)y_i, & \alpha^*_i = 1.
\end{cases}
$$
\end{itemize}
If we can prove that
\begin{equation}
\label{eq3.2}
1 - (\langle w^*, x_i \rangle - b^*)y_i
\begin{cases}
\leq 0, & \alpha^*_i = 0,\\
= 0, & 0 < \alpha^*_i < 1,\\
\geq 0, & \alpha^*_i = 1,
\end{cases}
\end{equation}
it will show that $\gamma^* \in \RR_+^N$
and therefore that $\Bw^* = (w^*, b^*, \gamma^*) \in \BW$.
It will also show that
$$
G(\alpha, \Bw^*) = \sum_{i=1}^N \gamma^*_i
+ \sum_{i, \alpha^*_i = 0}  \alpha_i \bigl[ 1 -
(\langle \Bw^*, x_i \rangle - b^*) y_i \bigr]
+ \tfrac{1}{2} \lVert \Bw^* \rVert^2,
$$
proving that
$G(\alpha^*, \Bw^*) = \sup_{\alpha \in \RR_+^N} G(\alpha,
\Bw^*)$. As obviously $G (\alpha^*, \Bw^*) = G \bigl[ \alpha^*,
(w^*,\break 0 , 0) \bigr]$, we already know that
$G(\alpha^*, \Bw^*) = \inf_{\Bw \in \BW} G(\alpha^*, \Bw)$.
This will show that $(\alpha^*, \Bw^*)$ is the saddle
point we were looking for, thus ending the proof of the
theorem.
\end{proof}\smallskip

\noindent{\sc Proof of equation \eqref{eq3.2}.}
Let us deal first with the case when there is $j \in \{1, \dots, N\}$
such that $0 < \alpha_j^* < 1$.

For any $i \in \{1, \dots, N\}$
such that $0< \alpha^*_i < 1$, there is $\epsilon > 0$ such
that for any $t \in (-\epsilon, \epsilon)$, $\alpha^* + t y_i e_i - t y_j e_j
\in \C{A}'$, where $(e_k)_{k=1}^N$ is the canonical base of $\RR^N$.
Thus $\frac{\partial}{\partial t}_{|t=0} F(\alpha^* + t y_i e_i -
t y_j e_j ) = 0$. Computing this derivative,
we obtain
\begin{align*}
\frac{\partial}{\partial t}_{|t=0}
F(\alpha^* + t y_i e_i - t y_j e_j)
& = y_i  - \langle w^*, x_i \rangle + \langle w^*, x_j \rangle - y_j \\
& = y_i \bigl[ 1 - \bigl(\langle w, x_i \rangle - b^* \bigr) y_i \bigr].
\end{align*}
Thus $1 - \bigl(\langle w, x_i \rangle - b^* \bigr) y_i = 0$,
as required. This shows also that the definition of $b^*$ does not
depend on the choice of $j$ such that $0 < \alpha^*_j < 1$.

For any $i \in \{1, \dots, N\}$ such that $\alpha^*_i = 0$,
there is $\epsilon > 0$ such that for any $t \in (0, \epsilon)$,
$\alpha^* + t e_i - t y_i y_j e_j \in \C{A}'$.
Thus $\frac{\partial}{\partial t}_{|t=0} F(\alpha^* + t e_i
- t y_i y_j e_j) \leq 0$, showing that
$1 - \bigl( \langle w^*, x_i \rangle - b^* \bigr) y_i \leq 0$ as
required.

For any $i \in \{1, \dots, N\}$ such that $\alpha^*_i
= 1$, there is $\epsilon > 0$ such that $
\alpha^* - t e_i + t y_i y_j e_j \in \C{A}'$.
Thus $\frac{\partial}{\partial t}_{| t = 0} F(
\alpha^* - t e_i + t y_i y_j e_j) \leq 0$, showing
that  $1 - \bigl( \langle w^*, x_i \rangle - b^* \bigr) y_i \geq 0$
as required. This shows that $(\alpha^*, \Bw^*)$
is a saddle point in this case.

Let us deal now with the case where $\alpha^* \in \{0, 1\}^N$.
If we are not in the trivial case where the vector $(y_i)_{i=1}^N$
is constant, the case $\alpha^* = 0$ is ruled out. Indeed,
in this case, considering $\alpha^* + t e_i + t e_j$, where
$y_i y_j = -1$, we would get the contradiction
$2 = \frac{\partial}{\partial t}_{|t=0} F(\alpha^*+te_i+te_j)
\leq 0$.

Thus there are values of $j$ such that $\alpha^*_j = 1$,
and since $\sum_{i=1}^N \alpha_i y_i = 0$, both classes are
present in the set $\{ j: \alpha^*_j = 1 \}$.

Now for any $i, j \in \{1, \dots, N\}$ such that
$\alpha^*_i = \alpha^*_j = 1$ and such that $y_i = +1$ and $y_j = -1$,
$ \frac{\partial}{\partial t}_{|t=0} F( \alpha^* - t e_i
- t e_j) = - 2 + \langle w^* , x_i \rangle - \langle
w^*, x_j \rangle \leq 0$.
Thus
$$
\sup \{ \langle w^*, x_i \rangle - 1: \alpha^*_i = 1, y_i = +1 \}
\leq \inf \{ \langle w^*, x_j \rangle + 1: \alpha^*_j = 1, y_j = -1 \},
$$
showing that
$$
1 - \bigl( \langle w^*, x_k \rangle - b^* \bigr) y_k \geq 0, \alpha^*_k = 1.
$$
Finally, for any $i$ such that $\alpha^*_i = 0$,
for any $j$ such that $\alpha^*_j = 1$ and
$y_j = y_i$, we have
$$
\frac{\partial}{\partial t}_{|t=0}F(\alpha^*
+ t e_i - t e_j) =  y_i \langle w^*, x_i - x_j \rangle  \leq 0,
$$
showing that $1 - \bigl( \langle w^*, x_i \rangle - b^* \bigr) y_i
\leq 0$. This shows that $(\alpha^*, \Bw^*)$ is always a  saddle
point.

\subsection{Support Vector Machines}
\begin{dfn}\mypoint
The symmetric measurable kernel $K: \C{X} \times \C{X}
\rightarrow \RR$ is said to
be positive (or more precisely positive semi-definite) if
for any $n \in \NN$, any $(x_i)_{i=1}^n \in \C{X}^n$,
$$
\inf_{\alpha \in \RR^n} \sum_{i=1}^n \sum_{j=1}^n \alpha_i K(x_i, x_j)
\alpha_j \geq 0.
$$
\end{dfn}
Let $Z = (x_i,y_i)_{i=1}^N$ be some training set. Let us consider
as previously
$$
\C{A} = \Biggl\{ \alpha \in \RR_+^N \,: \sum_{i=1}^N \alpha_i y_i = 0 \Biggr\}.
$$
Let
$$
F(\alpha) = \sum_{i=1}^N \sum_{j=1}^N \alpha_i y_iK(x_i,x_j)y_j \alpha_j
- 2 \sum_{i=1}^N \alpha_i.
$$
\begin{dfn}\mypoint
Let $K$ be a positive symmetric kernel.
The training set $Z$ is said to be $K$-separable
if
$$
\inf \bigl\{ F(\alpha)\,: \alpha \in \C{A} \bigr\} > - \infty.
$$
\end{dfn}
\begin{lemma}\mypoint
When $Z$ is $K$-separable, $\inf\{ F(\alpha)\,: \alpha \in \C{A} \}$ is
reached.
\end{lemma}
\begin{proof}
Consider the training set $Z' = (x_i',y_i)_{i=1}^N$, where
$$
x_i' = \biggl\{ \biggl[ \Bigl\{ K(x_k,x_{\ell})\Bigr\}_{k=1, \ell=1}^{N
\quad N} \biggr]^{1/2}(i,j) \biggr\}_{j=1}^N \in \RR^N.
$$
We see that $F(\alpha) = \lVert \sum_{i=1}^N \alpha_i y_i x_i'
\rVert^2 - 2 \sum_{i=1}^N \alpha_i$.
We  proved in the previous section that $Z'$ is linearly separable
if and only if $\inf \{ F(\alpha)\,: \alpha \in \C{A} \} > - \infty$,
and that the infimum is reached in this case.
\end{proof}

\begin{proposition}\mypoint
\label{chap4Prop4.1} Let $K$ be a symmetric positive kernel and let
$Z = (x_i, y_i)_{i=1}^N$ be some $K$-separable training set. Let
$\alpha^* \in \C{A}$ be such that $F(\alpha^*)
= \inf \{ F(\alpha) \,: \alpha \in \C{A} \}$.
Let
\begin{align*}
I_-^* & = \{ i \in \NN\,:1 \leq i \leq N, y_i = -1, \alpha_i^* > 0 \}\\
I_+^* & = \{ i \in \NN\,:1 \leq i \leq N, y_i = +1, \alpha_i^* > 0 \}\\
b^* & = \frac{1}{2} \Bigl\{
\sum_{j=1}^N \alpha_j^* y_j K(x_j,x_{i_-})
+ \sum_{j=1}^N \alpha_j^* y_j K(x_j,x_{i_+}) \Bigr\}, \qquad i_- \in
I_-^*, i_+ \in I_+^*,
\end{align*}
where the value of $b^*$ does not depend on the choice of $i_-$ and
$i_+$.
The classification rule $f: \C{X} \rightarrow \C{Y}$
defined by the formula
$$
f(x) = \sign \left( \sum_{i=1}^N \alpha_i^* y_i K(x_i,x) -
b^* \right)
$$
is independent of the choice of $\alpha^*$ and is called
the support vector machine defined by $K$ and $Z$.
The set
$\C{S} = \{ x_j\,: \sum_{i=1}^N \alpha_i^* y_i K(x_i,x_j) - b^* = y_j \}$
is called the set of support vectors. For any choice of $\alpha^*$,
$\{ x_i\,: \alpha_i^* > 0 \} \subset \C{S}$.
\end{proposition}

An important consequence of this proposition is that the support
vector machine defined by $K$ and $Z$ is also the support vector
machine defined by $K$ and $Z' = \{ (x_i, y_i): \alpha^*_i > 0,
1 \leq i \leq N \}$, since this restriction of the index set
contains the value $\alpha^*$ where the minimum of $F$ is reached.

\begin{proof}
The independence of the choice of $\alpha^*$, which is not
necessarily unique, is seen as follows.
Let $(x_i)_{i=1}^N$ and $x \in \C{X}$ be fixed.
Let us put for ease of notation $x_{N+1} = x$.
Let $M$ be the $(N+1) \times (N+1)$ symmetric
semi-definite matrix defined by $M(i,j) = K(x_i,x_j)$,
$i=1,\dots, N+1$, $j=1, \dots, N+1$.
Let us consider the mapping
$\Psi: \{ x_i\,:i=1, \dots, N+1 \} \rightarrow \RR^{N+1}$
defined by
\begin{equation}
\label{PsiDef}
\Psi(x_i) = \bigl[M^{1/2}(i,j)\bigr]_{j=1}^{N+1} \in \RR^{N+1}.
\end{equation}
Let us consider the training set $Z' = \bigl[ \Psi(x_i),y_i \bigr]_{i=1}^N$.
Then $Z'$ is linearly separable,
$$F(\alpha) =
\Bigl\lVert \sum_{i=1}^N \alpha_i y_i \Psi(x_i) \Bigr\rVert^2
- 2 \sum_{i=1}^N \alpha_i,$$
and we have proved that
for any choice of $\alpha^* \in \C{A}$ minimizing $F(\alpha)$,
\linebreak $w_{Z'} = \sum_{i=1}^N \alpha_i^* y_i \Psi(x_i)$.
Thus the support vector machine defined by $K$ and $Z$ can also be expressed by the formula
$$
f(x) = \sign \Bigl[ \langle w_{Z'}, \Psi(x) \rangle - b_{Z'} \bigr]
$$
which does not depend on $\alpha^*$. The definition of $\C{S}$
is such that $\Psi(\C{S})$ is the set of support vectors
defined in the linear case, where its stated property has already been
proved.
\end{proof}

We can in the same way use the box constraint and show
that any solution $\alpha^* \in \arg \min
\{ F(\alpha): \alpha \in \C{A}, \alpha_i \leq \lambda^2,
i = 1, \dots, N \}$ minimizes
\begin{multline}
\label{eq3.4}
\inf_{b \in \RR} \lambda^2 \sum_{i=1}^N \biggl[ 1 -
\biggl( \sum_{j=1}^N y_j \alpha_j K(x_j, x_i) - b
\biggr) y_i \biggr]_+ \\ + \frac{1}{2}
\sum_{i=1}^N \sum_{j=1}^N \alpha_i \alpha_j y_i y_j K(x_i, x_j).
\end{multline}

\subsection{Building kernels}

Except the last, the results of this section are drawn from
\cite{Cristianini}. We have no reference for the last
proposition of this section, although we believe it is well known.
We include them for the convenience of the reader.

\begin{prop}\mypoint
Let $K_1$ and $K_2$ be positive symmetric kernels on $\C{X}$.
Then for any $a \in \RR_+$
\begin{align*}
(a K_1 + K_2)(x,x') & \overset{\text{\rm def}}{=} a K_1(x,x')
+ K_2(x,x')\\
\text{ and }(K_1 \cdot K_2)(x,x') &\overset{\text{\rm def}}{=}
K_1(x,x') K_2(x,x')
\end{align*}
are also positive symmetric kernels.
Moreover, for any measurable function \linebreak $g: \C{X} \rightarrow \RR$,
$K_g(x,x') \overset{\text{\rm def}}{=} g(x)g(x')$ is also a positive symmetric kernel.
\end{prop}
\begin{proof}
It is enough to prove the proposition in the case when $\C{X}$ is
finite and kernels are just ordinary symmetric matrices.
Thus we can assume without loss of generality that
$\C{X} = \{ 1, \dots, n\}$. Then for any $\alpha \in \RR^N$,
using usual matrix notation,
\begin{align*}
\langle \alpha , (a K_1 + K_2) \alpha \rangle & =
a \langle \alpha, K_1 \alpha \rangle + \langle \alpha , K_2 \alpha \rangle
\geq 0,\\
\langle \alpha, (K_1 \cdot K_2) \alpha \rangle & =
\sum_{i,j} \alpha_i K_1(i,j) K_2(i,j) \alpha_j\\
& = \sum_{i,j,k} \alpha_i K_1^{1/2}(i,k) K_1^{1/2}(k,j)K_2(i,j) \alpha_j
\\ & = \sum_{k} \underbrace{\sum_{i,j} \bigl[K_1^{1/2}(k,i) \alpha_i \bigr] K_2(i,j)
\bigl[K_1^{1/2}(k,j) \alpha_j \bigr]}_{
\geq 0} \geq 0,\\
\langle \alpha, K_g \alpha \rangle & = \sum_{i,j} \alpha_i g(i) g(j) \alpha_j
= \left( \sum_i \alpha_i g(i) \right)^2 \geq 0.
\end{align*}
\end{proof}

\begin{prop}\mypoint
Let $K$ be some positive symmetric kernel on $\C{X}$. Let $p: \RR \rightarrow
\RR$ be a polynomial with positive coefficients.
Let $g: \C{X} \rightarrow \RR^d$ be a measurable function.
Then
\begin{align*}
p(K)(x,x') & \overset{\text{def}}{=}
p\bigl[ K(x,x')\bigr], \\
\exp(K)(x,x') & \overset{\text{def}}{=}
\exp \bigl[ K(x,x') \bigr]\\
\text{ and } G_{g}(x,x') & \overset{\text{def}}{=}
\exp \bigl( - \lVert g(x) - g(x') \rVert^2 \bigr)
\end{align*}are all
positive symmetric kernels.
\end{prop}
\begin{proof}
The first assertion is a direct consequence of the previous proposition.
The second comes from the fact that the exponential function is
the pointwise limit of a sequence of polynomial functions
with positive coefficients.
The third is seen from the second and the decomposition
$$
G_{g}(x,x') = \Bigl[ \exp\bigl( - \lVert g(x) \rVert^2 \bigr)
\exp \bigl( - \lVert g(x') \rVert^2 \bigr) \Bigr]
\exp \bigl[ 2 \langle g(x), g(x') \rangle \bigr]
$$
\end{proof}
\begin{prop}\mypoint
With the notation of the previous proposition,
\emph{any} training set $Z = (x_i,y_i)_{i=1}^N \in \bigl( \C{X}\times \{-1,+1\}
\bigr)^N$ is $G_g$-separable as soon as $g(x_i)$, $i = 1, \dots, N$ are
distinct points of $\RR^d$.
\end{prop}
\begin{proof}
It is clearly enough to prove the case when $\C{X} = \RR^d$ and
$g$ is the identity.
Let us consider some other generic point $x_{N+1} \in \RR^d$
and define $\Psi$ as in \eqref{PsiDef}.
It is enough to prove that
$\Psi(x_1), \dots, \Psi(x_N)$ are affine independent, since the
simplex, and therefore any affine independent set of points, can
be split in any arbitrary way by affine half-spaces.
 Let us assume that
$(x_1, \dots, x_N)$ are affine dependent; then
for some $(\lambda_1, \dots, \lambda_N) \neq 0$ such that
$\sum_{i=1}^N \lambda_i = 0$,
$$
\sum_{i=1}^N \sum_{j=1}^N \lambda_i G(x_i, x_j) \lambda_j = 0.
$$
Thus, $(\lambda_i)_{i=1}^{N+1}$, where we have put $\lambda_{N+1} = 0$
is in the kernel of the symmetric positive semi-definite matrix
$G(x_i,x_j)_{i,j \in \{1, \dots, N+1\}}$. Therefore
$$
\sum_{i=1}^N \lambda_i G(x_i, x_{N+1}) = 0,
$$
for any $x_{N+1} \in \RR^d$. This would mean that
the functions $x \mapsto \exp (- \lVert x - x_i \rVert^2)$ are
linearly dependent, which can be easily proved to be false.
Indeed, let $n \in \RR^d$ be such that $\lVert n \rVert = 1$
and $\langle n, x_i \rangle$, $i = 1, \dots, N$ are distinct
(such a vector exists, because it has to be outside the
union of a finite number of hyperplanes, which is of zero
Lebesgue measure on the sphere). Let us assume for
a while that for some $(\lambda_i)_{i=1}^N \in \RR^N$,
for any $x \in \RR^d$,
$$
\sum_{i=1}^N \lambda_i \exp( - \lVert x - x_i \rVert^2) = 0.
$$
Considering $x = t n$, for $t \in \RR$, we would get
$$
\sum_{i=1}^N \lambda_i \exp( 2 t \langle n, x_i \rangle
- \lVert x_i \rVert^2 ) = 0, \qquad t \in \RR.
$$
Letting $t$ go to infinity, we see that this is only
possible if $\lambda_i = 0$ for all values of $i$.
\end{proof}

\section{Bounds for Support Vector Machines}

\subsection{Compression scheme bounds}

We can use Support Vector Machines in the framework of compression
schemes and apply Theorem \thmref{thm2.3.3}.
More precisely, given some positive symmetric kernel $K$ on $\C{X}$,
we may consider for any training set $Z' = (x_i',y_i')_{i=1}^h$
the classifier $\Hat{f}_{Z'}: \C{X} \rightarrow \C{Y}$ which is
equal to the Support Vector Machine defined by $K$ and $Z'$
whenever $Z'$ is $K$-separable, and which is equal to some
constant classification rule otherwise; we take this convention
to stick to the framework described on page \pageref{compression}, we
will only use $\Hat{f}_{Z'}$ in the $K$-separable case,
so this extension of the definition is just a matter of
presentation. In the application of Theorem \ref{thm2.3.3}
in the case when the observed sample $(X_i,Y_i)_{i=1}^N$ is $K$-separable,
a natural if perhaps sub-optimal choice of $Z'$ is to choose for
$(x_i')$ the set of support vectors defined by $Z = (X_i,Y_i)_{i=1}^N$
and to choose for $(y_i')$ the corresponding values of $Y$.
This is justified by the fact that $\Hat{f}_{Z}=\Hat{f}_{Z'}$,
as shown in Proposition \ref{chap4Prop4.1} (page \pageref{chap4Prop4.1}).
If
$Z$ is not $K$-separable,
we can train a Support Vector Machine with the box constraint,
then remove all the errors to obtain a $K$-separable sub-sample
$Z' = \{ (X_i, Y_i): \alpha^*_i < \lambda^2, 1 \leq i \leq N \}$,
 using the same notation as in equation \eqref{eq3.4}
on page \pageref{eq3.4},
and then
consider its support vectors as the compression set.
Still using the notation of page \pageref{eq3.4},
this means we have to compute successively
$\alpha^* \in \arg\min \{ F(\alpha): \alpha \in \C{A},
\alpha_i \leq \lambda^2 \}$, and $\alpha^{**}
\in  \arg \min \{ F(\alpha): \alpha \in \C{A},
\alpha_i = 0 \text{ when } \alpha^*_i = \lambda^2 \}$,
to keep the compression set indexed by
$J = \{ i: 1 \leq i \leq N, \alpha^{**}_i > 0 \}$,
and the corresponding Support Vector Machine $\w{f}_{J}$.
Different values of $\lambda$ can be used at this
stage, producing different candidate compression
sets: when $\lambda$ increases, the number of
errors should decrease, on the other hand when
$\lambda$ decreases, the margin $\lVert w \rVert^{-1}$
of the separable subset $Z'$
increases, supporting the hope for a smaller set of
support vectors, thus we can use $\lambda$
to monitor the number of errors on the training set
we accept from the compression scheme.
As we can use whatever heuristic we want while
selecting the compression set, we can also try
to threshold in the previous construction $\alpha_i^{**}$
at different levels $\eta \geq 0$, to produce candidate
compression sets
$J_{\eta} = \{ i: 1 \leq i \leq N, \alpha^{**}_i > \eta \}$
of various sizes.

As the size $\lvert J \rvert$ of the compression
set is random in this construction, we must
use a version of Theorem \ref{thm2.3.3} (page
\pageref{thm2.3.3}) which handles compression
sets of arbitrary sizes. This is done by choosing
for each $k$ a $k$-partially exchangeable posterior distribution
$\pi_k$ which weights the compression sets of all dimensions.
We immediately see that we can choose $\pi_k$ such that
$- \log \bigl[ \pi_k (\Delta_k(J)) \bigr]
\leq \log \bigl[ \lvert J \rvert (\lvert J \rvert + 1)
\bigr] + \lvert J \rvert  \log \Bigl[
\tfrac{(k+1)eN}{\lvert J \rvert} \Bigr]$.

If we observe the shadow sample patterns, and if computer
resources permit, we can of
course use more elaborate bounds than Theorem \ref{thm2.3.3},
such as the transductive equivalent for Theorem \ref{thm1.24}
(page \pageref{thm1.24}) (where we may consider the submodels
made of all the compression sets of the same size). Theorems
based on relative bounds, such as Theorem \thmref{thm1.58} or Theorem \thmref{thm1.80} can also be used. Gibbs distributions
can be approximated by Monte Carlo techniques, where
a Markov chain with the proper invariant measure
consists in appropriate local perturbations of the
compression set.

Let us mention also that the use of compression schemes based
on Support Vector Machines
can be tailored to perform some kind of \emph{feature aggregation}.
Imagine that the kernel $K$ is defined as the scalar
product in $L_2(\pi)$, where $\pi \in \C{M}_+^1(\Theta)$.
More precisely let us consider for some set of
soft classification rules $\bigl\{ f_{\theta}: \C{X} \rightarrow
\RR\,; \theta \in \Theta \bigr\}$ the kernel
$$
K(x,x') = \int_{\theta \in \Theta} f_{\theta}(x) f_{\theta}(x')
\pi(d \theta).
$$
In this setting, the Support Vector Machine
applied to the training set $Z = (x_i,\break y_i)_{i=1}^N$
has the form
$$
f_{Z}(x) = \sign \left( \int_{\theta \in \Theta} f_{\theta}(x)
\sum_{i=1}^N y_i \alpha_i
f_{\theta}(x_i) \pi(d \theta) - b \right)
$$
and, if this is too burdensome to compute,
we can replace it with some finite approximation
$$
\widetilde{f}_{Z}(x) = \sign \left(
\frac{1}{m} \sum_{k=1}^m f_{\theta_k}(x) w_k - b \right),
$$
where the set $\{\theta_k,\, k=1, \dots, m\}$ and the
weights $\{ w_k,\,k=1, \dots, m\}$ are computed
in some suitable way from the set  $Z' = (x_i, y_i)_{i , \alpha_i > 0}$
of support vectors
of $f_Z$. For instance,
we can draw $\{ \theta_k,\,k=1, \dots, m\}$ at random according to
the probability distribution proportional to
$$
\left\lvert \sum_{i=1}^N y_i \alpha_i f_{\theta}(x_i) \right\rvert
\pi(d \theta),
$$
define the weights $w_k$ by
$$
w_k =
\sign \left( \sum_{i=1}^N y_i \alpha_i f_{\theta_k}(x_i)
\right) \int_{\theta \in \Theta} \left\lvert
\sum_{i = 1}^N y_i \alpha_i f_{\theta}(x_i) \right\rvert \pi(d\theta),
$$
and choose the smallest value of $m$ for which this approximation
still classifies $Z'$ without errors.
Let us remark that we have built
$\widetilde{f}_Z$ in such a way that
$$
\lim_{m \rightarrow + \infty}
\widetilde{f}_Z(x_i) = f_Z(x_i) = y_i, \quad \text{a.s.}
$$ for any support index
$i$ such that $\alpha_i > 0$.

Alternatively, given $Z'$, we can select a finite set of features
$\Theta' \subset \Theta$ such that $Z'$ is $K_{\Theta'}$ separable,
where
$K_{\Theta'}(x,x') = \sum_{\theta \in \Theta'}
f_{\theta}(x) f_{\theta}(x')$
and consider the Support Vector Machines $f_{Z'}$ built with the
kernel $K_{\Theta'}$. As soon as $\Theta'$ is chosen as a function
of $Z'$ only, Theorem \ref{thm2.3.3} (page \pageref{thm2.3.3}) applies
and provides
some level of confidence for the risk of $f_{Z'}$.

\subsection{The Vapnik--Cervonenkis dimension
of a family of subsets}

Let us consider some set $X$ and some set
$S \subset \{0,1\}^X$ of subsets of $X$.
Let $h(S)$ be the Vapnik--Cervonenkis dimension of $S$, defined as
$$
h(S) = \max \Bigl\{ \lvert A \rvert: A \subset X,
\lvert A \rvert < \infty \text{ and }
A \cap S = \{0,1\}^{A} \Bigr\},
$$
where by definition $A \cap S = \{ A \cap B: B \in S \}$
and $\lvert A \rvert$ is the number of points in $A$.
Let us notice that this definition does not depend on
the choice of the reference set $X$. Indeed $X$ can
be chosen to be $\bigcup S$, the union of all the sets in $S$
or any bigger set. Let us notice also that for any set $B$,
$h(B \cap S) \leq h(S)$, the reason being that
$A \cap (B \cap S) = B \cap (A \cap S)$.

This notion of Vapnik--Cervonenkis dimension is useful because, as we will see for Support Vector
Machines, it can be computed in some important special cases.
Let us prove here as an illustration that
$h(S) = d+1$ when $X = \RR^d$
and $S$ is made of all the half spaces:
$$
S = \{ A_{w,b}\,: w \in \RR^d, b \in \RR \},
\text{ where } A_{w,b} = \{ x \in X \,:
\langle w, x \rangle \geq b \}.
$$
\begin{prop}\mypoint
With the previous notation, $h(S) = d+1$.
\end{prop}
\begin{proof}
Let $(e_i)_{i=1}^{d+1}$ be the canonical base of $\RR^{d+1}$,
and let $X$ be the affine subspace it generates, which
can be identified with $\RR^d$. For any $(\epsilon_i)_{i=1}^{d+1}
\in \{-1,+1\}^{d+1}$, let $w = \sum_{i=1}^{d+1} \epsilon_i e_i$
and $b = 0$. The half space $A_{w,b} \cap X$ is such that
$\{e_i\,; i=1, \dots, d+1 \} \cap (A_{w,b} \cap X) = \{ e_i \,;
\epsilon_i = +1 \}$. This proves that $h(S) \geq d + 1$.

To prove that $h(S) \leq d + 1$, we have to show that
for any set $A \subset \RR^d$
of size $|A| = d+2$, there is $B \subset A$ such
that $B \not\in (A \cap S)$. Obviously this will
be the case if the convex hulls of $B$ and $A \setminus
B$ have a non-empty intersection: indeed if a hyperplane
separates two sets of points, it also separates
their convex hulls. As $\lvert A \rvert
> d+1$, $A$ is affine dependent: there is
$(\lambda_x)_{x \in A} \in \RR^{d+2} \setminus
\{0\}$ such that
$\sum_{x \in A} \lambda_x x = 0$ and $\sum_{x \in A}
\lambda_x = 0$. The set
$B = \{ x \in A\,: \lambda_x > 0\}$ and its complement $A \setminus B$ are non-empty,
because $\sum_{x \in A} \lambda_x = 0$ and $\lambda \neq
0$. Moreover $\sum_{x \in B} \lambda_x =
\sum_{x \in A \setminus B} - \lambda_x > 0$.
The relation
$$
\frac{1}{\sum_{x \in B} \lambda_x} \sum_{x \in B}
\lambda_x x = \frac{1}{\sum_{x \in B} \lambda_x}
\sum_{x \in A \setminus B} - \lambda_x x
$$
shows that the convex hulls of $B$ and $A \setminus B$
have a non-void intersection.
\end{proof}

Let us introduce the function of two integers
$$
\Phi_n^h = \sum_{k=0}^h \binom{n}{k},
$$
which can alternatively be defined
by the relations
$$
\Phi_n^h =
\begin{cases}
2^n & \text{ when } n \leq h,\\
\Phi_{n-1}^{h-1} + \Phi_{n-1}^h & \text{ when } n > h.
\end{cases}
$$
\begin{thm}\mypoint
\label{th1}
Whenever $\bigcup S$ is finite,
$$
\lvert S \rvert \leq \Phi\left( \left\lvert \bigcup S \right\rvert, h(S)
\right).
$$
\end{thm}
\begin{thm}\mypoint
\label{th2}
For any $h \leq n$,
$$
\Phi_n^h \leq \exp \bigl[ n H \bigl(\tfrac{h}{n}\bigr) \bigr]
\leq \exp \bigl[ h \bigl( \log ( \tfrac{n}{h} ) + 1 \bigr) \bigr],
$$
where $H(p) = - p \log(p) - (1-p)\log(1-p)$ is the Shannon
entropy of the Bernoulli distribution with parameter $p$.
\end{thm}
{\sc Proof of theorem \ref{th1}.}
Let us prove this theorem by induction on $\left\lvert \bigcup
S \right\rvert$. It is easy to check that it holds
true when $\left\lvert \bigcup
S \right\rvert = 1$.
Let $X = \bigcup S$, let
$x \in X$ and $X' = X \setminus \{x\}$. Define ($\bigtriangleup$
denoting the symmetric difference of two sets)
\begin{align*}
S' & = \{ A \in S: A \bigtriangleup \{x\} \in S \},\\
S'' & = \{ A \in S: A \bigtriangleup \{x\} \not\in S \}.
\end{align*}
Clearly, $\sqcup$ denoting the disjoint union,
$S = S' \sqcup S''$ and $S \cap X' = (S' \cap X')
\sqcup (S'' \cap X')$. Moreover $\lvert S' \rvert =
2 \lvert S' \cap X' \rvert$ and $\lvert S'' \rvert = \lvert
S'' \cap X' \rvert$. Thus $$\lvert S \rvert =
\lvert S' \rvert + \lvert S'' \rvert = 2 \lvert S' \cap X' \rvert
+ \lvert S'' \rvert = \lvert S \cap X' \rvert + \lvert S' \cap
X' \rvert.$$ Obviously $h(S \cap X') \leq h(S)$. Moreover
$h(S' \cap X') = h(S') - 1$, because if $A \subset X'$
is shattered by $S'$ (or equivalently by $S' \cap X'$),
then $A \cup \{x\}$ is shattered by $S'$ (we say that $A$
is shattered by $S$ when $A \cap S = \{0,1\}^A$).
Using the induction hypothesis, we then see that
$\lvert S \cap X' \rvert \leq \Phi_{\lvert X' \rvert}^{h(S)}
+ \Phi_{\lvert X' \rvert}^{h(S)-1}$. But as $\lvert X' \rvert =
\lvert X \rvert - 1$, the right-hand side of this inequality
is equal to $\Phi_{\lvert X \rvert}^{h(S)}$, according to
the recurrence equation satisfied by $\Phi$.

{\sc Proof of theorem \ref{th2}:}
This is the well-known Chernoff bound for the deviation of sums
of Bernoulli random variables: let $(\sigma_1, \dots, \sigma_n)$ be i.i.d.
Bernoulli random variables with parameter $1/2$. Let us notice that
$$
\Phi_n^h = 2^n \PP \left( \sum_{i=1}^n \sigma_i \leq h \right).
$$
For any positive real number $\lambda$ ,
\begin{align*}
\PP \biggl( \sum_{i=1}^n \sigma_i \leq h \biggr)
& \leq \exp (\lambda h) \EE \biggl[
\exp \biggl( - \lambda \sum_{i=1}^n \sigma_i \biggr) \biggr] \\ & =
\exp \Bigl\{ \lambda h + n \log \bigl\{
\EE \bigl[ \exp \bigl( - \lambda \sigma_1 \bigr)
\bigr] \bigr\} \Bigr\}.
\end{align*}
Differentiating the right-hand side in $\lambda$ shows that its
minimal value is \linebreak
$\exp \bigl[ - n \C{K}(\tfrac{h}{n},\tfrac{1}{2}) \bigr]$,
where $\C{K}(p,q) = p \log(\tfrac{p}{q}) + (1-p) \log(\tfrac{1-p}{1-q})$
is the Kullback divergence function between two Bernoulli distributions
$B_p$ and $B_q$
of parameters $p$ and $q$. Indeed the optimal value $\lambda^*$ of $\lambda$
is such that
$$
h = n \frac{\EE \bigl[\sigma_1 \exp ( - \lambda^* \sigma_1)
\bigr]}{\EE \bigl[ \exp ( - \lambda^* \sigma_1) \bigr]}
= n B_{h/n}(\sigma_1).
$$
Therefore, using the fact that two Bernoulli
distributions with the same expectations are equal,
$$
\log \bigl\{ \EE \bigl[ \exp ( - \lambda^* \sigma_1)\bigr] \bigr\}
= - \lambda^* B_{h/n}(\sigma_1) - \C{K}(B_{h/n},B_{1/2}) =
- \lambda^* \tfrac{h}{n} - \C{K}(\tfrac{h}{n},\tfrac{1}{2}).
$$
The announced result then follows from
the identity
\begin{multline*}
H(p) = \log(2) - \C{K}(p,\tfrac{1}{2}) \\= p \log(p^{-1})
+ (1- p) \log(1 + \frac{p}{1-p}) \leq p \bigl[ \log(p^{-1})+1\bigr].
\end{multline*}
\eject

\subsection{Vapnik--Cervonenkis dimension of linear rules with margin}
The proof of the following theorem was suggested to us
by a similar proof presented in \cite{Cristianini}.
\begin{thm}\mypoint
\label{chap5Th1.1}
Consider a family of points $(x_1, \dots, x_n)$ in some Euclidean
vector space $E$ and a family of affine functions
$$
\C{H} = \bigl\{ g_{w,b}: E \rightarrow \RR\,; w \in E, \lVert w \rVert = 1,
b \in \RR \bigr\},
$$
where
$$
g_{w,b}(x) = \langle w, x \rangle - b, \qquad x \in E.
$$

Assume that there is a set of thresholds $(b_i)_{i=1}^n
\in \RR^n$ such that for any \linebreak $(y_i)_{i=1}^n \in \{-1,+1\}^n$,
there is $g_{w,b} \in \C{H}$ such that
$$
\inf_{i=1}^n  \bigl( g_{w,b}(x_i) - b_i \bigr) y_i \geq
\gamma.
$$
Let us also introduce the empirical variance of $(x_i)_{i=1}^n$,
$$
\Var(x_1, \dots, x_n) = \frac{1}{n} \sum_{i=1}^n
\biggl\lVert x_i - \frac{1}{n} \sum_{j=1}^n x_j \biggr\rVert^2.
$$
In this case and with this notation,
\begin{equation}
\label{firstPart}
\frac{\Var(x_1, \dots, x_n)}{\gamma^2} \geq
\begin{cases}
n-1 & \text{ when } n \text{ is even,}\\
(n-1) \frac{n^2 - 1}{n^2} & \text{ when } n \text{ is odd.}
\end{cases}
\end{equation}
Moreover, equality is reached when $\gamma$ is optimal,
$b_i = 0$, $i = 1, \dots, n$
and $(x_1, \dots,\break x_n)$
is a regular simplex
(i.e. when $2 \gamma$ is the minimum distance
between the convex hulls of any two subsets of $\{x_1, \dots, x_n\}$
and $\lVert x_i - x_j \rVert$ does not depend on $i \neq j$).
\end{thm}
\begin{proof}
Let $(s_i)_{i=1}^n \in \RR^n$ be such that $\sum_{i=1}^n s_i = 0$.
Let $\sigma$ be a uniformly distributed random variable with values
in $\mathfrak{S}_{n}$, the set of permutations of the first $n$
integers $\{1, \dots, n \}$. By assumption, for any value of $\sigma$,
there is an affine function $g_{w,b} \in \C{H}$ such that
$$
\min_{i=1, \dots, n} \bigl[ g_{w,b}(x_i) - b_i \bigr] \bigl[
2 \B{1}(s_{\sigma(i)} > 0) - 1 \bigr] \geq \gamma.
$$
As a consequence
\begin{align*}
\left\langle \sum_{i=1}^n s_{\sigma(i)} x_i, w \right\rangle
& =
\sum_{i=1}^n s_{\sigma(i)} \bigl( \langle x_i, w \rangle - b - b_i\bigr)
+ \sum_{i=1}^n s_{\sigma(i)} b_i\\
& \geq \sum_{i=1}^n
\gamma \lvert s_{\sigma(i)} \rvert + s_{\sigma(i)} b_i.
\end{align*}
Therefore, using the fact that the map $x \mapsto
\Bigl(\max \bigl\{0,x\bigr\}\Bigr)^2$ is convex,
\begin{multline*}
\EE \left(
\biggl\lVert \sum_{i=1}^n s_{\sigma(i)} x_i \biggr\rVert^2 \right)
\geq
\EE \left[ \left( \max \left\{ 0,
\sum_{i=1}^n \gamma \lvert s_{\sigma(i)} \rvert + s_{\sigma(i)} b_i
\right\} \right)^2 \right] \\ \geq
\left(\max \left\{ 0, \sum_{i=1}^n \gamma \EE \bigl(
\lvert s_{\sigma(i)} \rvert \bigr) + \EE \bigl( s_{\sigma(i)} \bigr)
b_i \right\} \right)^2
= \gamma^2 \left( \sum_{i=1}^n \lvert s_i \rvert \right)^2,
\end{multline*}
where $\EE$ is the expectation with respect to the random permutation
$\sigma$.
On the other hand
$$
\EE \left( \biggl\lVert \sum_{i=1}^n s_{\sigma(i)} x_i \biggr\rVert^2 \right)
= \sum_{i=1}^n \EE(s_{\sigma(i)}^2) \lVert x_i \rVert^2 +
\sum_{i\neq j} \EE(s_{\sigma(i)} s_{\sigma(j)}) \langle x_i, x_j \rangle.
$$
Moreover
$$
\EE ( s_{\sigma(i)}^2 ) = \frac{1}{n} \EE \left(
\sum_{i=1}^n s_{\sigma(i)}^2 \right) = \frac{1}{n} \sum_{i=1}^n
s_i^2.
$$
In the same way, for any $i \neq j$,
\begin{align*}
\EE \left( s_{\sigma(i)} s_{\sigma(j)} \right) & =
\frac{1}{n(n-1)} \EE \left( \sum_{i \neq j} s_{\sigma(i)} s_{\sigma(j)}
\right) \\ & = \frac{1}{n(n-1)} \sum_{i\neq j} s_i s_j\\
& = \frac{1}{n(n-1)} \Biggl[
\Biggl( \underbrace{\sum_{i=1}^n s_i}_{=0} \Biggr)^2 - \sum_{i=1}^n s_i^2
\Biggr] \\ & = - \frac{1}{n(n-1)} \sum_{i=1}^n s_i^2.
\end{align*}
Thus
\begin{align*}
\EE \left( \biggl\lVert \sum_{i=1}^n s_{\sigma(i)} x_i \biggr\rVert^2 \right)
& = \left( \sum_{i=1}^n s_i^2 \right) \left[ \frac{1}{n} \sum_{i=1}^n \lVert
x_i \rVert^2 -
\frac{1}{n(n-1)} \sum_{i\neq j} \langle x_i, x_j \rangle \right] \\ & =
\left( \sum_{i=1}^n s_i^2 \right) \Biggl[
\left( \frac{1}{n} + \frac{1}{n(n-1)} \right) \sum_{i=1}^n \lVert x_i \rVert^2
\\ & \qquad - \frac{1}{n(n-1)} \biggl\lVert \sum_{i=1}^n x_i
\biggr\rVert^2 \Biggr] \\ & =
\frac{n}{n-1} \left( \sum_{i=1}^n s_i^2 \right) \Var(x_1, \dots, x_n).
\end{align*}
We have proved that
$$
\frac{\Var(x_1, \dots, x_n)}{\gamma^2} \geq \frac{\ds (n-1) \biggl(
\sum_{i=1}^n \lvert s_i \rvert \biggr)^2}{\ds n \sum_{i=1}^n s_i^2}.
$$
This can be used with $s_i = \B{1}( i \leq \frac{n}{2}) - \B{1}(
i > \frac{n}{2})$ in the case when $n$ is even and
$s_i = \frac{2}{(n-1)} \B{1}( i \leq \frac{n-1}{2} ) -
\frac{2}{n+1} \B{1}(i > \frac{n-1}{2} )$ in the case when
$n$ is odd, to establish the first inequality \eqref{firstPart} of the theorem.

Checking that equality is reached for the simplex is an easy computation
when the simplex $(x_i)_{i=1}^n \in (\RR^n)^n$ is parametrized in such a
way that
$$
x_i(j) = \begin{cases}
1 & \text{ if } i = j,\\
0 & \text{ otherwise.}
\end{cases}
$$
Indeed the distance between the convex hulls of any two subsets of
the simplex is the distance between their mean values (i.e. centers of mass).
\end{proof}

\subsection{Application to Support Vector Machines}

We are going to apply Theorem \ref{chap5Th1.1} (page
\pageref{chap5Th1.1}) to Support Vector
Machines in the transductive case. Let
$(X_i, Y_i)_{i=1}^{(k+1)N}$ be distributed according to some partially exchangeable
distribution $\PP$ and assume that $(X_i)_{i=1}^{(k+1)N}$ and
$(Y_i)_{i=1}^N$ are observed. Let us consider some positive
kernel $K$ on $\C{X}$. For any $K$-separable training set of
the form $Z' = (X_i,y_i')_{i=1}^{(k+1)N}$, where $(y_i')_{i=1}^{(k+1)N}
\in \C{Y}^{(k+1)N}$, let $\Hat{f}_{Z'}$ be the Support Vector Machine
defined by $K$ and $Z'$ and let $\gamma(Z')$ be its margin.
Let
\begin{multline*}
R^2 = \max_{i=1, \dots, (k+1)N} K(X_i,X_i) + \frac{1}{(k+1)^2 N^2}
\sum_{j=1}^{(k+1)N} \sum_{k=1}^{(k+1)N} K(X_j,X_k) \\
- \frac{2}{(k+1)N}
\sum_{j=1}^{(k+1)N} K(X_i,X_j).
\end{multline*}
This is an easily computable upper-bound for the radius
of some ball containing the image of $(X_1, \dots, X_{(k+1)N})$
in feature space.

Let us define for any integer $h$ the margins
\begin{equation}
\label{margin}
\gamma_{2h} = (2h - 1)^{-1/2}\quad
\text{and}\quad \gamma_{2h+1} = \left[ 2h\left(
1 - \frac{1}{(2h+1)^2}\right) \right]^{-1/2}.
\end{equation}
Let us consider for any $h =1, \dots, N$ the exchangeable model
$$
\C{R}_h = \bigl\{ \Hat{f}_{Z'}\,:Z' = (X_i, y_i')_{i=1}^{(k+1)N}
\text{ is $K$-separable and } \gamma(Z') \geq R \gamma_h \bigr\}.
$$
The family of models $\C{R}_h$, $h=1, \dots, N$ is nested,
and we know from Theorem \ref{chap5Th1.1} (page \pageref{chap5Th1.1}) and
Theorems \ref{th1} (page \pageref{th1}) and
\ref{th2} (page \pageref{th2}) that
$$
\log \bigl( \lvert \C{R}_h \rvert \bigr) \leq h \log
\bigl( \tfrac{(k+1)e N}{h} \bigr).
$$
We can then consider on the large model $\C{R} = \bigsqcup_{h=1}^N
\C{R}_h$ (the disjoint union of the sub-models)
an exchangeable prior $\pi$ which is uniform on each $\C{R}_h$
and is such that $\pi(\C{R}_h) \geq \frac{1}{h(h+1)}$.
Applying Theorem \ref{thm2.1.5}
(page \pageref{thm2.1.5})
we get
\begin{proposition}\mypoint
With $\PP$ probability at least $1 - \epsilon$, for any
$h = 1, \dots, N$, any Support Vector Machine $f \in \C{R}_h$,
\begin{multline*}
r_2(f)  \leq \\*
\frac{k+1}{k} \inf_{\lambda \in \RR_+}
\frac{1 - \exp \Bigl[ - \frac{\lambda}{N} r_1(f) - \frac{h}{N} \log
\Bigl( \frac{e(k+1)N}{h} \Bigr) - \frac{\log[h(h+1)] -
\log(\epsilon)}{N}
\Bigr]}{
1 - \exp( - \frac{\lambda}{N})} \\* - \frac{r_1(f)}{k}.
\end{multline*}
\end{proposition}

Searching the whole model $\C{R}_h$ to optimize the bound may require more computer
resources than are available, but any heuristic can be applied to choose $f$,
since the bound is uniform. For instance,
a Support Vector Machine $f'$ using a box constraint
can be trained from
the training set $(X_i, Y_i)_{i=1}^N$ and then $(y'_i)_{i=1}^{
(k+1)N}$ can be set to $y'_i = \sign(f'(X_i))$, $i = 1,
\dots, (k+1)N$.

\subsection[Inductive margin bounds]{Inductive margin bounds for Support
Vector Machines}

In order to establish inductive margin bounds, we will
need a different combinatorial lemma. It is due to \cite{Alon}.
We will reproduce their proof with some tiny improvements on
the values of constants.

Let us consider the finite case when $\C{X} = \{1, \dots, n\}$,
$\C{Y} = \{1, \dots, b\}$ and \linebreak $b \geq 3$.  The question
we will study would be meaningless when $b \leq 2$. Assume as usual that we are
dealing with a prescribed set of classification rules
$\C{R} = \bigl\{ f: \C{X} \rightarrow \C{Y} \bigr\}$.
Let us say that a pair $(A,s)$, where $A \subset \C{X}$
is a non-empty set of shapes
and $s: A \rightarrow \{2, \dots, b-1\}$ a threshold function,
is \emph{shattered}
by the set of functions $F \subset \C{R}$
if for any $(\sigma_x)_{x \in A} \in \{-1,+1\}^{A}$,
there exists some $f \in F$ such that $\min_{x \in A}
\sigma_x \bigl[ f(x) - s(x) \bigr] \geq 1$.

\begin{dfn}\mypoint
\label{fatDef}
Let the \emph{fat shattering
dimension} of $(\C{X},\C{R})$ be the maximal size $\lvert A \rvert$
of the first component of the pairs which are shattered by $\C{R}$.
\end{dfn}

Let us say that a subset of classification rules $F \subset
\C{Y}^{\C{X}}$ is \emph{separated} whenever for any pair
$(f,g) \in F^2$ such that $f\neq g$, $\lVert f - g \rVert_{\infty}
= \max_{x \in \C{X}} \lvert f(x) - g(x) \rvert \geq 2$.
Let $\mathfrak{M}(\C{R})$ be the maximum size $\lvert F \rvert$
of separated subsets $F$ of $\C{R}$. Note that if $F$ is a
separated subset of $\C{R}$ such that $\lvert F \rvert =
\mathfrak{M}(\C{R})$, then it is a $1$-net for the $\C{L}_{\infty}$
distance: for any function $f \in \C{R}$ there exists $g \in F$
such that $\lVert f - g \rVert_{\infty} \leq 1$ (otherwise $f$ could be
added to $F$ to create a larger separated set).

\begin{lemma}\mypoint
\label{lemma3.1}
With the above notation,
whenever the fat shattering dimension of
$(\C{X}, \C{R})$ is not greater than $h$,
\begin{multline*}
\log \bigl[ \mathfrak{M}(\C{R}) \bigr] < \log \bigl[ (b-1)(b-2) n \bigr]
\Biggl\{\frac{\log \bigl[ \sum_{i=1}^h \binom{n}{i} (b-2)^i \bigr]}{
\log(2)}+1 \Biggr\} + \log(2)
\\ \leq \log \bigl[ (b-1)(b-2) n \bigr]
\Biggl\{ \biggl[ \log \Bigl[ \tfrac{(b-2) n}{h}
\Bigr] + 1 \biggr] \frac{h}{\log(2)} + 1\Biggr\} + \log(2).
\end{multline*}
\end{lemma}
\begin{proof}
For any set of functions $F \subset \C{Y}^{\C{X}}$,
let $t(F)$ be the number of pairs $(A, s)$ shattered by $F$.
Let $t(m,n)$ be the minimum of $t(F)$ over
all \emph{separated} sets of functions $F \subset \C{Y}^{\C{X}}$ of size $\lvert
F \rvert = m$ ($n$ is here to recall that the shape space $\C{X}$
is made of $n$ shapes). For any $m$ such that $t(m,n) > \sum_{i=1}^h
\binom{n}{i} (b-2)^i$, it is clear that any separated set of functions
of size $\lvert F \rvert \geq m$ shatters at least one pair
$(A,s)$ such that $\lvert A \rvert > h$. Indeed, from its definition $t(m,n)$ is
clearly  a non-decreasing function of $m$,
so that $t(\lvert F \rvert, n) > \sum_{i=1}^h \binom{n}{i}
(b-2)^i$.
Moreover there are only $\sum_{i=1}^h \binom{n}{i}(b-2)^i$
pairs $(A,s)$ such that $\lvert A \rvert \leq h$.
As a consequence, whenever the fat shattering dimension
of $(\C{X}, \C{R})$ is not greater than $h$ we have $\mathfrak{M}(\C{R})
< m$.

It is clear that for any $n \geq 1$, $t(2,n) = 1$.
\begin{lemma}\mypoint
For any $m \geq 1$,
$t\bigl[mn(b-1)(b-2), n \bigr] \geq 2 t\bigl[ m, n-1 \bigr]$,
and therefore $t\bigl[ 2 n(n-1) \cdots (n-r+1) (b-1)^r(b-2)^r, n \bigr]
\geq 2^r$.
\end{lemma}
\begin{proof}
Let $F = \{f_1, \dots, f_{mn(b-1)(b-2)}\}$
be some separated set of functions of size
$mn(b-1)(b-2)$. For any pair $(f_{2i-1},f_{2i})$,
$i=1,\dots, mn(b-1)(b-2)/2$, there is $x_i \in \C{X}$
such that $\lvert f_{2i-1}(x_i) - f_{2i}(x_i) \rvert
\geq 2$. Since $\lvert \C{X} \rvert = n$, there is
$x \in \C{X}$ such that $\sum_{i=1}^{mn(b-1)(b-2)/2}
\B{1}(x_i = x) \geq m(b-1)(b-2)/2$. Let $I = \{ i \,:
x_i = x\}$.
Since there are
$(b-1)(b-2)/2$ pairs $(y_1,y_2) \in \C{Y}^2$
such that $1\leq y_1 < y_2 - 1 \leq b -1$, there is some pair
$(y_1,y_2)$, such that $1 \leq y_1 < y_2 \leq b$
and such that $\sum_{i\in I} \B{1}(\{y_1,y_2\} = \{f_{2i-1}(x),
f_{2i}(x)\}) \geq m$.
Let $J = \bigl\{i \in I\,: \{f_{2i-1}(x),f_{2i}(x)\} = \{y_1,y_2\}
\bigr\}$. Let
\begin{align*}
F_1 & =
\{ f_{2i-1} \,:i \in J, f_{2i-1}(x) = y_1\}
\cup
\{ f_{2i} \,:i \in J, f_{2i}(x) = y_1\},\\
F_2 & =
\{ f_{2i-1} \,:i \in J, f_{2i-1}(x) = y_2\}
\cup
\{ f_{2i} \,:i \in J, f_{2i}(x) = y_2\}.
\end{align*}
Obviously $\lvert F_1 \rvert = \lvert F_2 \rvert =
\lvert J \rvert = m$.  Moreover the restrictions
of the functions of $F_1$ to $\C{X} \setminus \{x\}$
are separated, and it is the same with $F_2$. Thus
$F_1$ strongly shatters at least $t(m,n-1)$
pairs $(A,s)$ such that $A \subset \C{X} \setminus \{x\}$
and it is the same with $F_2$. Finally,
if the pair $(A,s)$ where $A \subset \C{X} \setminus \{x\}$
is both shattered by $F_1$ and $F_2$, then
$F_1 \cup F_2$ shatters also $(A \cup \{x\}, s')$
where $s'(x') = s(x')$ for any $x' \in A$ and $s'(x) =
\lfloor \frac{y_1+y_2}{2} \rfloor$. Thus $F_1 \cup F_2$,
and therefore $F$, shatters at least $2t(m,n-1)$
pairs $(A,s)$.
\end{proof}

Resuming the proof of lemma \ref{lemma3.1}, let us choose
for $r$ the smallest integer such that
$2^r > \sum_{i=1}^h \binom{n}{i} (b-2)^i$, which is no greater than
\\ \mbox{} \hfill $\left\{ \frac{\log \bigl[ \sum_{i=1}^h \binom{n}{i} (b-2)^i \bigr]}{
\log(2)} + 1 \right\}$.
\hfill \mbox{}\\
In the case when $1 \leq n \leq r$,
$$
\log( \mathfrak{M}(\C{R}) ) < {\lvert \C{X} \rvert} \log(\lvert \C{Y} \rvert)
 = n \log(b) \leq r \log( b) \leq r \log \bigl[ (b-1)(b-2)n \bigr] + \log(2),
 $$
 which proves the lemma. In the remaining case $n > r$,
\begin{multline*}
 t \bigl[ 2 n^r (b-1)^r (b-2)^r, n \bigr]
\\ \geq t \bigl[ 2n(n-1) \dots (n-r+1)(b-1)^r(b-2)^r, n\bigr]
 \\ > \sum_{i=1}^h \binom{n}{i} (b-2)^i.
\end{multline*}
 Thus $\lvert \mathfrak{M}(\C{R}) \rvert < 2 \Bigl[(b-2)(b-1)n\Bigr]^r$ as
claimed.
\end{proof}

In order to apply this combinatorial lemma to Support Vector
Machines, let us consider now the case of separating
hyperplanes in $\RR^d$ (the generalization to Support Vector Machines
being straightforward).
Assume that $\C{X} = \RR^d$ and
$\C{Y}= \{-1,+1\}$.
For any sample $(X)_{i=1}^{(k+1)N}$, let
$$
R(X_1^{(k+1)N}) = \max \{ \lVert X_i \rVert \,: 1 \leq i \leq (k+1)N \}.
$$
Let us consider the set of parameters
$$
\Theta = \bigl\{ (w,b) \in \RR^d \times \RR\,: \lVert w \rVert = 1 \bigr\}.
$$
For any $(w,b) \in \Theta$, let
$g_{w,b}(x) = \langle w, x \rangle - b$.
Let $h$ be some fixed integer and let $\gamma = R(X_1^{(k+1)N})\gamma_h$,
where $\gamma_h$ is defined by equation \myeq{margin}.

Let us define $\zeta: \RR \rightarrow \ZZ$ by
$$
\zeta (r) =
\left\{
\begin{aligned}
-5  & & \text{ when }&&   & r \leq -4\gamma,\\
-3  & & \text{ when }&&   -4 \gamma < & r \leq -2 \gamma,\\
-1  & & \text{ when }&&  -2 \gamma < & r \leq 0,\\
+1  & & \text{ when }&&   0 < & r \leq 2 \gamma,\\
+3  & & \text{ when }&&   2 \gamma < & r \leq 4 \gamma,\\
+5  & & \text{ when }&&   4 \gamma < & r.
\end{aligned}\right.
$$
Let $G_{w,b}(x) = \zeta \bigl[ g_{w,b}(x) \bigr]$.
The fat shattering dimension (as defined in \ref{fatDef})
of
$$
\Bigl( X_1^{(k+1)N}, \bigl\{ (G_{w,b}+7)/2:
(w,b) \in \Theta \bigr\} \Bigr)
$$
is not greater than $h$ (according to Theorem \ref{chap5Th1.1}, page
\pageref{chap5Th1.1}),
therefore there is some set $\C{F}$
of functions from $X_1^{(k+1)N}$ to $\{-5,-3,-1,+1,+3,+5\}$
such that
$$
\log \bigl(\lvert \C{F} \rvert \bigr) \leq
\log\bigl[ 20(k+1) N \bigr] \Biggl\{ \frac{h}{\log(2)}
\biggl[ \log \left( \frac{4(k+1)N}{h} \right) + 1 \biggr]
+ 1 \Biggr\} + \log(2).
$$
and
for any $(w,b) \in \Theta$, there is
$f_{w,b} \in \C{F}$ such that $\sup \bigl\{ \lvert f_{w,b}
(X_i) - G_{w,b}(X_i) \rvert\,: i=1, \dots, (k+1)N \bigr\} \leq 2.$
Moreover, the choice of $f_{w,b}$ may be required to depend
on $(X_i)_{i=1}^{(k+1)N}$ in an exchangeable way.
Similarly to Theorem \ref{thm2.1.5} (page \pageref{thm2.1.5}),
it can be proved that for any partially exchangeable probability
distribution $\PP \in \C{M}_+^1 (\Omega)$,
with $\PP$ probability at least $1 - \epsilon$,
for any $f_{w,b} \in \C{F}$,

\begin{multline*}
\frac{1}{kN} \sum_{i=N+1}^{(k+1)N}
\B{1}\bigl[f_{w,b}(X_i) Y_i \leq 1 \bigr] \\
\begin{aligned} \leq \frac{k+1}{k} & \inf_{\lambda \in \RR_+}
\bigl[ 1 - \exp( - \tfrac{\lambda}{N} ) \bigr]^{-1}
\biggl\{ 1 - \\
& \exp \biggl[ - \frac{\lambda}{N^2}
\sum_{i=1}^N \B{1} \bigl[ f_{w,b}(X_i) Y_i \leq 1 \bigr]
- \frac{\log \bigl( \lvert \C{F} \rvert \bigr) - \log(\epsilon)}{N}
\biggr] \biggr\} \end{aligned}\\- \frac{1}{k N} \sum_{i=1}^{N} \B{1} \bigl[
f_{w,b}(X_i) Y_i \leq 1 \bigr].
\end{multline*}

Let us remark that
$$
\B{1} \Bigl\{
2 \B{1} \bigl[g_{w,b}(X_i) \geq 0 \bigr] - 1 \neq Y_i \Bigr\}
= \B{1}\bigl[ G_{w,b}(X_i) Y_i < 0 \bigr] \leq
\B{1} \bigl[ f_{w,b}(X_i) Y_i \leq 1 \bigr]
$$
and
$$
\B{1}\bigl[ f_{w,b}(X_i) Y_i \leq 1 \bigr]
\leq \B{1}\bigl[ G_{w,b}(X_i) Y_i \leq 3 \bigr]
\leq \B{1} \bigl[ g_{w,b}(X_i) Y_i \leq 4 \gamma \bigr].
$$
This proves the following theorem.
\begin{thm}\mypoint
\label{thm3.19}
Let us consider the sequence $(\gamma_h)_{h \in \NN^*}$ defined
by equation \myeq{margin}. With $\PP$ probability at least
$1 - \epsilon$, for any $(w,b) \in \Theta$,
\begin{multline*}
\frac{1}{kN} \sum_{i=N+1}^{(k+1)N}
\B{1} \Bigl\{ 2 \B{1} \bigl[ g_{w,b}(X_i) \geq 0 \bigr] - 1 \neq Y_i \Bigr\}\\
\begin{aligned} \leq \frac{k+1}{k} & \inf_{\lambda \in \RR_+, h \in \NN^*}
\bigl[ 1 - \exp( - \tfrac{\lambda}{N} ) \bigr]^{-1}
\Biggl\{ 1 - \\
\exp \Biggl[ - & \frac{\lambda}{N^2}
 \sum_{i=1}^N \B{1} \bigl[ g_{w,b}(X_i)Y_i \leq 4 R \gamma_h \bigr]
\\ - & \frac{\log
\bigl[ 20 (k+1)N \bigr] \Bigl\{
\tfrac{h}{\log(2)} \log \Bigl( \tfrac{4e (k+1)N}{h} \Bigr)
+ 1 \Bigr\} + \log\Bigl[ \tfrac{2h(h+1)}{\epsilon} \Bigr] }{N}
\Biggr] \Biggr\} \end{aligned}\\- \frac{1}{k N} \sum_{i=1}^{N} \B{1}
\bigl[ g_{w,b}(X_i)Y_i \leq 4 R \gamma_h \bigr].
\end{multline*}
\end{thm}
Properly speaking this theorem is not a margin
bound, but more precisely a \emph{margin quantile} bound, since it
covers the case where some fraction of the training sample falls
within the region defined by the margin parameter $\gamma_h$ which
optimizes the bound.

As a consequence though, we get a true (weaker) margin bound:
with $\PP$ probability at least $1 - \epsilon$,
for any $(w,b) \in \Theta$ such that
$$
\gamma = \min_{i=1, \dots, N}  g_{w,b}(X_i)Y_i > 0,
$$
\begin{multline*}
\frac{1}{kN} \sum_{i=N+1}^{(k+1)N}
\B{1} \bigl[ g_{w,b}(X_i) Y_i < 0 \bigr]
\\ \leq \tfrac{k+1}{k} \biggl\{
1 - \exp \biggl[ - \tfrac{\log\bigl[ 20(k+1)N \bigr] }{N}
\Bigl\{ \tfrac{16 R^2 + 2 \gamma^2}{\log(2) \gamma^2}
\log \Bigl( \tfrac{e (k+1)N \gamma^2}{4R^2} \Bigr)  + 1 \Bigr\}
\\ +  \frac{1}{N} \log ( \tfrac{\epsilon}{2} ) \biggr] \biggr\}.
\end{multline*}
This inequality compares favourably with similar inequalities
in \cite{Cristianini}, which moreover do not extend to the margin
quantile case as this one.

Let us also mention that it is easy to circumvent the fact that
$R$ is not observed when the test set
$X_{N+1}^{(k+1)N}$ is not observed.

Indeed, we can consider the sample obtained by projecting $X_1^{(k+1)N}$
on some ball of fixed radius $R_{\max}$, putting
$$
t_{R_{\max}}(X_i) = \min \left\{ 1, \frac{R_{\max}}{\lVert X_i \rVert} \right\} X_i.
$$
We can further consider an atomic prior distribution $\nu \in \C{M}_+^1(\RR_+)$
bearing on $R_{\max}$, to obtain a uniform result through a union bound.
As a consequence of the previous theorem, we have
\begin{cor}\mypoint
For any atomic prior $\nu \in \C{M}_+^1(\RR_+)$,
for any partially exchangeable probability measure $\PP \in \C{M}_+^1(\Omega)$,
with $\PP$ probability at least
$1 - \epsilon$, for any $(w,b) \in \Theta$, any $R_{\max} \in \RR_+$,
\begin{multline*}
\frac{1}{kN} \sum_{i=N+1}^{(k+1)N}
\B{1} \Bigl\{ 2 \B{1} \bigl[ g_{w,b} \circ t_{R_{\max}}(X_i)
\geq 0 \bigr] - 1 \neq Y_i \Bigr\}\\*
\begin{aligned} \leq \frac{k+1}{k} & \inf_{\lambda \in \RR_+, h \in \NN^*}
\bigl[ 1 - \exp( - \tfrac{\lambda}{N} ) \bigr]^{-1}
\Biggl\{ 1 - \\
\exp \Biggl[ - & \frac{\lambda}{N^2}
 \sum_{i=1}^N \B{1} \bigl[ g_{w,b} \circ t_{R_{\max}}(X_i)Y_i \leq 4 R_{\max}
 \gamma_h \bigr]
\\ - & \frac{\log
\bigl[ 20 (k+1)N \bigr] \Bigl\{
\tfrac{h}{\log(2)} \log \Bigl( \tfrac{4e (k+1)N}{h} \Bigr)
+ 1 \Bigr\} + \log\Bigl[ \tfrac{2h(h+1)}{\epsilon \nu(R_{\max})} \Bigr] }{N}
\Biggr] \Biggr\} \end{aligned}\\- \frac{1}{k N} \sum_{i=1}^{N} \B{1}
\bigl[ g_{w,b}\circ t_{R_{\max}} (X_i)Y_i \leq 4 R_{\max} \gamma_h \bigr].
\end{multline*}
\end{cor}

Let us remark that $t_{R_{\max}}(X_i) = X_i$, $i = N+1, \dots, (k+1)N$,
as soon as we consider only the values of $R_{\max}$ not smaller
than $\max_{i=N+1, \dots, (k+1)N} \lVert X_i \rVert$ in this corollary.
Thus we obtain a bound on the transductive generalization error
of the unthresholded classification rule $2 \B{1} \bigl[
g_{w,b} (X_i) \geq 0 \bigr] - 1$, as well as some incitation to
replace it with a thresholded rule when the value of $R_{\max}$
minimizing the bound falls below $\max_{i=N+1, \dots, (k+1)N}
\lVert X_i \rVert$.
\chapter*{Appendix: Classification by thresholding}
\addcontentsline{toc}{chapter}{Appendix: Classification by thresholding}
\markboth{Appendix}{}
\setcounter{chapter}{5}
\setcounter{section}{0}

In this appendix, we show how the bounds given in the first section
of this monograph can be computed in practice on a simple example:
the case when the classification is performed by comparing
a series of measurements to threshold values.
Let us mention that our description covers the case when
the same measurement is compared to several thresholds,
since it is enough to repeat a measurement in the list
of measurements describing a pattern to cover this case.

\section{Description of the model} Let us assume that the patterns we want to classify
are described through $h$ real valued measurements normalized
in the range $(0,1)$. In this setting the pattern space can thus be
defined as $\C{X} = (0,1)^h$.

Consider the threshold set $\C{T} = (0,1)^h$ and the response set
$\C{R} = \C{Y}^{\{0,1\}^h}$. For any $t \in (0,1)^h$ and
any $a: \{0,1\}^h \rightarrow \C{Y}$, let
$$
f_{(t,a)}(x) = a \Bigl\{ \bigl[ \B{1} ( x^j \geq t_j ) \bigr]_{j=1}^h \Bigr\},
\quad x \in \C{X},
$$
where $x^j$ is the $j$th coordinate of $x \in \C{X}$.
Thus our parameter set here is $\Theta = \C{T} \times \C{R}$.
Let us consider the Lebesgue measure $L$ on $\C{T}$  and the uniform probability distribution $U$
 on $\C{R}$. Let our prior distribution
be $\pi = L \otimes U$. Let us define for any threshold sequence $t \in \C{T}$
$$
\Delta_t = \Bigl\{ t' \in \C{T}: \ov{(t'_j, t_j)} \cap \{ X_i^j ;
i = 1, \dots, N \} = \varnothing, j = 1, \dots, h \Bigr\},
$$
where $X_i^j$ is the $j$th coordinate of the sample pattern $X_i$, and
where the interval $\ov{(t'_j, t_j)}$ of the real line is defined as the
convex hull of the two point set $\{t'_j, t_j\}$, whether $t'_j \leq t_j$
or not. We see that $\Delta_t$ is the set of
thresholds giving the same response as $t$ on the
training patterns. Let us consider for any $t \in \C{T}$ the middle
$$
m(\Delta_t) = \frac{ \int_{\Delta_t} t' L(d t')}{L(\Delta_t)}
$$
of $\Delta_t$. The set $\Delta_t$ being a product of intervals,
its middle is the point whose coordinates are the middle of these
intervals. Let us introduce the finite set $T$ composed of the middles
of the cells $\Delta_t$, which can be defined as
$$
T = \{ t \in \C{T}: t = m(\Delta_t) \}.
$$
It is easy to see that $\lvert T \rvert \leq (N+1)^h$ and that
$\lvert \C{R} \rvert = \lvert \C{Y} \rvert^{2^h}$.
\eject

\section{Computation of inductive bounds}

For any parameter $(t,a) \in \C{T} \times \C{R} = \Theta$, let
us consider the posterior distribution defined by its density
$$
\frac{d \rho_{(t,a)}}{d \pi} (t',a') =
\frac{\B{1}\bigl(t' \in \Delta_t\bigr) \B{1}\bigl(a' = a\bigr)}{ \pi \bigl(
\Delta_t \times \{ a \}
\bigr)}.
$$
In fact we are considering a finite number
of posterior distributions, since
$\rho_{(t,a)} = \rho_{(m(\Delta_t), a)}$, where $m(\Delta_t) \in T$.
Moreover, for any exchangeable sample distribution $\PP
\in \C{M}_+^1\bigl[ (\C{X} \times \C{Y})^{N+1} \bigr]$ and
any thresholds $t \in \C{T}$,
$$
\PP \Bigl[\, \ov{(X_{N+1}^j, t_j)} \cap \{X_i^j, i = 1, \dots, N\}
= \varnothing \Bigr] \leq \frac{2}{N+1}.
$$
Thus, for any $(t,a) \in \Theta$,
$$
\PP \Bigl\{ \rho_{(t,a)} \bigl[ f_{.}(X_{N+1}) \bigr]
\neq f_{(t,a)}(X_{N+1}) \Bigr\}
\leq \frac{2h}{N+1},
$$
showing that the classification produced by $\rho_{(t,a)}$ on new examples
is typically non-random; this result is only indicative, since it is
concerned with a non-random choice of $(t,a)$.

Let us compute the various quantities needed to apply the results of
the first section, focussing our attention on Theorem \ref{thm1.1.43}
(page \pageref{thm1.1.43}).

First note that $\rho_{(t,a)}(r) = r[ (t,a) ]$.
The entropy term is such that
$$
\C{K}(\rho_{t,a}, \pi ) = - \log \bigl[ \pi \bigl(
\Delta_t \times \{ r \} \bigr) \bigr] =
- \log \bigl[ L(\Delta_t) \bigr] + 2^h \log \bigl(\lvert \C{Y} \rvert\bigr).
$$
Let us notice accordingly that
$$
\min_{(t,a) \in \Theta} \C{K}(\rho_{(t,a)}, \pi ) \leq
h \log (N+1) + 2^h \log\bigl(\lvert \C{Y} \rvert\bigr).
$$
Let us introduce the counters
\begin{align*}
b_{y}^t(c) = \frac{1}{N} \sum_{i=1}^N \B{1} \Bigl\{ Y_i = y \text{ and }
\bigl[ \B{1} ( X_i^j \geq t_j ) \bigr]_{j=1}^h  = c & \Bigr\}, \\ & t \in T,
c \in \{0,1\}^h, y \in \C{Y},\\
b^t(c) = \sum_{y \in \C{Y}} b_y^t(c) =
\frac{1}{N} \sum_{i=1}^N \B{1} \Bigl\{ \bigl[ \B{1}(X_i^j \geq t_j) \bigr]_{j=1}^h
= c & \Bigr\}, \qquad t \in T, c \in \{0,1\}^h.
\end{align*}
Since
$$
r [ (t,a)] =  \sum_{c \in \{0,1\}^h} \bigl[
b^t(c) - b^t_{a(c)}(c) \bigr],
$$
the partition function of the Gibbs estimator can be computed as
\begin{align*}
\pi \bigl[ \exp (- \lambda r ) \bigr] & =
\sum_{t \in T} L(\Delta_t)  \sum_{a \in \C{R}} \frac{1}{\lvert \C{Y} \rvert^{2^h}}
\exp \biggl[ - \lambda \sum_{i=1}^N \B{1}\bigl[ Y_i \neq f_{(t,a)}(X_i)
\bigr] \biggr] \\
& = \sum_{t \in T} L(\Delta_t)  \sum_{a \in \C{R}} \frac{1}{\lvert \C{Y} \rvert^{2^h}}
\exp \biggl[ - \lambda \sum_{c \in \{0,1\}^h}
\bigl[ b^t(c) - b^t_{a(c)}(c)\bigr]
\biggr]\\
& = \sum_{t \in T} L(\Delta_t) \prod_{c \in \{0,1\}^h} \Biggl[
\frac{1}{\lvert \C{Y} \rvert} \sum_{y \in \C{Y}} \exp
\biggl( - \lambda \bigl[ b^t(c) - b^t_y(c) \bigr] \biggr) \Biggr].
\end{align*}
We see that the number of operations needed to compute
$\pi \bigl[ \exp( - \lambda r) \bigr]$ is proportional to
$\lvert T \rvert \times 2^h \times \lvert \C{Y} \rvert
\leq (N+1)^h2^h\lvert \C{Y} \rvert$.
An exact computation will therefore be feasible only for
small values of $N$ and $h$. For higher values, a Monte Carlo
approximation of this sum will have to be performed instead.

If we want to compute the bound provided by Theorem
\ref{thm1.1.43} (page \pageref{thm1.1.43}) or
by Theorem \thmref{thm1.1.39}, we need also
to compute, for any fixed parameter $\theta \in \Theta$,
quantities of the type
$$
\pi_{\exp( - \lambda r)} \Bigl\{ \exp \bigl[ \xi m'(\cdot, \theta) \bigr]
\Bigr\} = \pi_{\exp( - \lambda r)} \Bigl\{
\exp \bigl[ \xi \rho_{\theta} (m') \bigr] \Bigr\}, \quad \lambda, \xi
\in \RR_+.
$$
We need to introduce
$$
\ov{b}_y^t(\theta, c) = \frac{1}{N} \sum_{i=1}^N \Bigl\lvert \B{1}
\bigl[ f_{\theta} (X_i) \neq Y_i \bigr]
- \B{1} (y \neq Y_i) \Bigr\rvert
\B{1} \bigl\{ \bigl[ \B{1}(X_i^j \geq t_j) \bigr]_{j=1}^h = c \bigr\}.
$$
Similarly to what has been done previously, we obtain
\begin{multline*}
\pi \bigl\{ \exp  \bigl[ - \lambda r + \xi m'(\cdot, \theta)\bigr] \bigr\}
\\ = \sum_{t \in T} L(\Delta_t) \prod_{c \in \{0,1\}^h}
\biggl[ \frac{1}{\lvert \C{Y} \rvert}
\sum_{y \in \C{Y}}
\exp \biggl( - \lambda \bigl[ b^t(c) - b^t_y(c) \bigr] + \xi \ov{b}^t_y(\theta, c)
\biggr) \biggr].
\end{multline*}
We can then compute
\begin{align*}
\pi_{\exp( - \lambda r)}(r) & = -
\frac{\partial}{\partial \lambda} \log \bigl\{ \pi \bigl[ \exp( - \lambda r) \bigr]
\bigr\},\\
\pi_{\exp( - \lambda r)} \Bigl\{
\exp  \bigl[ \xi \rho_{\theta}(m')  \bigr] \Bigr\} & =
\frac{ \pi \bigl\{ \exp \bigl[ - \lambda r + \xi m'(\cdot, \theta) \bigr] \bigr\}}{
\pi \bigl[ \exp (- \lambda r) \bigr]},\\
\pi_{\exp( - \lambda r)}\bigl[ m'(\cdot, \theta) \bigr]
& = \frac{\partial}{\partial \xi}_{| \xi = 0}
\log \Bigl[ \pi
\bigl\{ \exp \bigl[ - \lambda r + \xi m'(\cdot, \theta) \bigr] \bigr\} \Bigr].
\end{align*}
This is all we need to compute $B(\rho_{\theta}, \beta, \gamma)$
(and also $B(\pi_{\exp( - \lambda r)}, \beta, \gamma)$)
in Theorem \ref{thm1.1.43} (page \pageref{thm1.1.43}), using the approximation
\begin{multline*}
\log \Bigl\{ \pi_{\exp( - \lambda_1 r)}
\Bigl[ \exp \bigl\{ \xi \pi_{\exp( - \lambda_2 r)}(m') \bigr\}
\Bigr] \Bigr\} \\* \leq \log \Bigl\{ \pi_{\exp( - \lambda_1 r)}
\Bigl[ \exp \bigl\{ \xi m'(\cdot, \theta) \bigr\} \Bigr] \Bigr\}
+ \xi \pi_{\exp( - \lambda_2 r)}\bigl[ m'(\cdot, \theta) \bigr], \quad \xi \geq  0.
\end{multline*}
Let us also explain how to apply the posterior distribution $\rho_{(t,a)}$,
in other words our randomized estimated classification rule, to a new
pattern $X_{N+1}$:
\begin{multline*}
\rho_{(t,a)} \bigl[ f_{\cdot}(X_{N+1}) = y \bigr]
= L(\Delta_t)^{-1} \int_{\Delta_t}
\B{1} \Bigl[ a \bigl\{ \bigl[ \B{1}(X_{N+1}^j \geq t_j')
\bigr]_{j=1}^h \bigr\} = y \Bigr]
L(d t')\\
= L(\Delta_t)^{-1} \sum_{c \in \{0,1\}^h}
L \Bigl( \Bigl\{ t' \in \Delta_t: \bigl[ \B{1}(X_{N+1}^j \geq t_j') \bigr]_{j=1}^h
= c \Bigr\} \Bigr) \B{1}\bigl[ a(c) = y \bigr].
\end{multline*}
Let us define for short
$$
\Delta_t(c) = \Bigl\{
t' \in \Delta_t: \bigl[ \B{1}(X_{N+1}^j \geq t_j') \bigr]_{j=1}^h = c
\Bigr\}, \qquad c \in \{0,1\}^h.
$$
With this notation
$$
\rho_{(t,a)} \bigl[ f_.(X_{N+1}) = y \bigr]
= L\bigl(\Delta_t\bigr)^{-1} \sum_{c \in \{0,1\}^h}
L \bigl[ \Delta_t(c) \bigr] \B{1} \bigl[ a(c) = y \bigr].
$$
We can compute in the same way the probabilities for the label
of the new pattern under the Gibbs posterior distribution:
\begin{multline*}
\pi_{\exp( - \lambda r)}\bigl[ f_{\cdot}(X_{N+1}) = y' \bigr]
\\ \shoveleft{\quad =  \Biggl\{ \sum_{t \in T}
\prod_{c \in \{0,1\}^h} \biggl[
\frac{1}{\lvert \C{Y} \rvert}
\sum_{y\in\C{Y}}
\exp \biggl( - \lambda \bigl[ b^t(c)
- b^t_y(c) \bigr] \biggr) \biggr]} \\ \times
\sum_{c \in \{0,1\}^h} L \bigl[ \Delta_t(c) \bigr]
\frac{ \sum_{y \in \C{Y}} \B{1}(y=y') \exp
\bigl\{ - \lambda \bigl[ b^t(c) - b^t_y(c) \bigr] \bigr\}}{\sum_{y \in \C{Y}}
\exp \bigl\{ - \lambda \bigl[ b^t(x) - b^t_y(c) \bigr] \bigr\}} \Biggr\}\\
\times \Biggl\{ \sum_{t \in T} L(\Delta_t) \prod_{c \in \{0,1\}^h}
\biggl[ \frac{1}{\lvert \C{Y} \rvert}
\sum_{y \in \C{Y}} \exp \Bigl( - \lambda \bigl[ b^t(c) - b^t_y(c) \bigr] \Bigr) \biggr]
\Biggr\}^{-1}.
\end{multline*}

\section{Transductive bounds}

In the case when we observe the patterns of a shadow sample $(X_i)_{i = N+1}^{(k+1)N}$
on top of the training sample $(X_i,Y_i)_{i=1}^N$,
we can introduce the set of thresholds responding as $t$ on the extended
sample $(X_i)_{i=1}^{(k+1)N}$
$$
\ov{\Delta}_t = \Bigl\{ t' \in \C{T}:
\ov{(t_j', t_j)} \cap \bigl\{
X_i^j; i = 1, \dots, (k+1)N \} = \varnothing,
j = 1, \dots, h \Bigr\},
$$
consider the set
$$
\ov{T} = \bigl\{ t \in \C{T}: t = m (\ov{\Delta}_t) \bigr\},
$$
of the middle points of the cells $\ov{\Delta}_t$, $t \in \C{T}$,
and replace the Lebesgue measure $L \in \C{M}_+^1 \bigl[ (0,1)^h \bigr]$
of the previous section with the uniform probability measure
$\ov{L}$ on $\ov{T}$.
We can then consider $\pi = \ov{L} \otimes U$, where $U$ is as previously
the uniform probability measure on $\C{R}$.
This gives obviously
an exchangeable posterior distribution and therefore qualifies $\pi$ for
transductive bounds. Let us notice that $\lvert \ov{T} \rvert
\leq \bigl[ (k+1)N+1 \bigr]^{h}$, and therefore that
$\pi(t,a) \geq \bigl[ (k+1)N+1 \bigr]^{-h} \lvert \C{Y} \rvert^{-2^h}$,
for any $(t,a) \in \ov{T} \times \C{R}$.

For any $(t,a) \in \C{T} \times \C{R}$
we may similarly to the inductive case consider
the posterior distribution $\rho_{(t,a)}$ defined by
$$
\frac{d \rho_{(t,a)}}{d \pi} (t',a')
= \frac{\B{1}(t' \in \Delta_t) \B{1} (a' = a)}{\pi
\bigl( \Delta_t \times \{ a \})},
$$
but we may also consider $\delta_{(m(\ov{\Delta}_t), a)}$,
which is such that $r_i\{ [ m(\ov{\Delta}_t), a ] \}
= r_i[(t,a)]$, $i = 1, 2$,
whereas only
$\rho_{(t,a)} (r_1) = r_1[(t,a)]$, while
$$
\rho_{(t,a)}(r_2) = \frac{1}{\lvert \ov{T} \cap
\Delta_t\rvert} \sum_{t' \in \ov{T} \cap \Delta_t} r_2[(t',a)].
$$
We get
\begin{multline*}
\C{K}(\rho_{(t,a)},\pi) = - \log \bigl[ \,\ov{L}(\Delta_t) \bigr] +
2^h \log \bigl( \lvert \C{Y} \rvert \bigr)
\\ \leq \log \bigl(\lvert \ov{T} \rvert \bigr)
+ 2^h \log( \lvert \C{Y} \rvert) =
\C{K}(\delta_{[m(\ov{\Delta}_t), a]}, \pi)
\\ \leq h \log \bigl[ (k+1)N+1 \bigr] + 2^h \log(\lvert \C{Y} \rvert),
\end{multline*}
whereas we had no such uniform bound in the inductive case.
Similarly to the inductive case
$$
\pi \bigl[ \exp ( - \lambda r_1 ) \bigr]
= \sum_{t \in T} \ov{L}(\Delta_t) \prod_{c \in \{0,1\}^h}
\biggl[ \frac{1}{\lvert \C{Y} \rvert} \sum_{y \in \C{Y}}
\exp \biggl( - \lambda \bigl[ b^t(c) - b^t_y(c) \bigr]  \biggr) \biggr].
$$
Moreover, for any $\theta \in \Theta$,
\begin{multline*}
\pi \bigl\{ \exp \bigl[
- \lambda r_1 + \xi \rho_{\theta}(m')
\bigr] \bigr\} = \pi \bigl\{ \exp \bigl[ - \lambda r_1 + \xi m'(\cdot, \theta) \bigr]
\bigr\} \\
= \sum_{t \in T} \ov{L}(\Delta_t)
\prod_{c \in \{0,1\}^h}
\biggl[ \frac{1}{\lvert \C{Y} \rvert}
\sum_{y \in \C{Y}} \exp \biggl( -
\lambda \bigl[ b^t(c) - b^t_y(c) \bigr] +
\xi \ov{b}(\theta, c) \biggr) \biggr].
\end{multline*}
The bound for the transductive counterpart to Theorems \ref{thm1.1.43}
(page \pageref{thm1.1.43}) or
\thmref{thm1.1.39}, obtained as explained page \pageref{page97},
can be computed as in the inductive case,
from these two partition functions and the above entropy
computation.

Let us mention finally that, using the same notation as in the inductive
case,
\begin{multline*}
\pi_{\exp( - \lambda r_1)}\bigl[ f_{\cdot}(X_{N+1}) = y' \bigr]
\\* \shoveleft{\quad =  \Biggl\{ \sum_{t \in T}
\prod_{c \in \{0,1\}^h} \biggl[
\frac{1}{\lvert \C{Y} \rvert}
\sum_{y\in\C{Y}}
\exp \biggl( - \lambda \bigl[ b^t(c)
- b^t_y(c) \bigr] \biggr) \biggr]} \\* \times
\sum_{c \in \{0,1\}^h} \ov{L} \bigl[ \Delta_t(c) \bigr]
\frac{ \sum_{y \in \C{Y}} \B{1}(y=y') \exp
\bigl\{ - \lambda \bigl[ b^t(c) - b^t_y(c) \bigr] \bigr\}}{\sum_{y \in \C{Y}}
\exp \bigl\{ - \lambda \bigl[ b^t(x) - b^t_y(c) \bigr] \bigr\}} \Biggr\}\\*
\times \Biggl\{ \sum_{t \in T} \ov{L}(\Delta_t) \prod_{c \in \{0,1\}^h}
\biggl[ \frac{1}{\lvert \C{Y} \rvert}
\sum_{y \in \C{Y}} \exp \Bigl( - \lambda \bigl[ b^t(c) - b^t_y(c) \bigr] \Bigr) \biggr]
\Biggr\}^{-1}.
\end{multline*}

To conclude this appendix on classification by thresholding, note
that similar factorized computations are feasible in the important case of
\emph{classification trees}. This can be achieved
using some variant of the \emph{context tree weighting
method} discovered by \citet{Willems1} and
successfully used in lossless compression theory.
The interested reader can find a description of this
algorithm applied to classification trees in \citet[page 62]{Cat7}.

\end{document}